\documentclass[11pt]{article}

\usepackage{neurips_2022}

\usepackage{parskip}

\RequirePackage{natbib}

\setcitestyle{authoryear,round,citesep={;},aysep={,},yysep={;}}

\usepackage{amsmath,amsfonts,bm}

\def\eqref#1{equation~\ref{#1}}

\def\1{\bm{1}}

\usepackage[utf8]{inputenc} 
\usepackage[T1]{fontenc}    
\usepackage{booktabs}       
\usepackage{amsfonts}       
\usepackage{nicefrac}       
\usepackage{microtype}      
\usepackage{xcolor}         

\usepackage{graphicx}
\usepackage{mathrsfs}
\usepackage{amsfonts, amsmath,dsfont,amssymb,amsthm,stmaryrd,bbm}
\usepackage{mathtools}
\usepackage{caption}
\usepackage{subcaption}
\usepackage{enumitem}
\setlist{nolistsep}
\setlist{nosep}
\usepackage{wrapfig}
\usepackage{algorithm}
\usepackage{algorithmic}
\usepackage{todonotes}
\usepackage{xspace}

 \newtheorem{conj}{Conjecture}
 \newtheorem{thm}{Theorem}
 \newtheorem*{thm*}{Theorem}
 \newtheorem*{lemma*}{Lemma}
 \newtheorem{rmk}{Remark}
 
 \newtheorem{lemma}[conj]{Lemma}

 \newtheorem{defn}[conj]{Definition}
 \newtheorem{coro}{Corollary}
 \newtheorem*{coro*}{Corollary}
 \newtheorem{assumption}{Assumption}

\newcommand{\f}[1]{\boldsymbol{#1}}
\newcommand{\bb}[1]{\mathbb{#1}}
\newcommand{\fl}[1]{\mathbf{#1}}
\newcommand{\ca}[1]{\mathcal{#1}}

\newcommand{\s}[1]{\mathsf{#1}}

\newcommand{\lr}[1]{{\left|\left|#1\right|\right|}}
\newcommand{\method}{AMHT-LRS\xspace}
\newcommand{\lrs}{\textsf{LRS}\xspace}

\def\vzero{{\bm{0}}}

\DeclareMathAlphabet{\mathsfit}{\encodingdefault}{\sfdefault}{m}{sl}
\SetMathAlphabet{\mathsfit}{bold}{\encodingdefault}{\sfdefault}{bx}{n}

\usepackage[utf8]{inputenc} 
\usepackage[T1]{fontenc}    
\usepackage{url}            
\usepackage{booktabs}       
\usepackage{amsfonts}       
\usepackage{nicefrac}       
\usepackage{microtype}      
\usepackage{xcolor}         

\usepackage{hyperref}

\usepackage{graphicx}
\usepackage{mathrsfs}
\usepackage{amsfonts, amsmath,dsfont,amssymb,amsthm,stmaryrd,bbm}
\usepackage{mathtools}
\usepackage{caption}
\usepackage{subcaption}
\usepackage{enumitem}
\setlist{nolistsep}
\setlist{nosep}
\usepackage{wrapfig}
\usepackage{algorithm}
\usepackage{algorithmic}
\usepackage{todonotes}
\usepackage{xspace}

\usepackage{authblk}
\usepackage{multirow}

\title{\textbf{Sample-Efficient Personalization: Modeling User Parameters as Low Rank Plus Sparse Components}}

\author[1]{Soumyabrata Pal\footnote{Equal Contribution.}}
\newcommand\CoAuthorMark{\footnotemark[\arabic{footnote}]} 
\author[1]{Prateek Varshney\protect\CoAuthorMark}
\author[1]{Prateek Jain}
\author[2]{Abhradeep Guha Thakurta}
\author[1]{Gagan Madan}
\author[1]{Gaurav Aggarwal}
\author[1]{Pradeep Shenoy}
\author[1]{Gaurav Srivastava}
\affil[1]{Google Research India}
\affil[2]{Google Brain}
\affil[1,2]{\textit {\{soumyabrata,vprateek,prajain,athakurta,gaganmadan, gauravaggarwal,shenoypradeep,gasrivastava\}@google.com}}

\newcommand{\fTheta}{\fl{\Theta}}

\begin{document}

\maketitle

\vspace{-20pt}

\begin{abstract}
Personalization of machine learning (ML) predictions for individual users/domains/enterprises is critical for practical recommendation systems. Standard personalization approaches involve learning a user/domain specific \emph{embedding} that is fed into a fixed global model which can be limiting. On the other hand, personalizing/fine-tuning model itself for each user/domain -- a.k.a meta-learning -- has high storage/infrastructure cost. Moreover, rigorous theoretical studies of scalable personalization approaches have been very limited. 
To address the above issues, we propose a novel meta-learning style approach that models network weights as a sum of low-rank and sparse components. This captures common information from multiple individuals/users together in the low-rank part while sparse part captures user-specific idiosyncrasies.   
We then study the framework in the linear setting, where the problem reduces to that of estimating the sum of a rank-$r$ and a $k$-column sparse matrix using a small number of linear measurements. We propose a computationally efficient alternating minimization method with iterative hard thresholding -- \method -- to learn the low-rank and sparse part. Theoretically, for the realizable Gaussian data setting, we show that \method solves the problem efficiently with nearly optimal sample complexity. Finally, a significant challenge in personalization is ensuring privacy of each user's sensitive data. We alleviate this problem by proposing a differentially private variant of our method that also is equipped with strong generalization guarantees. 
\end{abstract}


\section{Introduction}\label{sec:intro}

Typical industrial recommendation systems cater to a large number of users/domains/enterprises with a small amount of user-specific data \cite{10.1016/j.knosys.2016.04.018, 10.1007/s10115-020-01455-2}. For instance, YouTube has $\sim1.3$ billion unique monthly active users while the average likes per user is small.
So, personalizing predictions for each user corresponds to the challenging task of training ML models with a few data-points. Typically, personalization literature has approached this problem from collaborative  \cite{schafer1999recommender} or content filtering \cite{burke2003hybrid} approach. Both these approaches, in some sense learn a user embedding or user specific feature vector which is then consumed by a global model to provide the personalized predictions. This can also be seen as a variation of the popular prompt learning approach \cite{lester2021power}. 

Naturally, the prompt learning approach is limiting because the global model might not be able to capture all the variations across users/domains unless it is of extremely large size which in turn leads to large inference/training cost. Moreover, user-descriptive feature vectors, limited by privacy concerns, might not be able to capture the user-taste explicitly. 

On the other extreme, such user-specific personalized models can be learned by fine-tuning a global model for each user. 
Vanilla approaches for fine-tuning can be  categorized as: 1) \textit{Neighborhood Models}: these methods learn a global model, which is then entirely "fine-tuned" to specific tasks \citep{guo2020parameter,howard2018universal,zaken2021bitfit}
2) \textit{Representation Learning}: these methods learn a low-dimensional representation of points which can be used to train task-specific learners \citep{javed2019meta,raghu2019rapid,lee2019meta,bertinetto2018meta,hu2021lora}. 
Neighborhood fine-tuning techniques have two key limitations: \textrm{I}) Infrastructure costs of hosting such models is prohibitive. For example, consider a scenario where we have a pre-trained 1GB model which is fine-tuned for 1M users, then the total storage cost itself is 1PB, and \textrm{II}) they typically work for a small number of data-rich users/domains/tasks, but the key regime for personalization in our work is a long tail of data-starved users. In this regime, neighborhood fine-tuning might lead to over-fitting as well.
Simple fixes like fine-tuning only the last layer often lead to a significantly worse performance  \citep{chen2020simple,salman2020adversarially}. 
Further, note that representation learning techniques can learn only the (low dimensional) common information across the users but cannot capture the user-level peculiarities as they share a fixed smaller representation space. These limitations make \textit{model personalization for a large number of data-starved users a challenging problem from the perspective of both storage and inference}.

Recently, with the advent of large ML models, there has been a huge push towards practical Parameter Efficient Fine-Tuning (PEFT) with a variety of techniques (see the survey  \cite{lialin2023scaling} and references therein). However, as described in \cite{lialin2023scaling}, \textit{a rigorous theoretical study of such approaches that can be scaled to large number of tasks (users) has been limited.}  
In this work, for the multi-user personalization problem, our first contribution is to introduce the \lrs (Low Rank and Sparse) framework that combines representation learning and sparse fine-tuning (two distinctive class of methods for PEFT) with two main goals: 1) propose a practical/efficient algorithm for learning a model for each user with minimal memory/parameter overhead, and 2) theoretically establish strong generalization properties despite having a limited amount of data for many users.

Let us describe the \lrs framework for multiple users in its full generality for any parameterized class of functions $\ca{F}$. 
Let $\fl{x}\in \bb{R}^d$ be the input point, and let $\widehat{y}=f(\fl{x}; \fl{\Theta})\in \bb{R}$ be the predicted label using a function $f\in \ca{F}$ parameterized by $\fl{\Theta}$. For simplicity, assume that $\fl{\Theta}$, the set of parameters, is represented in the form of a high dimensional vector. Let there be $t$ users, each with a small amount of user-specific data and each associated with a set of learnable model parameters $\fTheta^{(i)}$. Then, the goal is to learn $\fTheta^{(i)}$ for each user $i\in [t]$ such that 
(a) $\fTheta^{(i)}$ does not over-fit despite a small number of data-points generated specifically from user $i$, (b) $\{\fTheta^{(i)},1\leq i\leq t\}$ can be stored and used for fast inference even for large $t$,
Our method
\lrs attempts to address all the three requirements using a simple low-rank + sparse approach: we model each $\fTheta^{(i)}$ as $\fTheta^{(i)}:=\fl{U}^{\star}\cdot \fl{w}^{\star(i)} + \fl{b}^{\star(i)}$, where $\fl{U}^{\star}$ is a orthonormal matrix representing a global set of shared parameters (across users) corresponding to a low (say $r$)-dimensional subspace, $\fl{w}^{\star(i)}$ is a $r$-dimensional vector
and $\fl{b}^{\star(i)}$ is a $k$-sparse vector. Thus, we represent 
$\fTheta^{(i)}:=\fTheta^{(i,1)}+\fTheta^{(i,2)}$, where the first term $\fTheta^{(i,1)} = \fl{U}^{\star}\cdot \fl{w}^{\star(i)}$ denotes the low rank component of $\fTheta^{(i)}$ that is, the model parameters of the $i^{\s{th}}$ user lying in a low-dimensional subspace (common for all users) and the second term $\fTheta^{(i,2)} = \fl{b}^{\star(i)}$ is restricted to be sparse.
Hence, for each user, we only need to store weights of the $r$-dimensional vector $\fl{w}^{\star(i)}$ and the non-zero weights of the $k$-sparse $\fl{b}^{\star(i)}$.
Therefore, if $r$ and $k$ are small, then the memory overhead per user is small thus allowing efficient deployment. Moreover, we can expect that a small amount of user-specific data should be sufficient to train these few additional parameters per user without over-fitting. Finally, our framework provides users with the flexibility to either contribute to a central model using privacy preserving bill-board models (see Section~\ref{sec:dp}) or learn their parameters locally without contributing to the central model. This allows for greater customization and adaptability based on individual user preferences and requirements. 

Theoretically speaking, a recent line of work \citep{thekumparampil2021statistically,du2020few,tripuraneni2021provable,boursier2022trace, modelperosnalization} analyzes a framework exclusively on representation learning. They model $\fTheta^{(i)}:=\fl{U}^{\star}\cdot \fl{w}^{\star(i)}$ as the parameters associated with the $i^{\s{th}}$ user and provide theoretical guarantees in the linear model setting. However, their algorithm/analysis do not extend to our case since their model does not capture the additional sparsity component in $\fTheta^{(i)}$ - this additional non-convex constraint introduces several technical challenges in the analysis. Moreover, these works do not explore the personalization framework with privacy constraints.


In our second and primary contribution, we consider the problem of analyzing the \lrs framework from a theoretical lens. 
Specifically,
we analyze the instantiation of \lrs in the context of linear models. In this case, the training data for each user corresponds to a few linear measurements of the underlying user-specific model parameter $\f{\Theta}^{(i)}$. Our objective is to estimate the individual components of $\f{\Theta}^{(i)}$.
Training a model in the \lrs framework involves learning a global set of parameters characterizing the low dimensional subspace and the local user-specific parameters jointly. To this end, we propose a simple alternating minimization (AM) style iterative technique. Our method \method alternatingly estimates the global parameters and user-specific parameters independently for each user. To ensure sparsity of $(\fl{b}^{\star(i)})$'s we use an iterative hard-thresholding style estimator \citep{jain2014iterative}. 
For linear models, even estimating the global parameters i.e. the low rank subspace induced by $\fl{U}^{\star}$ is an NP-hard problem \cite{thekumparampil2021statistically}. 
Therefore, similar to \cite{thekumparampil2021statistically}, we consider the realizable setting where the data is generated from a Gaussian distribution.
In this case, we provide a novel analysis demonstrating the efficient convergence of our method, \method, towards the optimal solution.  We believe our analysis techniques are of independent interest as we track entry-wise error of different parameter estimates across iterations of \method. Below, we state our main result (Thm. \ref{thm:main}) informally:

\begin{thm*}[Informal] Suppose we are given $m\cdot t$ samples from $t$ linear regression training tasks (each corresponding to a user) of dimension $d$ with $m$ samples each. In the \lrs framework, the goal is to learn a new regression task’s parameters using $m$
samples, i.e., learn the $r$-dimensional weight vector defining the user-specific low rank representation and the user-specific $k$-sparse vector. Then \method with total $m\dot t = \Omega(kdr^4)$ samples and $m=\Omega(\max(k,r^3))$ samples per user can recover all the parameters exactly and in time nearly linear in $m\cdot t$.
\end{thm*}

That is, in the linear instantiation of \lrs, \method is able to estimate the underlying model parameters up to desired precision as long as the total number of users in the recommendation system is large enough. Moreover, we show that the sufficient sample complexity per user for our estimation guarantees scales only linear in $k$ and cubically in $r$ (nearly optimal); recall that $r,k$ are much smaller than the ambient dimension $d$. The detailed analysis of Theorem \ref{thm:main} is provided in Appendix \ref{app:detailed_proof}, \ref{app:generalization}.
Furthermore, using the billboard model of $(\epsilon,\delta)$ differential privacy (DP)~\citep{modelperosnalization, chien2021private,kearns2014mechanism}, we can extend  \method to preserve privacy of each individual user. 
For similar sample complexity as in the above theorem albeit with slightly worse dependence on $r$, we can guarantee strong generalization error up to a standard  error term due to privacy (see Theorem~\ref{thm:private2} and Appendix \ref{app:detailed_proof}).

Finally, to validate our theoretical contributions, we demonstrate  experimental results on synthetic and  real-world datasets using linear models. Also, we experiment with our framework applied to neural networks architectures (see Appendix  \ref{app:experiments_detailed},\ref{app:experiments_second}). Our experiments demonstrate the advantage/efficacy of our framework compared to natural baseline frameworks applied to the same model architecture.

\subsection{Other Related Work}\label{sec:model}

\noindent \textbf{Comparison with \cite{hu2021lora}}:
LORA (Low Rank Adaptation of Large Language Models) was proposed by \cite{hu2021lora} for meta-learning with large number of tasks at scale. Although the authors demonstrate promising experimental results, LORA only allows a central model (in a low dimensional manifold) and does not incorporate sparse fine-tuning. Hence, LORA becomes ineffective when output dimension is small. 
Moreover, LORA does not have any theoretical guarantees even in simple settings.

\noindent \textbf{Comparison with Prompt-based and Batch-norm Fine-tuning}: 
Another popular approach for personalization is to use prompt-based or batch-norm based fine-tuning \citep{wang2022learning,liu2021pre,lester2021power}; this usually involves a task-based feature embedding concatenated with the covariate. Note that in a linear model, such an approach will only lead to an additional scalar bias which can be easily modeled in our framework; thus our framework is richer and more expressive with a smaller number of parameters. 

\subsection{Preliminaries - Private Personalization}
For model personalization in recommendation systems, where we wish to have a personalized model for each user, privacy guarantees are of utmost importance. Due to sensitivity of user-data, we would want to preserve privacy of each {\em user} for which we use user-level $(\epsilon,\delta)$-DP as the privacy notion (see Definition~\ref{defn:diffPriv}). In this  setting, each user $i\in[t]$ holds a set of data samples $D^{(i)}=\{\fl{x}^{(i)}_{j},1\leq j\leq m\}$. Furthermore, users interact via a  central algorithm that maintains the common representation matrix $\fl{U}^{\star}$ which is guaranteed to be differentially private with respect to all the data samples of any single user.  The central algorithm publishes the current $\fl{U}^{\star}$ to all the users (a.k.a. on a billboard) and obtains further updates from the users. It has been shown in prior works~\citep{modelperosnalization,chien2021private,thakkar2019differentially} that such a billboard mechanism allows for significantly more accurate privacy preserving methods while ensuring user-level privacy. In particular, it allows learning of $\fl{U}^{\star}$ effectively, while each user can keep a part of the model which is personal to them, for example, the $\fl{w}^{\star(i)}, (\fl{b}^{\star(i)})$'s in our context. See Section 3 of ~\cite{modelperosnalization} for more details about billboard model in the  personalization setting. Traditionally, such model of private computation is called the billboard model of DP and is a subclass of joint DP~\citep{kearns2014mechanism}.

\begin{defn}
[Differential Privacy~\cite{DMNS,ODO,bun2016concentrated}]
A randomized algorithm $\mathcal{A}$ is $(\varepsilon,\delta)$-differentially private (DP) if for any pair of data sets $D$ and $D'$ that differ in one user (i.e., $|D\triangle D'|=1$), and for all $S$ in the output range of $\mathcal{A}$, we have 
$$\Pr[\mathcal{A}(D)\in S] \leq e^{\varepsilon} \cdot \Pr[\mathcal{A}(D')\in S] +\delta,$$
where probability is over the randomness of $\mathcal{A}$. Similarly, an algorithm $\mathcal{A}$ is $\rho$-zero Concentrated DP (zCDP) if $D_{\alpha}\left(\mathcal{A}(D)||\mathcal{A}(D')\right)\leq \alpha\rho$, 
where $D_{\alpha}$ is the R\'enyi divergence of order $\alpha$. 
\label{defn:diffPriv}
\end{defn}

In Definition ~\ref{defn:diffPriv}, when we define the notion of neighborhood, we define it with respect to the addition (removal) of a single user (i.e., additional removal of all the data samples $D_i$ for any user $i\in[t]$). In the literature~\cite{dwork2014algorithmic}, the definition is referred to as user-level DP.



\section{Linear \lrs with Gaussian Data}\label{subsec:linear_lrs}

\subsection{Problem Statement and Algorithm \method}
\noindent \textbf{Notations:} $[m]$  denotes the set $\{1,2,\dots,m\}$. 
For a matrix $\fl{A}$,  $\fl{A}_i$  denotes $i^{\s{th}}$ row of $\fl{A}$. For a vector $\fl{x}$, $x_i$ denotes  $i^{\s{th}}$ element of $\fl{x}$. We sometimes use $\fl{x}_j$ to denote an indexed vector; in this case $x_{j,i}$ denotes the $i^{\s{th}}$ element of  $\fl{x}_j$. $\|\cdot\|_2$ denotes euclidean norm of a vector and the operator norm of a matrix. $\|\cdot\|_{\infty}, \|\cdot\|_0$ will denote the $\ell_{\infty}$ and $\ell_0$ norms of a vector respectively. $\|\fl{A}\|_{2,\infty}=\max_i \|\fl{A}_i\|_2$ and $\|\fl{A}\|_{\s{F}}=\sqrt{\sum_i \|\fl{A}_i\|_2^2}$ denote the $\s{L}_{2,\infty}$ and Frobenius norm of a matrix respectively. For a sparse vector $\fl{v}\in \bb{R}^d$, we define the support $\s{supp}(\fl{v})\subseteq [d]$ to be the set of indices $\{i\in [d]\mid v_i \neq 0\}$. We use $\fl{I}$ to denote the identity matrix. $\widetilde{O}(\cdot),\widetilde{\Omega}(\cdot)$ notations subsume logarithmic factors. For a matrix $\fl{V}\in \bb{R}^{d\times r}$, 
$\s{vec}(\fl{V})\in \bb{R}^{dr}$ vectorizes the matrix $\fl{V}$ by stacking columns sequentially. Similarly $\s{vec}^{-1}_{d\times r}(\fl{v})$ inverts the $\s{vec}(\cdot)$ operation by reconstructing a $d\times r$ matrix from a vector of dimension $dr$ i.e. for the matrix $\fl{V}$, $\s{vec}^{-1}_{d\times r}(\s{vec}(\fl{V}))=\fl{V}$. Finally, let $\s{HT}:\bb{R}^d \times \bb{R}\rightarrow \bb{R}^d$ be a {\em hard thresholding} function that takes a vector $\fl{v}\in \bb{R}^d$ and a parameter $\Delta$ as input and returns a vector $\fl{v}'\in \bb{R}^d$ such that $v'_i = v_i$ if $|v_i|>\Delta$ and $0$ otherwise.

In this section, we describe our \lrs framework for the linear setting with Gaussian data, provide an efficient algorithm for parameter estimation, and provide rigorous analysis under realizable setting.  Formally speaking, consider $t$, $d$-dimensional linear regression tasks indexed by $i\in [t]$ where the $i^{\s{th}}$ training task is associated with the $i^{\s{th}}$ user. Recall that according to the definition of \lrs framework, every user/task $i\in [t]$ is associated with a set of unknown learnable parameters $\f{\Theta}^{(i)}\in \bb{R}^d$ that can be decomposed as $\f{\Theta}^{(i)}=\fl{U}^{\star}\fl{w}^{\star (i)}+\fl{b}^{\star (i)}$. 
Here $\fl{U}^{\star}\in \bb{R}^{d\times r}$ (satisfying $(\fl{U}^{\star})^{\s{T}}\fl{U}^{\star}=\fl{I}$) is the global shared set of parameters (across users) corresponding to the orthonormal basis vectors of a $r$-dimensional subspace. For the $i^{\s{th}}$ user, the $r$-dimensional user-specific parameters $\fl{w}^{\star(i)}\in \bb{R}^r$ corresponds to the weights of the basis vectors of low dimensional subspace defined by columns of $\fl{U}^{\star}$; similarly, the user-specific $k$-sparse vector $\fl{b}^{\star(i)}\in \bb{R}^d$ with $\|\fl{b}^{\star(i)}\|_0=k$  corresponds to the sparse component of the unknown parameters of the $i^{\s{th}}$ user. Henceforth, we will refer to the problem of estimating the unknown parameters of user $i\in [t]$ to be the $i^{\s{th}}$ training task.

For each task $i\in [t]$, $m \ll d$  samples $\{(\fl{x}^{(i)}_j,y^{(i)}_j)\}_{j=1}^{m}\in (\bb{R}^d\times \bb{R})^m$
are provided labelled as task $i$.  Next, we assume the following generative model for the data:  $\text{ for all } i\in[t], \text{ for all } j\in [m]$, the covariates $\{\fl{x}^{(i)}_j\}_{i,j}$ are independently generated from a $d$-dimensional Gaussian with identity covariance (denoted as $\ca{N}(\f{0},\fl{I}_d) $) and the expected response is a linear function of the corresponding covariate. More precisely, we have:
\begin{align}\label{eq:samples}
    &\fl{x}^{(i)}_j \sim \ca{N}(\f{0},\fl{I}_d) 
    \text{ and }
    y_j^{(i)}\mid \fl{x}^{(i)}_j = \langle \fl{x}^{(i)}_j, \fl{U}^{\star} \fl{w}^{\star(i)}+\fl{b}^{\star(i)} \rangle+z^{(i)}_j \quad \forall i\in [t],j\in [m],
\end{align}
where $z^{(i)}_j \sim \ca{N}(0,\sigma^2)$ are zero mean Gaussian random variables with variance $\sigma^2$. Furthermore, $\{\fl{x}^{(i)}_j,z^{(i)}_j\}_{i,j}$ are independent random variables.
We use $\fl{X}^{(i)}\in \bb{R}^{m \times d}$  to represent the matrix of covariates for the $i^{\s{th}}$ task s.t. $\fl{X}^{(i)}_j = (\fl{x}^{(i)}_j)^{\s{T}}$. Similarly, we write $\fl{y}^{(i)},\fl{z}^{(i)}\in \bb{R}^m$ to represent the user-specific response vector and noise vector respectively. 

Therefore, given the data-set $\{(\fl{x}^{(i)}_j,y^{(i)}_j)\}_{i,j}$, the problem reduces to that of  designing  statistically and computationally efficient algorithms to estimate the common  representation learning parameter $\fl{U}^\star$ as well as task-specific parameters $\{\fl{w}^{\star(i)}\}_{i \in [t]},\{\fl{b}^{\star(i)}\}_{i \in [t]}$. The ERM (Empirical Risk Minimizer) for this model assuming squared loss is given by the following objective function: 
\begin{align}\label{prob:general}
\text{\textbf{LRS:} }\quad \quad &\text{minimize } \ca{L}(\fl{U},\fl{W},\fl{B})\quad = \sum_{i\in [t]}\sum_{j \in [m]} \frac{1}{2}\Big(y^{(i)}_j-\langle \fl{x}^{(i)}_j, \fl{U} \fl{w}^{(i)}+\fl{b}^{(i)} \rangle \Big)^2\nonumber\\ 
& \text{ such that }\fl{U}^{\s{T}}\fl{U}=\fl{I} \text{ and } \|\fl{b}^{(i)}\|_0 \le k \; \forall i\in [t] 
\end{align}
where $\fl{U}\in \bb{R}^{d \times r}$, $\fl{W} = [\fl{w}^{(1)} \; \fl{w}^{(2)} \; \dots \; \fl{w}^{(t)}]^{\s{T}} \in \bb{R}^{t \times r}$ stores the estimated task-specific coefficients of the low-dimensional subspace, and $\fl{B}=[\fl{b}^{(1)} \; \fl{b}^{(2)} \; \dots \; \fl{b}^{(t)}] \in \bb{R}^{d \times t}$ stores the estimated task-specific  sparse vectors for {\em fine-tuning}. Note that the \lrs objective is non-convex due to: a)  bilinearity of $\fl{U},\fl{W}$, and b) non-convexity of $\ell_0$ norm constraint. 


To optimize the \lrs objective, we propose an Alternating Minimizing algorithm \method that starts with an initialization of the unknown parameters and iteratively updates them. \method handles the non-convexity in the objective and the constrained set by sequentially updating $\fl{U}$, $\fl{W}$ and $\fl{B}$ (while keeping the others fixed) with hard thresholding (recall the function $\s{HT}(\cdot,\cdot)$) to ensure sparsity of the columns of $\fl{B}$. 
 Let  $\fl{U}^{+(\ell-1)}, \{\fl{w}^{(i,\ell-1)}\}_{i \in [t]}$ and $\{\fl{b}^{(i,\ell-1)}\}_{i\in [t]}$ be the latest iterates at the beginning of the $\ell^{\s{th}}$ iteration in \method where we ensure that $\fl{U}^{+(\ell-1)}$ is an orthonormal matrix $\in \mathbb{R}^{d\times r}$, $\{\fl{w}^{(i,\ell-1)}\}_{i\in [t]}$ are $r$-dimensional vectors and $\{\fl{b}^{(i,\ell-1)}\}_{i\in [t]}$ are at most $k$-sparse. 
First, for each task $i\in [t]$, given estimates $\fl{U}^{+(\ell-1)},\fl{w}^{(i,\ell-1)}$, we can compute $\fl{b}^{(i,\ell)}$ by solving:
\begin{align}\label{eq:optimize1}
    \small &\s{argmin}_{\fl{b}\in \bb{R}^d} \lr{\fl{X}^{(i)}(\fl{U}^{+(\ell-1)}\fl{w}^{(i,\ell-1)}+\fl{b})-\fl{y}^{(i)}}_2 \text{ such that } \|\fl{b}\|_{0} \le k.  
\end{align}
While the problem is non-convex, note that it is equivalent to sparse linear regression with the modified response vector $\fl{y}^{(i)}-\fl{X}^{(i)}\fl{U}^{+(\ell-1)}\fl{w}^{(i,\ell-1)}$. Therefore, we can still apply a projected gradient descent algorithm which reduces to iterative hard thresholding. More precisely, for each task indexed by $i\in [t]$, in Line 4 of \method, we invoke another iterative sub-routine \textsc{OptimizeSparseVector} (Alg. \ref{algo:optimize_b}). In each iteration of \textsc{OptimizeSparseVector}, we run a gradient descent step on the estimate $\fl{b}$ (of $\fl{b}^{\star (i)}$) and subsequently apply the hard-thresholding function $\s{HT}(\cdot,\Delta)$ where $\Delta>0$ is set appropriately 
 in order to ensure that the updated estimate $\fl{b}$ is sparse 
 at the end of every iteration in Algorithm \ref{algo:optimize_b}.
Next, given estimates $\fl{U}^{+(\ell-1)},\fl{b}^{(i,\ell)}$, we update $\fl{w}^{(i,\ell)}$ by solving the following task-specific optimization problem 
\begin{align}\label{eq:update_w}
\s{argmin}_{\fl{w}\in \bb{R}^r} \lr{\fl{X}^{(i)}(\fl{U}^{+(\ell-1)}\fl{w}+\fl{b}^{(i,\ell)})-\fl{y}^{(i)}}_2 \; \forall \; i\in [t].  
\end{align}
Note that the objective in equation \ref{eq:update_w} allows a closed form solution and therefore, the updated task-specific estimate $ \fl{w}^{(i, \ell)}$ is given as (computed in Line 5 of \method):
\begin{align}\label{eq:soln_w}
    \fl{w}^{(i, \ell)} = \Big((\fl{X}^{(i)}\fl{U}^{+(\ell-1)})^{\s{T}}(\fl{X}^{(i)}\fl{U}^{+(\ell-1)})\Big)^{-1}\Big((\fl{X}^{(i)}\fl{U}^{+(\ell-1)})^{\s{T}}(\fl{y}^{(i)} -  \fl{X}^{(i)}\fl{b}^{(i, \ell)})\Big)
\end{align}
Subsequently, given the updated estimates of the task-specific parameters $\{\fl{w}^{(i,\ell)}\}_{i \in [t]}$ and $\{\fl{b}^{(i,\ell)}\}_{i\in [t]}$, we update the global $\fl{U}^{+(\ell)}$ in two steps: first, we compute $\fl{U}^{(\ell)}$ by solving
\begin{align}\label{eq:update_u}
    \s{argmin}_{\fl{U}\in \bb{R}^{d\times r}} \sum_{i\in [t]}\lr{\fl{X}^{(i)}(\fl{U}\fl{w}^{(i,\ell)}+\fl{b}^{(i,\ell)})-\fl{y}^{(i)}}_2.
\end{align}
followed by a QR decomposition of the solution $\fl{U}^{(\ell)}$ to obtain the updated estimate $\fl{U}^{+(\ell)}$; the QR factorization ensures that $\fl{U}^{+(\ell)}$ is orthonormal ($\fl{U}^{(\ell)}=\fl{U}^{+(\ell)}\fl{R}$). The objective in eq. \ref{eq:update_u} allows a closed form solution and therefore, we can compute $\fl{U}^{(\ell)}$ as (computed in line 7 of \method): 
\begin{align}\label{eq:soln_u}
&\fl{U}^{(\ell)} =  \s{vec}^{-1}_{d\times r}(\fl{A}^{-1}\s{vec}(\fl{V})) \text{ where }
\fl{A} := \sum_{i \in [t]}\Big(\fl{w}^{(i, \ell)}(\fl{w}^{(i, \ell)})^{\s{T}} \otimes \Big(\sum_{j=1}^{m}\fl{x}^{(i)}_j(\fl{x}^{(i)}_j)^{\s{T}}\Big)\Big)  \nonumber \\
&\text{ and }
\fl{V} := \sum_{i \in [t]} (\fl{X}^{(i)})^{\s{T}}\Big(\fl{y}^{(i)} - \fl{b}^{(i, \ell)}\Big)(\fl{w}^{(i, \ell)})^{\s{T}}
\end{align}
Here $\otimes$ denotes the Kronecker product of two matrices. 
Finally, we must ensure independence of the estimates (which are random variables themselves) from the data that is used in a particular update. 
We can ensure such statistical independence by using a fresh batch of samples in every iteration. 

\begin{algorithm*}[t]
\caption{\textsc{\method} (Algorithm for estimating $\fl{U}^{\star},\{\fl{w}^{\star(i)}\}_{i\in [t]}$ and $\{\fl{b}^{\star(i)}\}_{i\in [t]}$.)}  \label{algo:optimize_lrs1}
\begin{algorithmic}[1]
\small
\REQUIRE Data $\{(\fl{x}^{(i)}_j\in \bb{R}^d,y^{(i)}_j\in \bb{R})\}_{j=1}^{m}$ for all $i\in [t]$, column sparsity $k$ of $\fl{B}$. Initialization $\fl{U}^{+(0)},\{\fl{w}^{(i,0)}\}_{i\in [t]},\{\fl{b}^{(i,0)}\}_{i\in [t]}$ and
parameters $\s{B},\gamma^{(0)},\epsilon>0$ such that
$\lr{(\fl{I}-\fl{U}^{\star}(\fl{U}^{\star})^{\s{T}})\fl{U}^{+(0)}}_{\s{F}} \le \s{B}$, $\max_i \|\fl{b}^{(i, 0)} - \fl{b}^{\star(i)}\|_\infty \leq \gamma^{(0)}$ and $\epsilon$ is the desired parameter estimation accuracy.
\FOR{$\ell = 1, 2, \dots$} 
    \STATE Set $T^{(\ell)} = \Omega\Big(\ell\log\Big(\frac{\gamma^{(\ell-1)}}{\epsilon}\Big)\Big)$ \quad \texttt{// Beginning of the $\ell^{\s{th}}$ iteration}
    \FOR{$i = 1,2, \dots, t$}
        \STATE $\fl{b}^{(i, \ell)} \gets \s{Optimize Sparse Vector}((\fl{X}^{(i)}, \fl{y}^{(i)}), \fl{v} = \fl{U}^{+(\ell-1)}\fl{w}^{(i,\ell-1)}, \fl{b}=\fl{b}^{(i,\ell-1)}, \alpha = O\Big(c_4^{\ell-1}\frac{\s{B}}{\sqrt{k}}\Big), \beta = O(c_5^{\ell-1}\s{B}), \gamma = \gamma^{(\ell-1)}, \s{T} = T^{(\ell)})$ for suitable constants $1\ge c_4,c_5>0$. \texttt{\\ // Update the estimate of $\fl{b}^{\star(i)}$}\\
        \STATE Compute $\fl{w}^{(i, \ell)}$ as in equation \ref{eq:soln_w} 
        \texttt{// Update the estimate of $\fl{w}^{\star(i)}$}
    \ENDFOR
    \STATE 
 Compute  $\fl{U}^{(\ell)}$ as in equation \ref{eq:soln_u} and update $\fl{U}^{+(\ell)}=\s{QR}(\fl{U}^{(\ell)})$ such that $\fl{U}^{+(\ell)}$ is orthonormal.
 \texttt{\\ // Update the estimate of $\fl{U}^{\star}$}
    \STATE $\gamma^{(\ell)} \gets (c_3)^{\ell-1}\epsilon\s{B}$ for a suitable constant $c_3<1$. 
\ENDFOR
\STATE Return  $\fl{w^{(\ell)}}$, $\fl{U}^{+(\ell)}$ and $\{\fl{b}^{(i, \ell)} \}_{i\in [t]}$.
\end{algorithmic}
\end{algorithm*}

\begin{algorithm}[t]
\caption{\small \textsc{OptimizeSparseVector} (Projected Gradient Descent for Estimating Sparse Vector)   \label{algo:optimize_b}}
\begin{algorithmic}[1]
\small
\REQUIRE Data $(\fl{X},\fl{y})\in \bb{R}^{m\times d} \times \bb{R}^m$
where we minimize $\lr{\fl{y}-\fl{X}(\fl{v}^{\star}+\fl{b}^{\star})}_2$ s.t. $\lr{\fl{b}^{\star}}_0 \le k$. Estimate $\fl{v}$ (of $\fl{v}^{\star}$) and initialization $\fl{b}$ (of $\fl{b}^{\star}$).  Iterations $\s{T}$, parameters $\alpha,\beta,\gamma>0$ and suitable constant $c_1 \in [0,1/2]$. 
\FOR{$j$ = 1,2,\dots, $\s{T}$}
\STATE        $\fl{c} \gets \fl{b}-\frac{1}{m}\cdot 
        (\fl{X}^{(i)})^{\s{T}}  (\fl{X}^{(i)}\fl{b}+\fl{X}^{(i)}\fl{v}-\fl{y}^{(i)})\quad  $\texttt{\\// Gradient Descent step on $\fl{b}$ to compute intermediate vector $\fl{c}$} 
        
\STATE   $\Delta \gets  \alpha+c_1\Big(\gamma+\frac{\beta}{\sqrt{k}}\Big)$
and $\fl{b} \gets \s{HT}(\fl{c},\Delta)$ 
\texttt{\\ // Hard-Thresholding operation on $\fl{c}$ to ensure sparsity of $\fl{b}$.}
\STATE $\gamma \leftarrow 2c_1\gamma + 2(\alpha+\frac{c_1}{\sqrt{k}}\beta)$
\ENDFOR
\STATE Return vector $\fl{b}$.
\end{algorithmic}
\end{algorithm}

\textbf{Technical Challenges:} 
To highlight the main technical challenges in analysis, we begin with a warm-up case: the rank-$1$ setting (Theorem \ref{thm:main_toy} in Appendix \ref{app:finetuning}) where $\{\fl{w}^{\star (i)}\}$'s are fixed to be $1$.
Here, the learnable parameters of the $i^{\s{th}}$ user can be written as $\f{\Theta}^{(i)}=\fl{u}^{\star}+\fl{b}^{\star (i)}$ where $\fl{u}^{\star}\in \bb{R}^d$ is a global parameter vector shared across users and $\fl{b}^{\star (i)}\in \bb{R}^{d\times t}$ is just a $k$-sparse vector (see Remark \ref{rmk:special}). \method alternately updates the global parameter $\fl{u}^{\star}$ and the sparse user-specific parameter matrix $\fl{B}^{\star}=[\fl{b}^{\star(1)},\dots,\fl{b}^{\star(t)}]\in \bb{R}^{d \times t}$. The key novel step in the analysis is to track the entry-wise error in the intermediate estimates of $\fl{u}^{\star},\fl{B}^{\star}$ and their true values -- this requires a careful matrix taylor series decomposition of the least squares estimator and a covering argument.   

The general rank-$r$ setting with unknown   $(\fl{W}^{\star})^\s{T}=[\fl{w}^{\star(1)},\dots,\fl{w}^{\star(t)}]\in \bb{R}^{r \times t}$ poses several additional challenges.  First of all, similar to the special case outlined above, simply tracking the norms of $\fl{U}^{+(\ell)}-\fl{U}^{\star}$ is insufficient. Instead we track the norms of the orthonormal matrix $\fl{U}^{+(\ell)}$ projected on the subspace orthogonal to the one spanned by the columns of $\fl{U}^{\star}$ --  given by $(\fl{I}-\fl{U}^{\star}(\fl{U}^{\star})^{\s{T}})\fl{U}^{+(\ell)}$. 
This is because, while analyzing $\fl{U}^{+(\ell)}-\fl{U}^{\star}$, the estimate error of $\fl{W}^{\star}$ leads to a bias term that does not go down with iterations. Our novel approach combines the complementary ideas presented in  \cite{thekumparampil2021statistically}, where the authors consider only the low rank component and not the sparse component and secondly, in  \cite{netrapalli2014non}, where the authors design an AM algorithm for the problem of reconstructing the low rank and sparse components of an input matrix. In contrast, we only observe gaussian linear measurements of the individual columns of the parameter matrix. 

Hence, our analysis involves several crucial steps in each iteration: 1) We track the incoherence of several intermediate matrices corresponding to the latest estimates $\fl{W}^{(\ell)},\fl{U}^{(\ell)}$ of $\fl{W}^{\star},\fl{U}^{\star}$ 2) We also track the $\s{L}_{2,\infty}$ norm of the matrix $(\fl{I}-\fl{U}^{\star}(\fl{U}^{\star})^{\s{T}})\fl{U}^{(\ell)}$ to make progress on learning $\fl{B}^{\star}$. In particular, the second step is the most technically involved component of our analysis.

\subsection{Theoretical Guarantees and Analysis}\label{sec:theory}
As in prior works, we are interested in the few-shot learning regime when there are only a few samples per task. From information theoretic viewpoint, we expect the number of samples per task to scale linearly with the sparsity $k$ and rank $r$ and logarithmically with the dimension $d$. On the other hand, $\fl{U}^{\star}$ has $dr$ parameters and therefore, it is expected that the total number of samples across all tasks scales linearly with $dr$ which implies we would want the number of tasks $t$ to scale linearly with dimension $d$. Note that if the sparse vectors $\{\fl{b}^{\star (i)}\}_{i\in [t]}$ have the same support (or a high overlap between the supports), then the model parameters might not be uniquely identifiable. This is because, in that case, the matrix $\fl{B}^{\star}=[\fl{b}^{\star(1)},\dots,\fl{b}^{\star(t)}]\in \bb{R}^{d \times t}$ can be represented as a low-rank matrix itself. 
To establish identifiability of $\fl{U}^\star$ and sparse matrix $\fl{B}^\star$, we make the following assumption: 
\begin{assumption}[A1]\label{assum:row_sparse}
Consider the matrix $\fl{B^\star}\in \bb{R}^{d \times t}$ whose $i^{\s{th}}$ column is the vector $\fl{b}^{\star (i)}$. Then each row  of $\fl{B^\star}$ is $\zeta$-sparse i.e. $\|\fl{B}^\star_i\|_0\le \zeta$ for all $i\in [d]$, and each column is $k$-sparse.
\end{assumption}

Note that the orthonormal matrix $\fl{U}^{\star}$ cannot have extremely sparse columns otherwise it would be information theoretically impossible to separate columns of $\fl{U}^\star$ from $\fl{B}^\star$. Moreover, similar to \cite{tripuraneni2021provable}, we need to ensure that each task contributes to learning the underlying representation $\fl{U}^\star$. These properties  can be ensured by the standard incoherence assumptions \cite{tripuraneni2021provable,collins2021exploiting,netrapalli2014non}  and thus, we have

\begin{assumption}[A2]\label{assum:task_diversity}
Let $\lambda_1^{\star}$,$\lambda_r^{\star}$ be the largest and smallest eigenvalues of the task diversity matrix $(r/t)(\fl{W}^{\star})^{\s{T}}\fl{W}^{\star}\in \bb{R}^{r\times r}$ where $\fl{W}^{\star}=[\fl{w}^{\star(1)} \; \dots \; \fl{w}^{\star (t)}]^{\s{T}}$. We assume $\fl{W}^{\star}$ and the representation matrix $\fl{U}^{\star}$ are $\mu^{\star}$-incoherent i.e. 
$    \lr{\fl{W}^{\star}}_{2,\infty} \le \sqrt{\mu^{\star}\lambda_r^{\star}}  \text{ and } \lr{\fl{U}^{\star}}_{2,\infty} \le \sqrt{\frac{\mu^{\star}r}{d}}.$
\end{assumption}

Now, we are ready to state our main theorem formally:
\begin{thm}\label{thm:main}
Consider the \lrs problem (\eqref{prob:general}) with $t$ linear regression tasks and samples obtained by \eqref{eq:samples}. Let model parameters satisfy assumptions A1 and A2. Let $\fl{B}^{\star}$ satisfy row sparsity  $\zeta = O\Big(t(r^2\mu^{\star})^{-1}\sqrt{\frac{\lambda_r^{\star}}{\lambda_1^{\star}}}\Big)$, and column sparsity $k=O\Big(d\cdot  (\frac{\lambda_r^{\star}}{\lambda_1^{\star}})^2\Big)$. 
Suppose Algorithm \ref{algo:optimize_lrs1} is initialized with $\fl{U}^{+(0)}$ s.t. $\lr{(\fl{I}-\fl{U}^{\star}(\fl{U}^{\star})^{\s{T}})\fl{U}^{+(0)}}_{\s{F}} = O\Big(\sqrt{\frac{\lambda_r^{\star}}{\lambda_1^{\star}}}\Big)$ and  $\lr{\fl{U}^{+(0)}}_{2,\infty}=O(\sqrt{\mu^{\star}r/d})$, and is run for 
$\s{L}=\widetilde{O}(1)$ iterations 
Then, with high probability, the outputs $\fl{U}^{+(\s{L})},\{\fl{b}^{(i,\s{L})}\}_{i\in [t]}$ satisfy:  
\begin{align}
\hspace*{-10pt}
     &\lr{(\fl{I}-\fl{U}^{\star}(\fl{U}^{\star})^{\s{T}})\fl{U}^{+(\s{L})}}_{\s{F}} = \frac{\widetilde{O}(1)\sigma\s{S}}{\sqrt{\mu^{\star}\lambda^{\star}_r}}, \text{ and }\left|\left|\fl{b}^{(i, \s{L})}-\fl{b}^{\star (i)}\right|\right|_{\infty} \le \frac{\widetilde{O}(1)\sigma \s{S}}{\sqrt{k}}, i\in [t],\nonumber
\end{align}
respectively, where $\s{S}=\Big(\mu^\star\sqrt{\frac{r^3d}{mt}}+ \sqrt{\frac{r^3}{m\lambda^{\star}_r}} + \sqrt{\frac{k}{m}}\Big)$ provided the total number of samples satisfies: 
\begin{align}
&m=\widetilde{\Omega}\Big(k + r^2\mu^\star\Big(\frac{\lambda_1^\star}{\lambda_r^\star}\Big)^2 + \frac{\sigma^2r^3}{\lambda_r^\star}\Big),\nonumber\\ &mt=\widetilde{\Omega}\Big(r^3d\mu^{\star}\Big(r(\mu^\star)^4(\lambda_r^\star)^2k + \mu^\star\Big(\frac{\lambda_1^\star}{\lambda_r^\star}\Big)^2 + \sigma^2\Big(1+\frac{1}{\lambda_r^\star}\Big)\Big)\Big)\nonumber.
\end{align}

For a new task with $m$ additional labelled samples and task specific parameters $\fl{w}^{\star}\in \bb{R}^r$, $\fl{b}^{\star}\in \bb{R}^d, \lr{\fl{b}^{\star}}_0=k$, a modified version of \method (see Alg. ~\ref{algo:generalize} in Appendix \ref{app:generalization}) computes an estimate $\fl{w},\fl{b}$ of $\fl{w}^{\star},\fl{b}^{\star}$ such that with high probability, we have the following generalization bound: 
$$\ca{L}(\fl{U}^{+(\s{L})},\fl{w},\fl{b})-\ca{L}(\fl{U}^{\star},\fl{w}^{\star},\fl{b}^{\star}) =\widetilde{O}\Big(\sigma^2\Big(\s{S}^2 + \frac{k+r}{m}\Big)\Big).$$
\end{thm}

Note that the per-task sample complexity of our method roughly scales as $m=(r^3+k)$, which is information theoretically optimal in $k$ and is roughly $r^2$ factor larger. Total sample complexity scales as $mt=kdr^4$, which is roughly $kr^3$ multiplicative factor larger than the information theoretic bound. Note that typically $r$ and $k$ are considered to be small, so the additional factors are small, but we leave further investigation into obtaining tighter bounds for future work. Finally, the generalization error scales as $\sigma^2 (r+k)/m$ which is nearly optimal. Note that, ignoring the \lrs framework, and directly learning the parameters of each task separately would lead to significantly larger error of $\sigma^2 d/m$.

\begin{rmk}[Runtime and Memory]
The run-time of Algorithm \ref{algo:optimize_lrs1} is dominated by the update for $\fl{U}^{+(\ell)}$. For each iteration $\ell$, Step 8 has a time complexity of $O((dr)^3+(mt)(dr)^2)$; however in practice, a gradient descent step for the update of $\fl{U}^{(\ell)}$ can bring down the time complexity to $O(mtdr)$. Moreover, the memory usage of Algorithm \ref{algo:optimize_lrs1} is $O((dr)^2+tr^2)$.
\end{rmk}

\begin{rmk}[Initialization]
Note that Algorithm \ref{algo:optimize_lrs1} has local convergence properties as described in Theorem \ref{thm:main}. In practice, typically we use random initialization for $\fl{U}^{+(0)}$. However, similar to the representation learning framework in \cite{tripuraneni2021provable}, we can use the Method of Moments to obtain a good initialization. See Appendix \ref{app:mom} for more details. 
\end{rmk}

\begin{rmk}[Special Settings]\label{rmk:special}
Consider the setting where, for each task, we just need to learn a single central model for all tasks and sparse 
fine-tune the weights for each task i.e. $\fl{w}^{\star (i)}=1$ for all $i \in [t]$ is fixed. In this case, \method obtains global convergence guarantees (Theorem \ref{thm:main_toy} in Appendix \ref{app:finetuning}). Moreover, if the central model $\fl{U}^{\star}$ is also frozen, then the task-based sparse fine-tuning reduces to standard compressed sensing. In the realizable setting, our framework recovers the standard generalization error of $\sigma^2 k/m$ in compressed sensing \citep{jain2017non}[Chapter 7].
\end{rmk}

\section{Private Linear \lrs: Privacy Preserving Personalization} \label{sec:dp}

We now extend our framework, algorithm and analysis to allow user-level differential privacy.
In this section, we provide a user level DP variant of Algorithm~\ref{algo:optimize_lrs1} in the billboard model. 
We obtain DP for the computation of each  $\fl{U}^{(\ell)}$ by perturbing the covariance matrix $\fl{A}$ and the linear term $\fl{V}$ in the algorithm with Gaussian noise to ensure that the contribution of any single user is protected. 
We start by introducing the function $\s{clip}:\bb{R}\times \bb{R}\rightarrow \bb{R}$ that takes as input a scalar $x$, parameter $\rho$ and returns $\s{clip}(x,\rho)=x\cdot\min\Big\{1, \frac{\rho}{x}\Big\}$. We can extend the definition of $\s{clip}$ to vectors and matrices by using $\s{clip}(\fl{v},\rho) = \fl{v}\cdot\min\Big\{1, \frac{\rho}{\|\fl{v}\|_2}\Big\}$ for a vector $\fl{v}$ and $\text{clip}(\fl{A},\rho) = \fl{A}\cdot\min\Big\{1, \frac{\rho}{\|\fl{A}\|_\s{F}}\Big\}$ for a matrix $\fl{A}$. In order to ensure that Algorithm \ref{algo:optimize_lrs1} is private, for input parameters $\s{A}_1,\s{A}_2,\s{A}_3,\s{A}_w$, we first clip the covariates and responses: for all $i\in [t],j \in [m]$, we will have $\widehat{\fl{x}^{(i)}_j}\leftarrow {\sf clip}\left(\fl{x}^{(i)}_j,\s{A}_1\right)$, $\widehat{y^{(i)}_j} \leftarrow {\sf clip}\left(y^{(i)}_j,\s{A}_2\right),$ $\widehat{(\fl{x}^{(i)}_j)^\s{T}\fl{b}^{(i, \ell)}}\leftarrow {\sf clip}\left((\fl{x}^{(i)}_j)^\s{T}\fl{b}^{(i, \ell)},\s{A}_3\right)$ and $\widehat{\fl{w}^{(i, \ell)}}\leftarrow {\sf clip}\left(\fl{w}^{(i, \ell)},\s{A}_w\right)$. Now, 
we can modify Line 7 in Algorithm \ref{algo:optimize_lrs1} as follows (let $\s{L}$ be the number of iterations of Alg. \ref{algo:optimize_lrs1}):
\begin{align}
    &\fl{A} := \frac{1}{mt}\Big(\sum_{i \in [t]}\Big(\widehat{\fl{w}^{(i, \ell)}}(\widehat{\fl{w}^{(i, \ell)}})^{\s{T}} \otimes \Big(\sum_{j=1}^{m}\widehat{\fl{x}^{(i)}_j}(\widehat{\fl{x}^{(i)}_j})^{\s{T}}\Big)\Big)  + \fl{N}^{(1)}\Big)    \label{line:opq1}\\
    &\fl{V} := \frac{1}{mt}\Big(\sum_{i \in [t]}\sum_{j \in [m]} \widehat{\fl{x}^{(i)}_j}\Big(\widehat{y^{(i)}_j} - (\widehat{\fl{x}^{(i)}_j)^{\s{T}}\fl{b}^{(i, \ell)}} \Big)(\widehat{\fl{w}^{(i, \ell)}})^{\s{T}} + \fl{N}^{(2)}\Big)\label{line:opq2}
\end{align}
where, for some $\sigma_{\s{DP}}>0$, each entry of $\fl{N}^{(1)}$ is independently generated from $\mathcal{N}\left(0,m^2\cdot \s{A}_1^4\cdot \s{A}_w^4\cdot {\sf L}\cdot\sigma_{\s{DP}}^2\right)$; similarly, each entry of $\fl{N}^{(2)}$ is independently generated from  $\mathcal{N}\left(0,m^2\cdot \s{A}_1^2(\s{A}_2+\s{A}_3)^2\s{A}_w^2\cdot {\sf L}\cdot\sigma_{\s{DP}}^2\right)$. We are now ready to state our main result:

\begin{thm}\label{thm:private1}
Algorithm \ref{algo:optimize_lrs1} (with modifications mentioned in~\eqref{line:opq1} and~\eqref{line:opq2}) satisfies $\sigma_{\s{DP}}^{-2}-\s{zCDP}$ and correspondingly satisfies $(\varepsilon, \delta)$-differential privacy in the billboard model, when we set the noise multiplier $\sigma_\s{DP} \geq 2\varepsilon^{-1}\sqrt{(\log(1/\delta)+\varepsilon)}
$. Furthermore, if $\varepsilon \leq \log(1/\delta)$, then $\sigma_\s{DP} \geq
\varepsilon^{-1}\sqrt{8 \log(1/\delta)}$ suffices to ensure $(\varepsilon, \delta)$-differential privacy.
\end{thm}


Next, we characterize the generalization properties of modified \method:

\begin{thm}\label{thm:private2}
Consider the LRS problem \eqref{prob:general} with all parameters $m,t,\zeta$ obeying the bounds stated in Theorem \ref{thm:main} and furthermore, $t=\widetilde{\Omega}(\frac{(rd)^{3/2}\sqrt{\log(1/\delta)+\epsilon}}{\epsilon}\mu^{\star})$.
Suppose we run \method (Step 7 in Alg. \ref{algo:optimize_lrs1} replaced with \ref{line:opq1} and \ref{line:opq2}) for $\s{L}=\widetilde{O}(1)$ iterations with $\s{A}_1 = \widetilde{O}(\sqrt{d}), \s{A}_2 = \widetilde{O}(\sqrt{\mu^\star\lambda_r^\star} + (\max_i \|\fl{b}^{\star(i)}\|_2)), 
\s{A}_3= \widetilde{O}\Big(\lambda_r^\star\sqrt{\frac{\mu^\star}{\lambda_1^\star}}\Big), \s{A}_w = \widetilde{O}(\sqrt{\mu^\star\lambda_r^\star})$. 
Then, with high probability, generalization error for a new task, with $m$ additional labelled samples and task specific parameters $\fl{w}^{\star}\in \bb{R}^r$, $\fl{b}^{\star}\in \bb{R}^d, \lr{\fl{b}^{\star}}_0=k$, satisfies:
\begin{align*}
&\ca{L}(\fl{U},\fl{w},\fl{b})-\ca{L}(\fl{U}^{\star},\fl{w}^{\star},\fl{b}^{\star}) =\widetilde{O}\Big(\sigma^2\s{S}^2+\frac{dr^2(\log(1/\delta)+\epsilon)(\lambda_r^{\star}\mu^{\star})^2}{\epsilon^2 t^2}\cdot (\kappa^2+r^2d^2) \Big)    
\end{align*}
where $(\fl{w},\fl{b})$ are estimates of $(\fl{w}^{\star},\fl{b}^{\star})$ obtained by Alg. \ref{algo:generalize}, $\s{S}=\Big(\mu^\star\sqrt{\frac{r^3d}{mt}}+ \sqrt{\frac{r^3}{m\lambda^{\star}_r}} + \sqrt{\frac{k}{m}}\Big)$, $\eta=\widetilde{O}\Big(t^{-1}\mu^\star r^2d^{3/2}\Big(1+\sqrt{\frac{\lambda_r^\star}{\lambda_1^\star}} + \max_{i\in[t]}\frac{\|\fl{b}^{\star(i)}\|_2}{\sqrt{\mu^\star\lambda_r^\star}}\Big)\sigma_{\s{DP}}\Big)$ and $\kappa=1+\sqrt{\frac{\lambda_r^\star}{\lambda_1^\star}} + \max_{i\in[t]}\frac{\|\fl{b}^{\star(i)}\|_2}{\sqrt{\mu^\star\lambda_r^\star}}$.
\end{thm}
Note that the modified \method ensures $(\epsilon,\delta)$ differential privacy 
without any assumptions. However, Theorem \ref{thm:private2} still has good generalization properties; moreover, the per-task sample complexity guarantee $m$ still only needs to scale polylogarithmically with the dimension $d$. In other words, our algorithm can ensure good generalization along with privacy in data-starved settings as long as the number of tasks is large -  scales as $\sim d^{3/2}/\epsilon$. Similarly, generalization error for a new task has two terms: the first has a standard dependence on noise $\sigma^2$ and the second has a scaling of $d^3(\epsilon t)^{-2}$ which is standard in private linear regression and private meta-learning \citep{smith2017interaction,modelperosnalization}. 
Detailed proofs of our main results namely Theorems \ref{thm:main},\ref{thm:private2} are delegated to Appendix \ref{app:detailed_proof}, \ref{app:generalization}.




\section{Empirical results}\label{sec:experiments}\vspace*{-1pt}

\begin{figure*}[!htbp]
\centering\vspace*{-10pt}
\begin{subfigure}[t]{0.33\textwidth}
    \includegraphics[width=\linewidth]
    {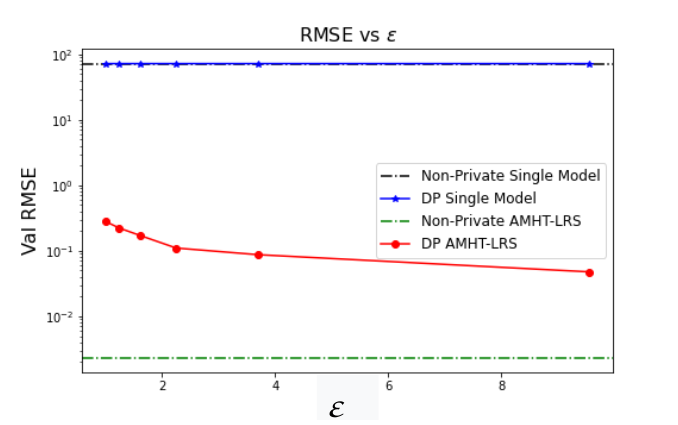}
    \caption{\small DP-\method Simulations}\label{fig:dp_simul_overall_RMSE}
\end{subfigure}
\begin{subfigure}[t]{.33\textwidth}
    \includegraphics[width=\linewidth]
    {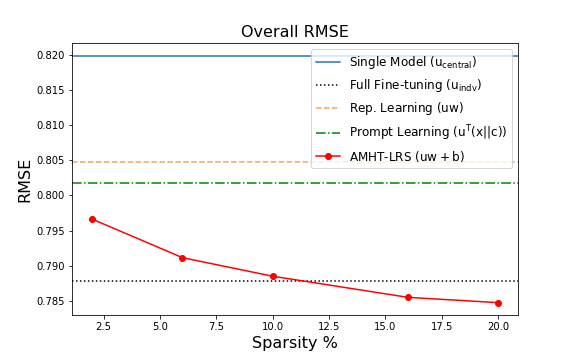}
    \caption{MovieLens}\label{fig:movieLens_overall_RMSE_main}
\end{subfigure}\hfill
\begin{subfigure}[t]{0.33\textwidth}
    \includegraphics[width=\linewidth]
    {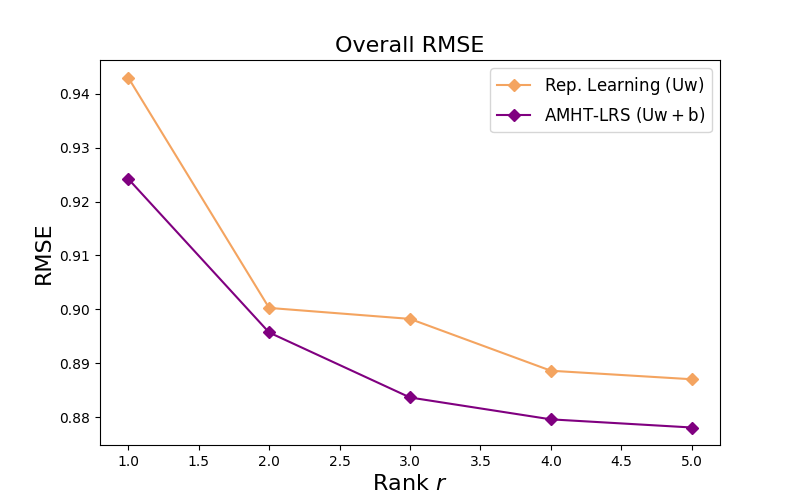}
    \caption{Netflix}\label{fig:netflix_1k_overall_RMSE}
\end{subfigure}\hfill
\caption{
\small
Comparison of \method with algorithms for other baseline frameworks on different datasets. Fig. a demonstrates that the overall RMSE on the simulated data for both the private and non-private versions of \method is far lower than the corresponding variants of Single Model. In Fig. b, on Movielens data, rank 1 version of \method (at 15\% sparsity) outperforms other baselines with small parameter overhead/user. In Fig. c, for Netflix data, we plot RMSE v/s rank curves for \method (at 2\% sparsity) and representation learning. We observe that for each rank value, \method improves upon the performance of representation learning.}
\label{fig:real_world_linear_overall}
\vspace*{-10pt}
\end{figure*}

In this section, the goal is to validate our theoretical guarantees for linear \lrs and compare \method with algorithms for other  baseline frameworks:
1) \textbf{Single Model ($\sf{u_{central}}$)}: learns a single model for all tasks, 2) \textbf{Full Fine-tuning ($\sf{u_{indv}}$)}  separate model for each task aka standard fine-tuning, 3) \textbf{Representation Learning or Rep. Learning $(\sf{uw})$}: only low rank model \citep{thekumparampil2021statistically} and 4) \textbf{Prompt Learning ($\sf{u^\s{T}(x \mid\mid c}$):} Modified covariate by  concatenation with a task-embedding vector. \\ 
{\bf a. \lrs on Recommendation Datasets}: We experiment with 2 datasets: MovieLens1M and Netflix. For the Movielens1M dataset, we have $241$ users with varying number of rated movies ($22-3070$). For the Netflix dataset, we have $1000$ users, each with at most $100$ rated movies. For Movielens1M dataset, we compare the rank-$1$ version of \method with other baselines where we vary the sparsity.
With only $15\%$ sparsity, \method becomes superior to other baselines.
For Netflix dataset, we compare the Representation Learning baseline with \method and vary the rank at $2\%$ sparsity - the remaining baselines perform significantly worse (Appendix \ref{app:experiments_detailed}). Clearly, \method provides gains on both datasets and corroborate our theory. Moreover, the number of additional parameters per task is also significantly small for \method (refer to Table~\ref{table:linear_memory} for exact numbers).     

{\bf b. DP \lrs Simulations}: Here, for a synthetic dataset, we note that while Differentially Private (DP) \method performs quite well for each $\epsilon$, both the private and non-private Single Model baselines fare badly even on higher values of $\epsilon$. Further, DP \method  achieves RMSE comparable to its non-private version by $\epsilon \approx 2$ mark.

Detailed version of these experiments is provided in Appendix \ref{app:experiments_detailed}. Additional experiments with neural network architectures and other synthetic datasets is provided in Appendix \ref{app:experiments_second}.

\section{Conclusion and Future Work}
\label{sec:conclusion}\vspace*{-6pt}
We presented a novel framework \lrs for model personalization that can scale to many users/domains, be accurate and preserves privacy provably. We proposed \method, that combines alternate minimization -- popular in representation learning -- with hard thresholding based methods. In the linear model with Gaussian data, we rigorously proved that \method is statistically and computationally efficient, and is able to generalize to new users with only $O(r+k)$ samples, where $r$ is the representation learning dimension and $k$ is the number of fine-tuning weights. Finally, we extended our result to ensure that privacy of each {\em user} is  preserved despite  sharing information. Extending our framework to non-realizable setting and adversarial settings are critical future directions.

\noindent \textbf{Limitations:} The main contributions of our works are theoretical. From a theoretical point of view, the limitations of our paper are discussed in Section \ref{sec:conclusion}. In particular, we believe that extending the theoretical analysis of \lrs framework to non-realizable settings is of significant importance - however, this will require non-trivial and novel technical/algorithmic ideas - we leave this as an important direction of future work.

\bibliography{bibfile}
\bibliographystyle{unsrtnat}

\newpage 
\appendix



\section{Empirical results}\label{app:experiments_detailed}\vspace*{-1pt}

\begin{figure*}[!htbp]
\centering
\begin{subfigure}[t]{.33\textwidth}
    \includegraphics[width=\linewidth]
    {figures/ICLR_movielens_overallRMSE.png}\vspace*{-5pt}
    \caption{MovieLens}\label{fig:movieLens_overall_RMSE}
\end{subfigure}\hfill
\begin{subfigure}[t]{0.33\textwidth}
    \includegraphics[width=\linewidth]
    {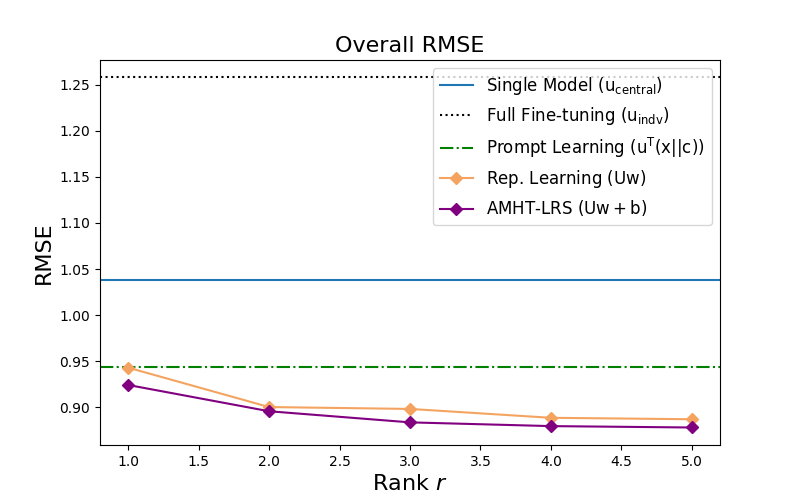}
    \caption{Netflix}\label{fig:netflix_overall_RMSE}
\end{subfigure}\hfill
\begin{subfigure}[t]{0.33\textwidth}
    \includegraphics[width=\linewidth]
    {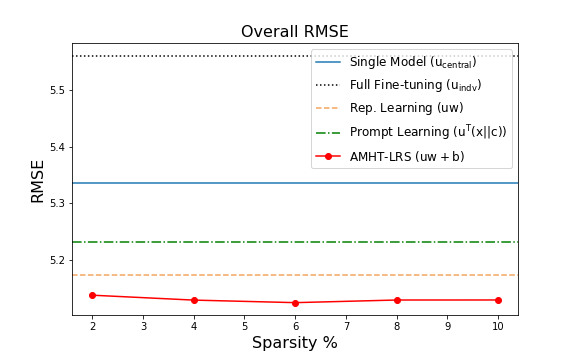}
    \caption{Jester}\label{fig:jester_overall_RMSE}
\end{subfigure}\vspace*{-10pt}
\caption{
\small
RMSE on different datasets for various methods as we increase the user-specific fine-tunable parameters. Note that \method outperforms other baselines with a small amount of parameter overhead/user. Full fine-tuning generalizes poorly in the more data-starved Netflix and Jester datasets.}
\label{fig:real_world_linear_overall_appendix}
\vspace*{-10pt}
\end{figure*}

\textbf{Compute Resources:} For all of our experiments we used Google Colaboratory, with a single V100 GPU having 12 GB RAM.

In this section, we are going to expand upon the information presented in Section~\ref{sec:experiments} and fill the gaps in implementation details. Recall that we developed \method with the following two goals: a) demonstrate that personalization with \method indeed improves accuracy for tasks with a small number of points, b) for a fixed budget of parameters, \method is significantly more accurate than existing baselines. 

Also, recall the nomenclature for the following baselines mentioned in the main paper: 1) \textbf{Single Model ($\sf{u_{central}}$)}: learns a single model for all tasks, 2) \textbf{Full Fine-tuning ($\sf{u_{indv}}$):}  separate model for each task aka standard fine-tuning, 3) \textbf{Representation Learning or Rep. Learning $(\sf{uw})$}: only low rank model \citep{thekumparampil2021statistically} and 4) \textbf{Prompt Learning ($\sf{u^\s{T}(x \mid\mid c}$):} Modified covariate by  concatenation with a task-embedding vector. Note that the models considered in \cite{hu2021lora, chua2021fine, denevi2018learning} all reduce to Full fine-tuning models (with much higher memory footprint) in the experiments that we consider. In the case of Rep. Learning and \method where $r > 1$, we will replace "$\sf{u}$" by "$\sf{U}$" in their respective nomenclature.
Lastly, the above approaches are not just restricted to linear models and can be extended to complex model classes such as Neural Networks (see Appendix~\ref{app:nn_experiments} for extension to 2-3 layer Neural Net architectures and Table~\ref{table:nn_memory} for memory footprint comparisons). 
Additional experiments on synthetic data can also be found in Appendix \ref{app:lrs_linear_synthetic_experiments}. 
Here, we focus on two sets of experiments: \\

\subsection{Linear Models on Recommendation Datasets}
We compare the performance of \method against the baselines mentioned above on 3 recommendation datasets: MovieLens1M, Netflix and Jester. For each dataset, we define training tasks at a user-level or with small groups of users. 
\subsubsection{MovieLens Dataset}\label{subsec:movielens_data}
The MovieLens 1M dataset comprises of 1M ratings of 6K users for 4K movies. Each user is associated with some demographic data namely gender, age group, and occupation in the MovieLens dataset. We partition the users into $241$ disjoint clusters where each cluster represents a unique combination of the demographic data. Each user group thus represents a "task" in the language of our paper. We partition the data into training and validation in the following way: for each task, we randomly choose $20\%$ movies rated by at least one user from that task and put all ratings made by users from that task for the chosen movie into the validation set. The remaining ratings belong to the training set. Based on the ratings in the training set, we fit a matrix of rank $50$ onto the ratings matrix and obtain a $50$ dimensional embedding of each movie. Thus, we ensure that there is no data leakage during embeddings generation. For each task, the samples consist of (movie embedding, average rating) tuples; the response is the average rating of the movie by users in that task. The number of samples per task varies from 22 to 3070 - clearly many clusters are data starved. We use the training data to learn the different models (see Section \ref{sec:experiments}; with hyper-parameter tuning) and use them to predict the ratings for the validation data. 
 
{\em Empirical Observations on MovieLens validation dataset}: 
With respect to the single model as reference, in the linear rank-$1$ case, the representation learning and the prompt learning based baselines have $1$ and $50$ additional parameters per task respectively; they are unable to personalize well. In contrast, with only $10\% (= 5)$ additional parameters per task, \method has smaller RMSE than fully fine-tuned model, which require 241x more parameters. See Figures \ref{fig:real_world_linear_overall} and Table~\ref{table:linear_memory} for clear comparisons.

\subsubsection{Netflix Dataset}\label{subsec:netflix_data}

We consider the Netflix Challenge dataset comprising of 17k users and 480k movies where ratings are provided as integers on a scale of $1-5$. We choose the top 1000 users who have rated the most movies and top 100 movies that have been rated the most and consider the $1000\times 100$ rating matrix restricted to these sets of top users and movies - this rating matrix comprises approximately 90k ratings. 
We perform the train-validation split in the following way: for $70$ users, we keep $10\%$ of their ratings in the training set (small data/user); for $70$ users, we kept $50\%$ of their ratings in their training set (medium data/user) and for the rest of $60$ users, we kept $90\%$ of their ratings in their training set (large data/user). 
All the observed ratings restricted to the $1000\times 100$ ratings matrix that are absent in the training set is inserted into the validation set. By using standard low rank matrix completion techniques \citep{chen2020noisy}, we complete the ratings matrix by minimizing the MSE w.r.t to the entries in the training set with a nuclear norm regularizer. Following this, we compute SVD $\fl{U}\f{\Sigma}\fl{V}^{\s{T}}$ and take the first $50$ columns of $\fl{V}$; this results in a truncated $1000\times 50$ dimensional matrix $\widehat{\fl{V}}$ where each row corresponds to a $50$-dimensional embedding of each movie. As before, the train-validation data split before embedding generation ensures that there is no data leakage while creating the embeddings. For each task representing each user, the samples consist of (movie embedding, average rating) tuples; the response is the average rating of the movie given by users in that task.  
For this experiment, instead of varying the number of fine-tunable parameters, we fix it at $2\%$ ($= 1$ parameter since $d=50$), and vary the rank $r$ for both rep. learning and \method. The different models (see Section \ref{sec:experiments}) are trained with some  hyper-parameter tuning) used to predict the ratings for the validation data. 

{\em Empirical Observations on Netflix validation data}:
We notice all the baselines other than rep. learning seem to perform significantly worse. Further, at each value of the rank $r$, \method with just $(r-1)*50 \text{ (from }\sf{U}) + (r-1)*1 \text{ (from }\sf{w}) =  (r-1)*51$ number of additional parameters shared across all users and 1 additional parameter (since $2\%$ of $d=50$ is 1) per user, outperforms all the other baselines. See Figures \ref{fig:real_world_linear_overall} and Table~\ref{table:linear_memory} for clear comparisons.

\subsubsection{Jester Dataset}\label{subsec:jester_data}

The Jester dataset comprises of 4.1M ratings from 73k users for 100 jokes with each rating being on a scale of $-10.0$ to $+10.0$. We choose 100 users who have rated all the 100 jokes and consider the $100\times 100$ rating matrix restricted to these users and jokes - this rating matrix is entirely filled. Similar to the Netflix dataset, we perform the train-validation split in the following way: for $30$ users, we keep $10\%$ of their ratings in the training set (small data/user); for $40$ users, we kept $50\%$ of their ratings in their training set (medium data/user) and for the rest of $30$ users, we kept $90\%$ of their ratings in their training set (large data/user). We use the training data to learn the different models (with some  hyper-parameter tuning) and use them to predict the ratings in the validation data. 

{\em Empirical Observations on Jester validation dataset}: 

The conclusions are mostly similar as in the Netflix and MovieLens cases. With only $2\% (= 1)$ additional parameter per task, \method has smaller RMSE than fully fine-tuned model, which require $100$x more parameters and the other baselines such as representation learning ($1$ additional parameter) and prompt learning ($50$ additional parameters). This holds true for both data-starved and data-surplus tasks. See Figures \ref{fig:real_world_linear_overall} and Table~\ref{table:linear_memory} for clear comparisons.

\subsubsection{DP \lrs Simulations} 
Here, for each task $i\in [t]$, we generate $m=100$ samples $\{(\fl{x}^{(i)}_j,y^{(i)}_j)\}_{j \in [m]}$ where $\fl{x}^{(i)}_j \sim \ca{N}(\f{0},\fl{I}_{d\times d})$, $y^{(i)}_j=\langle \fl{x}^{(i)}_j, \fl{u}^{\star}w^{\star(i)}+\fl{b}^{\star (i)} \rangle$ and $w^{\star(i)} = 1$ is fixed for simplicity. We select number of tasks $t = 5000$, data dimensions $d=10$, $k,\zeta$, and both the column and row sparsity level of $\{\fl{b}^{\star (i)}\}_{i\in [t]}$ to be $2$. We sample  $\fl{u}^{\star}$ uniformly from the unit sphere and the non-zero elements of $\{\fl{b}^{\star (i)}\}_{i\in [t]}$ are sampled i.i.d. from $\ca{N}(0,1)$ with the non-zero indices selected randomly. We run the algorithm for 15 epochs and use RDP sequential composition to compute the privacy risk accumulated over the epochs. We set smallest possible clipping norm values for $\s{A}_1, \s{A}_2$ and $\s{A}_3$ s.t. most samples don't get clipped. We fix $\delta = 10^{-5}$. Finally we plot the RMSE on the test set for different values of $\epsilon$.

\begin{figure}[!ht]
\vspace*{-10pt}
\centering
   \includegraphics[scale=0.4]{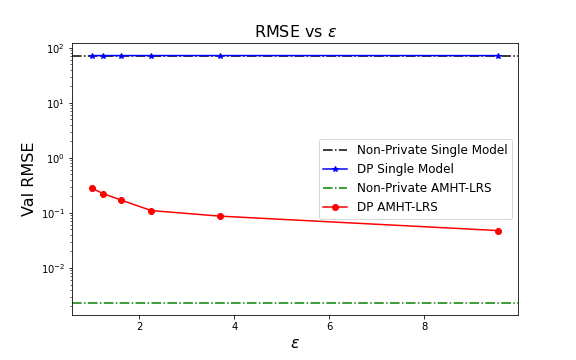}\vspace*{-5pt}
\vspace*{-15pt}
 \caption{\small  Comparison of Overall RMSE on the simulated data for both the private and non-private versions of \method and the Single Model Baseline.s}
\vspace*{-5pt}
\end{figure}\label{fig:dp_simul}

{\em Empirical Observations}: We note that while DP \method performs quite well for each $\epsilon$, both the private and non-private Single Model baselines fare badly even on higher values of $\epsilon$. Further, DP \method  achieves RMSE comparable to its non-private version by $\epsilon \approx 2$ mark.

\section{Additional Experiments and Setup}\label{app:experiments_second}

\begin{figure}[h]
\centering 
    \includegraphics[scale = 0.40]{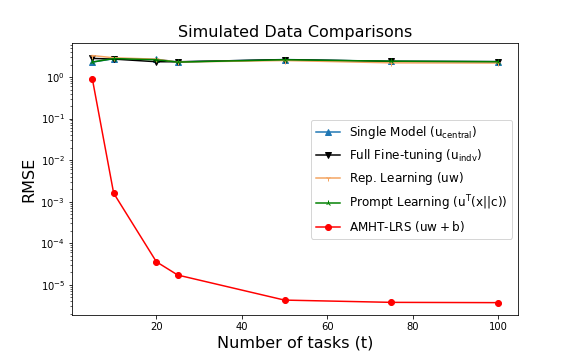}
  \caption{\small  Decrease in RMSE on Synthetic data for AMHT-LRS on increase in fine-tunable parameters}
  \label{fig:simulated}
  \vspace{-5pt}
  \smallskip
\end{figure}

\subsection{Synthetic experiments with non-private algorithms}\label{app:lrs_linear_synthetic_experiments}



\textbf{Synthetic dataset}: Here, for each task $i\in [t]$, we generate $m=100$ samples $\{(\fl{x}^{(i)}_j,y^{(i)}_j)\}_{j \in [m]}$ where $\fl{x}^{(i)}_j \sim \ca{N}(\f{0},\fl{I}_{d\times d})$, $y^{(i)}_j=\langle \fl{x}^{(i)}_j, \fl{u}^{\star}w^{\star(i)}+\fl{b}^{\star (i)} \rangle$. We select $d=150$, set $k,\zeta$, the column and row sparsity level of $\{\fl{b}^{\star (i)}\}_{i\in [t]}$ to be $10$ and $5$, respectively. We sample  $\fl{u}^{\star}$ uniformly from the unit sphere; non-zero elements of $\{\fl{b}^{\star (i)}\}_{i\in [t]}$ and $w^{\star(i)}$ are sampled i.i.d. from $\ca{N}(0,1)$ with the indices of zeros selected randomly.



Figure~\ref{fig:simulated} shows that not only having a single model can lead to poor performance, but a fully fine-tuned model per task can also be highly inaccurate as scarcity of data per task can leading to over-fitting.  Finally, low-rank representation learning as well as prompt tuning based techniques do not perform well due to lack of modeling power. In contrast, our method recovers the underlying parameters -- as also predicted by Theorem~\ref{thm:main} -- and provides 5 orders of magnitude better RMSE.

\subsection{Experiments with Neural Networks}\label{app:nn_experiments}

As described in Section \ref{sec:experiments}, our techniques/approaches can be extended to complex classes of non-linear model. To demonstrate this, we fix the class of models to Neural Networks (denoted by $\s{F}:\bb{R}^{d} \rightarrow \bb{R}$). Similar to Section \ref{sec:experiments}, we consider the following baselines:
 1) \textbf{Single Model ($\s{F}(x;\s{u}_{\text{central}})$)}: learns a single Neural Network model for all tasks, 2) \textbf{Full Fine-tuning ($\s{F}(x;\s{u}_{\text{indv}})$)} Learns a separate fully-trained Neural Network model for each task, 3) \textbf{Rep. Learning ($\s{F}(x;\s{uw})$)}: 
 Learns separate Neural Networks for each task such that the NN parameters of each task lie on a low dimensional manifold.
 4) \textbf{Prompt Learning ($\s{F}(x\mid \mid c;\s{u}_{\text{central}})$):} The covariate is concatenated with a trainable embedding of the corresponding task and a single Neural network model is trained with the modified covariates.
\method ($\s{F}(x;\s{uw}+\s{b}_{\text{sparse}})$) itself trains a separate Neural Network model for each task but assumes that the entire set of parameters of the  Neural networks for each task  can be represented as a Low Rank+Sparse matrix. And as before, for the the case of Rep. Learning or \method extensions to Neural Networks, if we use $r > 1$, we will replace "$\sf{u}$" by "$\sf{U}$" in their respective nomenclature.

As in the case of linear models, we use the training data to learn the parameters of different models (with some  hyper-parameter tuning) described above and use them to predict the ratings in the validation data. The overall average validation RMSE for \method and the $4$ different baselines (modified for neural networks) that we consider against different amounts of sparsity in $\s{b}_{\text{sparse}}$ is computed (shown in figures for each of three datasets that we experiment with). The memory footprint of the different methods (for each of the three datasets) has been provided in Table \ref{table:nn_memory}.

\paragraph{More Details on Experimental Setup:} To compute the single model and fully fine-tuned model metrics, we used batched gradient descent. To compute the low rank model metrics, we performed alternating optimization as per the algorithm described in  \citep{thekumparampil2021statistically}. Finally, to compute \method metrics, we used an $L_2$ regularization for each $\fl{b}^{(i, \ell)}$. For all the gradient based methods, we used Adam/AdamW Optimizer with weight decay and learning rate scheduler. We experimented with Cosine Annealing and Decay on Plateau schedulers. We performed a search over learning rates, $L_2$  weight decay values and learning rate scheduler hyperparameters (decay factor for Decay on Plateau and window size for Cosine Annealing) and reported the model metrics which gave the best overall RMSE on the validation dataset.

\subsubsection{Netflix Dataset}

\paragraph{\lrs with 2 layer Neural Net for Netflix:} 
For the Neural Network (NN) experiments on Netflix dataset, we consider the function class $\s{F}$ - a 2 layer Neural Net with a single hidden layer of 50 neurons and tanh activation.
 The training and the validation data on the Netflix dataset is same as created for the linear models (see Section \ref{subsec:netflix_data}). 
 The comparison of validation RMSE of \method and all the $4$ baselines corresponding to the Netflix dataset is given in Figure \ref{fig:Netflix_NN_overall}. 
 
 Observe that \method outperforms the rest of the baselines with a small memory overhead (see Table \ref{table:nn_memory}). In particular, the improvement in performance is achieved along with a significant improvement in memory cost compared to Full fine-tuning - \method (with only 1\% sparsity/tunable parameters for each user) outperforms the Full-Finetuning baseline at each rank $r$ value.
 
\begin{figure*}[!h]
\begin{subfigure}[t]{0.49\textwidth}
\centering
   \includegraphics[width=\linewidth]{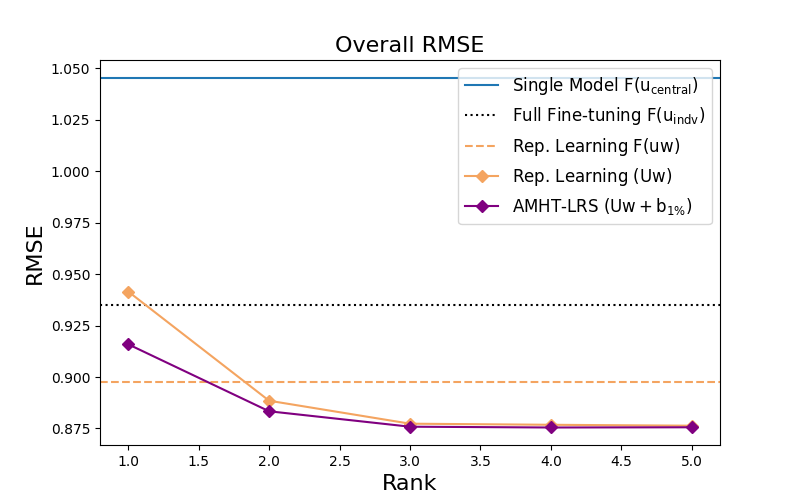}\vspace*{-5pt}
   \caption{\small Overall comparison of RMSE.}
       ~\label{fig:overall_netflix_1k_1}
 \end{subfigure}\hfill
 \begin{subfigure}[t]{0.49\textwidth}
\centering
   \includegraphics[width=\linewidth]{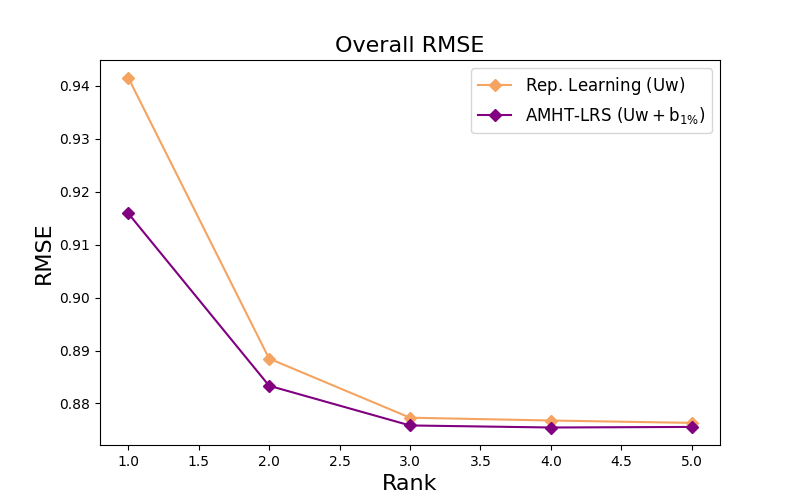}\vspace*{-5pt}
   \caption{\small Rep. Learning v/s \method}
       ~\label{fig:overall_netflix_1k_2}
 \end{subfigure}\hfill
\vspace*{-15pt}
 \caption{\small  Comparison of Overall RMSE on the Netflix validation data achieved by \method and the different baselines we consider where the training is modified for toy Neural Network models with a single hidden layer of $50$ neurons and tanh activation. \method outperforms other baselines at each rank value (see Table \ref{table:nn_memory} for exact numbers on model parameters).}
\end{figure*}\label{fig:Netflix_NN_overall}

\subsubsection{Jester Dataset}

\paragraph{\lrs with 2 layer Neural Net for Jester:}

For the Neural Network (NN) experiments on Jester dataset, we consider the function class $\s{F}$ - a 2 layer Neural Net with a single hidden layer comprising 50 neurons and tanh activation.
As before, the training and the validation data is the same that was created for the case of linear models (see Sec. \ref{subsec:jester_data}). 
The comparison of validation RMSE of \method and all the $4$ baselines corresponding to the Jester dataset is given in Figure \ref{fig:overall_jester2}. 
 
 Again, we observe that \method  outperforms the rest of the baselines with a small memory overhead (see Table \ref{table:nn_memory}). As before, the improvement in performance is achieved along with a significant improvement in memory cost compared to Full fine-tuning - \method (with only 2\% sparsity/tuneable parameters for each user) outperforms the Full-Finetuning baseline with only $1.5\%$ of the corresponding number of parameters.

\begin{figure*}[!h]
\begin{subfigure}[t]{\textwidth}
\centering
   \includegraphics[scale=0.5]{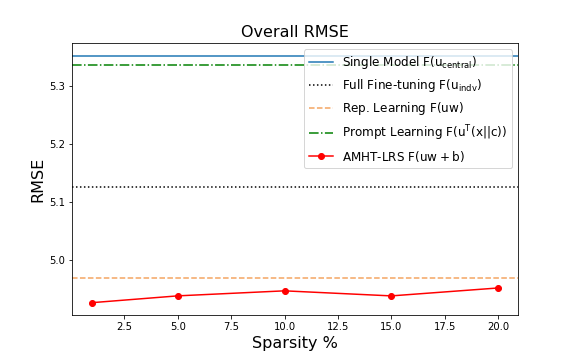}\vspace*{-5pt}
   \caption{\small Overall comparison of RMSE.}
       ~\label{fig:overall_jester2}
 \end{subfigure}%
 \vspace*{-15pt}
 \caption{\small  Comparison of Overall RMSE on the Jester validation data achieved by \method and the different baselines we consider where the training is modified for toy Neural Network models with a single hidden layer of $50$ neurons and tanh activation. \method outperforms other baselines along with a significantly smaller memory footprint (see Table \ref{table:nn_memory} for exact numbers on model parameters).}
\end{figure*}\label{fig:overall_jester}

\subsubsection{MovieLens Dataset}
 
\paragraph{\lrs with 3 layer Neural Net for MovieLens:} 
For the Neural Network (NN) experiments on MovieLens dataset, we consider the function class $\s{F}$ - a 3 layer Neural Net with 2 hidden layers of 50 neurons each and tanh activation. 
 The training and the validation data on the MovieLens dataset is created in a similar manner as discussed in Section \ref{sec:experiments}. The comparison of validation RMSE of \method and all the $4$ baselines corresponding to the MovieLens dataset is given in Figure \ref{fig:overall2}. Here, we can observe that \method has almost similar performance as the best performing baseline Full Fine-tuning ($\s{F}(x;\s{u}_{\text{indv}})$) while outperforming the other baselines. However, note that the individual models $\s{F}(x;\s{u}_{\text{indv}})$ have a high memory overhead since every trained model per task has the same memory usage as a single Neural Network model. In particular \method (with only 20\% sparsity/tunable parameters) matches the Full-finetuning baseline with approximately 20\% of the corresponding number of  parameters.

\begin{figure*}[!h]
\begin{subfigure}[t]{\textwidth}
\centering
   \includegraphics[scale = 0.5]{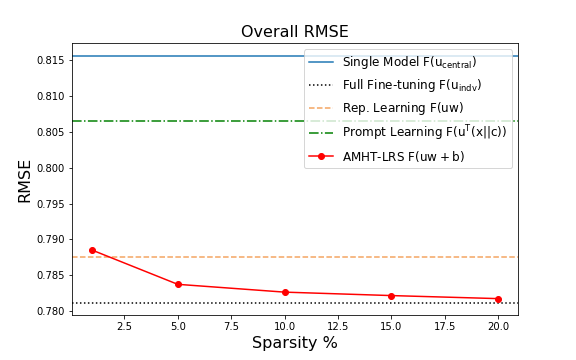}\vspace*{-5pt}
   \caption{\small Overall comparison of RMSE.}
       ~\label{fig:overall2}
 \end{subfigure}%
\vspace*{-15pt}
 \caption{\small  Comparison of RMSE on the MovieLens validation data achieved by \method and the different baselines we consider where the training is modified for toy Neural Network models with 2 hidden layers of $50$ neurons each. \method has similar performance as individual models $\s{F}(x;\s{u}_{\text{indv}})$ trained for each task; however our models have a significantly smaller memory footprint (see Table \ref{table:nn_memory} for exact numbers on model parameters).}
\end{figure*}\label{fig:movielens_NN_overall}

{\renewcommand{\arraystretch}{1.2}
\begin{table}[t!]
 \caption{Comparison of number of model parameters for Linear formulation of \method at different sparsity levels and different baselines. Note that our approach \method at 2\% and 10\% sparsity levels for MovieLens and Jester respectively has substantially less number of parameters than Full fine-tuning $\s{u}_{\text{indv}}$ and comparable number of model parameters with the other baselines. For Netflix, at each rank $r$ value, with just 1 additional parameter/user \method is able to outperform all other baselines.}
\vskip 0.15in
\begin{center}
\begin{small}
\begin{sc}
\begin{tabular}{|| c | l | r||} 
 \hline
 Dataset & Method & \#Parameters \\ [0.5ex] 
 \hline\hline
 \multirow{9}{5em}{MovieLens} & Single Model $(\s{u}_{\text{central}})$ & 50 \\
 & Full Fine-tuning $(\s{u}_{\text{indv}})$ & 12,050 \\
 & Rep. Learning $(\s{uw})$ & 291 \\
 & Prompt Learning $(\s{u}_{\text{central}}, x\mid \mid c)$ & 12,100\\
 & \method $(\s{uw}+\s{b}_{\text{2\% sparse}})$ & 532 \\
 & \method $(\s{uw}+\s{b}_{\text{6\% sparse}})$ & 1,014 \\
 & \method $(\s{uw}+\s{b}_{\text{10\% sparse}})$ & 1,496 \\
 & \method $(\s{uw}+\s{b}_{\text{16\% sparse}})$ & 2,219 \\
 & \method $(\s{uw}+\s{b}_{\text{20\% sparse}})$ & 2701\\ 
 \hline\hline
 \multirow{9}{5em}{Jester} & Single Model $(\s{u}_{\text{central}})$ & 50 \\ 
 & Full Fine-tuning $(\s{u}_{\text{indv}})$ & 5,000 \\
 & Rep. Learning $(\s{uw})$ & 150 \\
 & Prompt Learning $(\s{u}_{\text{central}}, x\mid \mid c)$ & 5,050\\
 & \method $(\s{uw}+\s{b}_{\text{2\% sparse}})$ & 250 \\
 & \method $(\s{uw}+\s{b}_{\text{4\% sparse}})$ & 350 \\
 & \method $(\s{uw}+\s{b}_{\text{6\% sparse}})$ & 450 \\
 & \method $(\s{uw}+\s{b}_{\text{8\% sparse}})$ & 550 \\
 & \method $(\s{uw}+\s{b}_{\text{10\% sparse}})$ & 650\\
 \hline\hline
 \multirow{9}{5em}{Netflix} & Single Model $(\s{u}_{\text{central}})$ & 50 \\ 
 & Full Fine-tuning $(\s{u}_{\text{indv}})$ & 50,000 \\
 & Rep. Learning $(\s{uw})$ Rank $r=1$& 1,050 \\
 & Rep. Learning $(\s{Uw})$ Rank $r=2$& 2,100 \\
 & Rep. Learning $(\s{Uw})$ Rank $r=3$& 3,150 \\
 & Rep. Learning $(\s{Uw})$ Rank $r=4$& 4,200 \\
 & Rep. Learning $(\s{Uw})$ Rank $r=5$& 5,250 \\
 & Prompt Learning $(\s{u}_{\text{central}},x\mid \mid c)$ & 10,050\\
 & \method $(\s{uw}+\s{b}_{\text{2\% sparse}})$ Rank $r=1$ & 2,050 \\
 & \method $(\s{Uw}+\s{b}_{\text{2\% sparse}})$ Rank $r=2$ & 3,100 \\
 & \method $(\s{Uw}+\s{b}_{\text{2\% sparse}})$ Rank $r=3$ & 4,150 \\
 & \method $(\s{Uw}+\s{b}_{\text{2\% sparse}})$ Rank $r=4$ & 5,200 \\
 & \method $(\s{Uw}+\s{b}_{\text{2\% sparse}})$ Rank $r=5$ & 6,250 \\
 \hline
 \end{tabular}
 \end{sc}
\end{small}
\end{center}
\vskip -0.1in
\end{table}\label{table:linear_memory}
}

{\renewcommand{\arraystretch}{1.2}
\begin{table}[t!]
 \caption{Comparison of number of model parameters for NN versions of \method at different sparsity levels and different baselines. As in Table~\ref{table:linear_memory}, our approach \method at 1\% and 5\% for MovieLens and Jester respectively sparsity levels has substantially less number of parameters than Full fine-tuning $\s{F}(x;\s{u}_{\text{indv}})$ and comparable number of model parameters with the other baselines. For Netflix, at each rank $r$ value, with just $1\%$ additional parameter/user \method is able to outperform all other baselines.}
\vskip 0.15in
\begin{center}
\begin{small}
\begin{sc}
\begin{tabular}{|| c | l | r||} 
 \hline
 Dataset & Method & \#Parameters \\ [0.5ex] 
 \hline\hline
 \multirow{9}{5em}{MovieLens} & Single Model $\s{F}(x;\s{u}_{\text{central}})$ & 5,151 \\
 & Full Fine-tuning $\s{F}(x;\s{u}_{\text{indv}})$ & 1,241,391 \\
 & Rep. Learning $\s{F}(x;\s{uw})$ & 5,392 \\
 & Prompt Learning $\s{F}(x\mid \mid c;\s{u}_{\text{central}})$ & 19,701\\
 & \method $\s{F}(x;\s{uw}+\s{b}_{\text{1\% sparse}})$ & 17,924 \\
 & \method $\s{F}(x;\s{uw}+\s{b}_{\text{5\% sparse}})$ & 67,570 \\
 & \method $\s{F}(x;\s{uw}+\s{b}_{\text{10\% sparse}})$ & 129,748 \\
 & \method $\s{F}(x;\s{uw}+\s{b}_{\text{15\% sparse}})$ & 191,685 \\
 & \method $\s{F}(x;\s{uw}+\s{b}_{\text{20\% sparse}})$ & 253,863\\ 
 \hline\hline
 \multirow{9}{5em}{Jester} & Single Model $\s{F}(x;\s{u}_{\text{central}})$ & 2,601 \\ 
 & Full Fine-tuning $\s{F}(x;\s{u}_{\text{indv}})$ & 260,100 \\
 & Rep. Learning $\s{F}(x;\s{uw})$ & 2,701 \\
 & Prompt Learning $\s{F}(x\mid \mid c;\s{u}_{\text{central}})$ & 10,101\\
 & \method $\s{F}(x;\s{uw}+\s{b}_{\text{1\% sparse}})$ & 5,302 \\
 & \method $\s{F}(x;\s{uw}+\s{b}_{\text{5\% sparse}})$ & 15,706 \\
 & \method $\s{F}(x;\s{uw}+\s{b}_{\text{10\% sparse}})$ & 28,711 \\
 & \method $\s{F}(x;\s{uw}+\s{b}_{\text{15\% sparse}})$ & 41,716 \\
 & \method $\s{F}(x;\s{uw}+\s{b}_{\text{20\% sparse}})$ & 54,721\\
 \hline\hline
 \multirow{9}{5em}{Netflix} & Single Model ($\s{F}(x;\s{u}_{\text{central}})$) & 2,601 \\ 
 & Full Fine-tuning $\s{F}(x;\s{u}_{\text{indv}})$ & 2,601,000 \\
 & Rep. Learning $\s{F}(x;\s{uw})$ Rank $r=1$ & 4,601 \\
 & Rep. Learning $\s{F}(x;\s{Uw})$ Rank $r=2$ & 9,202 \\
 & Rep. Learning $\s{F}(x;\s{Uw})$ Rank $r=3$ & 13,803 \\
 & Rep. Learning $\s{F}(x;\s{Uw})$ Rank $r=4$ & 18,404 \\
 & Rep. Learning $\s{F}(x;\s{Uw})$ Rank $r=5$ & 23,005 \\
 & Prompt Learning $\s{F}(x\mid \mid c;\s{u}_{\text{central}})$ & 55,101\\
 & \method $\s{F}(x;\s{uw}+\s{b}_{\text{1\% sparse}})$ Rank $r=1$ & 30,601 \\
 & \method $\s{F}(x;\s{Uw}+\s{b}_{\text{1\% sparse}})$ Rank $r=2$ & 35,202 \\
 & \method $\s{F}(x;\s{Uw}+\s{b}_{\text{1\% sparse}})$ Rank $r=3$ & 39,803 \\
 & \method $\s{F}(x;\s{Uw}+\s{b}_{\text{1\% sparse}})$ Rank $r=4$ & 44,404 \\
 & \method $\s{F}(x;\s{Uw}+\s{b}_{\text{1\% sparse}})$ Rank $r=5$ & 49,005 \\\hline
 \end{tabular}
 \end{sc}
\end{small}
\end{center}
\vskip -0.1in
\end{table}\label{table:nn_memory}
}



\section{Warm-Up: Central Model + Fine-tuning}\label{app:finetuning}

\subsection{Sparse Fine-tuning of Central Model}\label{subsec:finetuning}

Inspired by parameter efficient transfer learning applications shown in  \cite{guo2020parameter} where the authors propose learning a task-specific sparse vector, we consider the following simple variant of our problem in the noiseless rank-$1$ setting ($r=1$) with $(\fl{W}^{\star})^{\s{T}}$ being a multiple of an all $1$ $t$-dimensional vector. We will denote the representation vector by $\fl{u}^{\star}\in\bb{R}^d$ that is shared by all the tasks.
Therefore, the ERM for this model is given by the \lrs problem with $\fl{w}^{(i)}=1$ for all $i\in [t]$.
We can also pose this problem as the setting when the low rank representation of the datapoints corresponds to projection on a fixed unknown vector; there exists a central model (parameterized by the fixed unknown vector shared across tasks) and each task is a fine-tuned version of the central model. Our AM algorithm to solve the modified ERM problem is significantly simpler; in particular Steps 2-8 in Algorithm \ref{algo:optimize_lrs1} is replaced by the following set of updates given estimates $\fl{u}^{(\ell-1)}\in \bb{R}^d$ (of $\fl{u}^{\star}$) and $\{\fl{b}^{(i,\ell-1)}\}_{i\in [t]}$ (of $\{\fl{b}^{\star (i)}\}_{i\in [t]}$) in the $\ell^{\s{th}}$ iteration with a suitable choice of $\Delta^{(\ell)}$: 
\begin{align}
    &\fl{c}^{(i,\ell)} \gets \fl{b}^{(i,\ell-1)}- 
        (m^{-1}\fl{X}^{(i)})^{\s{T}}  (\fl{X}^{(i)}(\fl{u}^{(\ell-1)}+\fl{b}^{(i,\ell-1)})-\fl{y}^{(i)}) \label{eq:update_simple1} \\
     &\fl{b}^{(i,\ell)} \gets \s{HT}(\fl{c}^{(i,\ell)},\Delta^{(\ell)}) \label{eq:update_simple2} \\
     &\fl{u}^{(\ell)} \gets \Big(\sum_i  (\fl{X}^{(i)})^{\s{T}}\fl{X}^{(i)}\Big)^{-1} \Big(\sum_i (\fl{X}^{(i)})^{\s{T}}\Big(\fl{y}^{(i)}-
     \fl{X}^{(i)}\fl{b}^{(i,\ell)}\Big)\Big) \label{eq:update_simple3}
\end{align}
Notice that the updates in eq. \ref{eq:update_simple3} are only implemented once in each iteration (unlike Algorithm \ref{algo:optimize_b}) which improves the run-time as well as the sample complexity of the algorithm by logarithmic factors. The detailed Algorithm is provided in Appendix \ref{app:finetuning}. We present the main theorem below:

\begin{thm}\label{thm:main_toy}
Consider the \lrs problem with $t$ linear regression tasks and samples obtained by \eqref{eq:samples} where rank $r=1$, $\sigma=0$, $\fl{U}^{\star} \equiv \fl{u}^{\star}\in \bb{R}^d$ and $\fl{w}^{\star}_i \equiv w^{\star}\in \bb{R}$. Let model parameters $\{\fl{b}^{\star (i)}\}_{i\in [t]}$ satisfy assumption A1. Suppose Algorithm \ref{algo:optimize_lrs1} with modified updates (eq. \ref{eq:update_simple3}) is run for $\s{L}=\log\left(\epsilon_0^{-1}\Big(\max_{i\in [t]}\lr{\fl{b}^{\star(i)}}_{\infty}+\lr{\fl{u}^{\star}}_{\infty}+\frac{\lr{\fl{u}^{\star}}_2}{\sqrt{k}}\Big)\right)$ iterations. Then, w.p. $\geq 1-O(\delta_0)$, the outputs $\fl{u}^{(\s{L})},\{\fl{b}^{(i,\s{L})}\}_{i\in[t]}$ satisfy: $\lr{\fl{u}^{(\s{L})}-w^{\star}\fl{u}^{\star}}_{\infty} \le O(\epsilon_0)   \text{ and }\left|\left|\fl{b}^{(i, \s{L})}-\fl{b}^{\star (i)}\right|\right|_{\infty} \le O(\epsilon_0) \text{ for all }i\in [t]$ 
provided the total number of samples satisfy 
$    m=\widetilde{\Omega}(k),\; mt=\widetilde{\Omega}(d\sqrt{k}) \text{ and }mt^2=\widetilde{\Omega}(\zeta k d).$
\end{thm}

\begin{rmk}
 Notice from Theorem \ref{thm:main_toy} that our AM algorithm in the sparse fine-tuning setting enjoys global convergence guarantees and does not require any initialization conditions. 
Secondly, we do not need $\fl{u}^{\star}$ to satisfy any incoherence property for convergence guarantees of Theorem \ref{thm:main_toy} (unlike Theorem \ref{thm:main}). Therefore, Theorem \ref{thm:main_toy} is interesting in itself and significantly improves on the guarantees of Theorem \ref{thm:main} directly applied to the special setting.
\end{rmk}

\subsection{Detailed Analysis and Proof of Theorem \ref{thm:main_toy}}
\begin{algorithm}[t]
\caption{\textsc{\method} (Central Model+Finetuning)}  \label{alg:Algorithm-2}
\begin{algorithmic}[1]
\REQUIRE Data $\{(\fl{x}^{(i)}_j\in \bb{R}^d,y^{(i)}_j\in \bb{R})\}_{j=1}^{m}$ for all $i\in [t]$, column sparsity $k$ of $\fl{B}$. Initial Error Bounds $\max_{i\in [t]} \|\fl{b}^{(0,\ell)}-\fl{b}^{\star (i)}\|_{2} \leq \alpha^{(0)}$,  $\max_{i\in [t]} \|\fl{b}^{(i,0)}-\fl{b}^{\star (i)}\|_{\infty} \leq \gamma^{(0)}$, $\|\fl{u}^{(0)}-\fl{u}^{\star}\|_{\infty} \leq \beta^{(0)}$ and $\|\fl{u}^{(0)}-\fl{u}^{\star}\|_{2} \leq \tau^{(0)}$.
\STATE Suitable constants $c_1,c_2,c_3>0$.
\FOR{$\ell = 1, 2, \dots$ (Until Convergence)}
\STATE    $\Delta^{(\ell)} \leftarrow \beta^{(\ell-1)}+\frac{c_1}{\sqrt{k}}\Big(\tau^{(\ell-1)} + \alpha^{(\ell-1)})$
\STATE        $\fl{c}^{(i,\ell)} \leftarrow \fl{b}^{(i,\ell-1)}-\frac{1}{m}\cdot 
        (\fl{X}^{(i)})^{\s{T}}  (\fl{X}^{(i)}(\fl{u}^{(\ell-1)}+\fl{b}^{(i,\ell-1)})-\fl{y}^{(i)})$ \label{alg:c-udpate}
\STATE        $\fl{b}^{(i,\ell)} \leftarrow \s{HT}(\fl{c}^{(i,\ell)},\Delta^{(\ell)})$ 
\STATE $\fl{u}^{(\ell)} \gets \Big(\frac{1}{mt}\sum_i \sum_j \fl{x}^{(i)}_j(\fl{x}^{(i)}_j)^{\s{T}}\Big)^{-1} \Big(\frac{1}{mt}\sum_i \sum_j \fl{x}^{(i)}_j(y^{(i)}_j-(\fl{x}^{(i)}_j)^{\s{T}}\fl{b}^{(i,\ell)})\Big)$
\STATE Set, $\gamma^{(\ell)} \leftarrow 2\beta^{(\ell-1)} + \frac{2c_1}{\sqrt{k}}\tau^{(\ell-1)} + 2c_1\gamma^{(\ell-1)}$
\STATE Set $\tau^{(\ell)} \leftarrow c_2\sqrt{k}\gamma^{(\ell)}$,
 $\beta^{(\ell)} \leftarrow c_3\gamma^{(\ell)}$ and
$\alpha^{(\ell)} \gets \sqrt{k}\gamma^{(\ell)}$
\ENDFOR
\STATE Return  $\fl{w^{(\ell)}}$, $\fl{U}^{+(\ell)}$ and $\{\fl{b}^{(i, \ell)} \}_{i\in [t]}$.
\end{algorithmic}
\end{algorithm}

\noindent In the fine-tuning model described in Section \ref{subsec:finetuning}, we consider a system comprising of $t$ tasks, each of which (indexed by $i\in [t]$) is parameterized by an unknown task-specific sparse parameter vector $\fl{b}^{\star (i)}\in \bb{R}^d$ satisfying $\|\fl{b}^{\star (i)}\|_0 \le k$ along with a dense unknown parameter vector $\fl{u}^{\star}\in \bb{R}^d$ that is shared across all tasks. Now, for each task $i\in [t]$,  we obtain samples $\{(\fl{x}^{(i)}_j,y^{(i)}_j)\}_{j=1}^{m}$ according to the following model:
\begin{align}\label{eq:samples_toy}
    &\fl{x}^{(i)}_j \sim \ca{N}(\f{0},\fl{I}_d) 
    \text{ and }
    y_j^{(i)}\mid \fl{x}^{(i)}_j = \langle \fl{x}^{(i)}_j, \fl{u}^{\star} +\fl{b}^{\star(i)} \rangle \; \text{ for all } i\in[t],j\in [m] 
\end{align}

We will assume that the model parameters $\{\fl{b}^{\star (i)}\}_{i\in [t]}$ satisfy Assumption \hyperref[assum:row_sparse]{A1}. More importantly, we do not assume \hyperref[assum:task_diversity]{A2} and furthermore, we do not assume that $\fl{u}^{\star}$ is unit-norm. Since $\fl{u}^{\star}$ is not unit norm, we can write it as $\fl{u}^{\star}= \frac{\fl{u}^{\star}}{\lr{\fl{u}^{\star}}_2} \lr{\fl{u}^{\star}}_2$. In order to map it to the statement of Theorem \ref{thm:main_toy} and the general problem statement in \ref{eq:samples}, we can immediately write $w^{\star}\leftarrow \lr{\fl{u}^{\star}}_2$ and $\fl{u}^{\star} \leftarrow \frac{\fl{u}^{\star}}{\lr{\fl{u}^{\star}}_2}$ (since $\fl{u}^{\star}$ in the statement of Theorem \ref{thm:main_toy} is unit-norm). Hence, we can simplify the notation significantly by assuming that $\fl{u}^{\star}$ is not unit norm and by subsuming the scalar $w^{\star}$ (which is same across all tasks for this special setting) with the norm of vector $\fl{u}^{\star}$. 

\paragraph{Initialization and Notations:} For $\ell=0$, we will initialize $\fl{u}^{(0)}=\f{0}$ and $\fl{b}^{(i,0)}=\f{0}$ for all tasks indexed by $i\in [t]$. 
For any $\ell \ge 0$, at the beginning of the $(\ell+1)^{\s{th}}$ iteration, we will use $\alpha^{(\ell)},\tau^{(\ell)}$ to denote known upper bounds on the $\ell_2$-norm of the approximated parameters and $\gamma^{(\ell)},\beta^{(\ell)}$ to denote known upper bounds on the $\ell_{\infty}$-norm of the approximated parameters that will hold with high probability as described below:
\begin{align}
    \max_{i\in [t]} \left|\left|\fl{b}^{(i,\ell)}-\fl{b}^{\star (i)}\right|\right|_{2} &\leq \alpha^{(\ell)} \quad \text{ and }\quad
    \max_{i\in [t]} \left|\left|\fl{b}^{(i,\ell)}-\fl{b}^{\star (i)}\right|\right|_{\infty} \leq \gamma^{(\ell)},\nonumber\\
    \lr{\fl{u}^{(\ell)}-\fl{u}^{\star}}_{\infty} &\leq \beta^{(\ell)} \quad \text{ and }\quad
    \lr{\fl{u}^{(\ell)}-\fl{u}^{\star}}_{2} \leq \tau^{(\ell)}.\nonumber
\end{align}


\begin{lemma}\label{lemma:unknown-support-bounds}
For some constant $c>0$ and for any iteration indexed by $\ell\in [\s{L}]$, we can have the following updates 
\begin{align}
    \gamma^{(\ell)} &= 2\beta^{(\ell-1)}+2c\sqrt{\frac{\log (td/\delta_0)}{m}}\Big(\tau^{(\ell-1)}+\alpha^{(\ell-1)}\Big),\nonumber\\
    \alpha^{(\ell)} &= 2\sqrt{k}\beta^{(\ell-1)}+2c\sqrt{\frac{k\log (td/\delta_0)}{m}}\Big(\tau^{(\ell-1)}+\alpha^{(\ell-1)}\Big) \nonumber\\
    &\s{support}(\fl{b}^{(i,\ell}) \subseteq \s{support}(\fl{b}^{\star(i)}).\nonumber
\end{align}
with probability $1-O(\delta_0)$.
\end{lemma}

\begin{proof}
Fix any $i\in [t]$. It is easy to see that update step \ref{alg:c-udpate} of Algorithm \ref{alg:Algorithm-2} gives us
\begin{align}
    &\fl{c}^{(i,\ell)}-\fl{b}^{\star (i)} = \Big(\fl{I}-\frac{1}{m} (\fl{X}^{(i)})^{\s{T}}\fl{X}^{(i)} \Big)\Big(\fl{b}^{(i,\ell-1)}-\fl{b}^{\star (i)}\Big)+\frac{1}{m}(\fl{X}^{(i)})^{\s{T}}\fl{X}^{(i)}(\fl{u}^{\star}-\fl{u}^{(\ell-1)})\nonumber\\
    &\implies \fl{c}^{(i,\ell)}-\fl{b}^{\star (i)} -\fl{u}^{\star}+\fl{u}^{(\ell-1)} = \Big(\fl{I}-\frac{1}{m} (\fl{X}^{(i)})^{\s{T}}\fl{X}^{(i)} \Big)\Big(\fl{b}^{(i,\ell-1)}-\fl{b}^{\star (i)}\Big)\nonumber\\
    \qquad &+\Big(\fl{I}-\frac{1}{m}(\fl{X}^{(i)})^{\s{T}}\fl{X}^{(i)}\Big)(\fl{u}^{(\ell-1)}-\fl{u}^{\star}). \label{eq:lemma-unknown-support-c-unplugged}
\end{align}

Let $\fl{e}_s\in \bb{R}^d$ denote the $s^{\s{th}}$ basis vector for which the $s^{\s{th}}$ coordinate entry is $1$ and all other coordinate entries are $0$. Then, note that:
\begin{align}
    &\left|\Big(\fl{c}^{(i,\ell)}-\fl{b}^{\star (i)} -\fl{u}^{\star}+\fl{u}^{(\ell-1)}\Big)_s\right| \nonumber\\
    &= \left|\fl{e}_s^{\s{\s{T}}}\Big(\fl{I}-\frac{1}{m} (\fl{X}^{(i)})^{\s{T}}\fl{X}^{(i)} \Big)\Big(\fl{b}^{(i,\ell-1)}-\fl{b}^{\star (i)}\Big)+\fl{e}_s^{\s{T}}\Big(\fl{I}-\frac{1}{m}(\fl{X}^{(i)})^{\s{T}}\fl{X}^{(i)}\Big)(\fl{u}^{(\ell-1)}-\fl{u}^{\star})\right| \nonumber\\
    &\leq \left|\fl{e}_s^{\s{T}}\Big(\fl{I}-\frac{1}{m} (\fl{X}^{(i)})^{\s{T}}\fl{X}^{(i)} \Big)\Big(\fl{b}^{(i,\ell-1)}-\fl{b}^{\star (i)}\Big)\right|+\left|\fl{e}_s^{\s{T}}\Big(\fl{I}-\frac{1}{m}(\fl{X}^{(i)})^{\s{T}}\fl{X}^{(i)}\Big)(\fl{u}^{(\ell-1)}-\fl{u}^{\star})\right| \nonumber\\
    &\leq \left|\frac{1}{m} \fl{e}_s^{\s{T}}(\fl{X}^{(i)})^{\s{T}}\fl{X}^{(i)} (\fl{b}^{(i,\ell-1)}-\fl{b}^{\star (i)}) - \fl{e}_s^{\s{T}}(\fl{b}^{(i,\ell-1)}-\fl{b}^{\star (i)})\right|\nonumber\\
    &\qquad \qquad +\left|\frac{1}{m}\fl{e}_s^{\s{T}}(\fl{X}^{(i)})^{\s{T}}\fl{X}^{(i)}(\fl{u}^{(\ell-1)}-\fl{u}^{\star}) - \fl{e}_s^{\s{T}}(\fl{u}^{(\ell-1)}-\fl{u}^{\star})\right| \nonumber\\
    &\leq c\sqrt{\frac{\log (1/\delta_0)}{m}}\Big(\tau^{(\ell-1)}+\alpha^{(\ell-1)}\Big),\nonumber
\end{align}
w.p. $\geq 1-O(\delta_0)$, where we invoke Lemma \ref{lemma:useful1} in the last step and plugging $\fl{a} = \fl{e}_s$ and $\fl{b} = \fl{b}^{(i,\ell-1)}-\fl{b}^{\star (i)}$ and $\fl{u}^{(\ell-1)}-\fl{u}^{\star}$ for the two terms respectively. Therefore, by taking a union bound over all entries $s\in [d]$, and a further union bound over all tasks ($t$ of them), we can conclude that for all $i\in [t]$, we must have
\begin{align}
    &\lr{\fl{c}^{(i,\ell)}-\fl{b}^{\star (i)} -\fl{u}^{\star}+\fl{u}^{(\ell-1)}}_{\infty} \leq c\sqrt{\frac{\log (td/\delta_0)}{m}}\Big(\tau^{(\ell-1)}+\alpha^{(\ell-1)}\Big) \nonumber\\
    &\implies \lr{\fl{c}^{(i,\ell)}-\fl{b}^{\star (i)} }_{\infty} \leq \beta^{(\ell-1)}+c\sqrt{\frac{\log (td/\delta_0)}{m}}\Big(\tau^{(\ell-1)}+\alpha^{(\ell-1)}\Big) \label{eq:lemma-unknown-support-term1}
\end{align}
w.p. $1-O(\delta_0)$. Now, we have 
\begin{align}
    \fl{b}^{(i,l)} &= \s{HT}(\fl{c}^{(i,\ell)},\Delta^{(\ell)}) \nonumber\\
    \implies b^{(i,l)}_s &= \begin{cases}
    c^{(i,\ell)}_s & \text{if } |c^{(i,\ell)}_s| > \Delta^{(\ell)},\\
    0 & \text{ otherwise},
  \end{cases}\label{eq:lemma-unknown-support-term-temp} \\
  \implies |b^{(i,l)}_s - b^{\star(i)}_s | &= \begin{cases}
    |c^{(i,\ell)}_s - b^{\star(i)}_s| & \text{if } |c^{(i,\ell)}_s| > \Delta^{(\ell)},\\
    |b^{\star(i)}_s| & \text{ otherwise}.
  \end{cases}\label{eq:lemma-unknown-support-term2}
\end{align}
Therefore if we set $\Delta^{(\ell)} = \beta^{(\ell-1)}+c\sqrt{\frac{\log (td/\delta_0)}{m}}\Big(\tau^{(\ell-1)}+\alpha^{(\ell-1)}\Big)$ (as described in Step 2 of the algorithm), then, by using \eqref{eq:lemma-unknown-support-term1} and \eqref{eq:lemma-unknown-support-term2}, we have
$\lr{\fl{b}^{(i,\ell)}-\fl{b}^{\star (i)} }_{\infty}\le 2\Delta^{(\ell)}$ and therefore,  
\begin{align}
    \implies \lr{\fl{b}^{(i,\ell)}-\fl{b}^{\star (i)} }_{\infty} &\leq 2\beta^{(\ell-1)}+2c\sqrt{\frac{\log (td/\delta_0)}{m}}\Big(\tau^{(\ell-1)}+\alpha^{(\ell-1)}\Big) = \gamma^{(\ell)}\label{eq:lemma-unknown-support-bound1}\\
    \text{and } \lr{\fl{b}^{(i,\ell)}-\fl{b}^{\star (i)} }_{2} &\leq 2\sqrt{k}\beta^{(\ell-1)}+2c\sqrt{\frac{k\log (td/\delta_0)}{m}}\Big(\tau^{(\ell-1)}+\alpha^{(\ell-1)}\Big) = \alpha^{(\ell)},\label{eq:lemma-unknown-support-bound2}
\end{align}
with probability $1-O(\delta_0)$. Furthermore, from equation \eqref{eq:lemma-unknown-support-term1} we have for any coordinate $s$
\begin{align}
    \Big|\Big(\fl{c}^{(i,\ell)}-\fl{b}^{\star (i)}\Big)_s\Big| &\leq \Delta^{(\ell)}.\nonumber
\end{align}
Thus, if $s \notin	 \s{support}(\fl{b}^{\star(i)})$, then the above gives $|\fl{c}^{(i,\ell)}| \leq \Delta^{(\ell)}$. Using this in \eqref{eq:lemma-unknown-support-term-temp} gives $b^{(i,l)}_s = 0$. Hence, for all $s\in [d]$, we must have $ s \notin	 \s{support}(\fl{b}^{\star(i)}) \implies s \notin	 \s{support}(\fl{b}^{\star(i, \ell)})$ implying that $\s{support}(\fl{b}^{(i,\ell}) \subseteq \s{support}(\fl{b}^{\star(i)})$. Hence, the proof of the lemma is complete.
\end{proof}

\begin{lemma}
For some constant $c>0$ and for any iteration indexed by $\ell>0$, we  have 
\begin{align}
    \tau^{(\ell)} &= \frac{\sqrt{\frac{2\zeta k}{t}}\gamma^{(\ell)}+4\alpha^{(\ell)} \sqrt{\frac{d\log (d/\delta_0)}{mt}}}{1-c\sqrt{\frac{d\log (1/\delta_0)}{mt}}}\nonumber
\end{align}
with probability $1-O(\delta_0)$.
\end{lemma}

\begin{proof}

Update step 3 of the Algorithm for the $\ell^{\s{th}}$ iteration gives us
\begin{align}
\fl{u}^{(\ell)} &= \Big(\frac{1}{mt}\sum_i \sum_j \fl{x}^{(i)}_j(\fl{x}^{(i)}_j)^{\s{T}}\Big)^{-1} \Big(\frac{1}{mt}\sum_i \sum_j \fl{x}^{(i)}_j(\fl{x}^{(i)}_j)^{\s{T}}(\fl{u}^{\star}+\fl{b}^{\star (i)}-\fl{b}^{(i,\ell)})\Big) \nonumber\\
\implies \fl{u}^{(\ell)}-\fl{u}^{\star} &= \Big(\underbrace{\frac{1}{mt}\sum_i \sum_j \fl{x}^{(i)}_j(\fl{x}^{(i)}_j)^{\s{T}}}_{\fl{A}}\Big)^{-1} \Big(\underbrace{\frac{1}{mt}\sum_i \sum_j \fl{x}^{(i)}_j(\fl{x}^{(i)}_j)^{\s{T}}(\fl{b}^{\star (i)}-\fl{b}^{(i,\ell)})}_{\fl{v}}\Big).\nonumber
\end{align}
Let us denote the vector $\fl{b}^{\star (i)}-\fl{b}^{(i,\ell)}$ by $\fl{z}^{(i, \ell)}$ for simplicity. Notice that for any $h\in [d]$, we have
\begin{align}
 v_h &= \Big(\frac{1}{mt}\sum_i \sum_j \fl{x}^{(i)}_j(\fl{x}^{(i)}_j)^{\s{T}}\fl{z}^{(i, \ell)}\Big)_h\nonumber\\
 &= \frac{1}{mt}\sum_i\sum_j \Big(\Big(x^{(i)}_{j,h}\Big)^{2} z^{(i, \ell)}_h +\sum_{u:u\neq h}x^{(i)}_{j,h}x^{(i)}_{j,u}z^{(i, \ell)}_u\Big). \label{eq:lemma-unknown-support-unplugged}
\end{align}
Now, note that the random variable $\Big(x^{(i)}_{j,h}\Big)^{2} z^{(i, \ell)}_h$ is a $\Big(4(z^{(i, \ell)}_h)^2, 4|z^{(i, \ell)}_{h}|\Big)$ sub-exponential random variable. Similarly, $x^{(i)}_{j,h}x^{(i)}_{j,u}z^{(i, \ell)}_u$ is a $\Big(2(z^{(i, \ell)}_{u})^2, \sqrt{2} |z^{(i, \ell)}_{u}| \Big)$ sub-exponential random variable. Therefore, we must have  
\begin{align}
    &\Big(x^{(i)}_{j,h}\Big)^{2} z^{(i, \ell)}_h +\sum_{u:u\neq h}x^{(i)}_{j,h}x^{(i)}_{j,u}z^{(i, \ell)}_u \nonumber\\
    &= \Big( 4(z^{(i, \ell)}_h)^2 + 2\sum_{u:u\neq h} (z^{(i, \ell)}_{u})^2, \max\Big(4|z^{(i, \ell)}_{h}|, \max_{u:u\neq h}(\sqrt{2} |z^{(i, \ell)}_{u}|)\Big) \Big)\nonumber\\
    &= \Big( 4 \|\fl{z}^{(i, \ell)}\|^2_2, 4\|\fl{z}^{(i, \ell)}\|_{\infty}\Big) \text{ sub-exponential random variable.}\label{eq:lemma-unknown-support-1}
\end{align}
Furthermore, 
\begin{align}
    \E{v_h } &= \frac{1}{mt}\sum_i\sum_j \Big(\E{\Big(x^{(i)}_{j,h}\Big)^{2} z^{(i, \ell)}_h } + \E{\sum_{u:u\neq h}x^{(i)}_{j,h}x^{(i)}_{j,u}z^{(i, \ell)}_u }\Big)\nonumber\\
    &= \frac{1}{mt}\sum_i\sum_j \Big(z^{(i, \ell)}_h + 0\Big)\nonumber\\
    &= \frac{1}{t}\sum_iz^{(i, \ell)}_h. \label{eq:lemma-unknown-support-2}
\end{align}
Using \eqref{eq:lemma-unknown-support-1}, \eqref{eq:lemma-unknown-support-2} and Lemma~\ref{tail:sub_exp} in \eqref{eq:lemma-unknown-support-unplugged} implies that
\begin{align}
    \left|v_h - \frac{1}{t}\sum_i z^{(i, \ell)}_h\right| &\leq \max\Big(2 \|\fl{z}^{(i, \ell)}\|_2\sqrt{\frac{2\log (1/\delta_0)}{mt}}, 2 \|\fl{z}^{(i, \ell)}\|_{\infty}\frac{2\log (1/\delta_0)}{mt}\Big)\nonumber \\
    &\leq \underbrace{\max\Big(2 \alpha^{(\ell)} \sqrt{\frac{2\log (1/\delta_0)}{mt}}, 2 \gamma^{(\ell)} \frac{2\log (1/\delta_0)}{mt}\Big)}_{\epsilon_h}.\nonumber
\end{align}
will be true with probability at least $1-\delta_0$. On taking a union bound over all $h\in [d]$, we will have that
\begin{align}
     \left|v_h - \frac{1}{t}\sum_i z^{(i, \ell)}_h\right| \leq \underbrace{\max\Big(2 \alpha^{(\ell)} \sqrt{\frac{2\log (d/\delta_0)}{mt}}, 2 \gamma^{(\ell)} \frac{2\log (d/\delta_0)}{mt}\Big)}_{\epsilon_h}.\label{eq:lemma-unknown-support-plugged}
\end{align}
with probability $1-O(\delta_0)$. Note that $\|\fl{v}\|_2^2 = \sum_h v_h^2$. Hence, we have 
\begin{align}
    \sum_{h} v_h^2 &\leq \sum_{h} \Big(2\Big(\frac{1}{t}\sum_i z^{(i, \ell)}_h\Big)^2+2\epsilon_h^2\Big) \nonumber\\
    &\leq 2\zeta\sum_{h} \sum_i (\frac{z^{(i, \ell)}_h}{t})^2+2\sum_{h} \epsilon_h^2\nonumber\\
    &\leq 2\zeta\sum_{i} \sum_h \Big(\frac{z^{(i, \ell)}_h}{t}\Big)^2+2\sum_{h} \epsilon_h^2\nonumber\\
    &\leq \frac{2\zeta }{t}(\alpha^{(\ell)})^2+8(\alpha^{(\ell)})^2 \frac{2d\log (d/\delta_0)}{mt}\nonumber,
\end{align}
where we use that $2 \alpha^{(\ell)} \sqrt{\frac{2\log (d/\delta_0)}{mt}} > 2 \gamma^{(\ell)} \frac{2\log (d/\delta_0)}{mt}$.
Hence, with probability at least $1-O(\delta_0)$, we must have by using that $\alpha^{(\ell)} \le \sqrt{k}\gamma^{(\ell)}$
\begin{align}
    \|\fl{v}\|_2 \leq \sqrt{\frac{2\zeta k}{t}}\gamma^{(\ell)}+4\alpha^{(\ell)} \sqrt{\frac{d\log (d/\delta_0)}{mt}}\label{eq:lemma-unknown-support-num-val}
\end{align}
Furthermore, from Lemma~\ref{lemma:useful2}, we have with probability $1-\delta_0$ for any iterations $\ell\in [\s{L}]$
\begin{align}
    \|\frac{1}{mt}\sum_i \sum_j \fl{x}^{(i)}_j(\fl{x}^{(i)}_j)^{\s{T}} - \fl{I}\|_2 \leq c \sqrt{\frac{d\log (1/\delta_0)}{mt}}.\label{eq:lemma-unknown-support-denom-val}
\end{align}
implying that the minimum eigenvalue of the matrix $ \frac{1}{mt}\sum_i \sum_j \fl{x}^{(i)}_j(\fl{x}^{(i)}_j)^{\s{T}} - \fl{I}$ is at least $1-c \sqrt{\frac{d\log (1/\delta_0)}{mt}}$; hence the maximum eigenvalue of the matrix $ (\frac{1}{mt}\sum_i \sum_j \fl{x}^{(i)}_j(\fl{x}^{(i)}_j)^{\s{T}} - \fl{I})^{-1}$ is at most $(1-c \sqrt{\frac{d\log (1/\delta_0)}{mt}})^{-1}$. 
Using \eqref{eq:lemma-unknown-support-num-val} and \eqref{eq:lemma-unknown-support-denom-val}, we get for any iterations $\ell \in [\s{L}]$ with probability $1-O(\delta_0)$,
\begin{align}
    \|\fl{u}^{(\ell)}-\fl{u}^{\star}\|_2 \leq \frac{\sqrt{\frac{2\zeta k}{t}}\gamma^{(\ell)}+4\alpha^{(\ell)} \sqrt{\frac{d\log (d/\delta_0)}{mt}}}{1-c\sqrt{\frac{d\log (1/\delta_0)}{mt}}} \triangleq \tau^{(\ell)}.\label{eq:lemma-unknown-support-bound3}
\end{align}
\end{proof}

\begin{lemma}\label{lemma:u_infinity-norm-bound}
For some constant $c>0$ and for any iteration indexed by $\ell>0$, we have
\begin{align}
    \beta^{(\ell)} &= \Big(\frac{\zeta}{t} + 2c\sqrt{\frac{\log (d/\delta_0)}{mt}}\sqrt{\frac{2\zeta k}{t}}\Big)\gamma^{(\ell)} + \Big(c\sqrt{\frac{\log (d/\delta_0)}{mt}} + 8c\sqrt{d}\frac{\log (d/\delta_0)}{mt}\Big
  ) \alpha^{(\ell)}\nonumber
\end{align}
with probability at least $1-O(\delta_0)$.
\end{lemma}

\begin{proof}
With probability at least $1-O(\delta_0)$, we have that $\lr{\fl{E}}_2 \le \sqrt{\frac{d\log 9}{mt}}$. We fix $mt=\Omega(d)$ so that $\lr{\fl{E}}_2 < 1$. 
Our goal is to bound the quantity $\|\fl{u}^{(\ell)}-\fl{u}^{\star}\|_{\infty}$ from above.
Denoting $\fl{A}=\fl{I}+\fl{E}$ and using the fact that $(\fl{I}+\fl{E})^{-1}= \fl{I}-\fl{E}+\fl{E}^{2}+\dots$ (since $\lr{\fl{E}}_2 < 1$), by using Lemma~\ref{lemma:useful0} and taking a union bound over all entries $s\in [d]$, we have with probability at least $1-\delta_0$,
\begin{align}
    &\lr{\fl{v}-\frac{1}{t}\sum_{i}(\fl{b}^{\star (i)}-\fl{b}^{(i,\ell)})}_{\infty} \nonumber\\
    & = \max_s \left|\frac{1}{mt}\sum_i \sum_j \fl{e}_s^{\s{T}}\fl{x}^{(i)}_j(\fl{x}^{(i)}_j)^{\s{T}}(\fl{b}^{\star (i)}-\fl{b}^{(i,\ell)})-\frac{1}{t}\sum_{i}\fl{e}_s^{\s{T}}(\fl{b}^{\star (i)}-\fl{b}^{(i,\ell)})\right| \nonumber\\
    & = \max_s \left|\frac{1}{mt}\sum_i \sum_j (\fl{x}^{(i)}_j)^{\s{T}}(\fl{b}^{\star (i)}-\fl{b}^{(i,\ell)})\fl{e}_s^{\s{T}}\fl{x}^{(i)}_j-\frac{1}{t}\sum_{i}\fl{e}_s^{\s{T}}(\fl{b}^{\star (i)}-\fl{b}^{(i,\ell)})\right| \nonumber\\
    &\le c\sqrt{\sum_i\sum_j \|(\fl{b}^{\star (i)}-\fl{b}^{(i,\ell)})\fl{e}_s^{\s{T}}\|_{\s{F}}^2 \frac{\log (d/\delta_0)}{m^2t^2}} \nonumber\\
    &\le c\alpha^{(\ell)}\sqrt{\frac{\log (d/\delta_0)}{mt}}.\nonumber
\end{align}
Hence with probability at least $1-\delta_0$, we will have the following statement
\begin{align}
    \lr{\fl{v}}_{\infty} \leq c\alpha^{(\ell)}\sqrt{\frac{\log (d/\delta_0)}{mt}}+\frac{\zeta}{t}\gamma^{(\ell)} \label{eq:lemma-unknown-support-vinfty-val}.
\end{align}
Since $\fl{u}^{(\ell)}-\fl{u}^{\star} = (\fl{I}+\fl{E})^{-1}\fl{v}$ with $\fl{E}=\frac{1}{mt}\sum_i \sum_j \fl{x}^{(i)}_j(\fl{x}^{(i)}_j)^{\s{T}}-\fl{I}$, we will have 
\begin{align}
    \lr{\fl{u}^{(\ell)}-\fl{u}^{\star}}_{\infty} \leq \sum_{j=0}^{\infty} \lr{\fl{E}^j\fl{v}}_{\infty}.\label{eq:lemma-unknown-support-unplugged1}
\end{align}
Let $\fl{\cV} \triangleq \{\fl{z} \in \bR^d | \|\fl{z}\| = 1\}$. Then for $\epsilon \leq 1$, there exists an $\epsilon$-net, $N_\epsilon \subset \fl{\cZ}$, of size $(1 + 2/\epsilon)^
d$ w.r.t the Euclidean norm, i.e. $\forall$  $\fl{z} \in \fl{\cZ}$, $\exists$ $\fl{z}' \in N_\epsilon$ s.t. $\|\fl{z} - \fl{z}'\|_2 \leq \epsilon$.
Now consider any $\fl{z} \in N_\epsilon$. Then, Lemma~\ref{lemma:useful1} with $\fl{a} = \fl{e}_s$ and $\fl{b} = \fl{z}$ and taking a union bound over all entries $s\in [d]$ gives
\begin{align}
    &\left|\fl{e}^{\s{T}}_s\Big(\frac{1}{mt}\sum_i \sum_j \fl{x}^{(i)}_j(\fl{x}^{(i)}_j)^{\s{T}}-\fl{I}\Big)\Big)\fl{z}\right| \leq c\|\fl{z}\|^2_2\max\Big(\sqrt{\frac{\log (1/\delta_0)}{mt}},\frac{\log (1/\delta_0)}{mt}\Big)\nonumber\\
    \implies  &\|\fl{E}\fl{z}\|_\infty \leq c\max\Big(\sqrt{\frac{\log (d|N_\epsilon|/\delta_0)}{mt}},\frac{\log (d|N_\epsilon|/\delta_0)}{mt}\Big)\nonumber\\
     &\leq c\max\Big(\sqrt{\frac{\log(d(1+2/\epsilon)^d/\delta_0)}{mt}},\frac{\log (d(1+2/\epsilon)^d/\delta_0)}{mt}\Big), \quad \forall \fl{v} \in N_\epsilon
\end{align}
Further, $\exists$ $\fl{z} \in N_\epsilon$ s.t. $\|\fl{z}' - \fl{z}\|_2 \leq \epsilon$. This implies that setting $\epsilon \gets 1/4$ and $c \gets 2c$ gives:
\begin{align}
    \|\fl{E}\fl{z}'\|_\infty &\leq \|\fl{E}(\fl{z}-\fl{z}')\|_\infty + \|\fl{E}\fl{z}\|_\infty\nonumber\\
    &\leq \|\fl{E}(\fl{z}-\fl{z}')\|_2 + \|\fl{E}\fl{z}\|_\infty\nonumber\\
    &\leq c\sqrt{\frac{d\log(d/\delta_0)}{mt}}.
\end{align}
with probability at least $1-\delta_0$. Hence, with probability at least $1-O(\delta_0)$, we have that $\lr{\fl{E}}_2 \le \sqrt{\frac{d\log 9}{mt}}$ and $\lr{\fl{E}\fl{z}}_{\infty} \le c\sqrt{\frac{d\log (d\delta_0^{-1})}{mt}}$ for all $\fl{z}\in \ca{V}$. Therefore, let us conditioned on these events in order to prove the next steps. We will show an upper bound on $\|\fl{A}^{-1}\fl{v}\|_\infty$.
\begin{align}
    \|\fl{A}^{-1}\fl{v}\|_\infty &= \|(\fl{I}+\fl{E})^{-1}\fl{v}\|_\infty\nonumber\\
    &\leq \sum_{j=0}^{\infty} \lr{\fl{E}^j\fl{v}}_{\infty}.
\end{align}
We have with probability at least $1-\delta_0$
\begin{align}
    \|\fl{E}^p\fl{v}\|_\infty &= \|\fl{E}\fl{E}^{p-1}\fl{v}\|_\infty\nonumber\\
    &= \lr{\Big(\fl{E}\|\fl{E}^{p-1}\fl{v}\|_2\Big)\Big(\frac{\fl{E}^{p-1}\fl{v}}{\|\fl{E}^{p-1}\fl{v}\|_2}\Big)}_\infty\nonumber\\
    &= \|\fl{E}^{p-1}\fl{v}\|_2\lr{\fl{E}\Big(\frac{\fl{E}^{p-1}\fl{v}}{\|\fl{E}^{p-1}\fl{v}\|_2}\Big)}_\infty\nonumber\\
    &\leq \|\fl{E}^{p-1}\fl{v}\|_2c\sqrt{\frac{d\log(d/\delta_0)}{mt}}\nonumber\\
    &\leq \Big(c\sqrt{\frac{d\log(1/\delta_0)}{mt}}\Big)^{p-1}c\sqrt{\frac{d\log(d/\delta_0)}{mt}}\|\fl{v}\|_2
\end{align}
Therefore, if $mt=\Omega(d\log (d/\delta_0))$ by taking a union bound we must have with probability at least $1-\delta_0$,
\begin{align}\label{eq:corrected-lemma-unknown-support-uinfty-unplugged-temp}
    \sum_{p=1}^{\infty}\lr{\fl{E^p v}}_{\infty} = O\Big(\sqrt{\frac{d\log (d/\delta_0)}{mt}}\Big)\lr{\fl{v}}_2.
\end{align}

Therefore we have w.p. $\geq 1-O(\delta_0)$
\begin{align}
    \lr{\fl{u}^{(\ell)}-\fl{u}^{\star}}_{\infty} \leq \lr{\fl{v}}_{\infty}+2c\sqrt{\frac{d\log (d/\delta_0)}{mt}}\lr{\fl{v}}_2. \label{eq:corrected-lemma-unknown-support-uinfty-unplugged}
\end{align}
Plugging the bounds of $\lr{\fl{v}}_{\infty}$ and $\lr{\fl{v}}_2$ from \eqref{eq:lemma-unknown-support-vinfty-val} and \eqref{eq:lemma-unknown-support-num-val} in \eqref{eq:corrected-lemma-unknown-support-uinfty-unplugged}, we obtain that w.p. $\geq 1-\delta_0)$
\begin{align}
  &\lr{\fl{u}^{(\ell)}-\fl{u}^{\star}}_{\infty} \nonumber\\
  &\leq c\alpha^{(\ell)}\sqrt{\frac{\log (d/\delta_0)}{mt}}+\frac{\zeta}{t}\gamma^{(\ell)} + 2c\sqrt{\frac{d\log (d/\delta_0)}{mt}}\Big(\sqrt{\frac{2\zeta k}{t}}\gamma^{(\ell)}+4\alpha^{(\ell)} \sqrt{\frac{d\log (d/\delta_0)}{mt}}\Big)\nonumber\\
  &= \Big(\frac{\zeta}{t} + 2c\sqrt{\frac{d\log (d/\delta_0)}{mt}}\sqrt{\frac{2\zeta k}{t}}\Big)\gamma^{(\ell)} + \Big(c\sqrt{\frac{\log (d/\delta_0)}{mt}} + 8c\frac{d\log (d/\delta_0)}{mt}\Big
  ) \alpha^{(\ell)} = \beta^{(\ell)}\label{eq:corrected-lemma-unknown-support-bound4}
\end{align}

\end{proof}
\begin{lemma}\label{lemma:abgd-unknown-support-infty-bounds}

After $\s{L}$ iterations, for some constant $c>0$. we will have with probability $1-O(\s{L} \delta_0)$,
\begin{align}
    \lr{\fl{u}^{(\s{L})}-\fl{u}^{\star}}_{\infty} &\le c_32^{\s{L}-1}(c_3 + c_1c_2+ c_1)^{\s{L}-1}\s{Z},\nonumber\\
    \left|\left|\fl{b}^{(i,\s{L})}-\fl{b}^{\star (i)}\right|\right|_{\infty} &\le 2^{\s{L}-1}(c_3 + c_1c_2+ c_1)^{\s{L}-1}\s{Z},\nonumber\\
    \left|\left|\fl{b}^{(i,\s{L})}-\fl{b}^{\star (i)}\right|\right|_{2} &\leq \sqrt{k}2^{\s{L}-1}(c_3 + c_1c_2+ c_1)^{\s{L}-1}\s{Z},\nonumber\\
    \lr{\fl{u}^{(\s{L})}-\fl{u}^{\star}}_{2} &\leq c_2\sqrt{k}2^{\s{L}-1}(c_3 + c_1c_2+ c_1)^{\s{L}-1}\s{Z},\nonumber
\end{align}
where
\begin{align}
    \s{Z}&=\Big(2\lr{\fl{u}^{(0)}-\fl{u}^{\star}}_{\infty} + \frac{2c_1}{\sqrt{k}}\lr{\fl{u}^{(0)}-\fl{u}^{\star}}_{2} + 2c_1\max_{i\in[t]}\lr{\fl{b}^{(i, 0)}-\fl{b}^{\star(i)}}_{\infty}\Big), \nonumber\\
    c_1 &= c\sqrt{\frac{k\log (tdL/\delta_0)}{m}},\nonumber\\
    c_2 &= \frac{\sqrt{\frac{2\zeta}{t}}+4\sqrt{\frac{d\log (dL/\delta_0)}{mt}}}{1-c\sqrt{\frac{d\log (L/\delta_0)}{mt}}},\nonumber\\
    c_3 &=  \Big(\frac{\zeta}{t} + 2c\sqrt{\frac{d\log (d/\delta_0)}{mt}}\sqrt{\frac{2\zeta k}{t}}\Big) + \sqrt{k}\Big(c\sqrt{\frac{\log (d/\delta_0)}{mt}} + 8c\frac{d\log (d/\delta_0)}{mt}\Big).\nonumber
\end{align}
\end{lemma}

\begin{proof}
Using Lemma~\ref{lemma:unknown-support-bounds} and the fact that $\alpha^{(\ell)} \leq \sqrt{k}\gamma^{(\ell)}$, we have for $\ell \geq 1$
\begin{align}
    \gamma^{(\ell)} &\leq 2\beta^{(\ell-1)} + \frac{2c_1}{\sqrt{k}}\tau^{(\ell-1)} + 2c_1\gamma^{(\ell-1)},\label{eq:lemma-abgd-unknown-support-infty-1}\\
    \tau^{(\ell)} &\leq c_2\sqrt{k}\gamma^{(\ell)},\label{eq:lemma-abgd-unknown-support-infty-2}\\
    \beta^{(\ell)} &\leq c_3\gamma^{(\ell)},\label{eq:lemma-abgd-unknown-support-infty-3}
\end{align}
Using \eqref{eq:lemma-abgd-unknown-support-infty-2} and \eqref{eq:lemma-abgd-unknown-support-infty-3} in \eqref{eq:lemma-abgd-unknown-support-infty-1}, we get
\begin{align}
    \gamma^{(\ell)} &\leq (2c_3 + 2c_1c_2+ 2c_1)\gamma^{(\ell-1)}\nonumber\\
    &= 2(c_3 + c_1c_2+ c_1)\gamma^{(\ell-1)}\nonumber\\
    &\leq \dots\nonumber\\
    &\leq 2^{\ell-1}(c_3 + c_1c_2+ c_1)^{\ell-1}\gamma^{(1)}\nonumber\\
    &\leq 2^{\ell-1}(c_3 + c_1c_2+ c_1)^{\ell-1}\Big(2\beta^{(0)} + \frac{2c_1}{\sqrt{k}}\tau^{(0)} + 2c_1\gamma^{(0)}\Big),\label{eq:lemma-abgd-unknown-support-infty-gamma-val}
\end{align}
where in the last step we plug in the value $\gamma^{(1)} \leq 2\beta^{(0)} + \frac{2c_1}{\sqrt{k}}\tau^{(0)} + 2c_1\gamma^{(0)}$ from \eqref{eq:lemma-abgd-unknown-support-infty-1} at $\ell = 1$.

Using \eqref{eq:lemma-abgd-unknown-support-infty-gamma-val} in \eqref{eq:lemma-abgd-unknown-support-infty-3} gives
\begin{align}
    \beta^{(\ell)} &\leq c_32^{\ell-1}(c_3 + c_1c_2+ c_1)^{\ell-1}\Big(2\beta^{(0)} + \frac{2c_1}{\sqrt{k}}\tau^{(0)} + 2c_1\gamma^{(0)}\Big).\label{eq:lemma-abgd-unknown-support-infty-beta-val}
\end{align}

Using \eqref{eq:lemma-abgd-unknown-support-infty-gamma-val}, \eqref{eq:lemma-abgd-unknown-support-infty-2} and $\alpha^{(\ell)} \leq \sqrt{k}\gamma^{(\ell)}$ further gives:
\begin{align}
    \alpha^{(\ell)} &\leq \sqrt{k}2^{\ell-1}(c_3 + c_1c_2+ c_1)^{\ell-1}\Big(2\beta^{(0)} + \frac{2c_1}{\sqrt{k}}\tau^{(0)} + 2c_1\gamma^{(0)}\Big),\label{eq:lemma-abgd-unknown-support-infty-alpha-val}\\
    \tau^{(\ell)} &\leq c_2\sqrt{k}2^{\ell-1}(c_3 + c_1c_2+ c_1)^{\ell-1}\Big(2\beta^{(0)} + \frac{2c_1}{\sqrt{k}}\tau^{(0)} + 2c_1\gamma^{(0)}\Big).\label{eq:lemma-abgd-unknown-support-infty-tau-val}
\end{align}

\eqref{eq:lemma-abgd-unknown-support-infty-gamma-val}, \eqref{eq:lemma-abgd-unknown-support-infty-beta-val},\eqref{eq:lemma-abgd-unknown-support-infty-alpha-val} and \eqref{eq:lemma-abgd-unknown-support-infty-tau-val} give us the required result.
\end{proof}

\begin{thm*}[Restatement of Theorem \ref{thm:main_toy}]
Consider the \lrs problem  with $t$ linear regression tasks and samples obtained by \eqref{eq:samples} where rank $r=1$, $\sigma=0$, $\fl{U}^{\star} \equiv \fl{u}^{\star}\in \bb{R}^d$ and $\fl{w}^{\star}_i \equiv w^{\star}\in \bb{R}$. Let model parameters $\{\fl{b}^{\star (i)}\}_{i\in [t]}$ satisfy assumption A1 with $\zeta=O(t)$. Suppose Algorithm \ref{algo:optimize_lrs1} with modified updates (eqns. \ref{eq:update_simple1},\ref{eq:update_simple2},\ref{eq:update_simple3}) is run for $\s{L}=\log\left(\epsilon_0^{-1}\Big(\max_{i\in [t]}\lr{\fl{b}^{\star(i)}}_{\infty}+\lr{\fl{u}^{\star}}_{\infty}+\frac{\lr{\fl{u}^{\star}}_2}{\sqrt{k}}\Big)\right)$ iterations. Then, w.p. $\geq 1-O(\delta_0)$, the outputs $\fl{u}^{(\s{L})},\{\fl{b}^{(i,\s{L})}\}_{i\in[t]}$ satisfy:  
\begin{align}\label{eq:toy_final-temp}
    \lr{\fl{u}^{(\s{L})}-w^{\star}\fl{u}^{\star}}_{\infty} \le O(\epsilon_0)   \text{ and }
    \left|\left|\fl{b}^{(i, \s{L})}-\fl{b}^{\star (i)}\right|\right|_{\infty} \le O(\epsilon_0) \text{ for all }i\in [t]. \nonumber
\end{align}
provided the total number of samples satisfy 
\begin{align}
    m=\widetilde{\Omega}(k),\; mt=\widetilde{\Omega}(d\sqrt{k}) \text{ and }mt^2=\widetilde{\Omega}(\zeta k d).\nonumber
\end{align}
\end{thm*}

\begin{proof}
In order to map \ref{eq:samples_toy} to the statement of Theorem \ref{thm:main_toy} and the general problem statement in \ref{eq:samples}, recall that we can immediately write $w^{\star}\leftarrow \lr{\fl{u}^{\star}}_2$ and $\fl{u}^{\star} \leftarrow \frac{\fl{u}^{\star}}{\lr{\fl{u}^{\star}}_2}$ (since $\fl{u}^{\star}$ in the statement of Theorem \ref{thm:main_toy} is unit-norm). For the simplicity of notation, we had subsumed $w^{\star}$ within $\lr{\fl{u}^{\star}}_2$. Therefore, we directly use Lemma \ref{lemma:abgd-unknown-support-infty-bounds} to prove our theorem where the result is stated after mapping back to the setting in \ref{eq:samples_toy} and the Theorem statement. 
\end{proof}

\section{Algorithm and Proof of Theorem \ref{thm:main} (Parameter Recovery)}\label{app:detailed_proof}

\begin{algorithm}[!ht]
\caption{\textsc{AMHT-LRS} (Alternating Minimization for LRS in (\ref{prob:general}))}  \label{algo:optimize_lrs_main}
\begin{algorithmic}[1]
\REQUIRE Data $\{(\fl{x}^{(i)}_j\in \bb{R}^d,y^{(i)}_j\in \bb{R})\}_{j=1}^{m}$ for all $i\in [t]$, column sparsity $k$ of $\fl{B}$,  $\lr{\Delta(\fl{U}^{+(0)}, \fl{U}^{\star})}_\s{F} \le \s{B}$, $\max_i \|\fl{b}^{(i, 0)} - \fl{b}^{\star(i)}\|_\infty \leq \gamma^{(0)}$, Parameters $\epsilon>0$ and $\s{A}$.
\FOR{$\ell = 1, 2, \dots$}
    \STATE Set $T^{(\ell)} = \Omega\Big(\ell\log\Big(\frac{\gamma^{(\ell-1)}}{\epsilon}\Big)\Big)$
    \FOR{$i = 1,2, \dots, t$}
        \STATE $\fl{b}^{(i, \ell)} \gets \s{Optimize Sparse Vector}((\fl{X}^{(i)}, \fl{y}^{(i)}), \fl{v} = \fl{U}^{+(\ell-1)}\fl{w}^{(i,\ell-1)}, \alpha = O\Big(c_4^{\ell-1}\frac{\s{B}}{\sqrt{k}}+\s{A}\Big), \beta = O(c_5^{\ell-1}\s{B}+\s{A}), \gamma = \gamma^{(\ell-1)}+\s{A}, \s{T} = T^{(\ell)})$ \\
        \COMMENT{Use a fresh batch of data samples; $c_4,c_5$ are suitable constants}1
        \STATE  $\fl{w}^{(i, \ell)} = \Big((\fl{X}^{(i)}\fl{U}^{+(\ell-1)})^{\s{T}}(\fl{X}^{(i)}\fl{U}^{+(\ell-1)})\Big)^{-1}\Big((\fl{X}^{(i)}\fl{U}^{+(\ell-1)})^{\s{T}}(\fl{y}^{(i)} -  \fl{X}^{(i)}\fl{b}^{(i, \ell)})\Big)$ \COMMENT{Use a fresh batch of data samples}\label{optimize-lrs-w-update}
    \ENDFOR
    \STATE Set $\fl{A} := \sum_{i \in [t]}\Big(\fl{w}^{(i, \ell)}(\fl{w}^{(i, \ell)})^{\s{T}} \otimes \Big(\sum_{j=1}^{m}\fl{x}^{(i)}_j(\fl{x}^{(i)}_j)^{\s{T}}\Big)\Big)$  
    and $\fl{V} := \sum_{i \in [t]} (\fl{X}^{(i)})^{\s{T}}\Big(\fl{y}^{(i)} - \fl{b}^{(i, \ell)}\Big)(\fl{w}^{(i, \ell)})^{\s{T}}$  \COMMENT{Use a fresh batch of data samples}
    \STATE Compute  $\fl{U}^{(\ell)} =  \s{vec}^{-1}_{d\times r}(\fl{A}^{-1}\s{vec}(\fl{V}))$
    and $\fl{U}^{+(\ell)} \gets \s{QR}(\fl{U}^{(\ell)})$ \COMMENT{$\fl{U}^{(\ell)}=\fl{U}^{+(\ell)}\fl{R} $}
    \STATE $\gamma^{(\ell)} \gets (c_3)^{\ell-1}\epsilon\s{B}+\s{A}$ for a suitable constant $c_3<1$. 
\ENDFOR
\STATE Return  $\fl{w^{(\ell)}}$, $\fl{U}^{+(\ell)}$ and $\{\fl{b}^{(i, \ell)} \}_{i\in [t]}$.
\end{algorithmic}
\end{algorithm}

\begin{algorithm}[!ht]
\caption{\textsc{DP-AMHT-LRS} ( Private Alternating Minimization for LRS in (\ref{prob:general}))}  \label{algo:optimize_dp-lrs}
\begin{algorithmic}[1]
\REQUIRE Data $\{(\fl{x}^{(i)}_j\in \bb{R}^d,y^{(i)}_j\in \bb{R})\}_{j=1}^{m}$ for all $i\in [t]$, column sparsity $k$ of $\fl{B}$,  $\lr{\Delta(\fl{U}^{+(0)}, \fl{U}^{\star})}_\s{F} \le \s{B}$, $\max_i \|\fl{b}^{(i, 0)} - \fl{b}^{\star(i)}\|_\infty \leq \gamma^{(0)}$, Parameters $\epsilon>0$ and $\s{A}$.
\FOR{$\ell = 1, 2, \dots$}
    \STATE Set $T^{(\ell)} = \Omega\Big(\ell\log\Big(\frac{\gamma^{(\ell-1)}}{\epsilon}\Big)\Big)$
    \FOR{$i = 1,2, \dots, t$}
        \STATE $\fl{b}^{(i, \ell)} \gets \s{Optimize Sparse Vector}((\fl{X}^{(i)}, \fl{y}^{(i)}), \fl{v} = \fl{U}^{+(\ell-1)}\fl{w}^{(i,\ell-1)}, \alpha = O\Big(c_4^{\ell-1}\frac{\s{B}}{\sqrt{k}}+\s{A}\Big), \beta = O(c_5^{\ell-1}\s{B}+\s{A}), \gamma = \gamma^{(\ell-1)}+\s{A}, \s{T} = T^{(\ell)})$ for suitable constants $c_4,c_5$\\
        \STATE  $\fl{w}^{(i, \ell)} = \Big((\fl{X}^{(i)}\fl{U}^{+(\ell-1)})^{\s{T}}(\fl{X}^{(i)}\fl{U}^{+(\ell-1)})\Big)^{-1}\Big((\fl{X}^{(i)}\fl{U}^{+(\ell-1)})^{\s{T}}(\fl{y}^{(i)} -  \fl{X}^{(i)}\fl{b}^{(i, \ell)})\Big)$ 
        \label{optimize_dp-lrs-w-update}
    \ENDFOR
    \STATE $\forall i,j: \widehat{\fl{x}^{(i)}_j}\leftarrow {\sf clip}_{A_1}\left(\fl{x}^{(i)}_j\right)$, $\widehat{y^{(i)}_j} \leftarrow {\sf clip}_{A_2}\left(y^{(i)}_j\right)$, $\widehat{(\fl{x}^{(i)}_j)^\s{T}\fl{b}^{(i, \ell)}}\leftarrow {\sf clip}_{A_3}\left((\fl{x}^{(i)}_j)^\s{T}\fl{b}^{(i, \ell)}\right)$ and $\widehat{\fl{w}^{(i, \ell)}}\leftarrow {\sf clip}_{A_w}\left(\fl{w}^{(i, \ell)}\right)$  
    \STATE $\fl{A} := \frac{1}{mt}\Big(\sum_{i \in [t]}\Big(\widehat{\fl{w}^{(i, \ell)}}(\widehat{\fl{w}^{(i, \ell)}})^{\s{T}} \otimes \Big(\sum_{j=1}^{m}\widehat{\fl{x}^{(i)}_j}(\widehat{\fl{x}^{(i)}_j})^{\s{T}}\Big)\Big)  + \fl{N}_1\Big)$
    \STATE $\fl{V} := \frac{1}{mt}\Big(\sum_{i \in [t]}\sum_{j \in [m]} \widehat{\fl{x}^{(i)}_j}\Big(\widehat{y^{(i)}_j} - (\widehat{\fl{x}^{(i)}_j)^{\s{T}}\fl{b}^{(i, \ell)}} \Big)(\widehat{\fl{w}^{(i, \ell)}})^{\s{T}} + \fl{N}_2\Big)$ 
    \STATE $\fl{U}^{(\ell)} =  \s{vec}^{-1}_{d\times r}(\fl{A}^{-1}\s{vec}(\fl{V}))$
    \STATE $\fl{U}^{+(\ell)} \gets \s{QR}(\fl{U}^{(\ell)})$ \COMMENT{$\fl{U}^{(\ell)}=\fl{U}^{+(\ell)}\fl{R} $}
    \STATE $\gamma^{(\ell)} \gets (c_3)^{\ell-1}\epsilon\s{B}+\s{A}$ for a suitable constant $c_3<1$. 
\ENDFOR
\STATE Return  $\fl{w^{(\ell)}}$, $\fl{U}^{+(\ell)}$ and $\{\fl{b}^{(i, \ell)} \}_{i\in [t]}$.
\end{algorithmic}
\end{algorithm}

\begin{assumption}[A3]\label{assum:incoherence_u}
We assume that $\lr{\fl{U}^{\star}}_{2,\infty} \le \sqrt{\nu^{\star}/k}$ for some constant $\nu^{\star}>0$.
\end{assumption}

Note that Assumption A3 is weaker than Assumption A2 where $\lr{\fl{U}^{\star}}_{2,\infty} \le \sqrt{\mu^{\star}/d}$ provided $k\le \frac{d\nu^{\star}}{\mu^{\star}}$. We will use Assumption A3 in place of A2 for simplicity of exposition and for sharper guarantees as well. Recall that in the general setting described in eq. \ref{eq:samples}, we obtain samples that are generated according to the following process:
\begin{align}
    &\fl{x}^{(i)}_j \sim \ca{N}(\f{0},\fl{I}_d) 
    \text{ and }
    y_j^{(i)}\mid \fl{x}^{(i)}_j = \langle \fl{x}^{(i)}_j, \fl{U}^{\star} \fl{w}^{\star(i)}+\fl{b}^{\star(i)} \rangle+z^{(i)}_j  \; \text{ for all } i\in[t],j\in [m], 
\end{align}

where each $z^{(i)}_j\sim \ca{N}(0,\sigma^2)$ denotes the independent measurement noise with known variance $\sigma^2$. For each task $i\in [t]$, we will denote the noise vector to be $\fl{z}^{(i)}$ such that its $j^{\s{th}}$ co-ordinate is $\fl{z}^{(i)}_j$. Further, with some abuse of notation we will denote:
\begin{itemize}
    \item $\lambda_j \equiv \lambda_j\Big(\frac{r}{t}(\fl{W}^{(\ell)})^\s{T}\fl{W}^{(\ell)}\Big)$,
    $\lambda^\star_j$ = $\lambda_j\Big(\frac{r}{t}(\fl{W}^{\star})^\s{T}\fl{W}^{\star}\Big) \equiv \lambda_j\Big(\frac{r}{t}(\fl{W}^{\star})^\s{T}\fl{W}^{\star}\Big)$ $\forall$ $j \in [r]$,
    \item $\mu \equiv \mu^{(\ell)}$ and $\mu^\star$ for the incoherence factors for $\fl{W}^{(\ell)}$ and $\fl{W}^\star$ respectively,
    \item $\nu \equiv \nu^{(\ell)}$ for the incoherence factor of $\fl{U}^{+(\ell)}$.
\end{itemize}
We will now prove Theorem \ref{thm:main} via an inductive argument. We will start with the base case. 

\subsection{Base Case}\label{inductive-corollary:optimize_dp-lrs-base-case}
We initialize $\fl{W}^{(0)} = \vzero$ and recall $\lr{(\fl{I}-\fl{U}^{\star}(\fl{U}^{\star})^{\s{T}})\fl{U}^{+(0)}}_{\s{F}} = O\Big(\sqrt{\frac{\lambda_r^{\star}}{\lambda_1^{\star}}}\Big)$, $\lr{\fl{U}^{+(0)}}\le \sqrt{\frac{\nu^{(0)}}{k}}$ where $\nu^{(0)}$ is an appropriate constant less than $1$.
We use Lemma~\ref{inductive-corollary:optimize_dp-lrs-innerloop-output-bounds} that is proved later in its full generality. We have by using Lemma~\ref{inductive-corollary:optimize_dp-lrs-innerloop-output-bounds} :
\begin{align}
    \lr{\fl{b}^{(i,1)}-\fl{b}^{\star (i)} }_{2} \leq 2\varphi^{(i)} + \epsilon
    \text{ and } \lr{\fl{b}^{(i,1)}-\fl{b}^{\star (i)} }_{\infty} \leq \frac{1}{\sqrt{k}}\Big(2\varphi^{(i)} + \epsilon\Big)\nonumber
\end{align}
with probability at least $1-T^{(\ell)}\delta$, where $\varphi^{(i)}$ is an upper-bound on $\widehat{\varphi}^{(i)}$ s.t. 
\begin{align}
    \widehat{\varphi}^{(i)} &= 2\Big(\sqrt{k}\|\fl{U}^{\star}\fl{w}^{\star(i)}\|_{\infty}+c_1\|\fl{U}^{\star}\fl{w}^{\star(i)}\|_2+\sigma \sqrt{\frac{k\log (d\delta^{-1})}{m}}\Big)\nonumber\\
    &\leq 2\Big(\sqrt{k}\|\fl{U}^{\star}\|_{2,\infty}\|\fl{w}^{\star(i)}\|_2+c_1\|\fl{w}^{\star(i)}\|_2+\sigma \sqrt{\frac{k\log (d\delta^{-1})}{m}}\Big)\nonumber\\
    &\leq 2\Big(\sqrt{\nu^\star}+c_1\Big)\|\fl{w}^{\star(i)}\|_2+2\sigma \sqrt{\frac{k\log (d\delta^{-1})}{m}}.\nonumber
\end{align}
Choosing $\epsilon = 4\Big(\sqrt{\nu^\star}+c_1\Big)$ gives us the required expression for $\ell = 0$.
Hence, we have that for $c' =O\Big(\frac{1} {\s{B}_{\fl{U}^{(0)}}}\frac{\lambda^{\star}_1}{\lambda^{\star}_r}\Big)$, we will have that 
\begin{align}
    &\lr{\fl{b}^{(i,0)}-\fl{b}^{\star(i)}}_2 \le c'\max(\epsilon,\lr{\fl{w}^{\star(i)}}_2)\s{B}_{\fl{U}^{(0)}}\sqrt{\frac{\lambda^{\star}_r}{\lambda^{\star}_1}} +4\sigma \sqrt{\frac{k\log (d\delta^{-1})}{m}}\\
    &\lr{\fl{b}^{(i,0)}-\fl{b}^{\star(i)}}_\infty \le c'\max(\epsilon,\lr{\fl{w}^{\star(i)}}_2)\s{B}_{\fl{U}^{(0)}}\sqrt{\frac{\lambda^{\star}_r}{k\lambda^{\star}_1}}+4\sigma \sqrt{\frac{\log (d\delta^{-1})}{m}}
\end{align}
\subsection{Inductive Step}

We will begin with the inductive assumption. Note that these assumptions are true in the base case as well due to our initialization and optimizing the task-specific sparse vector. Let
\begin{align}
    &\resizebox{.97\textwidth}{!}{$\Lambda = \cO\Big( \sqrt{\lambda^{\star}_r\mu^{\star}}\Big(\frac{\sigma_2 r}{mt\lambda^{\star}_r}+\frac{\sigma_1r^{3/2}}{mt\lambda_r^\star}\sqrt{rd\log rd}+\sigma\sqrt{\frac{r^3d\mu^\star\log^2 (r\delta^{-1})}{mt\lambda_r^\star}}\Big) + \sigma\Big(\sqrt{\frac{r^3\log^2(r\delta^{-1})}{m\lambda^{\star}_r}}\Big) + \sqrt{\frac{k\log (d\delta^{-1})}{m}}\Big)\Big)$} \nonumber\\
    &\Lambda' = \cO\Big(\frac{\Lambda}{\sqrt{\mu^{\star}\lambda^{\star}_r}}\Big).\nonumber
\end{align}

\begin{assumption}[Inductive Assumption]\label{assum:init}
At the beginning of the $\ell^{\s{th}}$ iteration, we will use  $q^{(\ell-1)}, \s{B}_{\fl{u}^{+(\ell-1)}}$ to describe the following upper bounds on the quantities of interest:
\begin{align}
    &\text{1) } 1/2 < \lambda_{\min}(\fl{Q}^{(\ell-1)}) \leq \lambda_{\max}(\fl{Q}^{(\ell-1)}) < 1, \text{ where } \fl{Q}^{(\ell-1)} := \langle (\fl{U}^{\star})^\s{T} \fl{U}^{+(\ell-1)}\rangle\label{eq:rec-1}\\
    &\text{2) } 
    \|\Delta(\fl{U}^{+(\ell-1)}, \fl{U}^{\star})\|_\s{F} = \|(\fl{I} - \fl{U}^{\star}(\fl{U}^{\star})^{\s{T}})\fl{U}^{+(\ell-1)}\|_\s{F} = \|\fl{U}^{+(\ell-1)} - \fl{U}^{\star}\fl{Q}^{(\ell-1)}\|_{\s{F}} \nonumber\\
    &\le \s{B}_{\fl{U^{(\ell-1)}}}\sqrt{\frac{\lambda_r^\star}{\lambda_1^\star}} + \Lambda',\label{eq:rec-2}\\
    &\text{3) }\|\fl{b}^{(i, \ell)}-\fl{b}^{\star(i)}\|_2 \le c'\|(\fl{U}^{+(\ell-1)} - \fl{U}^{\star}\fl{Q}^{(\ell-1)})(\fl{Q}^{(\ell-1)})^{-1}
   \fl{w}^{\star(i)}\|_2 \nonumber\\
   &\leq c'\max\{\epsilon, \|\fl{w}^{\star(i)}\|_2\}\s{B}_{\fl{U^{(\ell-1)}}}\sqrt{\frac{\lambda_r^\star}{\lambda_1^\star}} + \Lambda,\label{eq:rec-3}\\
    &\text{4) }\|\fl{b}^{(i, \ell)}-\fl{b}^{\star(i)}\|_\infty \le  c'\|(\fl{U}^{+(\ell-1)} - \fl{U}^{\star}\fl{Q}^{(\ell-1)})(\fl{Q}^{(\ell-1)})^{-1}
   \fl{w}^{\star(i)}\|_2/\sqrt{k} \nonumber\\
   &\leq c'\max\{\epsilon, \|\fl{w}^{\star(i)}\|_2\}\s{B}_{\fl{U^{(\ell-1)}}}\sqrt{\frac{\lambda_r^\star}{\lambda_1^\star k}} + \frac{\Lambda}{\sqrt{k}}, \label{eq:rec-4}\\
    &\text{5) }\|\fl{U}^{+(\ell-1)}\|_{2, \infty} \le \sqrt{\nu^{(\ell-1)}/k} ,\label{eq:rec-5}
\end{align} 
where $\nu^{(\ell-1)} < \frac{1}{181c}$, $c' > 0$ and $\Lambda'<1/1000$. Note that $\Lambda,\Lambda'$ are fixed and do not change with iterations. 
\end{assumption}

Note that the base case satisfies the inductive assumption for our problem. Let us denote $\fl{h}^{(i, \ell)}\triangleq \fl{w}^{(i, \ell)} - (\fl{Q}^{(\ell-1)})^{-1}\fl{w}^{\star(i)}$ and $(\fl{H}^{(\ell)})^\s{T} = \begin{bmatrix}
\fl{h}^{(1, \ell)} & \fl{h}^{(2, \ell)} & \dots & \fl{h}^{(t, \ell)}
\end{bmatrix}_{\bR^{r\times t}}$.

\begin{lemma}\label{lemma:optimize_dp-lrs-wbound}
For some constant $c>0$ and for any iteration indexed by $\ell>0$, we  have
\begin{align}
    &\|\fl{h}^{(i, \ell)}\|_2 \leq \frac{1}{1 - c\sqrt{\frac{r\log(1/ \delta_0)}{m}}}\Big\{\|\fl{U}^{+(\ell-1)} - \fl{U}^{\star}\fl{Q}^{(\ell-1)}\|_\s{F}\|(\fl{Q}^{(\ell-1)})^{-1}\|\|\fl{w}^{\star(i)}\|_2\cdot\nonumber\\
    &\qquad \Big(\|\fl{U}^{+(\ell-1)} - \fl{U}^{\star}\fl{Q}^{(\ell-1)}\|_\s{F}+c\sqrt{\frac{\log(r/\delta_0)}{m}}\|\fl{U}^{+(\ell-1)}\|_{\s{F}}\Big)\nonumber\\
    &\qquad +\|\fl{b}^{\star(i)} - \fl{b}^{(i, \ell)}\|_2\Big(\sqrt{k}\|\fl{U}^{+(\ell-1)}\|_{2, \infty}+c\sqrt{\frac{\log(r/\delta_0)}{m}}\|\fl{U}^{+(\ell-1)}\|_{\s{F}} \Big)+\sigma\sqrt{\frac{r\log^2 (r\delta^{-1})}{m}}\Big\},\nonumber\\
    &\|\fl{H}^{(\ell)}\|_\s{F} \leq \frac{1}{1 - c\sqrt{\frac{r\log(1/ \delta_0)}{m}}}\cdot\Big\{\|\fl{U}^{+(\ell-1)} - \fl{U}^{\star}\fl{Q}^{(\ell-1)}\|_\s{F}\|(\fl{Q}^{(\ell-1)})^{-1}\|\sqrt{\frac{t}{r}\lambda^\star_1}\cdot \nonumber\\
    &\qquad \Big(\|\fl{U}^{+(\ell-1)} - \fl{U}^{\star}\fl{Q}^{(\ell-1)}\|_\s{F}+c\sqrt{\frac{\log(r/\delta_0)}{m}}\|\fl{U}^{(\ell-1)}\|_{\s{F}}\Big)\nonumber\\
    &\qquad + \sqrt{k\zeta}\|\fl{U}^\star\|_\s{F}\|\fl{b}^{\star(i)} - \fl{b}^{(i, \ell)}\|_\infty + \sqrt{\sum_{i \in [t]} \Big(c\|\fl{b}^{\star(i)} - \fl{b}^{(i, \ell)}\|_2\sqrt{\frac{\log(r/\delta_0)}{m}}\|\fl{U}^{+(\ell-1)}\|_{\s{F}}\Big)^2}\nonumber\\
    &\qquad +\sigma\sqrt{\frac{rt\log^2 (r\delta^{-1})}{m}}\Big\}\nonumber
\end{align}
with probability at least $1-\delta_0$, where $\fl{h}^{(i, \ell)} = \fl{w}^{(i, \ell)} - (\fl{Q}^{(\ell-1)})^{-1}\fl{w}^{\star(i)}$.
\end{lemma}
\begin{proof}
According to the update step \eqref{optimize_dp-lrs-w-update}, we have 
\begin{align}
    \fl{w}^{(i, \ell)} - (\fl{Q}^{(\ell-1)})^{-1}\fl{w}^{\star(i)} &= \Big(\frac{(\fl{X}^{(i)}\fl{U}^{+(\ell-1)})^{\s{T}}(\fl{X}^{(i)}\fl{U}^{+(\ell-1)})}{m}\Big)^{-1}\cdot\nonumber\\
    &\qquad \Big(\frac{(\fl{X}^{(i)}\fl{U}^{+(\ell-1)})^{\s{T}}(\fl{y}^{(i)} -  \fl{X}^{(i)}\fl{b}^{(i, \ell)})}{m}\Big) - (\fl{Q}^{(\ell-1)})^{-1}\fl{w}^{\star(i)}\nonumber
\end{align}
\begin{align}
    &\iff \fl{h}^{(i, \ell)}\nonumber\\
    &:=\underbrace{\Big(\frac{(\fl{X}^{(i)}\fl{U}^{+(\ell-1)})^{\s{T}}(\fl{X}^{(i)}\fl{U}^{+(\ell-1)})}{m}\Big)^{-1}}_{\fl{A}}\cdot\nonumber\\
    &\qquad \underbrace{\Big(\frac{(\fl{X}^{(i)}\fl{U}^{+(\ell-1)})^{\s{T}}(\fl{y}^{(i)} -  \fl{X}^{(i)}\fl{b}^{(i, \ell)})}{m} - \frac{(\fl{X}^{(i)}\fl{U}^{+(\ell-1)})^{\s{T}}(\fl{X}^{(i)}\fl{U}^{+(\ell-1)})(\fl{Q}^{(\ell-1)})^{-1}\fl{w}^{\star(i)}}{m}\Big)}_{\fl{z}}.  \label{eq:lemma-optimize_dp-lrs-wbound-w-unplugged}
\end{align}
Therefore, 
\begin{align}
    \|\fl{h}^{(i, \ell)}\|_\infty \leq \|\fl{A}\|\|\fl{z}\|_\infty \qquad \text{and} \qquad \|\fl{h}^{(i, \ell)}\|_2 \leq \|\fl{A}\|\|\fl{z}\|_2. \label{eq:lemma-optimize_dp-lrs-wbound-w-unplugged-2}
\end{align}
We will analyse the terms $\fl{A}$ and $\fl{z}$ separately.\\\\
\textbf{Analysis of $\fl{A}$:}\\\\
Note that:
\begin{align}
    \fl{A}^{-1} &= \frac{1}{m}(\fl{X}^{(i)}\fl{U}^{+(\ell-1)})^{\s{T}}(\fl{X}^{(i)}\fl{U}^{+(\ell-1)})\nonumber\\
    &= \frac{1}{m}(\fl{U}^{+(\ell-1)})^{\s{T}}(\fl{X}^{(i)})^{\s{T}}\fl{X}^{(i)}\fl{U}^{+(\ell-1)}\nonumber\\
    &= \frac{1}{m}(\fl{U}^{+(\ell-1)})^{\s{T}}\Big(\sum_j\fl{x}^{(i)}_j(\fl{x}^{(i)}_j)^{\s{T}}\Big)\fl{U}^{+(\ell-1)}\nonumber\\
    &= \frac{1}{m}\sum_j(\fl{U}^{+(\ell-1)})^{\s{T}}\fl{x}^{(i)}_j(\fl{x}^{(i)}_j)^{\s{T}}\fl{U}^{+(\ell-1)}.\label{eq:lemma-optimize_dp-lrs-wbound-Ainv-unplugged}
\end{align}
Now, let $\fl{\cV} \triangleq \{\fl{v} \in \bR^r | \|\fl{v}\| = 1\}$. Then for $\epsilon \leq 1$, there exists an $\epsilon$-net, $N_\epsilon \subset \fl{\cV}$, of size $(1 + 2/\epsilon)^
r$ w.r.t the Euclidean norm, i.e. $\forall$  $\fl{v} \in \fl{\cV}$, $\exists$ $\fl{v}' \in N_\epsilon$ s.t. $\|\fl{v} - \fl{v}'\|_2 \leq \epsilon$.
Then for any $\fl{v} \in N_\epsilon$,
\begin{align}
    &\fl{v}^{\s{T}}\Big(\frac{1}{m}\sum_j(\fl{U}^{+(\ell-1)})^{\s{T}}\fl{x}^{(i)}_j(\fl{x}^{(i)}_j)^{\s{T}}\fl{U}^{+(\ell-1)}\Big)\fl{v} \nonumber\\
    &\qquad = \frac{1}{m}\sum_j\Big(\fl{v}^{\s{T}}(\fl{U}^{+(\ell-1)})^{\s{T}}\Big)\fl{x}^{(i)}_j(\fl{x}^{(i)}_j)^{\s{T}}\Big(\fl{U}^{+(\ell-1)}\fl{v}\Big)\nonumber\\
    &\qquad = \frac{1}{m}\sum_j\Big(\fl{U}^{+(\ell-1)})\fl{v}\Big)^{\s{T}}\fl{x}^{(i)}_j(\fl{x}^{(i)}_j)^{\s{T}}\Big(\fl{U}^{+(\ell-1)}\fl{v}\Big). \label{eq:lemma-optimize_dp-lrs-wbound-A-temp1}
\end{align}
Further, note that 
\begin{align}
    \Big(\fl{U}^{+(\ell-1)}\fl{v}\Big)^{\s{T}}\Big(\fl{U}^{+(\ell-1)}\fl{v}\Big) &= \s{Tr}\Big(\Big(\fl{U}^{+(\ell-1)}\fl{v}\Big)^{\s{T}}\Big(\fl{U}^{+(\ell-1)}\fl{v}\Big)\Big)\nonumber\\
    &= \s{Tr}\Big(\fl{v}^{\s{T}}\Big((\fl{U}^{+(\ell-1)})^{\s{T}}\fl{U}^{+(\ell-1)}\Big)\fl{v}\Big). \label{eq:lemma-optimize_dp-lrs-wbound-A-temp2}
\end{align}
Using \eqref{eq:lemma-optimize_dp-lrs-wbound-A-temp1} and \eqref{eq:lemma-optimize_dp-lrs-wbound-A-temp2} in Lemma~\ref{lemma:useful1} with $\fl{a} = \fl{b} = \fl{U}^{+(\ell-1)}\fl{v}$ gives
\begin{align}
    &\left|\fl{v}^{\s{T}}\Big(\frac{1}{m}\sum_j(\fl{U}^{+(\ell-1)})^{\s{T}}\fl{x}^{(i)}_j(\fl{x}^{(i)}_j)^{\s{T}}\fl{U}^{+(\ell-1)} - (\fl{U}^{+(\ell-1)})^{\s{T}}\fl{U}^{+(\ell-1)}\Big)\Big)\fl{v}\right| \nonumber\\
    &\qquad \leq c\|\fl{U}^{+(\ell-1)}\fl{v}\|^2_2\max\Big(\sqrt{\frac{\log (1/\delta_0)}{m}},\frac{\log (1/\delta_0)}{m}\Big)\nonumber\\
    \implies  &\|\fl{v}^{\s{T}}\fl{E}\fl{v}\| \leq c\|\fl{U}^{+(\ell-1)}\fl{v}\|^2_2\max\Big(\sqrt{\frac{\log (|N_\epsilon|/\delta_0)}{m}},\frac{\log (|N_\epsilon|/\delta_0)}{m}\Big)\nonumber\\
     &\leq c\|\fl{U}^{+(\ell-1)}\fl{v}\|^2_2\max\Big(\sqrt{\frac{\log((1+2/\epsilon)^r/\delta_0)}{m}},\frac{\log ((1+2/\epsilon)^r/\delta_0)}{m}\Big) \quad \forall \fl{v} \in N_\epsilon\label{eq:lemma-optimize_dp-lrs-wbound-A-temp3}
\end{align}
where $\fl{E} \triangleq \frac{1}{m}\sum_j(\fl{U}^{+(\ell-1)})^{\s{T}}\fl{x}^{(i)}_j(\fl{x}^{(i)}_j)^{\s{T}}\fl{U}^{+(\ell-1)} - (\fl{U}^{+(\ell-1)})^{\s{T}}\fl{U}^{+(\ell-1)}$. Since $\fl{E}$ is symmetric, therefore $\|\fl{E}\| = (\fl{v}')^{\s{T}}\fl{E}\fl{v}'$ where $\fl{v}' \in \fl{\s{V}}$ is the largest eigenvector of $\fl{E}$. Further, $\exists$ $\fl{v} \in N_\epsilon$ s.t. $\|\fl{v}' - \fl{v}\| \leq \epsilon$. This implies
\begin{align}
    \|\fl{E}\| = (\fl{v}')^{\s{T}}\fl{E}\fl{v}' &= (\fl{v}' - \fl{v})^{\s{T}}\fl{E}\fl{v} + (\fl{v}')^{\s{T}}\fl{E}(\fl{v}' - \fl{v}) + \fl{v}^{\s{T}}\fl{E}\fl{v}\nonumber\\ 
    & \leq \|\fl{v}' - \fl{v}\|\|\fl{E}\|\|\fl{v}\| + \|\fl{v}'\|\|\fl{E}\|\|\fl{v}' - \fl{v}\| + \fl{v}^{\s{T}}\fl{E}\fl{v}\nonumber\\ 
    & \leq 2\epsilon\|\fl{E}\| + \fl{v}^{\s{T}}\fl{E}\fl{v}\nonumber\\ 
    \implies \|\fl{E}\| &\leq \frac{\fl{v}^{\s{T}}\fl{E}\fl{v}}{1 - 2\epsilon}.\label{eq:lemma-optimize_dp-lrs-wbound-A-temp4}
\end{align}
Using \eqref{eq:lemma-optimize_dp-lrs-wbound-A-temp3} and \eqref{eq:lemma-optimize_dp-lrs-wbound-A-temp4} and setting $\epsilon \gets 1/4$ and $c \gets 2c\sqrt{\log(9)}$ then gives:
\begin{align}
    \|\fl{E}\| \leq c\|\fl{U}^{+(\ell-1)}\fl{v}\|^2_2\max\sqrt{\frac{r\log(1/ \delta_0)}{m}}\label{eq:lemma-optimize_dp-lrs-wbound-E-value}
\end{align}
Using \eqref{eq:lemma-optimize_dp-lrs-wbound-E-value} in \eqref{eq:lemma-optimize_dp-lrs-wbound-Ainv-unplugged} then gives
\begin{align}
    \|\fl{A}^{-1}\| &\geq \|\fl{U}^{+(\ell-1)}\fl{v}\|^2_2\Big(1 - c\sqrt{\frac{r\log(1/ \delta_0)}{m}}\Big)\nonumber\\
    &\geq \lambda_{\min}\Big((\fl{U}^{+(\ell-1)})^{\s{T}}\fl{U}^{+(\ell-1)}\Big)\Big(1 - c\sqrt{\frac{r\log(1/ \delta_0)}{m}}\Big)\nonumber\\
    \implies \|\fl{A}\| &\leq \frac{1}{\lambda_{\min}\Big((\fl{U}^{+(\ell-1)})^{\s{T}}\fl{U}^{+(\ell-1)}\Big)\Big(1 - c\sqrt{\frac{r\log(1/ \delta_0)}{m}}\Big)}\nonumber\\
    &= \frac{1}{1 - c\sqrt{\frac{r\log(1/ \delta_0)}{m}}}\label{eq:lemma-optimize_dp-lrs-wbound-A-bound}\\
\end{align}
since $(\fl{U}^{+(\ell-1)})^\s{T}\fl{U}^{+(\ell-1)} = \fl{I}$.\\\\
\textbf{Analysis of $\fl{z}$:}\\\\
Similarly, we have
\begin{align}
    \fl{z}&= \frac{1}{m}(\fl{X}^{(i)}\fl{U}^{+(\ell-1)})^{\s{T}}(\fl{y}^{(i)} -  \fl{X}^{(i)}\fl{b}^{(i, \ell)}) - \frac{1}{m}(\fl{X}^{(i)}\fl{U}^{+(\ell-1)})^{\s{T}}(\fl{X}^{(i)}\fl{U}^{+(\ell-1)})(\fl{Q}^{(\ell-1)})^{-1}\fl{w}^{\star(i)}\nonumber\\
    &= \underbrace{\frac{1}{m}(\fl{U}^{+(\ell-1)})^{\s{T}}(\fl{X}^{(i)})^{\s{T}}\fl{X}^{(i)}(\fl{U}^{\star}\fl{Q}^{(\ell-1)} - \fl{U}^{+(\ell-1)})(\fl{Q}^{(\ell-1)})^{-1}\fl{w}^{\star(i)}}_{:= \fl{d}^{(i, \ell)}_1} \nonumber\\
    &\qquad + \underbrace{\frac{1}{m}(\fl{U}^{+(\ell-1)})^{\s{T}}(\fl{X}^{(i)})^{\s{T}}\fl{X}^{(i)}(\fl{b}^{\star(i)} - \fl{b}^{(i, \ell)})}_{:= \fl{d}^{(i, \ell)}_2}+\underbrace{\frac{1}{m}(\fl{U}^{+(\ell-1)})^{\s{T}}(\fl{X}^{(i)})^{\s{T}}\fl{z}^{(i)}}_{\fl{d}_3^{(i,\ell)}}.
\end{align}

\underline{Analysis of $\fl{d}^{(i, \ell)}_3$:}\\
Let us condition on the vector $\fl{z}^{(i)}$. In that case $(\fl{X}^{(i)})^{\s{T}}\fl{z}^{(i)}$ is a $d\times 1$ vector, each of whose entry is generated independently according to $\ca{N}(0,\lr{\fl{z}^{(i)}}_2^2)$. Therefore, if we consider any vector $\fl{v}$ satisfying $\lr{\fl{v}}_2=1$, we have
\begin{align}
    \fl{v}^{\s{T}} (\fl{X}^{(i)})^{\s{T}}\fl{z}^{(i)} \sim \ca{N}(0,\lr{\fl{z}^{(i)}}_2^2)\nonumber
\end{align}
and therefore, with probability $1-\delta$, we must have 
\begin{align}
    \lr{\frac{1}{m}(\fl{U}^{+(\ell-1)})^{\s{T}}(\fl{X}^{(i)})^{\s{T}}\fl{z}^{(i)}}_2  \le \sigma\sqrt{\frac{r\log^2 (r\delta^{-1})}{m}}.\nonumber
\end{align}

\underline{Analysis of $\fl{d}^{(i, \ell)}_1$:}
\begin{align}
    \fl{d}^{(i, \ell)}_1 &= \frac{1}{m}(\fl{U}^{+(\ell-1)})^{\s{T}}(\fl{X}^{(i)})^{\s{T}}\fl{X}^{(i)}(\fl{U}^{\star}\fl{Q}^{(\ell-1)} - \fl{U}^{+(\ell-1)})(\fl{Q}^{(\ell-1)})^{-1}\fl{w}^{\star(i)}\nonumber\\
    &= \underbrace{(\fl{U}^{+(\ell-1)})^{\s{T}}\Big(\frac{1}{m}(\fl{X}^{(i)})^{\s{T}}\fl{X}^{(i)}-\fl{I}\Big)(\fl{U}^{\star}\fl{Q}^{(\ell-1)} - \fl{U}^{+(\ell-1)})(\fl{Q}^{(\ell-1)})^{-1}\fl{w}^{\star(i)}}_{\fl{d}^{(i, \ell)}_{1,1}}\nonumber \\
    &\qquad + \underbrace{(\fl{U}^{+(\ell-1)})^{\s{T}}(\fl{U}^{\star}\fl{Q}^{(\ell-1)}) - \fl{U}^{+(\ell-1)})(\fl{Q}^{(\ell-1)})^{-1}\fl{w}^{\star(i)}}_{\fl{d}^{(i, \ell)}_{1,2}}.\label{eq:lemma-optimize_dp-lrs-wbound-z1-unpluggeed}
\end{align}
Note that
\begin{align}
    \E{\fl{d}^{(i, \ell)}_{1,1}} &= \E{(\fl{U}^{+(\ell-1)})^{\s{T}}\Big(\frac{1}{m}(\fl{X}^{(i)})^{\s{T}}\fl{X}^{(i)}-\fl{I}\Big)(\fl{U}^{\star}\fl{Q}^{(\ell-1)} - \fl{U}^{+(\ell-1)})(\fl{Q}^{(\ell-1)})^{-1}\fl{w}^{\star(i)}}\nonumber\\
    &= \vzero.\nonumber
\end{align}
Further,
\begin{align}
    (z^{(i, \ell)}_{1,1})_k &= \frac{1}{m}\sum_j(\fl{u}^{(k, \ell-1)})^{\s{T}}\fl{x}^{(i)}_j(\fl{x}^{(i)}_j)^{\s{T}}(\fl{U}^{\star}\fl{Q}^{(\ell-1)} - \fl{U}^{+(\ell-1)})(\fl{Q}^{(\ell-1)})^{-1}\fl{w}^{\star(i)} \nonumber\\
    &\qquad - (\fl{u}^{(k, \ell-1)})^{\s{T}}(\fl{U}^{\star}\fl{Q}^{(\ell-1)} - \fl{U}^{+(\ell-1)})(\fl{Q}^{(\ell-1)})^{-1}\fl{w}^{\star(i)}.\nonumber
\end{align}
Using Lemma \ref{lemma:useful1} in the above with $\fl{a} = \fl{u}^{(k, \ell-1)}$ and $\fl{b} = (\fl{U}^{\star}\fl{Q} - \fl{U}^{+(\ell-1)})\fl{Q}^{-1}\fl{w}^{\star(i)}$, we get
\begin{align}
    (z^{(i, \ell)}_{1,1})_k &\leq c\sqrt{\frac{\log(1/\delta_0)}{m}}\|\fl{u}^{(k, \ell-1)}\|_2\|(\fl{U}^{\star}\fl{Q}^{(\ell-1)} - \fl{U}^{+(\ell-1)})(\fl{Q}^{(\ell-1)})^{-1}\fl{w}^{\star(i)}\|_2.\label{eq:lemma-optimize_dp-lrs-wbound-z11k-bound}
\end{align}
Taking the Union Bound overall entries $k \in [r]$, we have
\begin{align}
    \|\fl{d}^{(i, \ell)}_{1,1}\|_2 &= \sqrt{\sum_{k \in [r]}|(z^{(i, \ell)}_{1,1})_k|^2}\nonumber\\
    &\leq c\sqrt{\frac{\log(r/\delta_0)}{m}}\sqrt{\sum_{k \in [r]}\|\fl{u}^{(k, \ell-1)}\|_2^2}\cdot\|(\fl{U}^{\star}\fl{Q}^{(\ell-1)} - \fl{U}^{+(\ell-1)})(\fl{Q}^{(\ell-1)})^{-1}\fl{w}^{\star(i)}\|_2\nonumber\\
    &= c\sqrt{\frac{\log(r/\delta_0)}{m}}\|\fl{U}^{+(\ell-1)}\|_{\s{F}}\|(\fl{U}^{\star}\fl{Q}^{(\ell-1)} - \fl{U}^{+(\ell-1)})(\fl{Q}^{(\ell-1)})^{-1}\fl{w}^{\star(i)}\|_2 .\label{eq:lemma-optimize_dp-lrs-wbound-z11k-2norm-bound}
\end{align}
Further,
\begin{align}
    \|\fl{d}^{(i, \ell)}_{1,2}\|_2 &= \|(\fl{U}^{+(\ell-1)})^{\s{T}}(\fl{U}^{\star}\fl{Q}^{(\ell-1)} - \fl{U}^{+(\ell-1)})(\fl{Q}^{(\ell-1)})^{-1}\fl{w}^{\star(i)}\|_2\nonumber\\
    &= \|(\fl{U}^{+(\ell-1)})^{\s{T}}(\fl{I} - \fl{U}^{\star}(\fl{U}^{\star})^{\s{T}})\fl{U}^{+(\ell-1)}(\fl{Q}^{(\ell-1)})^{-1}\fl{w}^{\star(i)}\|_2\nonumber\\
    &= \|(\fl{U}^{+(\ell-1)})^{\s{T}}(\fl{I} - \fl{U}^{\star}(\fl{U}^{\star})^{\s{T}})^2\fl{U}^{+(\ell-1)}(\fl{Q}^{(\ell-1)})^{-1}\fl{w}^{\star(i)}\|_2\nonumber\\
    &= \Big\|\Big((\fl{U}^{+(\ell-1)})^{\s{T}} - (\fl{U}^{+(\ell-1)})^{\s{T}}\fl{U}^{\star}(\fl{U}^{\star})^{\s{T}}\Big)\cdot\nonumber\\
    &\qquad \Big(\fl{U}^{+(\ell-1)} - \fl{U}^{\star}(\fl{U}^{\star})^{\s{T}}\fl{U}^{+(\ell-1)}\Big)(\fl{Q}^{(\ell-1)})^{-1}\fl{w}^{\star(i)}\Big\|_2\nonumber\\
    &\leq \|\fl{U}^{+(\ell-1)} - \fl{U}^{\star}(\fl{U}^{\star})^{\s{T}}\fl{U}^{+(\ell-1)}\|^2_\s{F}\|(\fl{U}^{(\ell-1)} - \fl{U}^{\star}(\fl{U}^{\star})^{\s{T}}\fl{U}^{+(\ell-1)})(\fl{Q}^{(\ell-1)})^{-1}\fl{w}^{\star(i)}\|_2\nonumber\\
    &= \|\fl{U}^{+(\ell-1)} - \fl{U}^{\star}\fl{Q}^{(\ell-1)}\|_\s{F}\|(\fl{U}^{+(\ell-1)} - \fl{U}^{\star}\fl{Q}^{(\ell-1)})(\fl{Q}^{(\ell-1)})^{-1}\fl{w}^{\star(i)}\|_2\nonumber\\
    &\leq \|\fl{U}^{+(\ell-1)} - \fl{U}^{\star}\fl{Q}^{(\ell-1)}\|^2_\s{F}\|(\fl{Q}^{(\ell-1)})^{-1}\|\|\fl{w}^{\star(i)}\|_2.\label{eq:lemma-optimize_dp-lrs-wbound-z12-2norm-bound}
\end{align}
We will also use the sharper bound below later for finding the Frobenius norm of $\fl{H}^{(\ell)}$
\begin{align}\label{eq:sharp1}
    \resizebox{0.99\textwidth}{!}{$\|\fl{d}^{(i, \ell)}_{1,2}\|_2 = \|(\fl{U}^{+(\ell-1)})^{\s{T}} - (\fl{U}^{+(\ell-1)})^{\s{T}}\fl{U}^{\star}(\fl{U}^{\star})^{\s{T}}\|_{\s{F}}\|(\fl{U}^{+(\ell-1)} - \fl{U}^{\star}\fl{Q}^{(\ell-1)})(\fl{Q}^{(\ell-1)})^{-1}\fl{w}^{\star(i)}\|_2$}.
\end{align}
Using \eqref{eq:lemma-optimize_dp-lrs-wbound-z11k-2norm-bound} and \eqref{eq:lemma-optimize_dp-lrs-wbound-z12-2norm-bound} in \eqref{eq:lemma-optimize_dp-lrs-wbound-z1-unpluggeed}, we have
\begin{align}
    \|\fl{d}^{(i, \ell)}_1\|_2 &\leq c\sqrt{\frac{\log(r/\delta_0)}{m}}\|\fl{U}^{+(\ell-1)}\|_{\s{F}}\|(\fl{U}^{\star}\fl{Q}^{(\ell-1)} - \fl{U}^{+(\ell-1)})(\fl{Q}^{(\ell-1)})^{-1}\fl{w}^{\star(i)}\|_2 \\
    &\qquad + \|\fl{U}^{+(\ell-1)} - \fl{U}^{\star}\fl{Q}^{(\ell-1)}\|^2_\s{F}\|(\fl{Q}^{(\ell-1)})^{-1}\|\|\fl{w}^{\star(i)}\|_2\\
    &\leq \|\fl{U}^{+(\ell-1)} - \fl{U}^{\star}\fl{Q}^{(\ell-1)}\|_\s{F}\|(\fl{Q}^{(\ell-1)})^{-1}\|\|\fl{w}^{\star(i)}\|_2\cdot\\
    &\qquad \Big(\|\fl{U}^{+(\ell-1)} - \fl{U}^{\star}\fl{Q}^{(\ell-1)}\|_\s{F}+c\sqrt{\frac{\log(r/\delta_0)}{m}}\|\fl{U}^{+(\ell-1)}\|_{\s{F}}\Big)\label{eq:lemma-optimize_dp-lrs-wbound-z1-2norm-bound}
\end{align}

As before, using \eqref{eq:lemma-optimize_dp-lrs-wbound-z11k-2norm-bound} and \eqref{eq:sharp1}, we have
\begin{align}
   \|\fl{d}^{(i, \ell)}_1\|_2 &\leq \|(\fl{U}^{+(\ell-1)} - \fl{U}^{\star}\fl{Q}^{(\ell-1)})(\fl{Q}^{(\ell-1)})^{-1}
   \fl{w}^{\star(i)}\|_2\cdot\nonumber\\
   &\qquad \Big(\|\fl{U}^{+(\ell-1)} - \fl{U}^{\star}\fl{Q}^{(\ell-1)}\|_\s{F}+c\sqrt{\frac{\log(r/\delta_0)}{m}}\|\fl{U}^{(\ell-1)}\|_{\s{F}}\Big).\label{eq:sharp2}
\end{align}
\textbf{Analysis of $\fl{d}^{(i, \ell)}_2$:}\\
Note that
\begin{align}
    \fl{d}^{(i, \ell)}_2 &= \frac{1}{m}(\fl{U}^{+(\ell-1)})^{\s{T}}(\fl{X}^{(i)})^{\s{T}}\fl{X}^{(i)}(\fl{b}^{\star(i)} - \fl{b}^{(i, \ell)})\nonumber\\
    &= \underbrace{(\fl{U}^{+(\ell-1)})^{\s{T}}\Big(\frac{1}{m}(\fl{X}^{(i)})^{\s{T}}\fl{X}^{(i)}-\fl{I}\Big)(\fl{b}^{\star(i)} - \fl{b}^{(i, \ell)})}_{\fl{d}^{(i, \ell)}_{2,1}} + \underbrace{(\fl{U}^{+(\ell-1)})^{\s{T}}(\fl{b}^{\star(i)} - \fl{b}^{(i, \ell)})}_{\fl{d}^{(i, \ell)}_{2,2}}\nonumber\\
    \text{and } \E{\fl{d}^{(i, \ell)}_{2,1}} &= \E{(\fl{U}^{+(\ell-1)})^{\s{T}}\Big(\frac{1}{m}(\fl{X}^{(i)})^{\s{T}}\fl{X}^{(i)}-\fl{I}\Big)(\fl{b}^{\star(i)} - \fl{b}^{(i, \ell)})} = \vzero.\nonumber
\end{align}
Further,
\begin{align}
    (z^{(i, \ell)}_{2,1})_k &= \frac{1}{m}\sum_j(\fl{u}^{(k, \ell-1)})^{\s{T}}\fl{x}^{(i)}_j(\fl{x}^{(i)}_j)^{\s{T}}(\fl{b}^{\star(i)} - \fl{b}^{(i, \ell)}) - (\fl{u}^{(k, \ell-1)})^{\s{T}}(\fl{b}^{\star(i)} - \fl{b}^{(i, \ell)}).\nonumber
\end{align}
Using Lemma \ref{lemma:useful1} in the above with $\fl{a} = \fl{u}^{(k, \ell-1)}$ and $\fl{b} = (\fl{b}^{\star(i)} - \fl{b}^{(i, \ell)})$, we get
\begin{align}
    (z^{(i, \ell)}_{2,1})_k &\leq c\sqrt{\frac{\log(1/\delta_0)}{m}}\|\fl{u}^{(k, \ell-1)}\|_2\|\fl{b}^{\star(i)} - \fl{b}^{(i, \ell)}\|_2.\label{eq:lemma-optimize_dp-lrs-wbound-z21k-bound}
\end{align}
Taking the Union Bound overall entries $k \in [r]$, using the above we have
\begin{align}
    \|\fl{d}^{(i, \ell)}_{2,1}\|_2 &= \sqrt{\sum_{k \in [r]}|(z^{(i, \ell)}_{2,1})_k|^2}\nonumber\\
    &\leq c\sqrt{\frac{\log(r/\delta_0)}{m}}\sqrt{\sum_{k \in [r]}\|\fl{u}^{(k, \ell-1)}\|_2^2}\cdot\|\fl{b}^{\star(i)} - \fl{b}^{(i, \ell)}\|_2\nonumber\\
    &= c\sqrt{\frac{\log(r/\delta_0)}{m}}\|\fl{U}^{+(\ell-1)}\|_{\s{F}}\|\fl{b}^{\star(i)} - \fl{b}^{(i, \ell)}\|_2 .\label{eq:lemma-optimize_dp-lrs-wbound-z21k-2norm-bound}
\end{align}
Further,
\begin{align}
    \|\fl{d}^{(i, \ell)}_{2,2}\|_2 &= \|(\fl{U}^{+(\ell-1)})^{\s{T}}(\fl{b}^{\star(i)} - \fl{b}^{(i, \ell)})\|_2\nonumber\\
    &= \|\Big((\fl{U}^{+(\ell-1)})^{\s{T}}(\fl{b}^{\star(i)} - \fl{b}^{(i, \ell)})\Big)_{\s{supp}(\fl{b}^{\star(i)})}\|_2\nonumber\\
    &\leq \|\fl{U}^{+(\ell-1)}_{\s{supp}(\fl{b}^{\star(i)})}\|_2\|(\fl{b}^{\star(i)} - \fl{b}^{(i, \ell)})_{\s{supp}(\fl{b}^{\star(i)})}\|_2\nonumber\\
    &\leq \sqrt{k}\|\fl{U}^{+(\ell-1)}\|_{2, \infty}\|\fl{b}^{\star(i)} - \fl{b}^{(i, \ell)}\|_2.\label{eq:lemma-optimize_dp-lrs-wbound-z22-2norm-bound}
\end{align}
Using \eqref{eq:lemma-optimize_dp-lrs-wbound-z21k-2norm-bound} and \eqref{eq:lemma-optimize_dp-lrs-wbound-z22-2norm-bound} we have
\begin{align}
    \|\fl{d}^{(i, \ell)}_2\|_2 \leq \|\fl{b}^{\star(i)} - \fl{b}^{(i, \ell)}\|_2\Big(\sqrt{k}\|\fl{U}^{+(\ell-1)}\|_{2, \infty}+c\sqrt{\frac{\log(r/\delta_0)}{m}}\|\fl{U}^{+(\ell-1)}\|_{\s{F}} \Big).\label{eq:lemma-optimize_dp-lrs-wbound-z2-2norm-bound}
\end{align}
Using \eqref{eq:lemma-optimize_dp-lrs-wbound-A-bound}, \eqref{eq:lemma-optimize_dp-lrs-wbound-z1-2norm-bound} and \eqref{eq:lemma-optimize_dp-lrs-wbound-z2-2norm-bound} we have
\begin{align}
    &\|\fl{h}^{(i, \ell)}\|_2 \leq \frac{1}{1 - c\sqrt{\frac{r\log(1/ \delta_0)}{m}}}\cdot\Big\{\|\fl{U}^{+(\ell-1)} - \fl{U}^{\star}\fl{Q}^{(\ell-1)}\|_\s{F}\|(\fl{Q}^{(\ell-1)})^{-1}\|\|\fl{w}^{\star(i)}\|_2\cdot\nonumber\\
    &\qquad \Big(\|\fl{U}^{+(\ell-1)} - \fl{U}^{\star}\fl{Q}^{(\ell-1)}\|_\s{F}+c\sqrt{\frac{\log(r/\delta_0)}{m}}\|\fl{U}^{(\ell-1)}\|_{\s{F}}\Big)\nonumber\\
    &\qquad +\|\fl{b}^{\star(i)} - \fl{b}^{(i, \ell)}\|_2\Big(\sqrt{k}\|\fl{U}^{+(\ell-1)}\|_{2, \infty}+c\sqrt{\frac{\log(r/\delta_0)}{m}}\|\fl{U}^{+(\ell-1)}\|_{\s{F}} \Big)+\sigma\sqrt{\frac{r\log^2 (r\delta^{-1})}{m}}\Big\}.\label{eq:lemma-optimize_dp-lrs-wbound-type1-hl-2norm-type1-bound}
\end{align}
Using \eqref{eq:lemma-optimize_dp-lrs-wbound-A-bound}, \eqref{eq:sharp2} and \eqref{eq:lemma-optimize_dp-lrs-wbound-z2-2norm-bound}, we also have the sharper bound
\begin{align}
    &\|\fl{h}^{(i, \ell)}\|_2 \leq \frac{1}{1 - c\sqrt{\frac{r\log(1/ \delta_0)}{m}}}\Big\{\|(\fl{U}^{+(\ell-1)} - \fl{U}^{\star}\fl{Q}^{(\ell-1)})(\fl{Q}^{(\ell-1)})^{-1}
   \fl{w}^{\star(i)}\|_2\cdot\nonumber\\
   &\qquad \Big(\|\fl{U}^{+(\ell-1)} - \fl{U}^{\star}\fl{Q}^{(\ell-1)}\|_\s{F}+c\sqrt{\frac{\log(r/\delta_0)}{m}}\|\fl{U}^{(\ell-1)}\|_{\s{F}}\Big)\nonumber\\
    &\qquad +\|\fl{b}^{\star(i)} - \fl{b}^{(i, \ell)}\|_2\Big(\sqrt{k}\|\fl{U}^{+(\ell-1)}\|_{2, \infty}+c\sqrt{\frac{\log(r/\delta_0)}{m}}\|\fl{U}^{+(\ell-1)}\|_{\s{F}} \Big)+\sigma\sqrt{\frac{r\log^2 (r\delta^{-1})}{m}}\Big\}.\label{eq:lemma-optimize_dp-lrs-wbound-hl-2norm-type2-bound}
\end{align}
Further note that $\sum_{i \in [t]} \|(\fl{U}^{+(\ell-1)} - \fl{U}^{\star}\fl{Q}^{(\ell-1)})(\fl{Q}^{(\ell-1)})^{-1}
   \fl{w}^{\star(i)}\|_2^2$:
\begin{align}
    &= \sum_{i \in [t]} \s{Tr}\Big(\Big((\fl{U}^{+(\ell-1)} - \fl{U}^{\star}\fl{Q}^{(\ell-1)})(\fl{Q}^{(\ell-1)})^{-1}
    \fl{w}^{\star(i)}\Big)^\s{T}\cdot\nonumber\\
    &\qquad\qquad \Big((\fl{U}^{+(\ell-1)} - \fl{U}^{\star}\fl{Q}^{(\ell-1)})(\fl{Q}^{(\ell-1)})^{-1}
    \fl{w}^{\star(i)}\Big)\Big)\nonumber\\
    &= \sum_{i \in [t]} \s{Tr}\Big((\fl{w}^{\star(i)})^\s{T}\Big((\fl{U}^{+(\ell-1)} - \fl{U}^{\star}\fl{Q}^{(\ell-1)})(\fl{Q}^{(\ell-1)})^{-1}
    \Big)^\s{T}\cdot\nonumber\\
    &\qquad\qquad \Big((\fl{U}^{+(\ell-1)} - \fl{U}^{\star}\fl{Q}^{(\ell-1)})(\fl{Q}^{(\ell-1)})^{-1}
    \Big)\fl{w}^{\star(i)}\Big)\nonumber\\
    &=  \s{Tr}\Big(\Big(\fl{U}^{+(\ell-1)} - \fl{U}^{\star}\fl{Q}^{(\ell-1)})(\fl{Q}^{(\ell-1)})^{-1}
    \Big)^\s{T}\cdot\nonumber\\
    &\qquad\qquad \Big((\fl{U}^{+(\ell-1)} - \fl{U}^{\star}\fl{Q}^{(\ell-1)})(\fl{Q}^{(\ell-1)})^{-1}
    \Big)\sum_{i \in [t]}\fl{w}^{\star(i)}(\fl{w}^{\star(i)})^\s{T}\Big)\nonumber\\
    &\leq  \s{Tr}\Big(\Big(\fl{U}^{+(\ell-1)} - \fl{U}^{\star}\fl{Q}^{(\ell-1)})(\fl{Q}^{(\ell-1)})^{-1}
    \Big)^\s{T}\cdot\nonumber\\
    &\qquad\qquad \Big((\fl{U}^{+(\ell-1)} - \fl{U}^{\star}\fl{Q}^{(\ell-1)})(\fl{Q}^{(\ell-1)})^{-1}
    \Big)\Big)\lambda_{\max}\Big(\sum_{i \in [t]}\fl{w}^{\star(i)}(\fl{w}^{\star(i)})^\s{T}\Big)\nonumber\\
    &=  \|\fl{U}^{+(\ell-1)} - \fl{U}^{\star}\fl{Q}^{(\ell-1)}\|_\s{F}^2\|(\fl{Q}^{(\ell-1)})^{-1}\|^2\cdot \frac{t}{r}\lambda^\star_1,\label{eq:lemma-optimize_dp-lrs-wbound-hl-useful1}
\end{align}
and $\sum_{i \in [t]} \|(\fl{U}^{+(\ell-1)})^{\s{T}}(\fl{b}^{\star(i)} - \fl{b}^{(i, \ell)})\|_2^2$
\begin{align}
    &= \sum_{i \in [t]} \sum_{j \in [r]}\Big((\fl{U}^{+(\ell-1, j)})^{\s{T}}(\fl{b}^{\star(i)} - \fl{b}^{(i, \ell)})\Big)^2\nonumber\\
    &= \sum_{i \in [t]} \sum_{j \in [r]}\Big( \sum_{p \in \s{supp}(\fl{b}^{\star(i)})}\fl{U}^{+(\ell-1, j)}_p(\fl{b}^{\star(i)}_p - \fl{b}^{(i, \ell)}_p)\Big)^2\nonumber\\
    &\leq k\sum_{i \in [t]} \sum_{j \in [r]}\sum_{p \in \s{supp}(\fl{b}^{\star(i)})}(\fl{U}^{+(\ell-1, j)}_p)^2(\fl{b}^{\star(i)}_p - \fl{b}^{(i, \ell)}_p)^2\nonumber\\
    &\leq k\zeta\sum_{j \in [r]}\sum_{p \in \s{supp}(\fl{b}^{\star(i)})}(\fl{U}^{+(\ell-1, j)}_p)^2\|\fl{b}^{\star(i)} - \fl{b}^{(i, \ell)}\|_\infty^2\nonumber\\
    &= k\zeta \|\fl{b}^{\star(i)} - \fl{b}^{(i, \ell)}\|_\infty^2 \sum_{j \in [r]}\sum_{p \in \s{supp}(\fl{b}^{\star(i)})}(\fl{U}^{+(\ell-1, j)}_p)^2\nonumber\\
    &\leq k\zeta \|\fl{U}^\star\|_\s{F}^2\|\fl{b}^{\star(i)} - \fl{b}^{(i, \ell)}\|_\infty^2.\nonumber
\end{align}
The above two equations with \eqref{eq:lemma-optimize_dp-lrs-wbound-hl-2norm-type2-bound} imply
\begin{align}
    &\|\fl{H}^{(\ell)}\|_\s{F} \leq \frac{1}{1 - c\sqrt{\frac{r\log(1/ \delta_0)}{m}}}\Big\{\|\fl{U}^{+(\ell-1)} - \fl{U}^{\star}\fl{Q}^{(\ell-1)}\|_\s{F}\|(\fl{Q}^{(\ell-1)})^{-1}\|\sqrt{\frac{t}{r}\lambda^\star_1}\cdot\nonumber\\
    &\qquad \Big(\|\fl{U}^{+(\ell-1)} - \fl{U}^{\star}\fl{Q}^{(\ell-1)}\|_\s{F}+c\sqrt{\frac{\log(r/\delta_0)}{m}}\|\fl{U}^{(\ell-1)}\|_{\s{F}}\Big)\nonumber\\
    &\qquad + \sqrt{k\zeta}\|\fl{U}^\star\|_\s{F}\|\fl{b}^{\star(i)} - \fl{b}^{(i, \ell)}\|_\infty + \sqrt{\sum_{i \in [t]} \Big(c\|\fl{b}^{\star(i)} - \fl{b}^{(i, \ell)}\|_2\sqrt{\frac{\log(r/\delta_0)}{m}}\|\fl{U}^{+(\ell-1)}\|_{\s{F}}\Big)^2}\nonumber\\
    &\qquad +\sigma\sqrt{\frac{rt\log^2 (r\delta^{-1})}{m}}\Big\}.\label{eq:lemma-optimize_dp-lrs-wbound-Hl-Fnorm-bound}
\end{align}
\end{proof}

\begin{coro}\label{inductive-corollary:optimize_dp-lrs-h,H-bounds}
    If $\s{B}_{\fl{U}^{(\ell-1)}} =  \cO\Big(\frac{1}{1\sqrt{r\mu^\star}}\Big)$, $\sqrt{\frac{r\log(1/ \delta_0)}{m}} = \cO(1)$, $\sqrt{\nu^{(\ell-1)}} = \cO\Big(\frac{1}{\sqrt{r\mu^\star}}\Big)$, $\sqrt{\frac{r^2\log(r/\delta_0)}{m}} = \cO\Big(\frac{1}{\sqrt{\mu^\star}}\Big)$, $\epsilon < \sqrt{\mu^\star\lambda_r^\star}$, $\sqrt{\frac{r^2\zeta}{t}} = \cO\Big(\frac{1}{\sqrt{\mu^\star}}\Big)$, $\Lambda' = \cO\Big(\sqrt{\frac{\lambda_r^\star}{\lambda_1^\star}}\Big)$ and Assumption \ref{assum:init} holds for iteration $\ell-1$, then w.p. $1-\cO(\delta_0)$
    \begin{align}
        & \resizebox{\textwidth}{!}{$\|\fl{h}^{(i, \ell)}\|_2 = \cO\Big(\frac{\max\{\epsilon, \|\fl{w}^{\star(i)}\|_2\}\s{B}_{\fl{U^{(\ell-1)}}}\sqrt{\frac{\lambda_r^\star}{\lambda_1^\star}}}{\sqrt{r\mu^\star}}\Big) + \cO\Big(\frac{\Lambda'\|\fl{w}^{\star(i)}\|_2}{\sqrt{r\mu^\star}}\Big) + \cO\Big(\frac{\Lambda}{\sqrt{r\mu^\star}}\Big) + \cO\Big(\sigma\sqrt{\frac{r\log^2 (r\delta^{-1})}{m}}\Big)$,}\nonumber\\
        & \resizebox{\textwidth}{!}{$\|\fl{H}^{(\ell)}\|_\s{F} \leq \cO\Big(\frac{\s{B}_{\fl{U^{(\ell-1)}}}\frac{\lambda_r^\star}{\lambda_1^\star}\sqrt{\frac{t}{r}\lambda^\star_1}}{\sqrt{r\mu^\star}}\Big) + \cO\Big(\Lambda'\sqrt{\frac{t}{r}\lambda^\star_r}\frac{1}{\sqrt{r\mu^\star}}\Big) + \cO\Big(\frac{1}{\sqrt{r\mu^\star}}\sqrt{\frac{t}{r}}\Lambda\Big) + \cO\Big(\sigma\sqrt{\frac{rt\log^2 (r\delta^{-1})}{m}}\Big)$.}\nonumber
    \end{align}
    \end{coro}
     \begin{proof} The proof follows from plugging the various constant bounds of the lemma statement and Inductive Assumption \ref{assum:init} in the expressions of Lemma~\ref{lemma:optimize_dp-lrs-wbound}:
    \begin{align}
        &\|\fl{h}^{(i, \ell)}\|_2  \leq \frac{1}{1 - c\sqrt{\frac{r\log(1/ \delta_0)}{m}}}\Big\{\|\fl{U}^{+(\ell-1)} - \fl{U}^{\star}\fl{Q}^{(\ell-1)}\|_\s{F}\|(\fl{Q}^{(\ell-1)})^{-1}\|\|\fl{w}^{\star(i)}\|_2\cdot\nonumber\\
        &\qquad \Big(\|\fl{U}^{+(\ell-1)} - \fl{U}^{\star}\fl{Q}^{(\ell-1)}\|_\s{F}+c\sqrt{\frac{\log(r/\delta_0)}{m}}\|\fl{U}^{+(\ell-1)}\|_{\s{F}}\Big)\nonumber\\
        &\qquad +\|\fl{b}^{\star(i)} - \fl{b}^{(i, \ell)}\|_2\Big(\sqrt{k}\|\fl{U}^{+(\ell-1)}\|_{2, \infty}+c\sqrt{\frac{\log(r/\delta_0)}{m}}\|\fl{U}^{+(\ell-1)}\|_{\s{F}} \Big)+ \sigma\sqrt{\frac{r\log^2 (r\delta^{-1})}{m}}\Big\}.
    \end{align}
    \begin{align}
        &\leq \frac{1}{1 - c\sqrt{\frac{r\log(1/ \delta_0)}{m}}}\cdot\nonumber\\
        & \Big\{\Big(\s{B}_{\fl{U^{(\ell-1)}}}\sqrt{\frac{\lambda_r^\star}{\lambda_1^\star}}+\Lambda'\Big)\cdot 2\|\fl{w}^{\star(i)}\|_2\Big(\Big(\s{B}_{\fl{U^{(\ell-1)}}}\sqrt{\frac{\lambda_r^\star}{\lambda_1^\star}} + \Lambda'\Big)+c\sqrt{\frac{\log(r/\delta_0)}{m}}\cdot \sqrt{r}\Big) \nonumber\\
        &\resizebox{0.9\textwidth}{!}{$+ \Big(c'\max\{\epsilon, \|\fl{w}^{\star(i)}\|_2\}\s{B}_{\fl{U^{(\ell-1)}}}\sqrt{\frac{\lambda_r^\star}{\lambda_1^\star}} + \Lambda\Big)\Big(\sqrt{\nu^{(\ell-1)}}+c\sqrt{\frac{\log(r/\delta_0)}{m}} \sqrt{r}\Big)+ \sigma\sqrt{\frac{r\log^2 (r\delta^{-1})}{m}} \Big\}$.}
    \end{align}
    Using $\Lambda' = \cO\Big(\frac{1}{\sqrt{r\mu^\star}}\sqrt{\frac{\lambda_r^\star}{\lambda_1^\star}}\Big)$, the above becomes
    \begin{align}
        &\leq \frac{\max\{\epsilon, \|\fl{w}^{\star(i)}\|_2\}\s{B}_{\fl{U^{(\ell-1)}}}\sqrt{\frac{\lambda_r^\star}{\lambda_1^\star}}}{1 - c\sqrt{\frac{r\log(1/ \delta_0)}{m}}}\cdot\nonumber\\
        &\qquad \Big\{2\Big(\s{B}_{\fl{U^{(\ell-1)}}}+\cO\Big(\frac{1}{\sqrt{r\mu^\star}}\sqrt{\frac{\lambda_r^\star}{\lambda_1^\star}}\Big)+c\sqrt{\frac{r\log(r/\delta_0)}{m}}\Big)+c'\Big(\sqrt{\nu^{(\ell-1)}}+c\sqrt{\frac{r\log(r/\delta_0)}{m}} \Big)\Big\} \nonumber\\
        & + \frac{1}{1 - c\sqrt{\frac{r\log(1/ \delta_0)}{m}}}\Big\{2\Lambda'\|\fl{w}^{\star(i)}\|_2\Big(\s{B}_{\fl{U^{(\ell-1)}}}+ \cO\Big(\frac{1}{\sqrt{r\mu^\star}}\sqrt{\frac{\lambda_r^\star}{\lambda_1^\star}}\Big)+c\sqrt{\frac{r\log(r/\delta_0)}{m}}\Big)\nonumber\\
        &\qquad + \Lambda\Big(\sqrt{\nu^{(\ell-1)}}+c\sqrt{\frac{r\log(r/\delta_0)}{m}} \Big)+ \sigma\sqrt{\frac{r\log^2 (r\delta^{-1})}{m}}\Big\}.
    \end{align}
    Further, using $c\sqrt{\frac{r\log(1/ \delta_0)}{m}} = \cO(1)$, $\s{B}_{\fl{U}^{(\ell-1)}} =  \cO\Big(\frac{1}{\sqrt{r\mu^\star}}\Big)$, $\sqrt{\nu^{(\ell-1)}} = \cO\Big(\frac{1}{\sqrt{r\mu^\star}}\Big)$, $\sqrt{\frac{r^2\log(r/\delta_0)}{m}} = \cO\Big(\frac{1}{\sqrt{\mu^\star}}\Big)$, $\lambda_r^\star \leq \lambda_1^\star$, $r \geq 1$ in the above, we get
    \begin{align}
        & \resizebox{0.9\textwidth}{!}{= $\cO\Big(\frac{\max\{\epsilon, \|\fl{w}^{\star(i)}\|_2\}\s{B}_{\fl{U^{(\ell-1)}}}\sqrt{\frac{\lambda_r^\star}{\lambda_1^\star}}}{\sqrt{r\mu^\star}}\Big) + \cO\Big(\frac{\Lambda'\|\fl{w}^{\star(i)}\|_2}{\sqrt{r\mu^\star}}\Big) + \cO\Big(\frac{\Lambda}{\sqrt{r\mu^\star}}\Big) + \cO\Big(\sigma\sqrt{\frac{r\log^2 (r\delta^{-1})}{m}}\Big)$.}
    \end{align}
    Similarly using the Inductive Assumption expressions from ~\ref{assum:init}, we also have
    \begin{align}
        &\|\fl{H}^{(\ell)}\|_\s{F} \leq \frac{1}{1 - c\sqrt{\frac{r\log(1/ \delta_0)}{m}}}\Big\{\|\fl{U}^{+(\ell-1)} - \fl{U}^{\star}\fl{Q}^{(\ell-1)}\|_\s{F}\|(\fl{Q}^{(\ell-1)})^{-1}\|\sqrt{\frac{t}{r}\lambda^\star_1}\cdot\nonumber\\
        &\qquad \Big(\|\fl{U}^{+(\ell-1)} - \fl{U}^{\star}\fl{Q}^{(\ell-1)}\|_\s{F}+c\sqrt{\frac{\log(r/\delta_0)}{m}}\|\fl{U}^{(\ell-1)}\|_{\s{F}}\Big)\nonumber\\
        &\qquad + \sqrt{k\zeta}\|\fl{U}^\star\|_\s{F}\|\fl{b}^{\star(i)} - \fl{b}^{(i, \ell)}\|_\infty + \sqrt{\sum_{i \in [t]} \Big(c\|\fl{b}^{\star(i)} - \fl{b}^{(i, \ell)}\|_2\sqrt{\frac{\log(r/\delta_0)}{m}}\|\fl{U}^{+(\ell-1)}\|_{\s{F}}\Big)^2}\nonumber\\
        &\qquad + \sigma\sqrt{\frac{rt\log^2 (r\delta^{-1})}{m}}\Big\}\\
        &\leq \frac{1}{1 - c\sqrt{\frac{r\log(1/ \delta_0)}{m}}}\Big\{\Big(\s{B}_{\fl{U^{(\ell-1)}}}\sqrt{\frac{\lambda_r^\star}{\lambda_1^\star}}+\Lambda'\Big)\cdot 2\sqrt{\frac{t}{r}\lambda^\star_1}\Big(\s{B}_{\fl{U^{(\ell-1)}}}\sqrt{\frac{\lambda_r^\star}{\lambda_1^\star}} + \Lambda'+c\sqrt{\frac{\log(r/\delta_0)}{m}}\sqrt{r}\Big)\nonumber\\
        &\qquad + \sqrt{k\zeta}\sqrt{r}\Big(c'\max\{\epsilon, \|\fl{w}^{\star(i)}\|_2\}\s{B}_{\fl{U^{(\ell-1)}}}\sqrt{\frac{\lambda_r^\star}{\lambda_1^\star k}} + \frac{\Lambda}{\sqrt{k}}\Big)\nonumber\\
        &\qquad +c\sqrt{\frac{\log(r/\delta_0)}{m}} \sqrt{r}\cdot \Big(c'\max\{\epsilon, \|\fl{w}^{\star(i)}\|_2\}\s{B}_{\fl{U^{(\ell-1)}}}\sqrt{\frac{\lambda_r^\star}{\lambda_1^\star}} + \Lambda\Big)\sqrt{t}+ \sigma\sqrt{\frac{rt\log^2 (r\delta^{-1})}{m}}\Big\}.
        \end{align}
        As before, using the bound on $\Lambda'$, the above becomes
        \begin{align}
        &\leq \frac{\s{B}_{\fl{U^{(\ell-1)}}}\sqrt{\frac{\lambda_r^\star}{\lambda_1^\star}}}{1 - c\sqrt{\frac{r\log(1/ \delta_0)}{m}}}\Big\{ 2\sqrt{\frac{t}{r}\lambda^\star_1}\Big(\s{B}_{\fl{U^{(\ell-1)}}}\sqrt{\frac{\lambda_r^\star}{\lambda_1^\star}}+\cO\Big(\frac{1}{\sqrt{r\mu^\star}}\sqrt{\frac{\lambda_r^\star}{\lambda_1^\star}}\Big)+c\sqrt{\frac{r\log(r/\delta_0)}{m}}\Big)\nonumber\\
        &\qquad + \sqrt{k\zeta}\sqrt{r}c'\max\{\epsilon, \|\fl{w}^{\star(i)}\|_2\}\frac{1}{\sqrt{k}} + c\sqrt{\frac{r\log(r/\delta_0)}{m}}\cdot c'\sqrt{t}\max\{\epsilon, \sqrt{\mu^\star\lambda^\star_r}\}\Big\}\nonumber\\
        &\qquad \frac{1}{1 - c\sqrt{\frac{r\log(1/ \delta_0)}{m}}}\cdot\Big\{2\Lambda'\sqrt{\frac{t}{r}\lambda^\star_1}\Big(\s{B}_{\fl{U^{(\ell-1)}}}\sqrt{\frac{\lambda_r^\star}{\lambda_1^\star}} + \cO\Big(\frac{1}{\sqrt{r\mu^\star}}\sqrt{\frac{\lambda_r^\star}{\lambda_1^\star}}\Big)+c\sqrt{\frac{\log(r/\delta_0)}{m}}\sqrt{r}\Big)\nonumber \\
        &\qquad + \sqrt{r\zeta}\Lambda + c\sqrt{\frac{\log(r/\delta_0)}{m}} \sqrt{r}\cdot \Lambda\sqrt{t}+ \sigma\sqrt{\frac{rt\log^2 (r\delta^{-1})}{m}}\Big\}
    \end{align}
        Further, using $c\sqrt{\frac{r\log(1/ \delta_0)}{m}} = \cO(1)$, $\s{B}_{\fl{U}^{(\ell-1)}} =  \cO\Big(\frac{1}{\sqrt{r\mu^\star}}\Big)$, $\sqrt{\nu^{(\ell-1)}} = \cO\Big(\frac{1}{\sqrt{r\mu^\star}}\Big)$, $\sqrt{\frac{r^2\log(r/\delta_0)}{m}} = \cO\Big(\frac{1}{\sqrt{\mu^\star}}\Big)$, $\lambda_r^\star \leq \lambda_1^\star$, $r \geq 1$ and $\sqrt{\frac{r^2\zeta}{t}} = \cO\Big(\frac{1}{\sqrt{\mu^\star}}\Big)$ in the above, we get
    \begin{align}
        &\leq \frac{\s{B}_{\fl{U^{(\ell-1)}}}\frac{\lambda_r^\star}{\lambda_1^\star}\sqrt{\frac{t}{r}\lambda^\star_1}}{1 - c\sqrt{\frac{r\log(1/ \delta_0)}{m}}}\Big\{ 2\Big(\s{B}_{\fl{U^{(\ell-1)}}}+\frac{1}{181\sqrt{r\mu^\star}}+c\sqrt{\frac{r\log(r/\delta_0)}{m}}\sqrt{\frac{\lambda_1^\star}{\lambda_r^\star}}\Big) \nonumber\\
        &\qquad + \frac{1}{\sqrt{\lambda_r^\star}}\max\{\epsilon, \sqrt{
        \mu^\star\lambda_r^\star}\}c'\Big(\sqrt{\frac{r^2\zeta}{t}} + c\sqrt{\frac{r^2\log(r/\delta_0)}{m}}\cdot \Big)\Big\}\nonumber\\
        &\qquad +\cO(1)\cdot\Big\{2\Lambda'\sqrt{\frac{t}{r}\lambda^\star_r}\cdot\cO\Big(\frac{1}{\sqrt{r\mu^\star}}\Big) + \sqrt{r\zeta}\Lambda + \cO\Big(\frac{1}{\sqrt{r\mu^\star}}\Big)\cdot\Lambda\sqrt{t}+ \sigma\sqrt{\frac{rt\log^2 (r\delta^{-1})}{m}}\Big\}\nonumber\\
        &= \cO\Big(\frac{\s{B}_{\fl{U^{(\ell-1)}}}\frac{\lambda_r^\star}{\lambda_1^\star}\sqrt{\frac{t}{r}\lambda^\star_1}}{\sqrt{r\mu^\star}}\Big) + \cO\Big(\Lambda'\sqrt{\frac{t}{r}\lambda^\star_r}\frac{1}{\sqrt{r\mu^\star}}\Big) + \cO\Big(\frac{1}{\sqrt{r\mu^\star}}\sqrt{\frac{t}{r}}\Lambda\Big) + \cO\Big(\sigma\sqrt{\frac{rt\log^2 (r\delta^{-1})}{m}}\Big).
    \end{align}
\end{proof}

\begin{coro}\label{inductive-inductive-corollary:optimize_dp-lrs-w-incoherence}
    If Assumption~\ref{assum:init} and Corollary~\ref{inductive-corollary:optimize_dp-lrs-h,H-bounds} hold, and $\Lambda' = \cO\Big(\sqrt{\frac{\lambda_r^\star}{\lambda_1^\star}}\Big)$, $\Lambda = \cO(\sqrt{\lambda_r^\star})$ and $\sigma\sqrt{\frac{r^2\log^2 (r\delta^{-1})}{m}} \\= \cO(\sqrt{\lambda_r^\star})$ then, $\fl{W}^{(\ell)}$ is incoherent w.p. probability $1-\cO(\delta_0)$ for $\exists c''>0$, s.t. 
    \begin{align}
        \|\fl{w}^{(i, \ell)}\|_2 &\leq (2+c'')\sqrt{\mu^\star\lambda_r\Big(\frac{r}{t}(\fl{W^\star})^\s{T}\fl{W^\star}\Big)},\nonumber\\
        (1-c'')\sqrt{\lambda^\star_r\Big(\frac{r}{t}(\fl{W}^{\star})^\s{T}\fl{W}^{\star}\Big)} &\leq \sqrt{\lambda_r\Big(\frac{r}{t}(\fl{W}^{(\ell)})^\s{T}\fl{W}^{(\ell)}\Big)}\leq (1+c'')\sqrt{\lambda^\star_r\Big(\frac{r}{t}(\fl{W}^{\star})^\s{T}\fl{W}^{\star}\Big)},\nonumber\\
        \|\fl{W}^{(\ell)}\| &\leq (1+c'')\sqrt{\frac{t}{r}}\sqrt{\lambda_1\Big(\frac{r}{t}(\fl{W}^{\star})^\s{T}\fl{W}^{\star}\Big)},\nonumber\\
        \sqrt{\mu^{(\ell)}} &:= \frac{2+c''}{1-c''}\sqrt{\mu^\star} \leq (1+2c'')\sqrt{\mu^\star}.\nonumber
    \end{align}
\end{coro}
\begin{proof}
    Using Triangle Inequality, Assumption~\ref{assum:init} and Corollary~\ref{inductive-corollary:optimize_dp-lrs-h,H-bounds}, we have
    \begin{align}
        \|\fl{w}^{(i, \ell)}\|_2 &\leq \|(\fl{Q}^{(\ell)})^{-1}\fl{w}^{\star(i)}\|_2 + \|\fl{h}^{(i, \ell)}\|_2\nonumber\\
        &\leq 2\|\fl{w}^{(\star)}\|_2 + \cO\Big(\frac{\max\{\epsilon, \|\fl{w}^{\star(i)}\|_2\}\s{B}_{\fl{U^{(\ell-1)}}}\sqrt{\frac{\lambda_r^\star}{\lambda_1^\star}}}{\sqrt{r\mu^\star}}\Big) + \cO\Big(\frac{\Lambda'\|\fl{w}^{\star(i)}\|_2}{\sqrt{r\mu^\star}}\Big) + \cO\Big(\frac{\Lambda}{\sqrt{r\mu^\star}}\Big) \nonumber\\
        &\qquad + \cO\Big(\sigma\sqrt{\frac{r\log^2 (r\delta^{-1})}{m}}\Big)\nonumber\\
        &= \Big(2+ \cO\Big(\frac{\s{B}_{\fl{U^{(\ell-1)}}}}{\sqrt{r\mu^\star}}\sqrt{\frac{\lambda_r^\star}{\lambda_1^\star}}\Big) +  \Big(\frac{\Lambda'}{\sqrt{r\mu^\star}}\Big) + \Big(\frac{1}{\sqrt{r\mu^\star}}\Big) + \cO(1)\Big)\sqrt{\mu^\star\lambda_r^\star}\nonumber\\
        &\leq  (2+c'')\sqrt{\mu^\star\lambda_r^\star},
    \end{align}
    for some $c''> 0$, where in the last two lines, we use the fact that $\epsilon < \sqrt{\mu^\star\lambda_r^\star}$, $\Lambda' = \cO\Big(\sqrt{\frac{\lambda_r^\star}{\lambda_1^\star}}$\Big), $\Lambda = \Big(\sqrt{\mu^\star\lambda_r^\star}\Big)$ and $\sigma\sqrt{\frac{r\log^2 (r\delta^{-1})}{m}} = \cO(\sqrt{\mu^\star\lambda_r^\star})$ and $r, \mu^\star > 1$.
     Using Lemma~\ref{lemma:psd-sigma-lambda} with $\fl{A} = (\fl{Q}^{(\ell-1)})^{-1}$ and $\fl{B} = \frac{r}{t}(\fl{W}^{\star})^\s{T}\fl{W}^{\star}$ we have
    \begin{align}
        \sigma^2_{\min}\Big((\fl{Q}^{(\ell-1)})^{-1}\Big)\lambda_r\Big(\frac{r}{t}(\fl{W}^{\star})^\s{T}\fl{W}^{\star}\Big)
        &\leq \frac{r}{t}\lambda_r\Big((\fl{Q}^{(\ell-1)})^{-1}(\fl{W}^{\star})^\s{T}\fl{W}^{\star}(\fl{Q}^{(\ell-1)})^{-\s{T}}\Big)\label{eq:corollary-w-incoherence-1}
    \end{align}
    Further, since $\sigma_{\min}\Big((\fl{Q}^{(\ell-1)})^{-1}\Big) = \sigma_{\min}\Big(((\fl{U}^{\star})^\s{T}\fl{U}^{+(\ell-1)})^{-1}\Big) > 1$, we have $\forall$ $j \in [r]$
    \begin{align}
        \lambda_j\Big(\frac{r}{t}(\fl{W}^{\star})^\s{T}\fl{W}^{\star}\Big) \leq \sigma^2_{\min}\Big((\fl{Q}^{(\ell-1)})^{-1}\Big)\lambda_j\Big(\frac{r}{t}(\fl{W}^{\star})^\s{T}\fl{W}^{\star}\Big)\label{eq:corollary-w-incoherence-2}
    \end{align}
    Using \eqref{eq:corollary-w-incoherence-2} in \eqref{eq:corollary-w-incoherence-1} we get
    \begin{align}
        \lambda_j\Big(\frac{r}{t}(\fl{W}^{\star})^\s{T}\fl{W}^{\star}\Big) &\leq \frac{r}{t}\lambda_j\Big((\fl{Q}^{(\ell-1)})^{-1}(\fl{W}^{\star})^\s{T}\fl{W}^{\star}(\fl{Q}^{(\ell-1)})^{-\s{T}}\Big)\\
        &\leq \frac{r}{t}\sigma_j^2\Big((\fl{Q}^{(\ell-1)})^{-1}(\fl{W}^{\star})^\s{T}\Big)\label{eq:corollary-w-incoherence-3}
    \end{align}
    Now, since $\fl{W}^{\star}(\fl{Q}^{(\ell-1)})^{-\s{T}} =  \fl{W}^{(\ell)} +\fl{W}^{\star}(\fl{Q}^{(\ell-1)})^{-\s{T}} - \fl{W}^{(\ell)}$, Using Lemma~\ref{lemma:weyl} with $\fl{A} = \fl{W}^{\star}(\fl{Q}^{(\ell-1)})^{-\s{T}}$, $\fl{B} = \fl{W}^{(\ell)}$ and $\fl{C} = \fl{W}^{\star}(\fl{Q}^{(\ell-1)})^{-\s{T}} - \fl{W}^{(\ell)}$, we have
    \begin{align}
        \Big|\sigma_j\Big(\fl{W}^{\star}(\fl{Q}^{(\ell-1)})^{-\s{T}}\Big) - \sigma_j\Big(\fl{W}^{(\ell)}\Big)\Big|&\leq  
         \|\fl{W}^{\star}(\fl{Q}^{(\ell-1)})^{-\s{T}} - \fl{W}^{(\ell)}\| \label{eq:corollary-w-incoherence-4}.
    \end{align}    
    Using \eqref{eq:corollary-w-incoherence-3} and Corollary~\ref{inductive-corollary:optimize_dp-lrs-h,H-bounds} in \eqref{eq:corollary-w-incoherence-4}, we have
    \begin{align}
        &\Big|\sqrt{\lambda_j\Big(\frac{r}{t}(\fl{W}^{\star})^\s{T}\fl{W}^{\star}\Big)} - \sqrt{\lambda_j\Big(\frac{r}{t}(\fl{W}^{(i, \ell)})^\s{T}\fl{W}^{(\ell)}\Big)}\Big|
        \leq  \sqrt{\frac{r}{t}}\|\fl{W}^{\star}(\fl{Q}^{(\ell-1)})^{-\s{T}} - \fl{W}^{(\ell)}\|\nonumber\\
        &\leq \sqrt{\frac{r}{t}}\|\fl{H}^{(\ell)}\|\nonumber\\
        &\leq \sqrt{\frac{r}{t}}\|\fl{H}^{(\ell)}\|_\s{F}\nonumber\\
        &\resizebox{\textwidth}{!}{$\leq  \sqrt{\frac{r}{t}}\Big\{\cO\Big(\frac{\s{B}_{\fl{U^{(\ell-1)}}}\frac{\lambda_r^\star}{\lambda_1^\star}\sqrt{\frac{t}{r}\lambda^\star_1}}{\sqrt{r\mu^\star}}\Big) + \cO\Big(\Lambda'\sqrt{\frac{t}{r}\lambda^\star_r}\frac{1}{\sqrt{r\mu^\star}}\Big) + \cO\Big(\frac{1}{\sqrt{r\mu^\star}}\sqrt{\frac{t}{r}}\Lambda\Big) + \cO\Big(\sigma\sqrt{\frac{rt\log^2 (r\delta^{-1})}{m}}\Big)\Big\}$}\nonumber\\
        &= \cO\Big(\frac{\s{B}_{\fl{U^{(\ell-1)}}}\frac{\lambda_r^\star}{\lambda_1^\star}\sqrt{\lambda^\star_1}}{\sqrt{r\mu^\star}}\Big) + \cO\Big(\frac{\Lambda'\sqrt{\lambda^\star_r}}{\sqrt{r\mu^\star}}\Big) + \cO\Big(\frac{\Lambda}{\sqrt{r\mu^\star}}\Big)+\cO\Big(\sigma\sqrt{\frac{r^2\log^2 (r\delta^{-1})}{m}}\Big)\nonumber\\
        &= \Big\{\cO\Big(\frac{\s{B}_{\fl{U^{(\ell-1)}}}}{\sqrt{r\mu^\star}}\Big) + \cO\Big(\frac{\Lambda'}{\sqrt{r\mu^\star}}\Big) + \Big(\frac{\Lambda}{\sqrt{r\mu^\star\lambda_r^\star}}\Big)+\Big(\sigma\sqrt{\frac{r^2\log^2 (r\delta^{-1})}{m\lambda_r^\star}}\Big\}\sqrt{\lambda_r^\star}\nonumber\\
        &= \cO(\sqrt{\lambda_r^\star}),
    \end{align}
    where in the last two steps we use $\sqrt{\frac{\lambda_r^\star}{\lambda_1^\star}} < 1$, $\Lambda' = \cO\Big(\sqrt{\frac{\lambda_r^\star}{\lambda_1^\star}}\Big)$, $\Lambda =  \cO(\sqrt{\lambda_r^\star})$ and $\sigma\sqrt{\frac{r^2\log^2 (r\delta^{-1})}{m}} = \cO(\sqrt{\lambda_r^\star})$. Note that for $j = r$, the above implies for some $c'' > 0$
    \begin{align}
        &\Big|\sqrt{\lambda_r\Big(\frac{r}{t}(\fl{W}^{(\ell)})^\s{T}\fl{W}^{(\ell)}\Big)} - \sqrt{\lambda^\star_r\Big(\frac{r}{t}(\fl{W}^{\star})^\s{T}\fl{W}^{\star}\Big)}\Big|
        \leq c''\sqrt{\lambda^\star_r\Big(\frac{r}{t}(\fl{W}^{\star})^\s{T}\fl{W}^{\star}\Big)}\nonumber\\
        &\iff (1-c'')\sqrt{\lambda^\star_r\Big(\frac{r}{t}(\fl{W}^{\star})^\s{T}\fl{W}^{\star}\Big)} \leq \sqrt{\lambda_r\Big(\frac{r}{t}(\fl{W}^{(\ell)})^\s{T}\fl{W}^{(\ell)}\Big)}
        \leq (1+c'')\sqrt{\lambda^\star_r\Big(\frac{r}{t}(\fl{W}^{\star})^\s{T}\fl{W}^{\star}\Big)}.
    \end{align}
    and for $j=1$, 
    \begin{align}
        \sqrt{\lambda_1\Big(\frac{r}{t}(\fl{W}^{(\ell)})^\s{T}\fl{W}^{(\ell)}\Big)} &\leq \sqrt{\lambda_1\Big(\frac{r}{t}(\fl{W}^{\star})^\s{T}\fl{W}^{\star}\Big)} + c''\sqrt{\lambda^\star_r\Big(\frac{r}{t}(\fl{W}^{\star})^\s{T}\fl{W}^{\star}\Big)}\nonumber\\
        \implies \|\fl{W}^{(\ell)}\| &\leq (1+c'')\sqrt{\frac{t}{r}}\sqrt{\lambda_1\Big(\frac{r}{t}(\fl{W}^{\star})^\s{T}\fl{W}^{\star}\Big)}\\
        &= \cO\Big(\sqrt{\frac{t}{r}}\sqrt{\lambda_1\Big(\frac{r}{t}(\fl{W}^{\star})^\s{T}\fl{W}^{\star}\Big)}\Big)
    \end{align}
\end{proof}

\begin{lemma}\label{lemma:Q-bound}
     If $\|\fl{U}^{(\ell)} - \fl{U}^\star(\fl{U}^\star)^\s{T}\fl{U}^{(\ell)}\| \leq \|\fl{U}^{(\ell)} - \fl{U}^\star(\fl{U}^\star)^\s{T}\fl{U}^{(\ell)}\|_\s{F} \leq \frac{3}{4}$ then $\frac{1}{2} \leq \|\fl{Q}^{(\ell)}\| \leq 1$.
\end{lemma}
\begin{proof}
\textbf{Upper Bound:}
\begin{align}
    \|\fl{Q}^{(\ell)}\| = \|(\fl{U}^{\star})^\s{T}\fl{U}^{(\ell)}\| \leq \|\fl{U}^{\star}\|\|\fl{U}^{(\ell)}\| \leq 1,\nonumber
\end{align}
since both $\fl{U}^\star, \fl{U}^{(\ell)} \in \bR^{d\times r}$ are orthonormal.

\textbf{Lower Bound:}\\
Now, let $\fl{E} := \fl{U}^{(\ell)} - \fl{U}^\star(\fl{U}^\star)^\s{T}\fl{U}^{(\ell)}$ and $\fl{Q} = (\fl{U}^\star)^\s{T}\fl{U}^{(\ell)}$. Then $(\fl{U}^{(\ell)})^\s{T}\fl{E} = \fl{I} - ((\fl{U}^{\star})^\s{T}\fl{U}^{(\ell)})^\s{T}(\fl{U}^{\star})^\s{T}\fl{U}^{(\ell)} = \fl{I} - (\fl{Q}^{(\ell)})^\s{T}\fl{Q}^{(\ell)}$. Then using Lemma~\ref{lemma:weyl} with $\fl{A} = \fl{I}, \fl{B} = (\fl{Q}^{(\ell)})^\s{T}\fl{Q}^{(\ell)}$ and $\fl{C} = (\fl{U}^{(\ell)})^\s{T}\fl{E}$, we get that
\begin{align}
    \sigma_k(\fl{I}) - \sigma_k(\fl{(\fl{Q}^{(\ell)})^\s{T}\fl{Q}^{(\ell)}}) &\leq \|\fl{(\fl{U}^{(\ell)})^\s{T}\fl{E}}\| \leq \|\fl{\fl{U}^{(\ell)}}\|\|\fl{E}\|\nonumber\\
    \implies 1 - \sigma_k^2(\fl{\fl{Q}^{(\ell)}}) &\leq \|\fl{U}^{(\ell)} - \fl{U}^\star(\fl{U}^\star)^\s{T}\fl{U}^{(\ell)}\|\nonumber\\
    \implies \sigma_k(\fl{\fl{Q}^{(\ell)}}) &\geq \sqrt{1 - \|\fl{U}^{(\ell)} - \fl{U}^\star(\fl{U}^\star)^\s{T}\fl{U}^{(\ell)}\|}.\nonumber
\end{align}
Therefore, $\|\fl{U}^{(\ell)} - \fl{U}^\star(\fl{U}^\star)^\s{T}\fl{U}^{(\ell)}\|_\s{F} \leq \frac{3}{4} \implies \sigma_k(\fl{\fl{Q}^{(\ell)}}) \geq \frac{1}{2}$ $\forall$ $k \in [r]$.
\end{proof}

\begin{lemma}\label{lemma:Fnorm-helper1}
Let $\fl{V} = \frac{1}{mt}\sum_i \sum_j \fl{x}^{(i)}_j(\fl{x}^{(i)}_j)^{\s{T}}\fl{R}^{(i)}$ where both $\fl{V}, \fl{R}^{(i)} \in \bR^{d\times r}$. Then, 
\begin{align}
        \|\fl{V}\|_\s{F}^2 &\leq 2\sum_p\sum_q \Big(\frac{1}{t}\sum_i \fl{R}^{(i, \ell)}_{p,q}\Big)^2 + 16\|\fl{R}^{(i, \ell)}\|_\s{F}^2\frac{d\log (rd/\delta_0)}{mt}.\nonumber
\end{align}
\end{lemma}
\begin{proof}
    Notice that then for any $(p, q)\in [d]\times[r]$, we have
    \begin{align}
     \fl{V}_{p,q} &= \Big(\frac{1}{mt}\sum_i \sum_j \fl{x}^{(i)}_j(\fl{x}^{(i)}_j)^{\s{T}}\fl{R}^{(i, \ell)}\Big)_{p,q}\nonumber\\
         &= \frac{1}{mt}\sum_i \sum_j \fl{x}^{(i)}_{j, p}(\fl{x}^{(i)}_j)^{\s{T}}\fl{R}^{(i, \ell, q)}\nonumber\\
         &= \frac{1}{mt}\sum_i \sum_j \fl{x}^{(i)}_{j, p}\Big(\sum_{u=1}^d \fl{x}^{(i)}_{j, u}\fl{R}^{(i, \ell)}_{u, q}\Big)\nonumber\\
        &= \frac{1}{mt}\sum_i\sum_j \Big(\Big(\fl{x}^{(i)}_{j,p}\Big)^{2} \fl{R}^{(i, \ell)}_{p,q} +\sum_{u:u\neq p}\fl{x}^{(i)}_{j,p}\fl{x}^{(i)}_{j,u}\fl{R}^{(i, \ell)}_{u,q}\Big). \label{eq:lemma-Fnorm-helper1-vpq-unplugged}
    \end{align}
    Now, note that the random variable $\Big(\fl{x}^{(i)}_{j,p}\Big)^{2} \fl{R}^{(i, \ell)}_{p,q}$ is a $\Big(4(\fl{R}^{(i, \ell)}_{p,q})^2, 4|\fl{R}^{(i, \ell)}_{p,q}|\Big)$ sub-exponential random variable. Similarly, $\fl{x}^{(i)}_{j,p}\fl{x}^{(i)}_{j,u}\fl{R}^{(i, \ell)}_{u,q}$ is a $\Big(2(\fl{R}^{(i, \ell)}_{u,q})^2, \sqrt{2} |\fl{R}^{(i, \ell)}_{u,q}| \Big)$ sub-exponential random variable. Therefore,
    \begin{align}
        &\Big(\fl{x}^{(i)}_{j,p}\Big)^{2} \fl{R}^{(i, \ell)}_{p,q} +\sum_{u:u\neq h}\fl{x}^{(i)}_{j,p}\fl{x}^{(i)}_{j,u}\fl{R}^{(i, \ell)}_{u,q}\nonumber \\
        &= \Big( 4(\fl{R}^{(i, \ell)}_{p,q})^2 + 2\sum_{u:u\neq p} (\fl{R}^{(i, \ell)}_{u,q})^2, \max\Big(4|\fl{R}^{(i, \ell)}_{p,q}|, \max_{u:u\neq p}(\sqrt{2} |\fl{R}^{(i, \ell)}_{u,q}|)\Big) \Big)\nonumber\\
        &= \Big( 4 \|\fl{R}^{(i, \ell, q)}\|^2_2, 4\|\fl{R}^{(i, \ell, q)}\|_{\infty}\Big) \text{ sub-exponential random variable.}\label{eq:lemma-Fnorm-helper1-1}
    \end{align}
    Furthermore, 
    \begin{align}
        &\E{\fl{V}_{p,q} } \nonumber\\ 
        &= \frac{1}{mt}\sum_i\sum_j \Big(\E{\Big(\fl{x}^{(i)}_{j,p}\Big)^{2} \fl{R}^{(i, \ell)}_{p,q} } + \E{\sum_{u:u\neq p}\fl{x}^{(i)}_{j,p}\fl{x}^{(i)}_{j,u}\fl{R}^{(i, \ell)}_{u,q} }\Big)\nonumber\\
        &= \frac{1}{mt}\sum_i\sum_j \Big(\fl{R}^{(i, \ell)}_{p,q} + \vzero\Big)\nonumber\\
        &= \frac{1}{t}\sum_i\fl{R}^{(i, \ell)}_{p,q}. \label{eq:lemma-Fnorm-helper1-2}
    \end{align}
    Using \eqref{eq:lemma-Fnorm-helper1-1}, \eqref{eq:lemma-Fnorm-helper1-2} and Lemma~\ref{tail:sub_exp} in \eqref{eq:lemma-Fnorm-helper1-vpq-unplugged} gives
    \begin{align}
        \left|\fl{V}_{p,q} - \frac{1}{t}\sum_i \fl{R}^{(i, \ell)}_{p,q}\right| &\leq \underbrace{\max\Big(2 \|\fl{R}^{(i, \ell, q)}\|_2\sqrt{\frac{2\log (1/\delta_0)}{mt}}, 2 \|\fl{R}^{(i, \ell, q)}\|_{\infty}\frac{2\log (1/\delta_0)}{mt}\Big)}_{\epsilon_{p,q}}.\label{eq:lemma-optimize_dp-lrs-U-Fnorm-bound-vh-plugged}
    \end{align}
    Note that $\|\fl{V}\|_\s{F}^2 = \sum_p\sum_q \fl{V}_{p,q}^2$. Hence taking a union bound over all entries $(p,q) \in [d]\times[r]$, we have
    \begin{align}
        \sum_{p}\sum_q \fl{V}_{p,q}^2 &\leq \sum_p\sum_q 2\Big(\Big(\frac{1}{t}\sum_i \fl{R}^{(i, \ell)}_{p,q}\Big)^2+\epsilon_{p,q}^2\Big) \nonumber\\
        &\leq 2\sum_p\sum_q \Big(\frac{1}{t}\sum_i \fl{R}^{(i, \ell)}_{p,q}\Big)^2+ 2\sum_p\sum_q \Big(2 \|\fl{R}^{(i, \ell, q)}\|_2\sqrt{\frac{2\log (rd/\delta_0)}{mt}}\Big)^2\nonumber\\
        &\leq 2\sum_p\sum_q \Big(\frac{1}{t}\sum_i \fl{R}^{(i, \ell)}_{p,q}\Big)^2 + 8\sum_p\sum_q \|\fl{R}^{(i, \ell, q)}\|_2^2\frac{2\log (rd/\delta_0)}{mt}\nonumber\\
        &\leq 2\sum_p\sum_q \Big(\frac{1}{t}\sum_i \fl{R}^{(i, \ell)}_{p,q}\Big)^2 + 8\|\fl{R}^{(i, \ell)}\|_\s{F}^2\frac{2d\log (rd/\delta_0)}{mt},
    \end{align}
    where we use that $2 \|\fl{R}^{(i, \ell, q)}\|_{2} \sqrt{\frac{2\log (rd/\delta_0)}{mt}} > 2 \|\fl{R}^{(i, \ell, q)}\|_{\infty} \frac{2\log (rd/\delta_0)}{mt}$.
    Hence, with probability at least $1-\delta_0$, we have
    \begin{align}
        \|\fl{V}\|_\s{F}^2 &\leq 2\sum_p\sum_q \Big(\frac{1}{t}\sum_i \fl{R}^{(i, \ell)}_{p,q}\Big)^2 + 16\|\fl{R}^{(i, \ell)}\|_\s{F}^2\frac{d\log (rd/\delta_0)}{mt}\label{eq:lemma-Fnorm-helper1-num-temp}.
    \end{align}
\end{proof}

\begin{lemma}\label{lemma:a^TEb-concentration}
    Let 
    \begin{align}
        \fl{A}_{rd\times rd} &= \frac{1}{mt}\sum_{i \in [t]}\Big(\fl{w}^{(i, \ell)}(\fl{w}^{(i, \ell)})^{\s{T}} \otimes (\fl{X}^{(i)})^{\s{T}}\fl{X}^{(i)}\Big) \nonumber\\
        &= \frac{1}{mt}\sum_{i \in [t]}\Big(\fl{w}^{(i, \ell)}(\fl{w}^{(i, \ell)})^{\s{T}} \otimes \Big(\sum_{j=1}^{m}\fl{x}^{(i)}_j(\fl{x}^{(i)}_j)^{\s{T}}\Big)\Big)\nonumber
    \end{align}
    s.t. $\fl{A}^{-1} = \fl{\E{\fl{A}}} + \fl{E}$ where $\fl{E}$ is the error matrix due to perturbation. Then for vectors $\fl{a}, \fl{b} \in \bR^{rd}$, we have
    \begin{align}
        \Big|\fl{a}^\s{T}\fl{E}\fl{b}\Big| &\leq c\|\fl{a}\|_2\|\fl{b}\|_2\sqrt{\sum_{i \in [t]}\sum_{j \in [m]}\|\fl{w}^{(i, \ell)}\|_2^4\frac{\log (1/\delta_0)}{m^2t^2}}.\nonumber
    \end{align}
\end{lemma}
\begin{proof}
    We can rewrite the vectors $\fl{a}$ and $\fl{b}$ s.t. $\fl{a}^\s{T} = \begin{bmatrix}\fl{a}^\s{T}_1, \fl{a}^\s{T}_2, \dots, \fl{a}^\s{T}_r \end{bmatrix}$ and $\fl{b}^\s{T} = \begin{bmatrix}\fl{b}^\s{T}_1, \fl{b}^\s{T}_2, \dots, \fl{b}^\s{T}_r \end{bmatrix}$ respectively where each $\fl{a}_i, \fl{b}_i \in \bR^{d}$. Note that 
    \begin{align}
        \E{\fl{A}} &= \frac{1}{t}\sum_{i \in [t]}\Big(\fl{w}^{(i, \ell)}(\fl{w}^{(i, \ell)})^{\s{T}} \otimes \fl{I}\Big),\nonumber\\
        \text{and } \fl{a}^\s{T}\fl{A}\fl{b} &= \fl{a}^\s{T}\Big(\frac{1}{mt}\sum_{i \in [t]}\Big(\fl{w}^{(i, \ell)}(\fl{w}^{(i, \ell)})^{\s{T}} \otimes \Big(\sum_{j=1}^{m}\fl{x}^{(i)}_j(\fl{x}^{(i)}_j)^{\s{T}}\Big)\Big)\Big)\fl{b}\nonumber\\
        &= \frac{1}{mt}\sum_{i \in [t]}\sum_{j \in [m]}\sum_{p \in [r]}\sum_{q \in [r]} \fl{a}^\s{T}_p\Big(w^{(i, \ell)}_p w^{(i, \ell)}_q \fl{x}^{(i)}_j(\fl{x}^{(i)}_j)^{\s{T}}\Big)\fl{b}_q\nonumber\\
        &= \frac{1}{mt}\sum_{i \in [t]}\sum_{j \in [m]}(\fl{x}^{(i)}_j)^\s{T}\Big(\sum_{p \in [r]}\sum_{q \in [r]} w^{(i, \ell)}_p w^{(i, \ell)}_q \fl{a}_p(\fl{b}_q)^{\s{T}}\Big)\fl{x}^{(i)}_j.\label{eq:lemma-a^TEb-concentration-denom-1}
    \end{align}
    Furthermore, 
    \begin{align}
        \s{Tr}\Big(\sum_{p \in [r]}\sum_{q \in [r]} w^{(i, \ell)}_p w^{(i, \ell)}_q \fl{a}_p(\fl{b}_q)^{\s{T}}\Big) &= \s{Tr}\Big( \fl{a}^\s{T}\Big(\fl{w}^{(i, \ell)}(\fl{w}^{(i, \ell)})^{\s{T}} \otimes \fl{I}\Big)\fl{b}\Big)\nonumber\\
        &= \fl{a}^\s{T}\Big(\fl{w}^{(i, \ell)}(\fl{w}^{(i, \ell)})^{\s{T}} \otimes \fl{I}\Big)\fl{b}\label{eq:lemma-a^TEb-concentration-denom-2}
    \end{align}
    and
    \begin{align}
        &\|\sum_{p \in [r]}\sum_{q \in [r]} w^{(i, \ell)}_p w^{(i, \ell)}_q \fl{a}_p(\fl{b}_q)^{\s{T}}\|_\s{F} \nonumber\\  
        &= \|\Big(\sum_{p \in [r]}w^{(i, \ell)}_p \fl{a}_p\Big) \Big(\sum_{q \in [r]} w^{(i, \ell)}_q \fl{b}_q\Big)^{\s{T}}\|_\s{F}\nonumber\\
        &= \|\sum_{p \in [r]}w^{(i, \ell)}_p \fl{a}_p\|_2\|\sum_{q \in [r]} w^{(i, \ell)}_q \fl{b}_q\|_2\nonumber\\
        &\leq \Big(\sum_{p \in [r]}|w^{(i, \ell)}_p|\|\fl{a}_p\|_2\Big)\Big(\sum_{q \in [r]}|w^{(i, \ell)}_q|\|\fl{b}_q\|_2\Big)\nonumber\\
        &\leq \sqrt{\Big(\sum_{p \in [r]}(w^{(i, \ell)}_p)^2\Big)\Big(\sum_{p \in [r]} \|\fl{a}_p\|_2^2\Big)}\sqrt{\Big(\sum_{q \in [r]}(w^{(i, \ell)}_q)^2\Big)\Big(\sum_{q \in [r]} \|\fl{b}_q\|_2^2\Big)}\nonumber\\
        &= \|\fl{w}^{(i, \ell)}\|_2^2\|\fl{a}\|_2\|\fl{b}\|_2.\label{eq:lemma-a^TEb-concentration-denom-3}
    \end{align}
    Therefore, using \eqref{eq:lemma-a^TEb-concentration-denom-1}, \eqref{eq:lemma-a^TEb-concentration-denom-2} and \eqref{eq:lemma-a^TEb-concentration-denom-3} in Lemma~\ref{lemma:useful0} we get
    \begin{align}
        \Big|\fl{a}^\s{T}\Big(\fl{A} - \frac{1}{mt}\sum_{i \in [t]}\sum_{j \in [m]} \Big(\fl{w}^{(i, \ell)}(\fl{w}^{(i, \ell)})^{\s{T}} \otimes \fl{I}\Big)\Big)\fl{b}\Big| &\leq c\sqrt{\sum_{i \in [t]}\sum_{j \in [m]}\|\fl{w}^{(i, \ell)}\|_2^4\|\fl{a}\|_2^2\|\fl{b}\|_2^2\frac{\log (1/\delta_0)}{m^2t^2}}\nonumber\\
        \iff \Big|\fl{a}^\s{T}\fl{E}\fl{b}\Big| &\leq c\|\fl{a}\|_2\|\fl{b}\|_2\sqrt{\sum_{i \in [t]}\sum_{j \in [m]}\|\fl{w}^{(i, \ell)}\|_2^4\frac{\log (1/\delta_0)}{m^2t^2}}.
    \end{align}
\end{proof}

\begin{lemma}\label{lemma:lemma-optimize_dp-lrs-U-Fnorm-bound}
For some constant $c>0$ and for any iteration indexed by $\ell>0$, we  have 
\begin{align}
    &\|\fl{U}^{(\ell)} - \fl{U}^{\star}\fl{Q}^{(\ell-1)}\|_\s{F} \nonumber\\
    &\leq \resizebox{.97\textwidth}{!}{$\frac{1}{\frac{1}{r}\lambda_r\Big(\frac{r}{t}(\fl{W}^{(\ell)})^\s{T}\fl{W}^{(\ell)} \Big) - c\sqrt{\sum_{i \in [t]}\sum_{j \in [m]}\|\fl{w}^{(i, \ell)}\|_2^4}\sqrt{\frac{rd\log(1/ \delta_0)}{m^2t^2}}-\frac{\sigma_1}{mt}\Big(2\sqrt{rd} + 4\sqrt{\log rd}\Big)} \cdot$}\nonumber\\
    &\qquad \Big\{\frac{2}{t}\|\fl{U}^{\star}\fl{Q}^{(\ell-1)}(\fl{H}^{(\ell)})^\s{T}\fl{W}^{(\ell)}\|_\s{F} + \sqrt{\frac{4\zeta}{t}}(\max_i\|\fl{w}^{(i, \ell)}\|_2)\|\fl{b}^{\star(i)} - \fl{b}^{(i, \ell)}\|_2\nonumber\\ 
    &\qquad +4\Big(\|\fl{U}^{\star}\fl{Q}^{(\ell-1)}\|\|\fl{h}^{(i, \ell)}\|_2\|\fl{w}^{(i, \ell)}\|_2 + \|\fl{b}^{\star(i)} - \fl{b}^{(i, \ell)}\|_2\|\fl{w}^{(i, \ell)}\|_2\Big)\sqrt{\frac{d\log (rd/\delta_0)}{mt}}\nonumber\\
    &\qquad + \frac{\sigma_2}{mt}6\sqrt{rd\log(rd)} + \frac{\sigma_1}{mt}\Big(2\sqrt{rd} + 4\sqrt{\log rd}\Big)\sqrt{r} + \frac{2\sigma\sqrt{d\mu^{\star}\lambda^{\star}_r}\log(2rdmt/\delta_0)}{\sqrt{mt}}\Big\}\nonumber,
    \end{align}
    \begin{align}
     &\|\fl{U}^{(\ell)}\|_\s{F} \nonumber\\
     &\leq \resizebox{.97\textwidth}{!}{$\frac{1}{\frac{1}{r}\lambda_r\Big(\frac{r}{t}(\fl{W}^{(\ell)})^\s{T}\fl{W}^{(\ell)} \Big) - c\sqrt{\sum_{i \in [t]}\sum_{j \in [m]}\|\fl{w}^{(i, \ell)}\|_2^4}\sqrt{\frac{rd\log(1/ \delta_0)}{m^2t^2}}-\frac{\sigma_1}{mt}\Big(2\sqrt{rd} + 4\sqrt{\log rd}\Big)} \cdot$}\nonumber\\
    &\qquad \Big\{\frac{2}{t}\|\fl{U}^{\star}(\fl{W}^{\star})^\s{T}\fl{W}^{(\ell)}\|_\s{F} + \sqrt{\frac{4\zeta}{t}}(\max_i\|\fl{w}^{(i, \ell)}\|_2)\|\fl{b}^{\star(i)} - \fl{b}^{(i, \ell)}\|_2\nonumber\\ 
    &\qquad +4\Big(\|\fl{U}^{\star}\|\|\fl{w}^{\star(i)}\|_2\|\fl{w}^{(i, \ell)}\|_2 + \|\fl{b}^{\star(i)} - \fl{b}^{(i, \ell)}\|_2\|\fl{w}^{(i, \ell)}\|_2\Big)\sqrt{\frac{d\log (rd/\delta_0)}{mt}}\nonumber\\
    &\qquad + \frac{\sigma_2}{mt}6\sqrt{rd\log(rd)} + \frac{2\sigma\sqrt{d\mu^{\star}\lambda^{\star}_r}\log(2rdmt/\delta_0)}{\sqrt{mt}}\Big\}\nonumber,
    \end{align}
    \begin{align}
    &\lr{\Delta(\fl{U}^{+(\ell)}, \fl{U}^{\star})}_\s{F} \nonumber\\
    &\leq \resizebox{.97\textwidth}{!}{$\Big\{\frac{c\sqrt{\sum_{i \in [t]}\sum_{j \in [m]}\|\fl{w}^{(i, \ell)}\|_2^4}\sqrt{\frac{rd\log(1/ \delta_0)}{m^2t^2}}+\frac{\sigma_1}{mt}\Big(2\sqrt{rd} + 4\sqrt{\log rd}\Big)}{\frac{1}{r}\lambda_r\Big(\frac{r}{t}(\fl{W}^{(\ell)})^\s{T}\fl{W}^{(\ell)} \Big) - c\sqrt{\sum_{i \in [t]}\sum_{j \in [m]}\|\fl{w}^{(i, \ell)}\|_2^4}\sqrt{\frac{rd\log(1/ \delta_0)}{m^2t^2}}-\frac{\sigma_1}{mt}\Big(2\sqrt{rd} + 4\sqrt{\log rd}\Big)}\cdot $}\nonumber\\
    &\qquad \Big(\frac{2}{t}\|\fl{U}^{\star}\fl{Q}^{(\ell-1)}(\fl{H}^{(\ell)})^\s{T}\fl{W}^{(\ell)}\|_\s{F}\Big) \nonumber\\
    &+ \resizebox{.97\textwidth}{!}{$\frac{1}{\frac{1}{r}\lambda_r\Big(\frac{r}{t}(\fl{W}^{(\ell)})^\s{T}\fl{W}^{(\ell)} \Big) - c\sqrt{\sum_{i \in [t]}\sum_{j \in [m]}\|\fl{w}^{(i, \ell)}\|_2^4}\sqrt{\frac{rd\log(1/ \delta_0)}{m^2t^2}}-\frac{\sigma_1}{mt}\Big(2\sqrt{rd} + 4\sqrt{\log rd}\Big)}\cdot$}\nonumber\\ 
    &\qquad \Big\{\Big(\sqrt{\frac{4\zeta}{t}}(\max_i\|\fl{w}^{(i, \ell)}\|_2)\|\fl{b}^{\star(i)} - \fl{b}^{(i, \ell)}\|_2\Big) \nonumber\\
    &+ 4\Big(\|\fl{U}^{\star}\fl{Q}^{(\ell-1)}\|\|\fl{h}^{(i, \ell)}\|_2\|\fl{w}^{(i, \ell)}\|_2 + \|\fl{b}^{\star(i)} - \fl{b}^{(i, \ell)}\|_2\|\fl{w}^{(i, \ell)}\|_2\Big)\sqrt{\frac{d\log (rd/\delta_0)}{mt}}\nonumber\\
    &\qquad \frac{\sigma_2}{mt}6\sqrt{rd\log(rd)} + \frac{\sigma_1}{mt}\Big(2\sqrt{rd} + 4\sqrt{\log rd}\Big)\sqrt{r} + \frac{2\sigma\sqrt{d\mu^{\star}\lambda^{\star}_r}\log(2rdmt/\delta_0)}{\sqrt{mt}}\Big\}\Big\}\|\fl{R}^{-1}\|\nonumber
\end{align}
with probability at least $1-\cO(\delta_0)$.
\end{lemma}
\begin{proof}
    \textbf{Analysis of $\|\fl{U}^{(\ell)} - \fl{U}^{\star}\fl{Q}^{(\ell-1)}\|_\s{F}$:}\\
    Update step for $\fl{U}$ of the Algorithm without DP Noise for the $\ell^{\s{th}}$ iteration gives us
    \begin{align}
    &\sum_{i \in [t]}(\fl{X}^{(i)})^{\s{T}}\fl{X}^{(i)}\fl{U}^{(\ell)}\fl{w}^{(i, \ell)}(\fl{w}^{(i, \ell)})^{\s{T}} \nonumber\\
    &=  \sum_{i \in [t]} (\fl{X}^{(i)})^{\s{T}}\fl{X}^{(i)}\Big(\fl{U}^{\star}\fl{w}^{\star(i)} + (\fl{b}^{\star(i)} - \fl{b}^{(i, \ell)})\Big)(\fl{w}^{(i, \ell)})^{\s{T}}+\sum_{i\in [t]}(\fl{X}^{(i)})^{\s{T}}\fl{\xi}^{(i)}(\fl{w}^{(i, \ell)})^{\s{T}} \nonumber\\ 
        \implies& \sum_{i \in [t]}(\fl{X}^{(i)})^{\s{T}}\fl{X}^{(i)}(\fl{U}^{(\ell)} - \fl{U}^{\star}\fl{Q}^{(\ell-1)})\fl{w}^{(i, \ell)}(\fl{w}^{(i, \ell)})^{\s{T}}\nonumber\\ 
        &=  \sum_{i \in [t]} (\fl{X}^{(i)})^{\s{T}}\fl{X}^{(i)}\Big(\fl{U}^{\star}(\fl{w}^{\star(i)} - \fl{Q}^{(\ell-1)}\fl{w}^{(i, \ell)}) + (\fl{b}^{\star(i)} - \fl{b}^{(i, \ell)})\Big)(\fl{w}^{(i, \ell)})^{\s{T}}\nonumber\\
        &\qquad +\sum_{i\in [t]}(\fl{X}^{(i)})^{\s{T}}\fl{\xi}^{(i)}(\fl{w}^{(i, \ell)})^{\s{T}}.\nonumber
    \end{align}
    Using Lemma~\ref{tensor-product:ABC}, the above can be written as:
    \begin{align}
        \fl{A}\s{vec}(\fl{U}^{(\ell)}) &=  \s{vec}(\fl{V}' + \fl{\Xi}),\nonumber\\
        \text{and }\fl{A}\s{vec}(\fl{U}^{(\ell)} - \fl{U}^\star\fl{Q}^{(\ell-1)}) &=  \s{vec}(\fl{V} + \fl{\Xi}),\label{lemma-optimize_dp-lrs-U-Fnorm-bound-vectorisedform}
    \end{align}
    where
    \begin{align}
        \fl{A}_{rd\times rd} &= \frac{1}{mt}\sum_{i \in [t]}\Big(\fl{w}^{(i, \ell)}(\fl{w}^{(i, \ell)})^{\s{T}} \otimes (\fl{X}^{(i)})^{\s{T}}\fl{X}^{(i)}\Big) \nonumber\\
        &= \frac{1}{mt}\sum_{i \in [t]}\Big(\fl{w}^{(i, \ell)}(\fl{w}^{(i, \ell)})^{\s{T}} \otimes \Big(\sum_{j=1}^{m}\fl{x}^{(i)}_j(\fl{x}^{(i)}_j)^{\s{T}}\Big)\Big),\nonumber\\
        \fl{V}_{d\times r} &= \frac{1}{mt}\sum_{i \in [t]} (\fl{X}^{(i)})^{\s{T}}\fl{X}^{(i)}\Big(\fl{U}^{\star}(\fl{w}^{\star(i)} - \fl{Q}^{(\ell-1)}\fl{w}^{(i, \ell)}) + (\fl{b}^{\star(i)} - \fl{b}^{(i, \ell)})\Big)(\fl{w}^{(i, \ell)})^{\s{T}}\nonumber\\
        &= \frac{1}{mt}\sum_{i \in [t]} (\fl{X}^{(i)})^{\s{T}}\fl{X}^{(i)}\Big(-\fl{U}^{\star}\fl{Q}^{(\ell-1)}\fl{h}^{(i, \ell)} + (\fl{b}^{\star(i)} - \fl{b}^{(i, \ell)})\Big)(\fl{w}^{(i, \ell)})^{\s{T}},\nonumber\\
        \fl{V}'_{d\times r} &= \frac{1}{mt}\sum_{i \in [t]} (\fl{X}^{(i)})^{\s{T}}\fl{X}^{(i)}\Big(\fl{U}^{\star}\fl{w}^{\star(i)} + (\fl{b}^{\star(i)} - \fl{b}^{(i, \ell)})\Big)(\fl{w}^{(i, \ell)})^{\s{T}},\nonumber
    \end{align}
    $\fl{h}^{(i, \ell)} = \fl{w}^{(i, \ell)} - (\fl{Q}^{(\ell-1)})^{-1}\fl{w}^{\star(i)} $, $\xi^{(i)}_j \sim \cN(0, \sigma^2)$ and $\fl{\Xi}_{d\times r} = \frac{1}{mt}\sum_{i\in [t]}(\fl{X}^{(i)})^{\s{T}}\fl{\xi}^{(i)}(\fl{w}^{(i, \ell)})^{\s{T}}$. 
    Now introducing DP noise we get:
    \begin{align}
        &\s{vec}(\fl{U}^{(\ell)}) =  \Big(\fl{A} + \frac{\fl{N}_1}{mt}\Big)^{-1}
        \Big(\s{vec}\Big(\fl{V}' + \frac{\fl{N}_2}{mt}+\fl{\Xi}\Big)\Big)\nonumber\\
        &\implies \s{vec}(\fl{U}^{(\ell)} - \fl{U}^\star\fl{Q}^{(\ell-1)})\nonumber\\ 
        &=  \Big(\fl{A} + \frac{\fl{N}_1}{mt}\Big)^{-1}\Big(\s{vec}\Big(\fl{V}' +\fl{\Xi} + \frac{\fl{N}_2}{mt}\Big) - \Big(\fl{A} + \frac{\fl{N}_1}{mt}\Big)\s{vec}(\fl{U}^\star\fl{Q}^{(\ell-1)})\Big)\nonumber\\
        &=  \Big(\fl{A} + \frac{\fl{N}_1}{mt}\Big)^{-1}\cdot\nonumber\\
        &\qquad \Big(\Big(\s{vec}(\fl{V}') - \fl{A}\s{vec}(\fl{U}^\star\fl{Q}^{(\ell-1)})\Big) + \s{vec}\Big(\frac{\fl{N}_2}{mt}\Big) -  \frac{\fl{N}_1}{mt}\s{vec}(\fl{U}^\star\fl{Q}^{(\ell-1)}) + \s{vec}(\fl{\Xi})\Big)\nonumber\\
        &=  \Big(\fl{A} + \frac{\fl{N}_1}{mt}\Big)^{-1}\Big(\s{vec}(\fl{V}) + \s{vec}\Big(\frac{\fl{N}_2}{mt}\Big) -  \frac{\fl{N}_1}{mt}\s{vec}(\fl{U}^\star\fl{Q}^{(\ell-1)})\Big) + \s{vec}(\fl{\Xi})\Big),\label{eq:lemma-optimize_dp-lrs-U-Fnorm-bound-exp1}
    \end{align}
    where
    \begin{align}
        \fl{N}_1 &\sim \sigma_1\cM\cN_{rd\times rd}(\vzero, \fl{I}_{rd\times rd}, \fl{I}_{rd\times rd}),\nonumber\\
        \fl{N}_2 &\sim \sigma_2\cM\cN_{d\times r}(\vzero, \fl{I}_{d\times d}, \fl{I}_{r\times r}),\nonumber
    \end{align}
    where $\cM\cN$ denotes the Matrix Normal Distribution. Note that \eqref{eq:lemma-optimize_dp-lrs-U-Fnorm-bound-exp1} gives:
    \begin{align}
        &\|\s{vec}(\fl{U}^{(\ell)} - \fl{U}^{\star}\fl{Q}^{(\ell-1)})\|_2 \nonumber\\
        &\qquad \leq \lr{\Big(\fl{A}+\frac{\fl{N}_1}{mt}\Big)^{-1}}\lr{\s{vec}(\fl{V}) + \s{vec}\Big(\frac{\fl{N}_2}{mt}\Big) -  \frac{\fl{N}_1}{mt}\s{vec}(\fl{U}^\star\fl{Q}^{(\ell-1)})+\s{vec}(\fl{\Xi})}_2\nonumber\\
        \iff &\|\fl{U}^{(\ell)} - \fl{U}^{\star}\fl{Q}^{(\ell-1)}\|_\s{F} \nonumber\\
        &\qquad \leq \lr{\Big(\fl{A}+\frac{\fl{N}_1}{mt}\Big)^{-1}}\Big(\|\fl{V}\|_\s{F} + \lr{\frac{\fl{N}_2}{mt}}_\s{F} + \lr{\frac{\fl{N}_1}{mt}}_2\|\fl{U}^{\star}\fl{Q}^{(\ell-1)}\|_\s{F}+\|\fl{\Xi}\|_{\s{F}}\Big). \label{eq:lemma-optimize_dp-lrs-U-Fnorm-bound-exp2}
    \end{align}
    
    First, with high probability, we will bound the following quantity
    \begin{align}
       \|\fl{\Xi}\|_{2, \infty} &= \lr{\frac{1}{mt}\sum_{i\in [t]}(\fl{X}^{(i)})^{\s{T}}\fl{\xi}^{(i)}(\fl{w}^{(i, \ell)})^{\s{T}}}_{2,\infty}. \nonumber
    \end{align}
    Condition on the vector $\fl{\xi}^{(i)}$.
    Notice that the random vector $(\fl{X}^{(i)})^{\s{T}}\fl{\xi}^{(i)}$  is a $d$ dimensional vectors where each entry is independently generated according to $\ca{N}(0,\lr{\fl{\xi}^{(i)}}^2_2)$.
    Therefore for a fixed row indexed by $s$, we have $\ell_2$ norm of the $s^{\s{th}}$ row of $\frac{1}{mt}\sum_{i\in [t]}(\fl{X}^{(i)})^{\s{T}}\fl{\xi}^{(i)}(\fl{w}^{(i, \ell)})^{\s{T}}$ is going to be $\frac{1}{mt}\lr{\fl{W}^{(\ell)}}_{\s{F}}\lr{\fl{\xi}^{(i)}}_2\sqrt{\log (dr\delta^{-1})}$. Hence, we have that with probability $1-\delta_0$ (provided that $m=\widetilde{\Omega}(\sigma^2)$),
    \begin{align}
        \lr{\frac{1}{mt}\sum_{i\in [t]}(\fl{X}^{(i)})^{\s{T}}\fl{\xi}^{(i)}(\fl{w}^{(i, \ell)})^{\s{T}}}_{2,\infty} \le \frac{2\sigma\sqrt{\mu^{\star}\lambda^{\star}_r}\log(2rdmt/\delta_0)}{\sqrt{mt}}\label{lemma-optimize_dp-lrs-U-Fnorm-bound-Xi-2,infty-norm}
    \end{align}
    and
    \begin{align}
        \lr{\frac{1}{mt}\sum_{i\in [t]}(\fl{X}^{(i)})^{\s{T}}\fl{\xi}^{(i)}(\fl{w}^{(i, \ell)})^{\s{T}}}_{\s{F}} \le \frac{2\sigma\sqrt{d\mu^{\star}\lambda^{\star}_r}\log(2rdmt/\delta_0)}{\sqrt{mt}}.\label{lemma-optimize_dp-lrs-U-Fnorm-bound-Xi-F-norm}
    \end{align}
    Now, We will analyse the two multiplicands separately.
    \begin{align}
        \E{\fl{V}} &= \frac{1}{t}\sum_{i \in [t]}\Big(-\fl{U}^{\star}\fl{Q}^{(\ell-1)}\fl{h}^{(i, \ell)} + (\fl{b}^{\star(i)} - \fl{b}^{(i, \ell)})\Big)(\fl{w}^{(i, \ell)})^{\s{T}}\nonumber\\
        &= \frac{1}{t}\sum_{i \in [t]}\Big(-\fl{U}^{\star}\fl{Q}^{(\ell-1)}\fl{h}^{(i, \ell)}(\fl{w}^{(i, \ell)})^{\s{T}} + (\fl{b}^{\star(i)} - \fl{b}^{(i, \ell)})(\fl{w}^{(i, \ell)})^{\s{T}}\Big)\nonumber\\
        &= \frac{-1}{t}\fl{U}^{\star}\fl{Q}^{(\ell-1)}(\fl{H}^{(\ell)})^\s{T}\fl{W}^{(\ell)} + \frac{1}{t}\sum_{i \in [t]}(\fl{b}^{\star(i)} - \fl{b}^{(i, \ell)})(\fl{w}^{(i, \ell)})^{\s{T}}.\nonumber
    \end{align}
    Using Lemma~\ref{lemma:Fnorm-helper1} with $\fl{R}^{(i, \ell)} = \Big(-\fl{U}^{\star}\fl{Q}^{(\ell-1)}\fl{h}^{(i, \ell)} + (\fl{b}^{\star(i)} - \fl{b}^{(i, \ell)})\Big)(\fl{w}^{(i, \ell)})^{\s{T}}$, we have with probability at least $1-\delta_0$
    \begin{align}
        &\sum_{p}\sum_q \fl{V}_{p,q}^2\nonumber\\
        &\leq 2\sum_p\sum_q \Big(\frac{1}{t}\sum_i \fl{R}^{(i, \ell)}_{p,q}\Big)^2 + 16\|\fl{R}^{(i, \ell)}\|_\s{F}^2\frac{d\log (rd/\delta_0)}{mt}\nonumber \\
        &\leq 2\sum_p\sum_q \Big(\frac{1}{t}\sum_i\Big(-\fl{U}^{\star}\fl{Q}^{(\ell-1)}\fl{h}^{(i, \ell)} + (\fl{b}^{\star(i)} - \fl{b}^{(i, \ell)})\Big)(\fl{w}^{(i, \ell)})^{\s{T}}\Big)_{p,q}^2 \nonumber\\
        &\qquad + 16\|\fl{R}^{(i, \ell)}\|_\s{F}^2\frac{d\log (rd/\delta_0)}{mt} \nonumber\\
        &\leq 2\sum_p\sum_q \frac{1}{t^2}\Big(-\fl{U}^{\star}\fl{Q}^{(\ell-1)}(\fl{H}^{(\ell)})^\s{T}\fl{W}^{(\ell)} + \sum_i (\fl{b}^{\star(i)} - \fl{b}^{(i, \ell)})(\fl{w}^{(i, \ell)})^{\s{T}}\Big)_{p,q}^2\nonumber\\
        &\qquad + 16\|\fl{R}^{(i, \ell)}\|_\s{F}^2\frac{d\log (rd/\delta_0)}{mt} \nonumber\\
        &\leq \frac{4}{t^2}\sum_p\sum_q \Big(\fl{U}^{\star}\fl{Q}^{(\ell-1)}(\fl{H}^{(\ell)})^\s{T}\fl{W}^{(\ell)}\Big)_{p,q}^2 + \frac{4}{t^2}\sum_p\sum_q \Big(\sum_i (\fl{b}^{\star(i)} - \fl{b}^{(i, \ell)})(\fl{w}^{(i, \ell)})^{\s{T}}\Big)_{p,q}^2 \nonumber\\
        &\qquad + 16\|\fl{R}^{(i, \ell)}\|_\s{F}^2\frac{d\log (rd/\delta_0)}{mt}\nonumber\\
        &\leq \frac{4}{t^2}\|\fl{U}^{\star}\fl{Q}^{(\ell-1)}(\fl{H}^{(\ell)})^\s{T}\fl{W}^{(\ell)}\|_\s{F}^2 + \frac{4}{t}\sum_p\sum_q \sum_i (\fl{b}^{\star(i)}_p - \fl{b}^{(i, \ell)}_p)^2(\fl{w}^{(i, \ell)}_s)^2 \nonumber\\
        &\qquad + 16\|\fl{R}^{(i, \ell)}\|_\s{F}^2\frac{d\log (rd/\delta_0)}{mt}\nonumber\\
        &\leq \frac{4}{t^2}\|\fl{U}^{\star}\fl{Q}^{(\ell-1)}(\fl{H}^{(\ell)})^\s{T}\fl{W}^{(\ell)}\|_\s{F}^2 + \frac{4}{t}\zeta (\max_i \|\fl{w}^{(i, \ell)}\|^2_2)\|\fl{b}^{\star(i)} - \fl{b}^{(i, \ell)}\|^2_2 + 16\|\fl{R}^{(i, \ell)}\|_\s{F}^2\frac{d\log (rd/\delta_0)}{mt}.\nonumber
    \end{align}
    \begin{align}
        \implies \|\fl{V}\|_\s{F} &\leq \frac{2}{t}\|\fl{U}^{\star}\fl{Q}^{(\ell-1)}(\fl{H}^{(\ell)})^\s{T}\fl{W}^{(\ell)}\|_\s{F} + \sqrt{\frac{4\zeta}{t}}(\max_i\|\fl{w}^{(i, \ell)}\|_2)\|\fl{b}^{\star(i)} - \fl{b}^{(i, \ell)}\|_2 \nonumber\\
        &\qquad + 4\|\fl{R}^{(i, \ell)}\|_\s{F}\sqrt{\frac{d\log (rd/\delta_0)}{mt}}\label{eq:lemma-optimize_dp-lrs-u-2norm-bound-type1-num-temp}.
    \end{align}
    Since $\fl{R}^{(i, \ell)} = -\fl{U}^{\star}\fl{Q}^{(\ell-1)}\fl{h}^{(i, \ell)}(\fl{w}^{(i, \ell)})^{\s{T}} + (\fl{b}^{\star(i)} - \fl{b}^{(i, \ell)})(\fl{w}^{(i, \ell)})^{\s{T}} $, we have
    \begin{align}
        \|\fl{R}^{(i, \ell)}\|_{\s{F}} &= \|-\fl{U}^{\star}\fl{Q}^{(\ell-1)}\fl{h}^{(i, \ell)}(\fl{w}^{(i, \ell)})^{\s{T}} + (\fl{b}^{\star(i)} - \fl{b}^{(i, \ell)})(\fl{w}^{(i, \ell)})^{\s{T}}\|_{\s{F}}\nonumber\\
        &\leq \|\fl{U}^{\star}\fl{Q}^{(\ell-1)}\fl{h}^{(i, \ell)}(\fl{w}^{(i, \ell)})^{\s{T}}\|_{\s{F}} + \|(\fl{b}^{\star(i)} - \fl{b}^{(i, \ell)})(\fl{w}^{(i, \ell)})^{\s{T}} \|_\s{F}\nonumber\\
        &\leq \|\fl{U}^{\star}\fl{Q}^{(\ell-1)}\|\|\fl{h}^{(i, \ell)}\|_2\|\fl{w}^{(i, \ell)}\|_2 + \|\fl{b}^{\star(i)} - \fl{b}^{(i, \ell)}\|_2\|\fl{w}^{(i, \ell)}\|_2 \label{eq:lemma-optimize_dp-lrs-u-2norm-bound-type1-num-term2}
    \end{align}
    Using \eqref{eq:lemma-optimize_dp-lrs-u-2norm-bound-type1-num-term2} in \eqref{eq:lemma-optimize_dp-lrs-u-2norm-bound-type1-num-temp} gives
    \begin{align}
        \|\fl{V}\|_\s{F} &\leq \frac{2}{t}\|\fl{U}^{\star}\fl{Q}^{(\ell-1)}(\fl{H}^{(\ell)})^\s{T}\fl{W}^{(\ell)}\|_\s{F} + \sqrt{\frac{4\zeta}{t}}(\max_i\|\fl{w}^{(i, \ell)}\|_2)\|\fl{b}^{\star(i)} - \fl{b}^{(i, \ell)}\|_2\nonumber\\ 
        &\qquad +4\Big(\|\fl{U}^{\star}\fl{Q}^{(\ell-1)}\|\|\fl{h}^{(i, \ell)}\|_2\|\fl{w}^{(i, \ell)}\|_2 + \|\fl{b}^{\star(i)} - \fl{b}^{(i, \ell)}\|_2\|\fl{w}^{(i, \ell)}\|_2\Big)\sqrt{\frac{d\log (rd/\delta_0)}{mt}}\label{eq:lemma-optimize_dp-lrs-u-2norm-bound-type1-num-val}
    \end{align}
    Now, let $\fl{\cV} \triangleq \{\fl{v} \in \bR^{rd} | \|\fl{v}\|_2 = 1\}$. Then for $\epsilon \leq 1$, there exists an $\epsilon$-net, $N_\epsilon \subset \fl{\cV}$, of size $(1 + 2/\epsilon)^
    {rd}$ w.r.t the Euclidean norm, i.e. $\forall$  $\fl{v} \in \fl{\cV}$, $\exists$ $\fl{v}' \in N_\epsilon$ s.t. $\|\fl{v} - \fl{v}'\|_2 \leq \epsilon$.
    Now consider any $\fl{v}^\s{T} = \begin{bmatrix}\fl{v}^\s{T}_1, \fl{v}^\s{T}_2, \dots, \fl{v}^\s{T}_r \end{bmatrix} \in N_\epsilon$ where each $\fl{v}_i \in \bR^{d}$. Then using Lemma~\ref{lemma:a^TEb-concentration} with $\fl{a} = \fl{b} = \fl{v}$, we get:
    \begin{align}
        \Big|\fl{v}^\s{T}\Big(\fl{A} - \frac{1}{mt}\sum_{i \in [t]}\sum_{j \in [m]} \Big(\fl{w}^{(i, \ell)}(\fl{w}^{(i, \ell)})^{\s{T}} \otimes \fl{I}\Big)\Big)\fl{v}\Big| &\leq c\sqrt{\sum_{i \in [t]}\sum_{j \in [m]}\|\fl{w}^{(i, \ell)}\|_2^4\frac{\log (1/\delta_0)}{m^2t^2}}\nonumber
    \end{align}
    \begin{align}
        \implies  \|\fl{v}^{\s{T}}\fl{E}\fl{v}\| &\leq c\sqrt{\sum_{i \in [t]}\sum_{j \in [m]}\|\fl{w}^{(i, \ell)}\|_2^4}\sqrt{\frac{\log (|N_\epsilon|/\delta_0)}{m^2t^2}}\nonumber\\
        &\leq c\sqrt{\sum_{i \in [t]}\sum_{j \in [m]}\|\fl{w}^{(i, \ell)}\|_2^4}\sqrt{\frac{\log ((1+2/\epsilon)^{rd}/\delta_0)}{m^2t^2}}, \quad \forall \fl{v} \in N_\epsilon\label{eq:lemma-optimize_dp-lrs-u-2norm-bound-type1-A-temp3}
    \end{align}
    w.p. $1-\delta_)$ where $\fl{E} \triangleq \fl{A} - \frac{1}{mt}\sum_{i \in [t]}\sum_{j \in [m]} \Big(\fl{w}^{(i, \ell)}(\fl{w}^{(i, \ell)})^{\s{T}} \otimes \fl{I}\Big)$. Since $\fl{E}$ is symmetric, therefore $\|\fl{E}\| = (\fl{v}')^{\s{T}}\fl{E}\fl{v}'$ where $\fl{v}' \in \fl{\s{V}}$ is the largest eigenvector of $\fl{E}$. Further, $\exists$ $\fl{v} \in N_\epsilon$ s.t. $\|\fl{v}' - \fl{v}\| \leq \epsilon$. This implies:
    \begin{align}
        \|\fl{E}\| = (\fl{v}')^{\s{T}}\fl{E}\fl{v}' &= (\fl{v}' - \fl{v})^{\s{T}}\fl{E}\fl{v} + (\fl{v}')^{\s{T}}\fl{E}(\fl{v}' - \fl{v}) + \fl{v}^{\s{T}}\fl{E}\fl{v}\nonumber\\ 
        & \leq \|\fl{v}' - \fl{v}\|\|\fl{E}\|\|\fl{v}\| + \|\fl{v}'\|\|\fl{E}\|\|\fl{v}' - \fl{v}\| + \fl{v}^{\s{T}}\fl{E}\fl{v}\nonumber\\
        & \leq 2\epsilon\|\fl{E}\| + \fl{v}^{\s{T}}\fl{E}\fl{v}\nonumber\\ 
        \implies \|\fl{E}\| &\leq \frac{\fl{v}^{\s{T}}\fl{E}\fl{v}}{1 - 2\epsilon}.\label{eq:lemma-optimize_dp-lrs-u-2norm-bound-type1-A-temp4}
    \end{align}
    Using \eqref{eq:lemma-optimize_dp-lrs-u-2norm-bound-type1-A-temp3} and \eqref{eq:lemma-optimize_dp-lrs-u-2norm-bound-type1-A-temp4} and setting $\epsilon \gets 1/4$ and $c \gets 2c\sqrt{\log(9)}$ then gives:
    \begin{align}
        \|\fl{E}\| \leq c\sqrt{\sum_{i \in [t]}\sum_{j \in [m]}\|\fl{w}^{(i, \ell)}\|_2^4}\sqrt{\frac{rd\log(1/ \delta_0)}{m^2t^2}}\label{eq:lemma-optimize_dp-lrs-u-2norm-bound-type1-A-E-value}
    \end{align}
    Using \eqref{eq:lemma-optimize_dp-lrs-u-2norm-bound-type1-A-E-value} then gives
    \begin{align}
        \|\fl{A}\| &\geq \min_{\fl{v}}\Big(\frac{1}{mt}\sum_{i \in [t]}\sum_{j \in [m]} \fl{v}^\s{T}\Big(\fl{w}^{(i, \ell)}(\fl{w}^{(i, \ell)})^{\s{T}} \otimes \fl{I}\Big)\fl{v}\Big) - c\sqrt{\sum_{i \in [t]}\sum_{j \in [m]}\|\fl{w}^{(i, \ell)}\|_2^4}\sqrt{\frac{rd\log(1/ \delta_0)}{m^2t^2}}\nonumber\\
        &\geq \lambda_r\Big(\frac{1}{mt}\sum_{i \in [t]}\sum_{j \in [m]} \fl{w}^{(i, \ell)}(\fl{w}^{(i, \ell)})^{\s{T}} \otimes \fl{I}\Big) - c\sqrt{\sum_{i \in [t]}\sum_{j \in [m]}\|\fl{w}^{(i, \ell)}\|_2^4}\sqrt{\frac{rd\log(1/ \delta_0)}{m^2t^2}}\nonumber\\
        &\geq \frac{1}{r}\lambda_r\Big(\frac{r}{t}(\fl{W}^{(\ell)})^\s{T}\fl{W}^{(\ell)} \Big) - c\sqrt{\sum_{i \in [t]}\sum_{j \in [m]}\|\fl{w}^{(i, \ell)}\|_2^4}\sqrt{\frac{rd\log(1/ \delta_0)}{m^2t^2}}.\nonumber
    \end{align}
    Next, note that the matrix $\s{N}/mt$ can be written as $\alpha \ca{N}(0,\fl{I}_{rd\times rd})$ where $\alpha=\frac{\sigma_1}{mt}$. Therefore, with probability at least $1-(rd)^{-8}$, the minimum eigenvalue of the matrix is at least $-\frac{4\sigma_1\sqrt{\log rd}}{mt}$. Further we have using standard gaussian concentration inequalities, 
    \begin{align}
        \lr{\frac{\fl{N}_1}{mt}} &\leq \frac{\sigma_1}{mt}\Big(2\sqrt{rd} + 4\sqrt{\log rd}\Big)\label{eq:lemma-optimize_dp-lrs-U-Fnorm-bound-N1-norm},\\
        \lr{\frac{\fl{N}_2}{mt}} &\leq \frac{\sigma_2}{mt}\Big(\sqrt{d} + \sqrt{r} + 4\sqrt{\log rd}\Big)\label{eq:lemma-optimize_dp-lrs-U-Fnorm-bound-N2-norm},\\
        \lr{\frac{\fl{N}_2}{mt}}_\s{F} &= \frac{\sigma_2}{mt}\|\fl{N}_2\|_\s{F}\nonumber\\
        &\leq \frac{\sigma_2}{mt}\sqrt{\sum_{i \in [d]}\sum_{j \in [r]}2\log((rd)^{2\cdot 9})}\nonumber\\
        &= \frac{\sigma_2}{mt}6\sqrt{rd\log(rd)}\label{eq:lemma-optimize_dp-lrs-U-Fnorm-bound-N2-Fnorm}.
    \end{align}
    
    Hence, the minimum eigenvalue of the matrix $\fl{A} + \frac{\fl{N}}{mt}$ is bounded from below by   
\begin{align}
\lambda_{\min}\Big(\fl{A} + \frac{\fl{N}}{mt}\Big) &\ge   \frac{1}{r}\lambda_r\Big(\frac{r}{t}(\fl{W}^{(\ell)})^\s{T}\fl{W}^{(\ell)} \Big) \\
&- c\sqrt{\sum_{i \in [t]}\sum_{j \in [m]}\|\fl{w}^{(i, \ell)}\|_2^4}\sqrt{\frac{rd\log(1/ \delta_0)}{m^2t^2}}-\frac{\sigma_1}{mt}\Big(2\sqrt{rd} + 4\sqrt{\log rd}\Big).
\end{align}
Therefore, the maximum eigenvalue of $(\fl{A}+\frac{\fl{N}_1}{mt})^{-1}$ is bounded from above by, $\lr{\Big(\fl{A}+\frac{\fl{N}_1}{mt}\Big)^{-1}}$
    \begin{align}
        &\leq \frac{1}{\frac{1}{r}\lambda_r\Big(\frac{r}{t}(\fl{W}^{(\ell)})^\s{T}\fl{W}^{(\ell)} \Big) - c\sqrt{\sum_{i \in [t]}\sum_{j \in [m]}\|\fl{w}^{(i, \ell)}\|_2^4}\sqrt{\frac{rd\log(1/ \delta_0)}{m^2t^2}}-\frac{\sigma_1}{mt}\Big(2\sqrt{rd} + 4\sqrt{\log rd}\Big)} .\label{eq:lemma-optimize_dp-lrs-u-2norm-bound-type1-A-bound}
    \end{align}
    Using \eqref{eq:lemma-optimize_dp-lrs-u-2norm-bound-type1-num-val}, \eqref{eq:lemma-optimize_dp-lrs-u-2norm-bound-type1-A-bound}, \eqref{eq:lemma-optimize_dp-lrs-U-Fnorm-bound-N1-norm}, \eqref{eq:lemma-optimize_dp-lrs-U-Fnorm-bound-N2-norm}, \eqref{eq:lemma-optimize_dp-lrs-U-Fnorm-bound-N2-Fnorm} and \eqref{lemma-optimize_dp-lrs-U-Fnorm-bound-Xi-F-norm} in \eqref{eq:lemma-optimize_dp-lrs-U-Fnorm-bound-exp2} gives $\|\fl{U}^{(\ell)} - \fl{U}^{\star}\fl{Q}^{(\ell-1)}\|_\s{F}$
    \begin{align}
            &\leq \frac{1}{\frac{1}{r}\lambda_r\Big(\frac{r}{t}(\fl{W}^{(\ell)})^\s{T}\fl{W}^{(\ell)} \Big) - c\sqrt{\sum_{i \in [t]}\sum_{j \in [m]}\|\fl{w}^{(i, \ell)}\|_2^4}\sqrt{\frac{rd\log(1/ \delta_0)}{m^2t^2}}-\frac{\sigma_1}{mt}\Big(2\sqrt{rd} + 4\sqrt{\log rd}\Big)} \cdot\nonumber\\
            &\qquad \Big\{\frac{2}{t}\|\fl{U}^{\star}\fl{Q}^{(\ell-1)}(\fl{H}^{(\ell)})^\s{T}\fl{W}^{(\ell)}\|_\s{F} + \sqrt{\frac{4\zeta}{t}}(\max_i\|\fl{w}^{(i, \ell)}\|_2)\|\fl{b}^{\star(i)} - \fl{b}^{(i, \ell)}\|_2\nonumber\\ 
            &\qquad +4\Big(\|\fl{U}^{\star}\fl{Q}^{(\ell-1)}\|\|\fl{h}^{(i, \ell)}\|_2\|\fl{w}^{(i, \ell)}\|_2 + \|\fl{b}^{\star(i)} - \fl{b}^{(i, \ell)}\|_2\|\fl{w}^{(i, \ell)}\|_2\Big)\sqrt{\frac{d\log (rd/\delta_0)}{mt}}\nonumber\\
            &\qquad + \frac{\sigma_2}{mt}6\sqrt{rd\log(rd)} + \frac{\sigma_1}{mt}\Big(2\sqrt{rd} + 4\sqrt{\log rd}\Big)\sqrt{r} + \frac{2\sigma\sqrt{d\mu^{\star}\lambda^{\star}_r}\log(2rdmt/\delta_0)}{\sqrt{mt}}\Big\}. \label{eq:lemma-optimize_dp-lrs-DeltaU-Fnorm-bound-val}
    \end{align}
    
    \textbf{Analysis of $\lr{\fl{U}^{(\ell)}}_\s{F}$:}
    
    The analysis will follow along similar lines as in the previous section except that we will now have:
    \begin{align}
        \s{vec}(\fl{U}^{(\ell)}) &=  \Big(\fl{A} + \frac{\fl{N}_1}{mt}\Big)^{-1}\s{vec}\Big(\fl{V}' + \frac{\fl{N}_2}{mt} + \fl{\Xi}\Big)
         \label{eq:lemma-optimize_dp-lrs-U-Fnorm-bound-type2-exp1}
    \end{align}
    where
    \begin{align}
        \fl{A}_{rd\times rd} &= \frac{1}{mt}\sum_{i \in [t]}\Big(\fl{w}^{(i, \ell)}(\fl{w}^{(i, \ell)})^{\s{T}} \otimes (\fl{X}^{(i)})^{\s{T}}\fl{X}^{(i)}\Big) \nonumber\\
        &= \frac{1}{mt}\sum_{i \in [t]}\Big(\fl{w}^{(i, \ell)}(\fl{w}^{(i, \ell)})^{\s{T}} \otimes \Big(\sum_{j=1}^{m}\fl{x}^{(i)}_j(\fl{x}^{(i)}_j)^{\s{T}}\Big)\Big),\nonumber\\
        \fl{V}'_{d\times r} &= \frac{1}{mt}\sum_{i \in [t]} (\fl{X}^{(i)})^{\s{T}}\fl{X}^{(i)}\Big(\fl{U}^{\star}\fl{w}^{\star(i)} + (\fl{b}^{\star(i)} - \fl{b}^{(i, \ell)})\Big)(\fl{w}^{(i, \ell)})^{\s{T}}.\nonumber
    \end{align}
    i.e. we have the term $-\fl{U}^{\star}\fl{Q}^{(\ell-1)}\fl{h}^{(i, \ell)}$ replaced by $\fl{U}^{\star}\fl{w}^{\star(i)}$. The above gives:
    \begin{align}
        \|\s{vec}(\fl{U}^{(\ell)}\|_2 &\leq \lr{\Big(\fl{A}+\frac{\fl{N}_1}{mt}\Big)^{-1}}\lr{\s{vec}\Big(\fl{V}' + \frac{\fl{N}_2}{mt} + \fl{\Xi}\Big)}_2\nonumber\\
        \iff \|\fl{U}^{(\ell)}\|_\s{F} &\leq \lr{\Big(\fl{A}+\frac{\fl{N}_1}{mt}\Big)^{-1}}\Big(\|\fl{V}'\|_\s{F} + \lr{\frac{\fl{N}_2}{mt}}_\s{F} + \|\fl{\Xi}\|_\s{F}\Big). \label{eq:lemma-optimize_dp-lrs-U-Fnorm-bound-type2-exp2}
    \end{align}
    We can compute the above following similar lines as before.
    \begin{align}
        \E{\fl{V}'} &= \frac{1}{t}\sum_{i \in [t]}\Big(\fl{U}^{\star}\fl{w}^{\star(i)} + (\fl{b}^{\star(i)} - \fl{b}^{(i, \ell)})\Big)(\fl{w}^{(i, \ell)})^{\s{T}}\nonumber\\
        &= \frac{1}{t}\sum_{i \in [t]}\Big(\fl{U}^{\star}\fl{w}^{\star(i)}(\fl{w}^{(i, \ell)})^{\s{T}} + (\fl{b}^{\star(i)} - \fl{b}^{(i, \ell)})(\fl{w}^{(i, \ell)})^{\s{T}}\Big)\nonumber\\
        &= \frac{1}{t}\fl{U}^{\star}(\fl{W}^{\star})^\s{T}\fl{W}^{(\ell)} + \frac{1}{t}\sum_{i \in [t]}(\fl{b}^{\star(i)} - \fl{b}^{(i, \ell)})(\fl{w}^{(i, \ell)})^{\s{T}}.\nonumber
    \end{align}
    Using Lemma~\ref{lemma:Fnorm-helper1} with $\fl{R}^{(i, \ell)} = \Big(\fl{U}^{\star}\fl{w}^{\star(i)} + (\fl{b}^{\star(i)} - \fl{b}^{(i, \ell)})\Big)(\fl{w}^{(i, \ell)})^{\s{T}}$, we have with probability at least $1-\delta_0$
    \begin{align}
        &\sum_{p}\sum_q (\fl{V}'_{p,q})^2 \nonumber\\ 
        &\leq 2\sum_p\sum_q \Big(\frac{1}{t}\sum_i \fl{R}^{(i, \ell)}_{p,q}\Big)^2 + 16\|\fl{R}^{(i, \ell)}\|_\s{F}^2\frac{d\log (rd/\delta_0)}{mt} \nonumber\\
        &\leq 2\sum_p\sum_q \Big(\frac{1}{t}\sum_i\Big(\fl{U}^{\star}\fl{w}^{\star(i)} + (\fl{b}^{\star(i)} - \fl{b}^{(i, \ell)})\Big)(\fl{w}^{(i, \ell)})^{\s{T}}\Big)_{p,q}^2 + 16\|\fl{R}^{(i, \ell)}\|_\s{F}^2\frac{d\log (rd/\delta_0)}{mt} \nonumber\\
        &\leq 2\sum_p\sum_q \frac{1}{t^2}\Big(\fl{U}^{\star}(\fl{W}^\star)^\s{T}\fl{W}^{(\ell)} + \sum_i (\fl{b}^{\star(i)} - \fl{b}^{(i, \ell)})(\fl{w}^{(i, \ell)})^{\s{T}}\Big)_{p,q}^2++ 16\|\fl{R}^{(i, \ell)}\|_\s{F}^2\frac{d\log (rd/\delta_0)}{mt} \nonumber\\
        &\leq \frac{4}{t^2}\sum_p\sum_q \Big(\fl{U}^{\star}(\fl{W}^\star)^\s{T}\fl{W}^{(\ell)}\Big)_{p,q}^2 + \frac{4}{t^2}\sum_p\sum_q \Big(\sum_i (\fl{b}^{\star(i)} - \fl{b}^{(i, \ell)})(\fl{w}^{(i, \ell)})^{\s{T}}\Big)_{p,q}^2 \nonumber\\
        &\qquad + + 16\|\fl{R}^{(i, \ell)}\|_\s{F}^2\frac{d\log (rd/\delta_0)}{mt}\nonumber\\
        &\leq \frac{4}{t^2}\|\fl{U}^{\star}(\fl{W}^\star)^\s{T}\fl{W}^{(\ell)}\|_\s{F}^2 + \frac{4}{t}\sum_p\sum_q \sum_i (\fl{b}^{\star(i)}_p - \fl{b}^{(i, \ell)}_p)^2(\fl{w}^{(i, \ell)}_s)^2 + 16\|\fl{R}^{(i, \ell)}\|_\s{F}^2\frac{d\log (rd/\delta_0)}{mt}\nonumber\\
        &\leq \frac{4}{t^2}\|\fl{U}^{\star}(\fl{W}^\star)^\s{T}\fl{W}^{(\ell)}\|_\s{F}^2 + \frac{4}{t}\zeta (\max_i \|\fl{w}^{(i, \ell)}\|^2_2)\|\fl{b}^{\star(i)} - \fl{b}^{(i, \ell)}\|^2_2 + 16\|\fl{R}^{(i, \ell)}\|_\s{F}^2\frac{d\log (rd/\delta_0)}{mt}.\nonumber
    \end{align}
    \begin{align}
        \implies \|\fl{V}'\|_\s{F} &\leq \frac{2}{t}\|\fl{U}^{\star}(\fl{W}^\star)^\s{T}\fl{W}^{(\ell)}\|_\s{F} + \sqrt{\frac{4\zeta}{t}}(\max_i\|\fl{w}^{(i, \ell)}\|_2)\|\fl{b}^{\star(i)} - \fl{b}^{(i, \ell)}\|_2 \nonumber\\
        &\qquad + 4\|\fl{R}^{(i, \ell)}\|_\s{F}\sqrt{\frac{d\log (rd/\delta_0)}{mt}}\label{eq:lemma-optimize_dp-lrs-u-2norm-bound-type2-num-temp}.
    \end{align}
    Since $\fl{R}^{(i, \ell)} = \fl{U}^{\star}\fl{w}^{\star(i)}(\fl{w}^{(i, \ell)})^{\s{T}} + (\fl{b}^{\star(i)} - \fl{b}^{(i, \ell)})(\fl{w}^{(i, \ell)})^{\s{T}} $, we have
    \begin{align}
        \|\fl{R}^{(i, \ell)}\|_{\s{F}} &= \|\fl{U}^{\star}\fl{w}^{\star(i)}(\fl{w}^{(i, \ell)})^{\s{T}} + (\fl{b}^{\star(i)} - \fl{b}^{(i, \ell)})(\fl{w}^{(i, \ell)})^{\s{T}}\|_{\s{F}}\nonumber\\
        &\leq \|\fl{U}^{\star}\fl{w}^{\star(i)}(\fl{w}^{(i, \ell)})^{\s{T}}\|_{\s{F}} + \|(\fl{b}^{\star(i)} - \fl{b}^{(i, \ell)})(\fl{w}^{(i, \ell)})^{\s{T}} \|_\s{F}\nonumber\\
        &\leq \|\fl{U}^{\star}\|\|\fl{w}^{\star(i)}\|_2\|\fl{w}^{(i, \ell)}\|_2 + \|\fl{b}^{\star(i)} - \fl{b}^{(i, \ell)}\|_2\|\fl{w}^{(i, \ell)}\|_2 \label{eq:lemma-optimize_dp-lrs-u-2norm-bound-type2-num-term2}
    \end{align}
    Using \eqref{eq:lemma-optimize_dp-lrs-u-2norm-bound-type2-num-term2} in \eqref{eq:lemma-optimize_dp-lrs-u-2norm-bound-type2-num-temp} gives
    \begin{align}
        \|\fl{V}'\|_\s{F} &\leq \frac{2}{t}\|\fl{U}^{\star}(\fl{W}^{\star})^\s{T}\fl{W}^{(\ell)}\|_\s{F} + \sqrt{\frac{4\zeta}{t}}(\max_i\|\fl{w}^{(i, \ell)}\|_2)\|\fl{b}^{\star(i)} - \fl{b}^{(i, \ell)}\|_2\nonumber\\ 
        &\qquad +4\Big(\|\fl{U}^{\star}\|\|\fl{w}^{\star(i)}\|_2\|\fl{w}^{(i, \ell)}\|_2 + \|\fl{b}^{\star(i)} - \fl{b}^{(i, \ell)}\|_2\|\fl{w}^{(i, \ell)}\|_2\Big)\sqrt{\frac{d\log (rd/\delta_0)}{mt}}\label{eq:lemma-optimize_dp-lrs-u-2norm-bound-type2-num-val}
    \end{align}
    Using \eqref{eq:lemma-optimize_dp-lrs-u-2norm-bound-type2-num-val}. \eqref{eq:lemma-optimize_dp-lrs-U-Fnorm-bound-N2-Fnorm} and \eqref{eq:lemma-optimize_dp-lrs-u-2norm-bound-type1-A-bound} in \eqref{eq:lemma-optimize_dp-lrs-U-Fnorm-bound-type2-exp2} gives $\|\fl{U}^{(\ell)}\|_\s{F}$
    \begin{align}
            &\leq \frac{1}{\frac{1}{r}\lambda_r\Big(\frac{r}{t}(\fl{W}^{(\ell)})^\s{T}\fl{W}^{(\ell)} \Big) - c\sqrt{\sum_{i \in [t]}\sum_{j \in [m]}\|\fl{w}^{(i, \ell)}\|_2^4}\sqrt{\frac{rd\log(1/ \delta_0)}{m^2t^2}}-\frac{\sigma_1}{mt}\Big(2\sqrt{rd} + 4\sqrt{\log rd}\Big)} \cdot\nonumber\\
            &\qquad \Big\{\frac{2}{t}\|\fl{U}^{\star}(\fl{W}^{\star})^\s{T}\fl{W}^{(\ell)}\|_\s{F} + \sqrt{\frac{4\zeta}{t}}(\max_i\|\fl{w}^{(i, \ell)}\|_2)\|\fl{b}^{\star(i)} - \fl{b}^{(i, \ell)}\|_2\nonumber\\ 
            &\qquad +4\Big(\|\fl{U}^{\star}\|\|\fl{w}^{\star(i)}\|_2\|\fl{w}^{(i, \ell)}\|_2 + \|\fl{b}^{\star(i)} - \fl{b}^{(i, \ell)}\|_2\|\fl{w}^{(i, \ell)}\|_2\Big)\sqrt{\frac{d\log (rd/\delta_0)}{mt}} \nonumber\\
            &\qquad + \frac{\sigma_2}{mt}6\sqrt{rd\log(rd)} + \frac{2\sigma\sqrt{d\mu^{\star}\lambda^{\star}_r}\log(2rdmt/\delta_0)}{\sqrt{mt}}\Big\}. \label{eq:lemma-optimize_dp-lrs-U-Fnorm-bound-val}
    \end{align}

    \textbf{Analysis of $\|\Delta(\fl{U}^{+(\ell)}, \fl{U}^{\star})\|_\s{F}$:}
    
    \begin{align}
        \|\Delta(\fl{U}^{+(\ell)}, \fl{U}^{\star})\|_\s{F} &= \|(\fl{I} - \fl{U}^{\star}(\fl{U}^{\star})^{\s{T}})\fl{U}^{+(\ell)}\|_\s{F} \nonumber\\
        &= \min_{\fl{Q}^{+}}\|\fl{U}^{+(\ell)} - \fl{U}^{\star}\fl{Q}^{+}\|_{\s{F}}\nonumber\\
        &= \|\fl{U}^{(\ell)} - \fl{U}^{\star}(\fl{U}^{\star})^\s{T}\fl{U}^{(\ell)}\|_{\s{F}}\|\fl{R}^{-1}\|
        \label{eq:lemma-optimize_dp-lrs-U-Fnorm-bound-Delta-exp1}
    \end{align}
    From \eqref{eq:lemma-optimize_dp-lrs-U-Fnorm-bound-exp1}, we have:
    \begin{align}
        &\s{vec}(\fl{U}^{(\ell)} - \fl{U}^\star\fl{Q}^{(\ell-1)}) =  \Big(\fl{A} + \frac{\fl{N}_1}{mt}\Big)^{-1}\Big(\s{vec}(\fl{V}) + \Big(\s{vec}\Big(\frac{\fl{N}_2}{mt}\Big) -  \frac{\fl{N}_1}{mt}\s{vec}(\fl{U}^\star\fl{Q}^{(\ell-1)})\Big)\Big)\nonumber\\
        &= \Big(\fl{A} + \frac{\fl{N}_1}{mt}\Big)^{-1}\cdot\nonumber\\
        &\qquad \Big(\s{vec}\Big(\frac{1}{mt}\sum_{i \in [t]} (\fl{X}^{(i)})^{\s{T}}\fl{X}^{(i)}\Big(-\fl{U}^{\star}\fl{Q}^{(\ell-1)}\fl{h}^{(i, \ell)} + (\fl{b}^{\star(i)} - \fl{b}^{(i, \ell)})\Big)(\fl{w}^{(i, \ell)})^{\s{T}}\Big)\nonumber \\
        &\qquad +\Big(\s{vec}\Big(\frac{\fl{N}_2}{mt}\Big) -  \frac{\fl{N}_1}{mt}\s{vec}(\fl{U}^\star\fl{Q}^{(\ell-1)})\Big)\Big).\nonumber
    \end{align}
    Note that $\s{vec}^{-1}\Big(\E{\Big(\fl{A} + \frac{\fl{N}_1}{mt}\Big)^{-1}\s{vec}\Big(\frac{1}{mt}\sum_{i \in [t]} (\fl{X}^{(i)})^{\s{T}}\fl{X}^{(i)}\Big(-\fl{U}^{\star}\fl{Q}^{(\ell-1)}\fl{h}^{(i, \ell)}\Big)(\fl{w}^{(i, \ell)})^{\s{T}}\Big)}\Big) = \s{vec}^{-1}\Big(\E{\fl{A}}^{-1}\s{vec}\Big(\frac{-1}{t}\fl{U}^{\star}\fl{Q}^{(\ell-1)}(\fl{H}^{(\ell)})^\s{T}\fl{W}^{(\ell)}\Big)\Big)$ lies in the subspace parallel to $\fl{U}^\star$ and therefore does not contribute to the distance $\|\Delta(\fl{U}^{(\ell)}, \fl{U}^{\star})\|_\s{F}$. Subtracting this in \eqref{eq:lemma-optimize_dp-lrs-U-Fnorm-bound-Delta-exp1}, we get $\|\Delta(\fl{U}^{+(\ell)}, \fl{U}^{\star})\|_\s{F}$
    \begin{align}
        &\leq \|\fl{U}^{(\ell)} - \fl{U}^{\star}\fl{Q}^{(\ell-1)} + \s{vec}^{-1}\Big(\E{\fl{A}}^{-1}\s{vec}\Big(\frac{1}{t}\fl{U}^{\star}\fl{Q}^{(\ell-1)}(\fl{H}^{(\ell)})^\s{T}\fl{W}^{(\ell)}\Big)\Big)\|\fl{R}^{-1}\|\\
        &= \|\s{vec}\Big(\fl{U}^{(\ell)} - \fl{U}^{\star}\fl{Q}^{(\ell-1)}\Big) + \E{\fl{A}}^{-1}\s{vec}\Big(\frac{1}{t}\fl{U}^{\star}\fl{Q}^{(\ell-1)}(\fl{H}^{(\ell)})^\s{T}\fl{W}^{(\ell)}\Big)\|_{\s{2}}\|\fl{R}^{-1}\|\\
        &= \|\Big(\fl{A} + \frac{\fl{N}_1}{mt}\Big)^{-1}\Big(\s{vec}(\fl{V}) + \Big(\s{vec}\Big(\frac{\fl{N}_2}{mt}\Big) -  \frac{\fl{N}_1}{mt}\s{vec}(\fl{U}^\star\fl{Q}^{(\ell-1)})\Big)\Big) \\
        &+ \E{\fl{A}}^{-1}\s{vec}\Big(\frac{1}{t}\fl{U}^{\star}\fl{Q}^{(\ell-1)}(\fl{H}^{(\ell)})^\s{T}\fl{W}^{(\ell)}\Big)\|_2\|\fl{R}^{-1}\|\\
        &= \|\fl{R}^{-1}\|\Big\{\Big\|\Big(\fl{A} + \frac{\fl{N}_1}{mt}\Big)^{-1}\Big(\s{vec}(\fl{V}) + \s{vec}\Big(\frac{\fl{N}_2}{mt}\Big) -  \frac{\fl{N}_1}{mt}\s{vec}(\fl{U}^\star\fl{Q}^{(\ell-1)})\Big) \\
        &\qquad + \Big(\fl{A} + \frac{\fl{N}_1}{mt}\Big)\E{\fl{A}}^{-1}\s{vec}\Big(\frac{1}{t}\fl{U}^{\star}\fl{Q}^{(\ell-1)}(\fl{H}^{(\ell)})^\s{T}\fl{W}^{(\ell)}\Big)\Big\|_2\Big\}
        \end{align}
        \begin{align}
        &\leq\Big\{\frac{c\sqrt{\sum_{i \in [t]}\sum_{j \in [m]}\|\fl{w}^{(i, \ell)}\|_2^4}\sqrt{\frac{rd\log(1/ \delta_0)}{m^2t^2}}+\frac{\sigma_1}{mt}\Big(2\sqrt{rd} + 4\sqrt{\log rd}\Big)}{\frac{1}{r}\lambda_r\Big(\frac{r}{t}(\fl{W}^{(\ell)})^\s{T}\fl{W}^{(\ell)} \Big) - c\sqrt{\sum_{i \in [t]}\sum_{j \in [m]}\|\fl{w}^{(i, \ell)}\|_2^4}\sqrt{\frac{rd\log(1/ \delta_0)}{m^2t^2}}-\frac{\sigma_1}{mt}\Big(2\sqrt{rd} + 4\sqrt{\log rd}\Big)}\cdot \nonumber\\
        &\qquad \Big(\frac{2}{t}\|\fl{U}^{\star}\fl{Q}^{(\ell-1)}(\fl{H}^{(\ell)})^\s{T}\fl{W}^{(\ell)}\|_\s{F}\Big) \nonumber\\
        &+ \frac{1}{\frac{1}{r}\lambda_r\Big(\frac{r}{t}(\fl{W}^{(\ell)})^\s{T}\fl{W}^{(\ell)} \Big) - c\sqrt{\sum_{i \in [t]}\sum_{j \in [m]}\|\fl{w}^{(i, \ell)}\|_2^4}\sqrt{\frac{rd\log(1/ \delta_0)}{m^2t^2}}-\frac{\sigma_1}{mt}\Big(2\sqrt{rd} + 4\sqrt{\log rd}\Big)}\cdot\nonumber\\ 
        &\qquad \Big\{\Big(\sqrt{\frac{4\zeta}{t}}(\max_i\|\fl{w}^{(i, \ell)}\|_2)\|\fl{b}^{\star(i)} - \fl{b}^{(i, \ell)}\|_2\Big) \nonumber\\
        &\qquad + 4\Big(\|\fl{U}^{\star}\fl{Q}^{(\ell-1)}\|\|\fl{h}^{(i, \ell)}\|_2\|\fl{w}^{(i, \ell)}\|_2 + \|\fl{b}^{\star(i)} - \fl{b}^{(i, \ell)}\|_2\|\fl{w}^{(i, \ell)}\|_2\Big)\sqrt{\frac{d\log (rd/\delta_0)}{mt}}\nonumber\\
        &\qquad \frac{\sigma_2}{mt}6\sqrt{rd\log(rd)} + \frac{\sigma_1}{mt}\Big(2\sqrt{rd} + 4\sqrt{\log rd}\Big)\sqrt{r} + \frac{2\sigma\sqrt{d\mu^{\star}\lambda^{\star}_r}\log(2rdmt/\delta_0)}{\sqrt{mt}}\Big\}\Big\} \|\fl{R}^{-1}\|.
    \end{align}
\end{proof}

\begin{coro}\label{inductive-corollary:optimize_dp-lrs-U-F-norm-bounds}
    If $\s{B}_{\fl{U}^{(\ell-1)}} = \cO\Big(\frac{1}{\sqrt{r\mu^\star}}\Big)$, $\sqrt{\frac{r\log(1/ \delta_0)}{m}} = \cO(1)$, $\sqrt{\nu^{(\ell-1)}} = \cO\Big(\frac{1}{\sqrt{r\mu^\star}}\Big)$, $\sqrt{\frac{r^2\log(r/\delta_0)}{m}} = \cO\Big(\frac{1}{\sqrt{\mu^\star}}\Big)$, $\epsilon < \sqrt{\mu^\star\lambda_r^\star}$, $\sqrt{\frac{r^2\zeta}{t}} = \cO\Big(\frac{1}{\sqrt{\mu^\star}}\Big)$, $\Lambda' = \cO\Big(\sqrt{\frac{\lambda_r^\star}{\lambda_1^\star}}\Big)$, $\Lambda = \cO(\sqrt{\lambda_r^\star})$, $\sigma\sqrt{\frac{r^2\log^2 (r\delta^{-1})}{m}} = \cO(\sqrt{\lambda_r^\star})$, $\sqrt{\frac{r^3d\log(1/ \delta_0)}{mt}} = \min\{\cO\Big(\frac{1}{\mu^\star}\Big), \cO\Big(\frac{1}{\mu^\star\lambda_r^\star}\Big)\}$, $\frac{\sigma_1}{mt}\Big(2\sqrt{rd} + 4\sqrt{\log rd}\Big) = \min\{\cO\Big(\frac{\lambda_r^\star}{r}\Big), \cO\Big(\frac{1}{r}\Big)\}$, $mt=\widetilde{\Omega}(dr^2
    \mu^{\star}(1\\ + \frac{1}{\lambda_r^{\star}}))$, $\zeta=\widetilde{O}(t(\mu^{\star} \lambda^{\star}_r)^{-1})$, $m=\widetilde{\Omega}(\sigma^2r^3/\lambda^{\star}_r)$ and Assumption \ref{assum:init} holds for iteration $\ell-1$, then, with probability $1-\cO(\delta_0)$,
    \begin{align}
        &\|\fl{U}^{(\ell)} - \fl{U}^{\star}\fl{Q}^{(\ell-1)}\|_\s{F} \nonumber\\ 
        &= \cO\Big(\s{B}_{\fl{U^{(\ell-1)}}}\sqrt{\frac{\lambda_r^\star}{\lambda_1^\star}}\Big) +  \cO\Big(\Lambda' + \frac{\Lambda}{\sqrt{\mu^\star\lambda_r^\star}} + \frac{\sigma_2r}{mt\lambda_r^\star}\sqrt{rd\log(rd)} + \frac{\sigma_1r^{3/2}}{mt\lambda_r^\star}\sqrt{rd\log rd} \nonumber\\
        &\qquad + \sigma\Big(\sqrt{\frac{r^3d\mu^\star\log^2 (r\delta^{-1})}{mt\lambda_r^\star}}+\sqrt{\frac{r^3\log^2(r\delta^{-1})}{m\lambda^{\star}_r}}\Big)\Big),\nonumber
    \end{align}
    \begin{align}
        \text{ and } &\lr{\Delta(\fl{U}^{+(\ell)}, \fl{U}^{\star})}_\s{F}\nonumber\\  
        &\leq  \cO\Big(\|\fl{R}^{-1}\|\s{B}_{\fl{U^{(\ell-1)}}}\sqrt{\frac{\lambda_r^\star}{\lambda_1^\star}}\Big) \nonumber\\
        &+ \resizebox{0.92\textwidth}{!}{$\cO\Big(\|\fl{R}^{-1}\|\Big\{\frac{\Lambda'\sqrt{\lambda^{\star}_r}}{r} + \frac{\Lambda}{r} + \frac{\sigma_2r}{mt\lambda_r^\star}\sqrt{rd\log(rd)} +
        \frac{\sigma_1r\sqrt{r}}{mt\lambda_r^\star}\sqrt{rd\log rd} + \sigma\Big(\sqrt{\frac{r^3d\mu^\star\log^2 (rdmt/\delta_0)}{mt\lambda_r^\star}}\Big)\Big\}\Big)$.}\nonumber
    \end{align}
\end{coro}
\begin{proof}
     The proof follows from plugging the various constant bounds from the corollary statement and Inductive Assumption \ref{assum:init} in the expressions of Lemma~\ref{lemma:lemma-optimize_dp-lrs-U-Fnorm-bound}. Note that, $\|\fl{U}^{(\ell)} - \fl{U}^{\star}\fl{Q}^{(\ell-1)}\|_\s{F}$
    \begin{align}
        &\leq \frac{1}{\frac{1}{r}\lambda_r - c\mu\lambda_r\sqrt{\frac{rd\log(1/ \delta_0)}{mt}}-\frac{\sigma_1}{mt}\Big(2\sqrt{rd} + 4\sqrt{\log rd}\Big)} \cdot\nonumber\\
        &\qquad \Big\{\frac{2}{t}\|\fl{U}^{\star}\fl{Q}^{(\ell-1)}\|\|(\fl{H}^{(\ell)})^\s{T}\fl{W}^{(\ell)}\|_\s{F} + \sqrt{\frac{4\zeta}{t}}(\max_i\|\fl{w}^{(i, \ell)}\|_2\|\fl{b}^{\star(i)} - \fl{b}^{(i, \ell)}\|_2)\nonumber\\ 
        &\qquad +4\Big(\|\fl{U}^{\star}\fl{Q}^{(\ell-1)}\|\|\fl{h}^{(i, \ell)}\|_2\|\fl{w}^{(i, \ell)}\|_2 + \|\fl{b}^{\star(i)} - \fl{b}^{(i, \ell)}\|_2\|\fl{w}^{(i, \ell)}\|_2\Big)\sqrt{\frac{d\log (rd/\delta_0)}{mt}}\nonumber\\
        &\qquad + \frac{\sigma_2}{mt}6\sqrt{rd\log(rd)} + \frac{\sigma_1}{mt}\Big(2\sqrt{rd} + 4\sqrt{\log rd}\Big)\sqrt{r} + \frac{2\sigma\sqrt{d\mu^{\star}\lambda^{\star}_r}\log(2rdmt/\delta_0)}{\sqrt{mt}}\Big\}. 
    \end{align}
    Using Assumption~\ref{assum:init} for $\fl{b}^{(i, \ell)}$ and $\fl{Q}^{(\ell-1)}$ terms, the fact that $\fl{U}^\star$ is orthonormal and eigenvalue ratios and incoherence bounds for $\fl{H}^{(\ell)}$ and $\fl{W}^{(\ell)}$ from Corollaries \ref{inductive-corollary:optimize_dp-lrs-h,H-bounds} and \ref{inductive-inductive-corollary:optimize_dp-lrs-w-incoherence}, the above becomes, $\|\fl{U}^{(\ell)} - \fl{U}^{\star}\fl{Q}^{(\ell-1)}\|_\s{F}$
    \begin{align}
        &\leq \frac{1}{\frac{1}{r}\lambda_r - c\mu\lambda_r\sqrt{\frac{rd\log(1/ \delta_0)}{mt}}-\frac{\sigma_1}{mt}\Big(2\sqrt{rd} + 4\sqrt{\log rd}\Big)} \cdot\nonumber\\
        & \resizebox{0.97\textwidth}{!}{$\Big\{\frac{2}{t}\Big(\cO\Big(\frac{\s{B}_{\fl{U^{(\ell-1)}}}\frac{\lambda_r^\star}{\lambda_1^\star}\sqrt{\frac{t}{r}\lambda^\star_1}}{\sqrt{r\mu^\star}}\Big) + \cO\Big(\Lambda'\sqrt{\frac{t}{r}\lambda^\star_r}\frac{1}{\sqrt{r\mu^\star}}\Big) + \cO\Big(\frac{1}{\sqrt{r\mu^\star}}\sqrt{\frac{t}{r}}\Lambda\Big) + \cO\Big(\sigma\sqrt{\frac{rt\log^2 (r\delta^{-1})}{m}}\Big)\Big)\cdot$} \nonumber\\
        &\qquad \cO(\sqrt{t\mu^\star\lambda_r^\star}) + \sqrt{\frac{4\zeta}{t}}\cdot \cO(\sqrt{\mu^\star\lambda_r^\star})\cdot \Big(c'\sqrt{\mu^\star\lambda_r^\star}\s{B}_{\fl{U^{(\ell-1)}}}\sqrt{\frac{\lambda_r^\star}{\lambda_1^\star}} + \Lambda\Big)\nonumber
        \end{align}
        \begin{align}
        &+4\sqrt{\frac{d\log (rd/\delta_0)}{mt}}\cO(\sqrt{\mu^\star\lambda_r^\star})\Big(\cO\Big(\frac{\max\{\epsilon, \|\fl{w}^{\star(i)}\|_2\}\s{B}_{\fl{U^{(\ell-1)}}}\sqrt{\frac{\lambda_r^\star}{\lambda_1^\star}}}{\sqrt{r\mu^\star}}\Big) + \cO\Big(\frac{\Lambda'\|\fl{w}^{\star(i)}\|_2}{\sqrt{r\mu^\star}}\Big)\nonumber\\
        &\qquad + \cO\Big(\frac{\Lambda}{\sqrt{r\mu^\star}}\Big) + \cO\Big(\sigma\sqrt{\frac{r\log^2 (r\delta^{-1})}{m}}\Big) + c'\max\{\epsilon, \|\fl{w}^{\star(i)}\|_2\}\s{B}_{\fl{U^{(\ell-1)}}}\sqrt{\frac{\lambda_r^\star}{\lambda_1^\star}} + \Lambda\Big)\nonumber\\
        & + \frac{\sigma_2}{mt}6\sqrt{rd\log(rd)} + \frac{\sigma_1}{mt}\Big(2\sqrt{rd} + 4\sqrt{\log rd}\Big)\sqrt{r}+ \frac{2\sigma\sqrt{d\mu^{\star}\lambda^{\star}_r}\log(2rdmt/\delta_0)}{\sqrt{mt}}\Big\}\nonumber\\
        &= J_1 + J_2\label{eq:inductive-corollary-optimize_dp-lrs-DeltaU-F-norm-bounds-unplugged}
    \end{align}
    where $J_1$ denotes the terms which arise from analysing the problem in the noiseless setting and $J_2$ denotes the contribution of noise terms ($\sigma_1, \sigma_2, \sigma, \Lambda, \Lambda'$). We will analyse both separately. Note that $J_1$
    \begin{align}
        &= \frac{1}{\frac{1}{r}\lambda_r - \frac{\mu}{\mu^\star}\cdot\frac{c\mu^\star\lambda_r}{r}\cdot\sqrt{\frac{r^3d\log(1/ \delta_0)}{mt}}-\frac{\lambda_r}{r}\cdot\frac{\lambda_r^\star}{\lambda_r}\cdot\frac{\sigma_1}{mtr\lambda_r^\star}\Big(2\sqrt{rd} + 4\sqrt{\log rd}\Big)} \cdot\nonumber\\
        &\Big\{\frac{2}{t}\cdot \cO\Big(\frac{\s{B}_{\fl{U^{(\ell-1)}}}\frac{\lambda_r^\star}{\lambda_1^\star}\sqrt{\frac{t}{r}\lambda^\star_1}}{\sqrt{r\mu^\star}}\cdot \sqrt{t\mu^\star\lambda_r^\star}\Big)+ \cO\Big(\sqrt{\frac{\zeta}{t}}\cdot \sqrt{\mu^\star\lambda_r^\star}\cdot \sqrt{\mu^\star\lambda_r^\star}\s{B}_{\fl{U^{(\ell-1)}}}\sqrt{\frac{\lambda_r^\star}{\lambda_1^\star}}\Big)\nonumber\\
        &+ \resizebox{0.93\textwidth}{!}{$\cO\Big(\sqrt{\frac{d\log (rd/\delta_0)}{mt}}\cdot\sqrt{\mu^\star\lambda_r^\star}\cdot\Big(\frac{\max\{\epsilon, \|\fl{w}^{\star(i)}\|_2\}\s{B}_{\fl{U^{(\ell-1)}}}\sqrt{\frac{\lambda_r^\star}{\lambda_1^\star}}}{\sqrt{r\mu^\star}} + \max\{\epsilon, \|\fl{w}^{\star(i)}\|_2\}\s{B}_{\fl{U^{(\ell-1)}}}\sqrt{\frac{\lambda_r^\star}{\lambda_1^\star}}\Big)\Big)\Big\}.$}\nonumber
        \end{align}
        Using $\s{B}_{\fl{U}^{(\ell-1)}} =  \cO\Big(\frac{1}{\sqrt{r\mu^\star}}\Big)$, $\lambda_r^\star \leq \lambda_1^\star$, $r\geq1$,$\mu^\star \geq 1$, the bracketed expression in the above simplifies to
        \begin{align}
        &\leq \frac{\s{B}_{\fl{U^{(\ell-1)}}}\sqrt{\frac{\lambda_r^\star}{\lambda_1^\star}}}{\frac{1}{r}\lambda_r - \frac{\mu}{\mu^\star}\frac{c\mu^\star\lambda_r}{r}\sqrt{\frac{r^3d\log(1/ \delta_0)}{mt}}-\frac{\lambda_r}{r}\cdot\frac{\lambda_r^\star}{\lambda_r}\cdot\frac{\sigma_1}{mtr\lambda_r^\star}\Big(2\sqrt{rd} + 4\sqrt{\log rd}\Big)} \cdot\nonumber\\
        &\qquad \cO\Big(\frac{\lambda_r^\star}{r} + \sqrt{\frac{\zeta}{t}}\mu^\star\lambda_r^\star + \sqrt{\frac{d\log (rd/\delta_0)}{mt}}\mu^\star\lambda_r^\star\Big(\frac{1}{\sqrt{r\mu^\star}} + 1\Big)\Big).\nonumber
    \end{align}
    Further, using $\sqrt{\frac{r^3d\log(1/ \delta_0)}{mt}} = \cO(\frac{1}{\mu^\star})$, $\frac{\sigma_1}{mtr}\Big(2\sqrt{rd} + 4\sqrt{\log rd}\Big) = \cO(\lambda_r^\star)$, $\sqrt{\frac{r^2\zeta}{t}} = \cO\Big(\frac{1}{\mu^\star}\Big)$, and eigenvalue and incoherence ratios from Corollary~\ref{inductive-inductive-corollary:optimize_dp-lrs-w-incoherence} in the above, we have
    \begin{align}
        J_1 &= \frac{\s{B}_{\fl{U^{(\ell-1)}}}\sqrt{\frac{\lambda_r^\star}{\lambda_1^\star}}}{1 - \cO(1)} \cdot\cO\Big(\frac{1}{\sqrt{r\mu^\star}} + 1\Big)\nonumber\\
        &= \cO\Big(\s{B}_{\fl{U^{(\ell-1)}}}\sqrt{\frac{\lambda_r^\star}{\lambda_1^\star}}\Big).\label{eq:inductive-corollary-optimize_dp-lrs-DeltaU-F-norm-bounds-J1-value}
    \end{align}
    Similarly, we have $J_2$ 
    \begin{align}
        &= \frac{1}{\frac{1}{r}\lambda_r - \frac{\mu}{\mu^\star}\frac{c\mu^\star\lambda_r}{r}\sqrt{\frac{r^3d\log(1/ \delta_0)}{mt}}-\frac{\lambda_r}{r}\cdot\frac{\lambda_r^\star}{\lambda_r}\cdot\frac{\sigma_1}{mtr\lambda_r^\star}\Big(2\sqrt{rd} + 4\sqrt{\log rd}\Big)} \cdot\nonumber\\
        &\qquad \Big\{\frac{2}{t}\cO\Big(\sqrt{\frac{t}{r}\lambda^\star_r}\frac{\Lambda'}{\sqrt{r\mu^\star}} + \frac{1}{\sqrt{r\mu^\star}}\sqrt{\frac{t}{r}}\Lambda+\sigma\sqrt{\frac{rt\log^2 (r\delta^{-1})}{m}}\Big)\cdot \cO(\sqrt{t\mu^\star\lambda_r^\star})\nonumber\\ 
        &\qquad + \resizebox{0.92\textwidth}{!}{$\cO\Big(\sqrt{\frac{\zeta}{t}}\cdot \sqrt{\mu^\star\lambda_r^\star}\Lambda\Big)
        +\cO\Big(\sqrt{\frac{d\log (rd/\delta_0)}{mt}}\sqrt{\mu^\star\lambda_r^\star}\Big(\frac{\Lambda'\|\fl{w}^{\star(i)}\|_2}{\sqrt{r\mu^\star}} + \Lambda + \frac{1}{\sqrt{r\mu^\star}}\sigma\sqrt{\frac{r\log^2 (r\delta^{-1})}{m}}\Big)\Big)$} \nonumber\\
        &\qquad +\frac{\sigma_2}{mt}6\sqrt{rd\log(rd)} + \frac{\sigma_1}{mt}\Big(2\sqrt{rd} + 4\sqrt{\log rd}\Big)\sqrt{r} + \frac{2\sigma\sqrt{d\mu^{\star}\lambda^{\star}_r}\log(2rdmt/\delta_0)}{\sqrt{mt}}\Big\}.\nonumber
        \end{align}
        Now, using $\sqrt{\frac{r^3d\log(1/ \delta_0)}{mt}} = \cO(\frac{1}{\mu^\star})$, $\frac{\sigma_1}{mtr}\Big(2\sqrt{rd} + 4\sqrt{\log rd}\Big) = \cO(\lambda_r^\star)$, eigenvalue and incoherence ratios from Corollary~\ref{inductive-inductive-corollary:optimize_dp-lrs-w-incoherence} for the denominator term and $\sqrt{\frac{r^2\zeta}{t}} = \cO\Big(\frac{1}{\mu^\star}\Big)$, $\lambda_r^\star \leq \lambda_1^\star$, $r\geq1$,$\mu^\star \geq 1$ for the bracketed terms in the above, we get $J_2$ 
        \begin{align}
        &= \frac{1}{1 - \cO(1)}\Big\{\cO\Big(\Lambda' + \frac{\Lambda}{\sqrt{\lambda_r^\star}} +\sigma \sqrt{\frac{r^3\mu^\star\log^2 (r\delta^{-1})}{m\lambda_r^\star}}\Big) + \cO\Big(\sqrt{\frac{\zeta}{t}}\cdot \sqrt{\mu^\star\lambda_r^\star}\Lambda\Big)\nonumber\\ 
        &\qquad + \cO\Big(\sqrt{\frac{r^2d\log (rd/\delta_0)}{mt}}\sqrt{\mu^\star}\Big(\frac{\Lambda'}{\sqrt{r}} + \frac{\Lambda}{\sqrt{\lambda_r^\star}} +\frac{\sigma}{\sqrt{r\mu^\star\lambda_r^\star}}\sqrt{\frac{r\log^2 (r\delta^{-1})}{m}}\Big)\Big)\nonumber\\
        &\qquad +\frac{\sigma_2r}{mt\lambda_r^\star}6\sqrt{rd\log(rd)} + \frac{\sigma_1r}{mt\lambda_r^\star}\Big(2\sqrt{rd} + 4\sqrt{\log rd}\Big)\sqrt{r}+ \frac{2\sigma r\sqrt{d\mu^{\star}\lambda^{\star}_r}\log(2rdmt/\delta_0)}{\sqrt{mt}\lambda_r^\star}\Big\}.
        \end{align}
        Rearranging the terms in above gives
        \begin{align}
        &= \cO\Big(\Big\{\Lambda'\Big(1 + \sqrt{\frac{r^2d\log (rd/\delta_0)}{mt}}\frac{\sqrt{\mu^\star}}{\sqrt{r}}\Big) + \Lambda\Big(\frac{1}{\sqrt{\lambda_r^\star}} +\sqrt{\frac{\zeta}{t}}\cdot \sqrt{\mu^\star\lambda_r^\star} + \sqrt{\frac{r^2d\log (rd/\delta_0)}{mt}}\frac{\sqrt{\mu^\star}}{\sqrt{\lambda_r^\star}}\Big)\nonumber\\
        & +\frac{\sigma_2r}{mt\lambda_r^\star}\sqrt{rd\log(rd)} + \frac{\sigma_1r}{mt\lambda_r^\star}\Big(\sqrt{rd} + \sqrt{\log rd}\Big)\sqrt{r}\nonumber\\
        & + \resizebox{0.97\textwidth}{!}{$\sigma\Big(\sqrt{\frac{r^3\log^2 (r\delta^{-1})}{m\lambda_r^\star}}  + \sqrt{\frac{r^2d\log (rd/\delta_0)}{mt}}\cdot\frac{1}{\sqrt{r\lambda_r^\star}}\sqrt{\frac{r\log^2 (r\delta^{-1})}{m}} + \frac{ r\sqrt{d\mu^{\star}\lambda^{\star}_r}\log(2rdmt/\delta_0)}{\sqrt{mt}\lambda_r^\star}\Big)\Big\}.$}
    \end{align}
    Substituting $mt=\widetilde{\Omega}(dr^2
    \mu^{\star}(1+\frac{1}{\lambda_r^{\star}}))$, $\zeta=\widetilde{O}(t(\mu^{\star} \lambda^{\star}_r)^{-1})$, we have that 
    \begin{align}
        J_2 &= \cO\Big(\Lambda' + \frac{\Lambda}{\sqrt{\mu^\star\lambda_r^\star}} + \frac{\sigma_2r}{mt\lambda_r^\star}\sqrt{rd\log(rd)} + \frac{\sigma_1r^{3/2}}{mt\lambda_r^\star}\sqrt{rd\log rd} \nonumber\\
        &\qquad + \sigma\Big(\sqrt{\frac{r^3d\mu^\star\log^2 (r\delta^{-1})}{mt\lambda_r^\star}}+\sqrt{\frac{r^3\log^2(r\delta^{-1})}{m\lambda^{\star}_r}}\Big)\Big).\label{eq:inductive-corollary-optimize_dp-lrs-DeltaU-F-norm-bounds-J2-value}
    \end{align}
    Using \eqref{eq:inductive-corollary-optimize_dp-lrs-DeltaU-F-norm-bounds-J1-value}, \eqref{eq:inductive-corollary-optimize_dp-lrs-DeltaU-F-norm-bounds-J2-value} in \eqref{eq:inductive-corollary-optimize_dp-lrs-DeltaU-F-norm-bounds-unplugged} gives us the required norm bound for $\|\fl{U}^{(\ell)} - \fl{U}^{\star}\fl{Q}^{(\ell-1)}\|_\s{F}$.
    
    Similarly, we can simplify the following $\lr{\Delta(\fl{U}^{+(\ell)}, \fl{U}^{\star})}_\s{F}$ bound expression from the Lemma statement.
    \begin{align}
     &\leq \Big\{\frac{c\mu\lambda_r\sqrt{\frac{rd\log(1/ \delta_0)}{mt}}+\frac{\sigma_1}{mt}\Big(2\sqrt{rd} + 4\sqrt{\log rd}\Big)}{\frac{1}{r}\lambda_r - c\mu\lambda_r\sqrt{\frac{rd\log(1/ \delta_0)}{mt}}-\frac{\sigma_1}{mt}\Big(2\sqrt{rd} + 4\sqrt{\log rd}\Big)}\cdot \Big(\frac{2}{t}\|\fl{U}^{\star}\fl{Q}^{(\ell-1)}(\fl{H}^{(\ell)})^\s{T}\fl{W}^{(\ell)}\|_\s{F}\Big) \nonumber\\
    &\qquad + \frac{1}{\frac{1}{r}\lambda_r - c\mu\lambda_r\sqrt{\frac{rd\log(1/ \delta_0)}{mt}}-\frac{\sigma_1}{mt}\Big(2\sqrt{rd} + 4\sqrt{\log rd}\Big)}\cdot\nonumber\\ 
    &\qquad \Big\{\Big(\sqrt{\frac{4\zeta}{t}}(\max_i\|\fl{w}^{(i, \ell)}\|_2)\|\fl{b}^{\star(i)} - \fl{b}^{(i, \ell)}\|_2\Big) + \nonumber\\ 
    & 4\Big(\|\fl{U}^{\star}\fl{Q}^{(\ell-1)}\|\|\fl{h}^{(i, \ell)}\|_2\|\fl{w}^{(i, \ell)}\|_2 + \|\fl{b}^{\star(i)} - \fl{b}^{(i, \ell)}\|_2\|\fl{w}^{(i, \ell)}\|_2\Big)\sqrt{\frac{d\log (rd/\delta_0)}{mt}}\nonumber\\
    &\qquad \frac{\sigma_2}{mt}6\sqrt{rd\log(rd)} + \frac{\sigma_1}{mt}\Big(2\sqrt{rd} + 4\sqrt{\log rd}\Big)\sqrt{r} + \frac{2\sigma\sqrt{d\mu^{\star}\lambda^{\star}_r}\log(2rdmt/\delta_0)}{\sqrt{mt}}\Big\}\Big\}\|\fl{R}^{-1}\|
    \end{align}
    Using the stated bounds on $\sqrt{\frac{r^3d\log(1/ \delta_0)}{mt}}$, $\frac{\sigma_1}{mtr}\Big(2\sqrt{rd} + 4\sqrt{\log rd}\Big)$, Corollary~\ref{inductive-inductive-corollary:optimize_dp-lrs-w-incoherence} as well as $\fl{H}^{(\ell)}$ and $\fl{W}^{(\ell)}$ bounds from Corollaries \ref{inductive-corollary:optimize_dp-lrs-h,H-bounds} and \ref{inductive-inductive-corollary:optimize_dp-lrs-w-incoherence} in the above, we have
    \begin{align}
    &\leq \|\fl{R}^{-1}\|\Big\{\frac{\cO(1)}{1 - \cO(1)}\frac{1}{\lambda_r^\star}\cdot \frac{2}{t}\cO\Big(\Big(\frac{\s{B}_{\fl{U^{(\ell-1)}}}\frac{\lambda_r^\star}{\lambda_1^\star}\sqrt{\frac{t}{r}\lambda^\star_1}}{\sqrt{r\mu^\star}} + \sqrt{\frac{t}{r}\lambda^\star_r}\frac{\Lambda'}{\sqrt{r\mu^\star}} \nonumber\\
    &\qquad + \frac{1}{\sqrt{r\mu^\star}}\sqrt{\frac{t}{r}}\Lambda + \sigma\sqrt{\frac{r\log^2 (r\delta^{-1})}{m}}\Big)\sqrt{t\mu^\star\lambda_r^\star}\Big) + \frac{1}{1 - \cO(1)}\frac{r}{\lambda_r}\cdot \nonumber\\
    &\qquad \Big\{\sqrt{\frac{4\zeta}{t}}\cdot \cO(\sqrt{\mu^\star\lambda_r^\star})\Big( c'\max\{\epsilon, \|\fl{w}^{\star(i)}\|_2\}\s{B}_{\fl{U^{(\ell-1)}}}\sqrt{\frac{\lambda_r^\star}{\lambda_1^\star}} + \Lambda\Big) \nonumber\\
    &\qquad + 4\sqrt{\frac{d\log (rd/\delta_0)}{mt}}\cdot\cO(\sqrt{\mu^\star\lambda_r^\star})\cdot\cO\Big(\frac{\max\{\epsilon, \|\fl{w}^{\star(i)}\|_2\}\s{B}_{\fl{U^{(\ell-1)}}}\sqrt{\frac{\lambda_r^\star}{\lambda_1^\star}}}{\sqrt{r\mu^\star}} + \frac{\Lambda'\|\fl{w}^{\star(i)}\|_2}{\sqrt{r\mu^\star}}\nonumber\\
    &\qquad + \frac{\Lambda}{\sqrt{r\mu^\star}} +  \sigma\sqrt{\frac{r\log^2 (r\delta^{-1})}{m}} +  \max\{\epsilon, \|\fl{w}^{\star(i)}\|_2\}\s{B}_{\fl{U^{(\ell-1)}}}\sqrt{\frac{\lambda_r^\star}{\lambda_1^\star}} + \Lambda\Big)\nonumber\\
    &\qquad+\frac{\sigma_2}{mt}6\sqrt{rd\log(rd)} + \frac{\sigma_1}{mt}\Big(2\sqrt{rd} + 4\sqrt{\log rd}\Big)\sqrt{r} + \frac{2\sigma\sqrt{d\mu^{\star}\lambda^{\star}_r}\log(2rdmt/\delta_0)}{\sqrt{mt}}\Big\}\Big\}\\
    &= J_1' + J_2'\label{eq:inductive-corollary-optimize_dp-lrs-DeltaU-F-norm-bounds-unplugged-2}
    \end{align}
    where as before, $J_1'$ denotes the terms which arise from analysing the problem in the noiseless setting and $J_2'$ denotes the contribution of noise terms ($\sigma_1, \sigma_2, \sigma, \Lambda, \Lambda'$). Analysing both the terms separately, we have  $J_1'$
    \begin{align}
        &= \|\fl{R}^{-1}\|\Big\{\cO(1)\frac{1}{\lambda_r^\star} \cdot \frac{2}{t}\cO\Big(\frac{\s{B}_{\fl{U^{(\ell-1)}}}\frac{\lambda_r^\star}{\lambda_1^\star}\sqrt{\frac{t}{r}\lambda^\star_1}}{\sqrt{r\mu^\star}}\Big)\sqrt{t}\cO(\sqrt{\mu^\star\lambda_r^\star})\nonumber\\
        &\qquad + \cO\Big(\frac{r}{\lambda_r}\Big)\Big\{\sqrt{\frac{4\zeta}{t}}\cdot \cO(\sqrt{\mu^\star\lambda_r^\star})\cdot \max\{\epsilon, \|\fl{w}^{\star(i)}\|_2\}\s{B}_{\fl{U^{(\ell-1)}}}\sqrt{\frac{\lambda_r^\star}{\lambda_1^\star}}\nonumber\\
        &\qquad + 4\sqrt{\frac{d\log (rd/\delta_0)}{mt}}\cdot\cO(\sqrt{\mu^\star\lambda_r^\star})\cdot \cO\Big(\frac{\max\{\epsilon, \|\fl{w}^{\star(i)}\|_2\}\s{B}_{\fl{U^{(\ell-1)}}}\sqrt{\frac{\lambda_r^\star}{\lambda_1^\star}}}{\sqrt{r\mu^\star}}\Big)\Big\}\Big\}\\
        &= \cO\Big(\|\fl{R}^{-1}\|\s{B}_{\fl{U^{(\ell-1)}}}\sqrt{\frac{\lambda_r^\star}{\lambda_1^\star}}\Big\{ \frac{1}{r} + \Big\{\sqrt{\frac{r^2\zeta}{t}}\mu^\star + \sqrt{\frac{r^2d\log (rd/\delta_0)}{mt}}\cdot\frac{\mu^\star}{\sqrt{r\mu^\star}}\Big\}\Big\}\Big)\nonumber\\
        &= \cO\Big(\|\fl{R}^{-1}\|\s{B}_{\fl{U^{(\ell-1)}}}\sqrt{\frac{\lambda_r^\star}{\lambda_1^\star}}\Big),\label{eq:inductive-corollary-optimize_dp-lrs-DeltaU-F-norm-bounds-J1'-value}
    \end{align}
    where in the last two steps we use the fact that $\lambda_r^\star \leq \lambda_1^\star$, $r\geq1$,$\mu^\star \geq 1$, $\sqrt{\frac{r^2\zeta}{t}} = \cO\Big(\frac{1}{\mu^\star}\Big)$ and $\sqrt{\frac{r^3d\log(1/ \delta_0)}{mt}} = \cO(\frac{1}{\mu^\star})$. Similarly we have $J_2'$
    \begin{align}
        &= \|\fl{R}^{-1}\|\Big\{\cO(1)\cdot\frac{2}{t}\cO\Big(\sqrt{\frac{t}{r}\lambda^\star_r}\frac{\Lambda'}{\sqrt{r\mu^\star}} + \frac{1}{\sqrt{r\mu^\star}}\sqrt{\frac{t}{r}}\Lambda + \sigma\sqrt{\frac{r\log^2 (r\delta^{-1})}{m}}\Big)\cdot\sqrt{t}\cO(\sqrt{\mu^\star\lambda_r^\star})\nonumber\\
        &\qquad + \cO\Big(\frac{r}{\lambda_r}\Big)\cdot \Big\{\sqrt{\frac{4\zeta}{t}}\cdot \cO(\sqrt{\mu^\star\lambda_r^\star})\Lambda \nonumber\\
        &\qquad + 4\sqrt{\frac{d\log (rd/\delta_0)}{mt}}\cdot\cO(\sqrt{\mu^\star\lambda_r^\star})\cO\Big(\frac{\Lambda'\|\fl{w}^{\star(i)}\|_2}{\sqrt{r\mu^\star}} + \frac{\Lambda}{\sqrt{r\mu^\star}} + \sigma\sqrt{\frac{r\log^2 (r\delta^{-1})}{m}} + \Lambda\Big)\nonumber\\
        &\qquad + \frac{\sigma_2}{mt}6\sqrt{rd\log(rd)} + \frac{\sigma_1}{mt}\Big(2\sqrt{rd} + 4\sqrt{\log rd}\Big)\sqrt{r} + \frac{2\sigma\sqrt{d\mu^{\star}\lambda^{\star}_r}\log(2rdmt/\delta_0)}{\sqrt{mt}}\Big\}\Big\}.\nonumber
    \end{align}
    Rearranging the terms in above gives
    \begin{align}
        &= \cO\Big(\|\fl{R}^{-1}\|\Big\{\Lambda'\Big(\frac{\sqrt{\lambda_r^\star}}{r} + \sqrt{\frac{r^2d\log (rd/\delta_0)}{mt}}\cdot\frac{\mu^\star}{\sqrt{r\mu^\star}}\Big) \nonumber\\
        &\qquad + \Lambda\Big(\frac{1}{r} +  \sqrt{\frac{r^2\zeta}{t}}\cdot \frac{\sqrt{\mu^\star\lambda_r^\star}}{\lambda_r^\star}  + \sqrt{\frac{d\log (rd/\delta_0)}{mt}}\frac{\sqrt{\mu^\star\lambda_r^\star}}{\lambda_r^\star}\cdot\Big(1+\frac{1}{\sqrt{r\mu^\star}}\Big) \Big)\nonumber\\
        &\qquad + \frac{\sigma_2r}{mt\lambda_r^\star}\sqrt{rd\log(rd)} + \frac{\sigma_1r}{mt\lambda_r^\star}\Big(\sqrt{rd} + \sqrt{\log rd}\Big)\sqrt{r}\nonumber\\
        &\qquad + \sigma\frac{r\sqrt{\mu^\star\lambda_r^\star}}{\lambda_r^\star}\Big(\sqrt{\frac{r\log^2 (r\delta^{-1})}{mt}}+ \sqrt{\frac{d\log (rd/\delta_0)}{mt}} \sqrt{\frac{r\log^2 (r\delta^{-1})}{m}} + \frac{\sqrt{d}\log(2rdmt/\delta_0)}{\sqrt{mt}}\Big)
        \Big\}\Big).
    \end{align}
    Assuming $mt=\widetilde{\Omega}(dr^2
    \mu^{\star}(1+\frac{1}{\lambda_r^{\star}}))$, $mt=\widetilde{\Omega}\Big(\frac{\sqrt{dr^3}}{\lambda^{\star}_r}\max(\sigma_1\sqrt{r},\sigma_2)\Big)$, $m=\widetilde{\Omega}(\sigma^2r^3/\lambda^{\star}_r)$, we have the above as
    \begin{align}
        \resizebox{0.98\textwidth}{!}{$J_2' = \cO\Big(\|\fl{R}^{-1}\|\Big\{\frac{\Lambda'\sqrt{\lambda^{\star}_r}}{r} + \frac{\Lambda}{r} + \frac{\sigma_2r}{mt\lambda_r^\star}\sqrt{rd\log(rd)} +
        \frac{\sigma_1r\sqrt{r}}{mt\lambda_r^\star}\sqrt{rd\log rd} + \sigma\Big(\sqrt{\frac{r^3d\mu^\star\log^2 (rdmt/\delta_0)}{mt\lambda_r^\star}}\Big)\Big\}\Big)$}\label{eq:inductive-corollary-optimize_dp-lrs-DeltaU-F-norm-bounds-J2'-value}
    \end{align}
    Using \eqref{eq:inductive-corollary-optimize_dp-lrs-DeltaU-F-norm-bounds-J1'-value} and \eqref{eq:inductive-corollary-optimize_dp-lrs-DeltaU-F-norm-bounds-J2'-value} in \eqref{eq:inductive-corollary-optimize_dp-lrs-DeltaU-F-norm-bounds-unplugged-2} gives us the required bound for $\lr{\Delta(\fl{U}^{+(\ell)}, \fl{U}^{\star})}_\s{F}$.
\end{proof}

\begin{coro}\label{inductive-corollary:optimize_dp-lrs-R-bounds}
    If $\frac{1}{2} \leq \sigma_{\min}(\fl{Q}^{(\ell-1)}) \leq \sigma_{\max}(\fl{Q}^{(\ell-1)}) < 1$, $\frac{2r}{\lambda_r^\star}\Big(\frac{\sigma_2}{mt}6\sqrt{rd\log(rd)} + \frac{\sigma_1}{mt}\Big(2\sqrt{rd} + 4\sqrt{\log rd}\Big)\sqrt{r} + \frac{2\sigma\sqrt{d\mu^{\star}\lambda^{\star}_r}\log(2rdmt/\delta_0)}{\sqrt{mt}}\Big) = \cO(1)$ then $\fl{R}$ is invertible and $\|\fl{R}^{-1}\| \leq 2 + c''$ and $\lr{\fl{R}} \leq 1+c''$ for some $c'' > 0$ w.p. $1-\cO(\delta_0)$.
\end{coro}
\begin{proof}
    $\forall$ $\fl{z} \in \bR^r$, we have:
    \begin{align}
        \|\fl{U}^{+(\ell)}\fl{R}\fl{z}\|_2 &= \|\fl{U}^{(\ell)}\fl{z}\|\nonumber\\
        &= \|\fl{U}^{\star}\fl{Q}^{(\ell-1)}\fl{z} - \s{vec}^{-1}\Big(\E{\fl{A}}^{-1}\s{vec}\Big(\frac{1}{t}\fl{U}^{\star}\fl{Q}^{(\ell-1)}(\fl{H}^{(\ell)})^\s{T}\fl{W}^{(\ell)}\Big)\Big) \nonumber\\
        &\qquad + \Big(\fl{U}^{(\ell)} - \fl{U}^{\star}\fl{Q}^{(\ell-1)} + \s{vec}^{-1}\Big(\E{\fl{A}}^{-1}\s{vec}\Big(\frac{1}{t}\fl{U}^{\star}\fl{Q}^{(\ell-1)}(\fl{H}^{(\ell)})^\s{T}\fl{W}^{(\ell)}\Big)\Big)\Big)\fl{z}\Big\|_2\label{eq:lemma-inductive-corollary-R-bounds-1}
    \end{align}   
    Using \eqref{eq:lemma-inductive-corollary-R-bounds-1}, we have:
    \begin{align}
        &\min_{\|\fl{z}\|_2 = 1}\|\fl{U}^{+(\ell)}\fl{R}\fl{z}\|_2 \nonumber\\ 
        &\geq \resizebox{0.97\textwidth}{!}{$\min_{\|\fl{z}\|_2 = 1}\sqrt{\fl{z}^\s{T}(\fl{Q}^{(\ell-1)})^\s{T}(\fl{U}^{\star})^\s{T}\fl{U}^{\star}\fl{Q}^{(\ell-1)}\fl{z}} - \|\s{vec}^{-1}\Big(\E{\fl{A}}^{-1}\s{vec}\Big(\frac{1}{t}\fl{U}^{\star}\fl{Q}^{(\ell-1)}(\fl{H}^{(\ell)})^\s{T}\fl{W}^{(\ell)}\Big)\Big)\|$} \nonumber\\
        &\qquad - \|\fl{U}^{(\ell)} - \fl{U}^{\star}\fl{Q}^{(\ell-1)} + \s{vec}^{-1}\Big(\E{\fl{A}}^{-1}\s{vec}\Big(\frac{1}{t}\fl{U}^{\star}\fl{Q}^{(\ell-1)}(\fl{H}^{(\ell)})^\s{T}\fl{W}^{(\ell)}\Big)\Big)\|\nonumber\\
        &\geq \min_{\|\fl{z}\|_2 = 1}\sqrt{\fl{z}^\s{T}(\fl{Q}^{(\ell-1)})^\s{T}\fl{Q}^{(\ell-1)}\fl{z}} - \|\s{vec}^{-1}\Big(\E{\fl{A}}^{-1}\s{vec}\Big(\frac{1}{t}\fl{U}^{\star}\fl{Q}^{(\ell-1)}(\fl{H}^{(\ell)})^\s{T}\fl{W}^{(\ell)}\Big)\Big)\|_\s{F}\nonumber\\
        &\qquad - \|\fl{U}^{(\ell)} - \fl{U}^{\star}\fl{Q}^{(\ell-1)} - \s{vec}^{-1}\Big(\E{\fl{A}}^{-1}\s{vec}\Big(\frac{1}{t}\fl{U}^{\star}\fl{Q}^{(\ell-1)}(\fl{H}^{(\ell)})^\s{T}\fl{W}^{(\ell)}\Big)\Big)\|_\s{F}\nonumber\\
        &\geq \sigma_{\min}(\fl{Q}^{(\ell-1)}) - \|\E{\fl{A}}^{-1}\|\|\frac{1}{t}\fl{U}^{\star}\fl{Q}^{(\ell-1)}(\fl{H}^{(\ell)})^\s{T}\fl{W}^{(\ell)}\|_\s{F}\nonumber\\
        &\qquad - \|\fl{U}^{(\ell)} - \fl{U}^{\star}\fl{Q}^{(\ell-1)} + \s{vec}^{-1}\Big(\E{\fl{A}}^{-1}\s{vec}\Big(\frac{1}{t}\fl{U}^{\star}\fl{Q}^{(\ell-1)}(\fl{H}^{(\ell)})^\s{T}\fl{W}^{(\ell)}\Big)\Big)\|_\s{F}.\nonumber
    \end{align}
        Using the bounds from Corollary~\ref{inductive-corollary:optimize_dp-lrs-h,H-bounds} and \ref{inductive-corollary:optimize_dp-lrs-U-F-norm-bounds} and Inductive Assumption \ref{assum:init} in the above we get,
    \begin{align}
        &\geq \frac{1}{2} - \frac{r}{t\lambda_r}\sqrt{t\mu\lambda_r}\cdot \cO\Big(\frac{\s{B}_{\fl{U^{(\ell-1)}}}\frac{\lambda_r^\star}{\lambda_1^\star}\sqrt{\frac{t}{r}\lambda^\star_1}}{\sqrt{r\mu^\star}} + \sqrt{\frac{t}{r}\lambda^\star_r}\frac{\Lambda'}{\sqrt{r\mu^\star}} + \frac{\Lambda}{\sqrt{r\mu^\star}}\sqrt{\frac{t}{r}} + \sigma\sqrt{\frac{r\log^2 (r\delta^{-1})}{m}}\Big)\nonumber\\
        &\qquad - \cO\Big(\s{B}_{\fl{U^{(\ell-1)}}}\sqrt{\frac{\lambda_r^\star}{\lambda_1^\star}}+\Lambda' + \frac{\Lambda}{\sqrt{\mu^\star\lambda_r^\star}} + \frac{\sigma_2r}{mt\lambda_r^\star}\sqrt{rd\log(rd)}\nonumber \\  
        &\qquad + \frac{\sigma_1r\sqrt{r}}{mt\lambda_r^\star}\sqrt{rd\log rd} + \sigma\sqrt{\frac{r^3d\mu^\star\log^2 (rdmt/\delta_0)}{mt\lambda_r^\star}}\Big)\label{eq:inductive-corollary-optimize_dp-lrs-R-bounds-temp1}.
    \end{align}
    Rearranging the terms in the above gives
    \begin{align}
        &\geq \frac{1}{2} - \cO\Big(\s{B}_{\fl{U^{(\ell-1)}}}\Big(\sqrt{\frac{\mu}{\mu^\star}}\sqrt{\frac{\lambda_r}{\lambda_r^\star}} +1\Big) + \Lambda'\Big(\sqrt{\frac{\mu}{\mu^\star}}\sqrt{\frac{\lambda_r}{\lambda_r^\star}} + 1\Big) + \frac{\Lambda}{\mu^\star\lambda_r^\star}\Big(\sqrt{\frac{\mu}{\mu^\star}}\sqrt{\frac{\lambda_r}{\lambda_r^\star}} + 1\Big) \nonumber\\
        &\qquad + \frac{\sigma_2r}{mt\lambda_r^\star}\sqrt{rd\log(rd)} + \frac{\sigma_1r\sqrt{r}}{mt\lambda_r^\star}\sqrt{rd\log rd}\nonumber\\
        &\qquad + \sigma\Big(\sqrt{\frac{\mu}{\mu^\star}}\sqrt{\frac{\lambda_r}{\lambda_r^\star}}\sqrt{\frac{\mu^\star r^2\log^2 (r\delta^{-1})}{mt\lambda_r^\star}}+\sqrt{\frac{r^3d\mu^\star\log^2 (rdmt/\delta_0)}{mt\lambda_r^\star}}+\sqrt{\frac{r^3\log^2(r\delta^{-1})}{m\lambda^{\star}_r}}\Big)\Big)\nonumber\\
        &\geq \frac{1}{2} - c''.\label{eq:inductive-corollary-optimize_dp-lrs-R-bounds-val1}
    \end{align}
where $c'' > 0$
Here, we need to use that $\Lambda'<10^{-3}, \Lambda <10^{-3}\mu^{\star}\lambda^{\star}_r$, $mt=\Omega(\max(\sigma_2,\sigma_1\sqrt{r})\sqrt{rd\log d}/\lambda_r^{\star})$ and $mt=\widetilde{\Omega}(\sigma^2 dr^3\mu^{\star}/\lambda^{\star}_r)$, $m=\widetilde{\Omega}(\sigma^2r^3/\lambda^{\star}_r)$.    
    
    Similarly, using \eqref{eq:lemma-inductive-corollary-R-bounds-1}, we also have:
    \begin{align}
        &\max_{\|\fl{z}\|_2 = 1}\|\fl{U}^{+(\ell)}\fl{R}\fl{z}\|_2 \nonumber\\ 
        &\leq \max_{\|\fl{z}\|_2 = 1}\sqrt{\fl{z}^\s{T}(\fl{Q}^{(\ell-1)})^\s{T}(\fl{U}^{\star})^\s{T}\fl{U}^{\star}\fl{Q}^{(\ell-1)}\fl{z}}  + \|\s{vec}^{-1}\Big(\E{\fl{A}}^{-1}\s{vec}\Big(\frac{1}{t}\fl{U}^{\star}\fl{Q}^{(\ell-1)}(\fl{H}^{(\ell)})^\s{T}\fl{W}^{(\ell)}\Big)\Big)\| \nonumber\\
        &\qquad + \|\fl{U}^{(\ell)} - \fl{U}^{\star}\fl{Q}^{(\ell-1)} + \s{vec}^{-1}\Big(\E{\fl{A}}^{-1}\s{vec}\Big(\frac{1}{t}\fl{U}^{\star}\fl{Q}^{(\ell-1)}(\fl{H}^{(\ell)})^\s{T}\fl{W}^{(\ell)}\Big)\Big)\|\nonumber\\
        &\leq \max_{\|\fl{z}\|_2 = 1}\sqrt{\fl{z}^\s{T}(\fl{Q}^{(\ell-1)})^\s{T}\fl{Q}^{(\ell-1)}\fl{z}}  + \|\s{vec}^{-1}\Big(\E{\fl{A}}^{-1}\s{vec}\Big(\frac{1}{t}\fl{U}^{\star}\fl{Q}^{(\ell-1)}(\fl{H}^{(\ell)})^\s{T}\fl{W}^{(\ell)}\Big)\Big)\|_\s{F}\nonumber\\
        &\qquad + \|\fl{U}^{(\ell)} - \fl{U}^{\star}\fl{Q}^{(\ell-1)} - \s{vec}^{-1}\Big(\E{\fl{A}}^{-1}\s{vec}\Big(\frac{1}{t}\fl{U}^{\star}\fl{Q}^{(\ell-1)}(\fl{H}^{(\ell)})^\s{T}\fl{W}^{(\ell)}\Big)\Big)\|_\s{F}\nonumber\\
        &\leq \sigma_{\max}(\fl{Q}^{(\ell-1)}) + \|\E{\fl{A}}^{-1}\|\|\frac{1}{t}\fl{U}^{\star}\fl{Q}^{(\ell-1)}(\fl{H}^{(\ell)})^\s{T}\fl{W}^{(\ell)}\|_\s{F}\nonumber\\
        &\qquad + \|\fl{U}^{(\ell)} - \fl{U}^{\star}\fl{Q}^{(\ell-1)} + \s{vec}^{-1}\Big(\E{\fl{A}}^{-1}\s{vec}\Big(\frac{1}{t}\fl{U}^{\star}\fl{Q}^{(\ell-1)}(\fl{H}^{(\ell)})^\s{T}\fl{W}^{(\ell)}\Big)\Big)\|_\s{F}.\nonumber
        \end{align}
        Directly using the bounds of  $\|\E{\fl{A}}^{-1}\|\|\frac{1}{t}\fl{U}^{\star}\fl{Q}^{(\ell-1)}(\fl{H}^{(\ell)})^\s{T}\fl{W}^{(\ell)}\|_\s{F}$ and $\|\fl{U}^{(\ell)} - \fl{U}^{\star}\fl{Q}^{(\ell-1)} + \s{vec}^{-1}\Big(\E{\fl{A}}^{-1}\s{vec}\Big(\frac{1}{t}\fl{U}^{\star}\fl{Q}^{(\ell-1)}(\fl{H}^{(\ell)})^\s{T}\fl{W}^{(\ell)}\Big)\Big)\|_\s{F}$ from the previous calculations as well inductive Assumption~\ref{assum:init}, we get
        \begin{align}
        \max_{\|\fl{z}\|_2 = 1}\|\fl{U}^{+(\ell)}\fl{R}\fl{z}\|_2  &\leq 1 + c''\nonumber\\
        \implies \|\fl{R}\| &\leq \max_{\|\fl{z}\|_2 = 1}\|\fl{R}\fl{z}\|_2  = \max_{\|\fl{z}\|_2 = 1}\|\fl{U}^{+(\ell)}\fl{R}\fl{z}\|_2 \leq 1+c''.\label{eq:inductive-corollary-optimize_dp-lrs-R-bounds-val2} 
    \end{align}
    Bounds \eqref{eq:inductive-corollary-optimize_dp-lrs-R-bounds-val1} and \eqref{eq:inductive-corollary-optimize_dp-lrs-R-bounds-val2} complete the proof.
\end{proof}

\begin{lemma}\label{lemma:optimize_dp-lrs-2,infty-norm}
    \begin{align}
    \|\fl{U}^{(\ell)}\|_{2, \infty}
    &\leq \|\fl{U}^\star\|_{2, \infty}\|\frac{1}{t}\sum_{i \in [t]}\Big(\frac{1}{t}(\fl{W}^{(\ell)})^\s{T}\fl{W}^{(\ell)}\Big)^{-1}\fl{w}^{\star(i)}(\fl{w}^{(i, \ell)})^{\s{T}}\|_{2} \nonumber\\
    &\qquad + c\|\fl{U}^\star\|_\s{F}\|\Big(\frac{1}{t}(\fl{W}^{(\ell)})^\s{T}\fl{W}^{(\ell)}\Big)^{-1}\|\|\fl{w}^{\star(i)}(\fl{w}^{(i, \ell)})^{\s{T}}\|_\s{F}\sqrt{\frac{\log (1/\delta_0)}{mt}}\nonumber\\
    &\qquad + \zeta(\max_i \|\fl{b}^{\star(i)} - \fl{b}^{(i, \ell)}\|_\infty\|\fl{w}^{(i, \ell)}\|_2)\|\Big((\fl{W}^{(\ell)})^\s{T}\fl{W}^{(\ell)}\Big)^{-1}\|_{2} \nonumber\\
    &\qquad + c\|\Big(\frac{1}{t}(\fl{W}^{(\ell)})^\s{T}\fl{W}^{(\ell)}\Big)^{-1}\|_\s{F}\|\fl{b}^{\star(i)} - \fl{b}^{(i, \ell)}\|_2\|\fl{w}^{(i, \ell)}\|_\s{2}\sqrt{\frac{\log (1/\delta_0)}{mt}}\nonumber\\
    &\qquad + \resizebox{0.8\textwidth}{!}{$\frac{2\sigma_2}{mt}\sqrt{\log(rd/\delta_0)}\cdot\frac{r\sqrt{r}}{\lambda_r\Big(\frac{r}{t}(\fl{W}^{(\ell)})^\s{T}\fl{W}^{(\ell)}\Big)} + \frac{2\sigma\sqrt{\mu^{\star}\lambda^{\star}_r}\log(2rdmt/\delta_0)}{\sqrt{mt}}\frac{r}{\lambda_r\Big(\frac{r}{t}(\fl{W}^{(\ell)})^\s{T}\fl{W}^{(\ell)}\Big)}$}\nonumber\\
    &\qquad + \frac{\sqrt{r}\Big(c\sqrt{\sum_{i \in [t]}\sum_{j \in [m]}\|\fl{w}^{(i, \ell)}\|_2^4}\sqrt{\frac{rd\log(rd/ \delta_0)}{m^2t^2}} + \frac{\sigma_1}{mt}\Big(2\sqrt{rd} + 2\sqrt{2rd\log(2rd/\delta_0)}\Big)\Big)}{\frac{1}{r}\lambda_r\Big(\frac{r}{t}(\fl{W}^{(\ell)})^\s{T}\fl{W}^{(\ell)} \Big)}\cdot\nonumber\\
    &\qquad \Big\{\frac{2}{t}\|\fl{U}^{\star}(\fl{W}^{\star})^\s{T}\fl{W}^{(\ell)}\|_\s{F} + \sqrt{\frac{4\zeta}{t}}(\max_i\|\fl{w}^{(i, \ell)}\|_2)\|\fl{b}^{\star(i)} - \fl{b}^{(i, \ell)}\|_2\nonumber\\ 
    &\qquad +4\Big(\|\fl{U}^{\star}\|\|\fl{w}^{\star(i)}\|_2\|\fl{w}^{(i, \ell)}\|_2 + \|\fl{b}^{\star(i)} - \fl{b}^{(i, \ell)}\|_2\|\fl{w}^{(i, \ell)}\|_2\Big)\sqrt{\frac{d\log (rd/\delta_0)}{mt}}\nonumber\\
    &\qquad + \frac{\sigma_2}{mt}6\sqrt{rd\log(rd)} + \frac{2\sigma\sqrt{d\mu^{\star}\lambda^{\star}_r}\log(2rdmt/\delta_0)}{\sqrt{mt}}\Big\}\nonumber
    \end{align}
    and  $\|\fl{U}^{(\ell)} - \fl{U}^{\star}\fl{Q}^{(\ell-1)}\|_{2,\infty}$
    \begin{align}
    &\leq \|\fl{U}^\star\|_{2, \infty}\|\frac{1}{t}\sum_{i \in [t]}\Big(\frac{1}{t}(\fl{W}^{(\ell)})^\s{T}\fl{W}^{(\ell)}\Big)^{-1}\fl{h}^{(i, \ell)}(\fl{w}^{(i, \ell)})^{\s{T}}\|_{2} \nonumber\\
    &\qquad + c\|\fl{U}^\star\|_\s{F}\|\Big(\frac{1}{t}(\fl{W}^{(\ell)})^\s{T}\fl{W}^{(\ell)}\Big)^{-1}\|\|\fl{h}^{(i, \ell)}(\fl{w}^{(i, \ell)})^{\s{T}}\|_\s{F}\sqrt{\frac{\log (1/\delta_0)}{mt}}\nonumber\\
    &\qquad + \zeta(\max_i \|\fl{b}^{\star(i)} - \fl{b}^{(i, \ell)}\|_\infty\|\fl{w}^{(i, \ell)}\|_2)\|\Big((\fl{W}^{(\ell)})^\s{T}\fl{W}^{(\ell)}\Big)^{-1}\|_{2} \nonumber\\
    &\qquad + c\|\Big(\frac{1}{t}(\fl{W}^{(\ell)})^\s{T}\fl{W}^{(\ell)}\Big)^{-1}\|_\s{F}\|\fl{b}^{\star(i)} - \fl{b}^{(i, \ell)}\|_2\|\fl{w}^{(i, \ell)}\|_\s{2}\sqrt{\frac{\log (1/\delta_0)}{mt}}\nonumber\\
    &\qquad + \frac{2\sigma_2}{mt}\sqrt{\log(rd/\delta_0)}\cdot\frac{r\sqrt{r}}{\lambda_r\Big(\frac{r}{t}(\fl{W}^{(\ell)})^\s{T}\fl{W}^{(\ell)}\Big)} + \frac{r^2\sigma_1\sqrt{rd\cdot2\log(2r^2d^2/\delta_0)}}{mt\lambda_r\Big(\frac{r}{t}\fl{W}^{(\ell)})^\s{T}\fl{W}^{(\ell)}\Big)} \nonumber\\
    &\qquad + \frac{2\sigma\sqrt{\mu^{\star}\lambda^{\star}_r}\log(2rdmt/\delta_0)}{\sqrt{mt}}\frac{r}{\lambda_r\Big(\frac{r}{t}(\fl{W}^{(\ell)})^\s{T}\fl{W}^{(\ell)}\Big)}\nonumber\\
    &\qquad + \frac{\sqrt{r}\Big(c\sqrt{\sum_{i \in [t]}\sum_{j \in [m]}\|\fl{w}^{(i, \ell)}\|_2^4}\sqrt{\frac{rd\log(rd/ \delta_0)}{m^2t^2}} + \frac{\sigma_1}{mt}\Big(2\sqrt{rd} + 2\sqrt{2rd\log(2rd/\delta_0)}\Big)\Big)}{\frac{1}{r}\lambda_r\Big(\frac{r}{t}(\fl{W}^{(\ell)})^\s{T}\fl{W}^{(\ell)} \Big)}\cdot\nonumber\\
    &\qquad \Big\{\frac{2}{t}\|\fl{U}^{\star}\fl{Q}^{(\ell-1)}(\fl{H}^{(\ell)})^\s{T}\fl{W}^{(\ell)}\|_\s{F} + \sqrt{\frac{4\zeta}{t}}(\max_i\|\fl{w}^{(i, \ell)}\|_2)\|\fl{b}^{\star(i)} - \fl{b}^{(i, \ell)}\|_2\nonumber\\ 
    &\qquad +4\Big(\|\fl{U}^{\star}\fl{Q}^{(\ell-1)}\|\|\fl{h}^{(i, \ell)}\|_2\|\fl{w}^{(i, \ell)}\|_2 + \|\fl{b}^{\star(i)} - \fl{b}^{(i, \ell)}\|_2\|\fl{w}^{(i, \ell)}\|_2\Big)\sqrt{\frac{d\log (rd/\delta_0)}{mt}}\nonumber\\
    &\qquad + \frac{\sigma_2}{mt}6\sqrt{rd\log(rd)} + \frac{\sigma_1}{mt}\Big(2\sqrt{rd} + 4\sqrt{\log rd}\Big)\sqrt{r} + \frac{2\sigma\sqrt{d\mu^{\star}\lambda^{\star}_r}\log(2rdmt/\delta_0)}{\sqrt{mt}}\Big\}\nonumber
    \end{align}
    w.p. $1 - \cO(\delta_0)$
\end{lemma}
\begin{proof}
    Recall that
    \begin{align}
        \s{vec}(\fl{U}^{(\ell)}) &=  \Big(\underbrace{\fl{A'} + \frac{\fl{N}_1}{mt}}_{\fl{A}}\Big)^{-1}\s{vec}\Big(\fl{V} + \frac{\fl{N}_2}{mt} + \fl{\Xi}\Big) \label{eq:lemma-optimize_dp-lrs-2,infty-norm-bound-type1-exp1}
    \end{align}
    where
    \begin{align}
        \fl{A'}_{rd\times rd} &= \frac{1}{mt}\sum_{i \in [t]}\Big(\fl{w}^{(i, \ell)}(\fl{w}^{(i, \ell)})^{\s{T}} \otimes (\fl{X}^{(i)})^{\s{T}}\fl{X}^{(i)}\Big) \nonumber\\
        &= \frac{1}{mt}\sum_{i \in [t]}\Big(\fl{w}^{(i, \ell)}(\fl{w}^{(i, \ell)})^{\s{T}} \otimes \Big(\sum_{j=1}^{m}\fl{x}^{(i)}_j(\fl{x}^{(i)}_j)^{\s{T}}\Big)\Big),\nonumber\\
        \fl{V}'_{d\times r} &= \frac{1}{mt}\sum_{i \in [t]} (\fl{X}^{(i)})^{\s{T}}\fl{X}^{(i)}\Big(\fl{U}^{\star}\fl{w}^{\star(i)} + (\fl{b}^{\star(i)} - \fl{b}^{(i, \ell)})\Big)(\fl{w}^{(i, \ell)})^{\s{T}}.\nonumber
    \end{align}
    Note that
    \begin{align}
        \E{\fl{A}} &= \frac{1}{t}\sum_{i \in [t]}\Big(\fl{w}^{(i, \ell)}(\fl{w}^{(i, \ell)})^{\s{T}} \otimes \fl{I}\Big) + \vzero\nonumber\\
        &= \frac{1}{t}\cdot(\fl{W}^{(\ell)})^\s{T}\fl{W}^{(\ell)} \otimes \fl{I}.\label{eq:lemma-optimize_dp-lrs-2,infty-norm-EA}
    \end{align}
    Let $\fl{C} := \E{\fl{A}}^{-1}$ and $\fl{A}' = \fl{I} + \fl{E}$. 
    Then, we have:
    \begin{align}
        \|\fl{C}\|_2 &= \|\E{\fl{A}}^{-1}\|_2\nonumber\\
        &\leq \frac{1}{\lambda_{\min}\Big(\frac{1}{t}\cdot(\fl{W}^{(\ell)})^\s{T}\fl{W}^{(\ell)} \otimes \fl{I}\Big)}\nonumber\\
        &\leq \frac{1}{\frac{1}{r}\lambda_r\Big(\frac{r}{t}(\fl{W}^{(\ell)})^\s{T}\fl{W}^{(\ell)} \Big)}.\label{eq:lemma-optimize_dp-lrs-2,infty-norm-Cnorm}
    \end{align}
    and
    \begin{align}
        \fl{A}^{-1} &= \Big(\E{\fl{A}}\E{\fl{A}}^{-1}\fl{A}\Big)^{-1}\nonumber\\
        &= \Big(\E{\fl{A}}\Big(\fl{I} + \E{\fl{A}}^{-1}\fl{E} + \E{\fl{A}}^{-1}\frac{\fl{N}_1}{mt}\Big)\Big)^{-1}\nonumber\\
        &= \Big(\fl{I} + \E{\fl{A}}^{-1}\fl{E} + \E{\fl{A}}^{-1}\frac{\fl{N}_1}{mt}\Big)^{-1}\E{\fl{A}}^{-1}\nonumber\\
        &= \Big(\fl{I} + \fl{C}\fl{E} + \frac{\fl{C}\fl{N}_1}{mt}\Big)^{-1}\fl{C}\nonumber\\
        &= \sum_{p=0}^\infty(-1)^p\Big(\fl{C}\fl{E} + \frac{\fl{C}\fl{N}_1}{mt}\Big)^p\fl{C}\nonumber\\
        &= \fl{C} + \sum_{p=1}^\infty(-1)^p\Big(\fl{C}\fl{E} + \frac{\fl{C}\fl{N}_1}{mt}\Big)^p\fl{C}\nonumber
    \end{align}
    since $\|\fl{C}\fl{E}\| < 1$. 
    Now, let $\fl{\cZ} \triangleq \{\fl{z} \in \bR^{rd} | \|\fl{z}\|_2 = 1\}$. Then for $\epsilon \leq 1$, there exists an $\epsilon$-net, $N_\epsilon \subset \fl{\cZ}$, of size $(1 + 2/\epsilon)^
    {rd}$ w.r.t the Euclidean norm, i.e. $\forall$  $\fl{z} \in \fl{\cZ}$, $\exists$ $\fl{z}' \in N_\epsilon$ s.t. $\|\fl{z} - \fl{z}'\|_2 \leq \epsilon$. We will now bound $\lr{\Big(\fl{C}\fl{E} + \frac{\fl{C}\fl{N}_1}{mt}\Big)\fl{z}}_\infty$ $\forall$ $\fl{z} \in \fl{Z}$.
    Now consider any $\fl{z}^\s{T} = \begin{bmatrix}\fl{z}^\s{T}_1, \fl{z}^\s{T}_2, \dots, \fl{z}^\s{T}_r \end{bmatrix} \in N_\epsilon$ where each $\fl{z}_i \in \bR^{d}$. Then for $s$-th standard basis vector $\fl{e}_s \in \bR^{rd}$, we have
    using Lemma~\ref{lemma:a^TEb-concentration}
    \begin{align}
        |\fl{e}_s^\s{T}\fl{C}\fl{E}\fl{z}| &\leq c\|\fl{C}^\s{T}\fl{e}_s\|_2\|\fl{z}\|_2\sqrt{\sum_{i \in [t]}\sum_{j \in [m]}\|\fl{w}^{(i, \ell)}\|_2^4\frac{\log (1/\delta_0)}{m^2t^2}}\nonumber\\
        &\leq c\|\fl{C}\|_2\sqrt{\sum_{i \in [t]}\sum_{j \in [m]}\|\fl{w}^{(i, \ell)}\|_2^4\frac{\log (1/\delta_0)}{m^2t^2}}
        \\
        \text{and } \Big|\fl{e}_s^\s{T}\frac{\fl{C}\fl{N}_1}{mt}\fl{z}\Big| &\leq \|\fl{C}^\s{T}\fl{e}_s\|_2\|\frac{\fl{N}_1}{mt}\|_2\|\fl{z}\|_2\nonumber\\
        &\leq \|\fl{C}\|_2\frac{\sigma_1}{mt}\Big(2\sqrt{rd} + 2\sqrt{2\log(1/\delta_0)}\Big).\nonumber
    \end{align}
    \begin{align}
        &\implies  \lr{\fl{C}\Big(\fl{E} + \frac{\fl{N}_1}{mt}\Big)\fl{z}}_\infty\nonumber\\
        & \leq \|\fl{C}\|_2\Big(c\sqrt{\sum_{i \in [t]}\sum_{j \in [m]}\|\fl{w}^{(i, \ell)}\|_2^4}\sqrt{\frac{\log (|N_\epsilon|/\delta_0)}{m^2t^2}} + \frac{\sigma_1}{mt}\Big(2\sqrt{rd} + \sqrt{2\log(2|N_\epsilon|/\delta_0)}\Big)\Big)\nonumber\\
        &\leq \resizebox{0.95\hsize}{!}{$\|\fl{C}\|_2\Big(c\sqrt{\sum_{i \in [t]}\sum_{j \in [m]}\|\fl{w}^{(i, \ell)}\|_2^4}\sqrt{\frac{\log ((1+2/\epsilon)^{rd}/\delta_0)}{m^2t^2}} + \frac{\sigma_1}{mt}\Big(2\sqrt{rd} + \sqrt{2\log(2(1+2/\epsilon)^{rd}/\delta_0)}\Big)\Big)$}.\nonumber
    \end{align}
    Further, $\exists$ $\fl{z} \in N_\epsilon$ s.t. $\|\fl{z}' - \fl{z}\|_2 \leq \epsilon$. This implies that setting $\epsilon \gets 1/4$ and $c \gets 2c\sqrt{\log(9)}$ gives $\lr{\fl{C}\Big(\fl{E} + \frac{\fl{N}_1}{mt}\Big)\fl{z}'}_\infty$
    \begin{align}
        &\leq \lr{\fl{C}\Big(\fl{E} + \frac{\fl{N}_1}{mt}\Big)(\fl{z}-\fl{z}')}_\infty + \lr{\fl{C}\Big(\fl{E} + \frac{\fl{N}_1}{mt}\Big)\fl{z}}_\infty\nonumber\\
        &\leq \lr{\fl{C}\Big(\fl{E} + \frac{\fl{N}_1}{mt}\Big)(\fl{z}-\fl{z}')}_2 + \lr{\fl{C}\Big(\fl{E} + \frac{\fl{N}_1}{mt}\Big)\fl{z}}_\infty\nonumber\\
        &\leq \lr{\fl{C}\Big(\fl{E} + \frac{\fl{N}_1}{mt}\Big)}_2\epsilon + \lr{\fl{C}\Big(\fl{E} + \frac{\fl{N}_1}{mt}\Big)\fl{z}}_\infty\nonumber\\
        &\leq \|\fl{C}\|\cdot \frac{1}{4} \cdot \Big(c\sqrt{\sum_{i \in [t]}\sum_{j \in [m]}\|\fl{w}^{(i, \ell)}\|_2^4}\sqrt{\frac{rd\log(rd/ \delta_0)}{m^2t^2}} + \frac{\sigma_1}{mt}\Big(2\sqrt{rd} + 4\sqrt{rd\log(rd/ \delta_0)}\Big) \Big) \nonumber\\
        &\qquad + \|\fl{C}\|_2\Big(c\sqrt{\sum_{i \in [t]}\sum_{j \in [m]}\|\fl{w}^{(i, \ell)}\|_2^4}\sqrt{\frac{rd\log(rd/ \delta_0)}{m^2t^2}} + \frac{\sigma_1}{mt}\Big(2\sqrt{rd} + 4\sqrt{rd\log(rd/ \delta_0)}\Big)\Big)\nonumber\\
        &\leq \|\fl{C}\|_2\Big(c\sqrt{\sum_{i \in [t]}\sum_{j \in [m]}\|\fl{w}^{(i, \ell)}\|_2^4}\sqrt{\frac{rd\log(rd/ \delta_0)}{m^2t^2}} + \frac{\sigma_1}{mt}\Big(2\sqrt{rd} + 4\sqrt{rd\log(rd/ \delta_0)}\Big)\Big).\label{eq:lemma-optimize_dp-lrs-2,infty-norm-CECN1norm}
    \end{align}
 with probability at least $1-\delta_0$ where we use \eqref{eq:lemma-optimize_dp-lrs-u-2norm-bound-type1-A-E-value}, \eqref{eq:lemma-optimize_dp-lrs-U-Fnorm-bound-N1-norm} and the fact that $\|\fl{M}\fl{N}\|_2 \leq \|\fl{M}\|_2\|\fl{N}\|_2$. Hence, with probability at least $1-O(\delta_0)$, we have $\lr{\fl{CE}}_2$ and $\lr{\fl{CE}\fl{z}}_{\infty}$ for all $\fl{z}\in \ca{Z}$. Therefore, let us condition on these events in order to prove the next steps. We will now show an upper bound on $\lr{\fl{A}^{-1}\s{vec}\Big(\fl{V}' + \frac{\fl{N}_2}{mt}\Big)}_\infty$. Note that
$\lr{\fl{A}^{-1}\s{vec}\Big(\fl{V}' + \frac{\fl{N}_2}{mt} + \fl{\Xi}\Big)}_\infty$
\begin{align}
    &= \lr{\fl{C}\s{vec}\Big(\fl{V}' + \frac{\fl{N}_2}{mt}+ \fl{\Xi}\Big) + \sum_{p=1}^\infty(-1)^p\Big(\fl{C}\fl{E}+\frac{\fl{C}\fl{N}_1}{mt}\Big)^p\fl{C}\s{vec}\Big(\fl{V}' + \frac{\fl{N}_2}{mt} + \fl{\Xi}\Big)}_\infty\nonumber\\
    &\leq \lr{\fl{C}\s{vec}\Big(\fl{V}' + \frac{\fl{N}_2}{mt} + \fl{\Xi}\Big)}_\infty + \sum_{p=1}^\infty\lr{\Big(\fl{C}\fl{E} + \frac{\fl{C}\fl{N}_1}{mt}\Big)^p\fl{C}\s{vec}\Big(\fl{V}' + \frac{\fl{N}_2}{mt} + \fl{\Xi}\Big)}_\infty.\label{eq:lemma-optimize_dp-lrs-2,infty-norm-AinvV-infty-unplugged}
\end{align}
We have with probability at least $1-\delta_0$, $\sum_{p=1}^\infty
    \lr{\Big(\fl{C}\fl{E} + \frac{\fl{C}\fl{N}_1}{mt}\Big)^p\s{vec}\Big(\fl{V}' + \frac{\fl{N}_2}{mt} + \fl{\Xi}\Big)}_\infty$
\begin{align}
     &\resizebox{0.95\hsize}{!}{$= \sum_{p=1}^\infty \lr{\Big(\fl{C}\fl{E} + \frac{\fl{C}\fl{N}_1}{mt}\Big)\cdot\lr{\Big(\fl{C}\fl{E} + \frac{\fl{C}\fl{N}_1}{mt}\Big)^{p-1}\s{vec}\Big(\fl{V}' + \frac{\fl{N}_2}{mt} + \fl{\Xi}\Big)}_2\cdot\frac{\Big(\fl{C}\fl{E} + \frac{\fl{C}\fl{N}_1}{mt}\Big)^{p-1}\s{vec}\Big(\fl{V}' + \frac{\fl{N}_2}{mt} + \fl{\Xi}\Big)}{\lr{\Big(\fl{C}\fl{E} + \frac{\fl{C}\fl{N}_1}{mt}\Big)^{p-1}\s{vec}\Big(\fl{V}' + \frac{\fl{N}_2}{mt} + \fl{\Xi}\Big)}_2}}_\infty$}\nonumber\\
    &= \resizebox{0.95\hsize}{!}{$\sum_{p=1}^\infty\lr{\Big(\fl{C}\fl{E} + \frac{\fl{C}\fl{N}_1}{mt}\Big)^{p-1}\s{vec}\Big(\fl{V}' + \frac{\fl{N}_2}{mt} + \fl{\Xi}\Big)}_2\lr{\Big(\fl{C}\fl{E}+\frac{\fl{C}\fl{N}_1}{mt}\Big)\cdot\frac{\Big(\fl{C}\fl{E} + \frac{\fl{C}\fl{N}_1}{mt}\Big)^{p-1}\s{vec}\Big(\fl{V}' + \frac{\fl{N}_2}{mt} + \fl{\Xi}\Big)}{\lr{\Big(\fl{C}\fl{E} + \frac{\fl{C}\fl{N}_1}{mt}\Big)^{p-1}\s{vec}\Big(\fl{V}' + \frac{\fl{N}_2}{mt} + \fl{\Xi}\Big)}_2}}_\infty$}\nonumber\\
    &\leq \sum_{p=1}^\infty\lr{\fl{C}\fl{E}+\frac{\fl{C}\fl{N}_1}{mt}}^{p-1}\lr{\s{vec}\Big(\fl{V}' + \frac{\fl{N}_2}{mt} + \fl{\Xi}\Big)}_2\cdot \nonumber\\
    &\qquad \|\fl{C}\|_2\Big(c\sqrt{\sum_{i \in [t]}\sum_{j \in [m]}\|\fl{w}^{(i, \ell)}\|_2^4}\sqrt{\frac{rd\log(rd/ \delta_0)}{m^2t^2}} + \frac{\sigma_1}{mt}\Big(2\sqrt{rd} + 4\sqrt{rd\log(rd/\delta_0)}\Big)\Big)\label{eq:lemma-optimize_dp-lrs-2,infty-norm-temp2}\\
    &\leq \sum_{p=1}^\infty \Big(\frac{c\sqrt{\sum_{i \in [t]}\sum_{j \in [m]}\|\fl{w}^{(i, \ell)}\|_2^4}\sqrt{\frac{rd\log(rd/ \delta_0)}{m^2t^2}}+\frac{\sigma_1}{mt}\Big(2\sqrt{rd} + 4\sqrt{\log rd}\Big)}{\frac{1}{r}\lambda_r\Big(\frac{r}{t}(\fl{W}^{(\ell)})^\s{T}\fl{W}^{(\ell)} \Big)}\Big)^{p-1}\nonumber\\
    &\qquad \|\fl{C}\|_2\Big(c\sqrt{\sum_{i \in [t]}\sum_{j \in [m]}\|\fl{w}^{(i, \ell)}\|_2^4}\sqrt{\frac{rd\log(1/ \delta_0)}{m^2t^2}} \nonumber\\
    &\qquad + \frac{\sigma_1}{mt}\Big(2\sqrt{rd} + 2\sqrt{2rd\log(2/\delta_0)}\Big)\Big)\lr{\s{vec}\Big(\fl{V}' + \frac{\fl{N}_2}{mt} + \fl{\Xi}\Big)}_2\nonumber\\
    &\leq \sum_{p=1}^\infty \Big(c'\sqrt{\frac{r^3d\log(rd/ \delta_0)}{mt}} + c''\frac{\sigma_1}{mt\lambda_r}\sqrt{r^3d\log(rd/ \delta_0)}\Big)^{p-1}\cdot\nonumber\\
    &\qquad \|\fl{C}\|_2\Big(c\sqrt{\sum_{i \in [t]}\sum_{j \in [m]}\|\fl{w}^{(i, \ell)}\|_2^4}\sqrt{\frac{rd\log(1/ \delta_0)}{m^2t^2}} \nonumber\\
    &\qquad + \frac{\sigma_1}{mt}\Big(2\sqrt{rd} + 2\sqrt{2rd\log(2/\delta_0)}\Big)\Big)\lr{\s{vec}\Big(\fl{V}' + \frac{\fl{N}_2}{mt} + \fl{\Xi}\Big)}_2,\label{eq:lemma-optimize_dp-lrs-2,infty-norm-Ainv-infty-temp}
\end{align}    
where in \eqref{eq:lemma-optimize_dp-lrs-2,infty-norm-temp2} we use \eqref{eq:lemma-optimize_dp-lrs-2,infty-norm-CECN1norm} and in \eqref{eq:lemma-optimize_dp-lrs-2,infty-norm-Ainv-infty-temp} the fact that $mt=\Omega(\max\{r^3d\log (rd/\delta_0), \frac{\sigma_1}{\lambda_r}\sqrt{r^3d\log(rd/ \delta_0)}\})$. By taking a union bound we must have with probability at least $1-\delta_0$, $\sum_{p=1}^{\infty}\lr{\Big(\fl{C}\fl{E} + \frac{\fl{C}\fl{N}_1}{mt}\Big)^p\s{vec}\Big(\fl{V}' + \frac{\fl{N}_2}{mt} + \fl{\Xi}\Big)}_{\infty}$
\begin{align}\label{eq:lemma-optimize_dp-lrs-2,infty-uinfty-unplugged-temp}
    &= \cO\Big(\|\fl{C}\|_2\Big(c\sqrt{\sum_{i \in [t]}\sum_{j \in [m]}\|\fl{w}^{(i, \ell)}\|_2^4}\sqrt{\frac{rd\log(1/ \delta_0)}{m^2t^2}}+ \frac{\sigma_1}{mt}\Big(2\sqrt{rd} \nonumber\\
    &\qquad + 2\sqrt{2rd\log(2/\delta_0)}\Big)\Big)\lr{\s{vec}\Big(\fl{V}' + \frac{\fl{N}_2}{mt} + \fl{\Xi}\Big)}_2\Big).
\end{align}
Using the above and \eqref{eq:lemma-optimize_dp-lrs-2,infty-norm-Cnorm} in \eqref{eq:lemma-optimize_dp-lrs-2,infty-norm-AinvV-infty-unplugged}, we have $\lr{\fl{A}^{-1}\s{vec}\Big(\fl{V}' + \frac{\fl{N}_2}{mt}\Big)}_\infty$
\begin{align}
    &\leq \|\fl{C}\s{vec}\Big(\fl{V}' + \frac{\fl{N}_2}{mt} \fl{\Xi}\Big)\|_\infty  + \|\fl{C}\|_2\Big(c\sqrt{\sum_{i \in [t]}\sum_{j \in [m]}\|\fl{w}^{(i, \ell)}\|_2^4}\sqrt{\frac{rd\log(1/ \delta_0)}{m^2t^2}} \nonumber\\
    &\qquad + \frac{\sigma_1}{mt}\Big(2\sqrt{rd} + 2\sqrt{2rd\log(2/\delta_0)}\Big)\Big)\lr{\s{vec}\Big(\fl{V}' + \frac{\fl{N}_2}{mt} + \fl{\Xi}\Big)}_2\nonumber
    \end{align}
    \begin{align}
    &\leq \lr{\fl{C}\s{vec}\Big(\fl{V}' + \frac{\fl{N}_2}{mt} + \fl{\Xi}\Big)}_\infty \nonumber\\
    &\qquad + \frac{\Big(c\sqrt{\sum_{i \in [t]}\sum_{j \in [m]}\|\fl{w}^{(i, \ell)}\|_2^4}\sqrt{\frac{rd\log(rd/ \delta_0)}{m^2t^2}} + \frac{\sigma_1}{mt}\Big(2\sqrt{rd} + 2\sqrt{2rd\log(2rd/\delta_0)}\Big)\Big)}{\frac{1}{r}\lambda_r\Big(\frac{r}{t}(\fl{W}^{(\ell)})^\s{T}\fl{W}^{(\ell)} \Big)}\cdot\nonumber\\
    &\qquad \lr{\fl{V}' + \frac{\fl{N}_2}{mt} + \fl{\Xi}}_\s{F}.\label{eq:lemma-optimize_dp-lrs-2,infty-norm-AinvV-infty-unplugged-2}
\end{align}
Similarly, we also have the following bound $\|\fl{U}^{(\ell)}\|_{2, \infty} = \lr{\s{vec}^{-1}\Big(\fl{A}^{-1}\s{vec}\Big(\fl{V}' + \frac{\fl{N}_2}{mt} + \fl{\Xi} \Big)\Big)}_{2, \infty}$
\begin{align}
    &\leq \lr{\s{vec}^{-1}\Big(\fl{C}\s{vec}\Big(\fl{V}' + \frac{\fl{N}_2}{mt} + \fl{\Xi}\Big)\Big)}_{2, \infty} \nonumber\\
    &\qquad + \frac{\sqrt{r}\Big(c\sqrt{\sum_{i \in [t]}\sum_{j \in [m]}\|\fl{w}^{(i, \ell)}\|_2^4}\sqrt{\frac{rd\log(rd/ \delta_0)}{m^2t^2}} + \frac{\sigma_1}{mt}\Big(2\sqrt{rd} + 2\sqrt{2rd\log(2rd/\delta_0)}\Big)\Big)}{\frac{1}{r}\lambda_r\Big(\frac{r}{t}(\fl{W}^{(\ell)})^\s{T}\fl{W}^{(\ell)} \Big)}\cdot\nonumber\\
    &\qquad \lr{\fl{V}' + \frac{\fl{N}_2}{mt} + \fl{\Xi}}_\s{F}.\label{eq:lemma-optimize_dp-lrs-2,infty-norm-AinvV-2,infty-unplugged-2}
\end{align}
Now $\fl{C}\s{vec}(\fl{V}')$
\begin{align}
    &= \Big(\frac{1}{t}(\fl{W}^{(\ell)})^\s{T}\fl{W}^{(\ell)} \otimes \fl{I}\Big)^{-1}\s{vec}\Big(\frac{1}{mt}\sum_{i \in [t]} (\fl{X}^{(i)})^{\s{T}}\fl{X}^{(i)}\Big(\fl{U}^{\star}\fl{w}^{\star(i)} + (\fl{b}^{\star(i)} - \fl{b}^{(i, \ell)})\Big)(\fl{w}^{(i, \ell)})^{\s{T}}\Big)\nonumber\\
    &= \Big(\Big(\frac{1}{t}(\fl{W}^{(\ell)})^\s{T}\fl{W}^{(\ell)}\Big)^{-1} \otimes \fl{I}\Big)\cdot\nonumber\\
    &\qquad \frac{1}{mt}\sum_{i \in [t]}\Big(\s{vec}\Big( (\fl{X}^{(i)})^{\s{T}}\fl{X}^{(i)}\fl{U}^{\star}\fl{w}^{\star(i)}(\fl{w}^{(i, \ell)})^{\s{T}}\Big) + \s{vec}\Big( (\fl{X}^{(i)})^{\s{T}}\fl{X}^{(i)}(\fl{b}^{\star(i)} - \fl{b}^{(i, \ell)})(\fl{w}^{(i, \ell)})^{\s{T}}\Big)\Big)\nonumber\\
    &= \Big(\Big(\frac{1}{t}(\fl{W}^{(\ell)})^\s{T}\fl{W}^{(\ell)}\Big)^{-1} \otimes \fl{I}\Big)\cdot\nonumber\\
    &\qquad \frac{1}{mt}\sum_{i \in [t]}\Big(\Big(\fl{w}^{\star(i)}(\fl{w}^{(i, \ell)})^{\s{T}}\otimes(\fl{X}^{(i)})^{\s{T}}\fl{X}^{(i)}\Big)\s{vec}( \fl{U}^{\star}) \nonumber\\
    &\qquad\qquad + \s{vec}\Big( (\fl{X}^{(i)})^{\s{T}}\fl{X}^{(i)}(\fl{b}^{\star(i)} - \fl{b}^{(i, \ell)})(\fl{w}^{(i, \ell)})^{\s{T}}\Big)\Big)\nonumber\\
    &= \frac{1}{mt}\sum_{i \in [t]} \Big(\Big(\frac{1}{t}(\fl{W}^{(\ell)})^\s{T}\fl{W}^{(\ell)}\Big)^{-1} \otimes \fl{I}\Big)\cdot\Big(\fl{w}^{\star(i)}(\fl{w}^{(i, \ell)})^{\s{T}}\otimes(\fl{X}^{(i)})^{\s{T}}\fl{X}^{(i)}\Big)\s{vec}( \fl{U}^{\star}) \nonumber\\
    &\qquad + \frac{1}{mt}\sum_{i \in [t]}\Big(\Big(\frac{1}{t}(\fl{W}^{(\ell)})^\s{T}\fl{W}^{(\ell)}\Big)^{-1} \otimes \fl{I}\Big)\s{vec}\Big( (\fl{X}^{(i)})^{\s{T}}\fl{X}^{(i)}(\fl{b}^{\star(i)} - \fl{b}^{(i, \ell)})(\fl{w}^{(i, \ell)})^{\s{T}}\Big)\nonumber\\
    &= \underbrace{\frac{1}{mt}\sum_{i \in [t]}\Big(\Big(\frac{1}{t}(\fl{W}^{(\ell)})^\s{T}\fl{W}^{(\ell)}\Big)^{-1}\fl{w}^{\star(i)}(\fl{w}^{(i, \ell)})^{\s{T}} \otimes (\fl{X}^{(i)})^{\s{T}}\fl{X}^{(i)}\Big)\s{vec}( \fl{U}^{\star})}_{\s{vec}(\fl{V}'_1)} \nonumber\\
    &\qquad + 
    \underbrace{\frac{1}{mt}\sum_{i \in [t]}\Big(\Big(\frac{1}{t}(\fl{W}^{(\ell)})^\s{T}\fl{W}^{(\ell)}\Big)^{-1} \otimes \fl{I}\Big)\s{vec}\Big( (\fl{X}^{(i)})^{\s{T}}\fl{X}^{(i)}(\fl{b}^{\star(i)} - \fl{b}^{(i, \ell)})(\fl{w}^{(i, \ell)})^{\s{T}}\Big)}_{\s{vec}(\fl{V}'_2)}.\label{eq:lemma-optimize_dp-lrs-2,infty-norm-CvecV-unplugged}
\end{align}

\textbf{Analysis for $\s{vec}(\fl{V}'_1)$:}

Let $\fl{J}^{(i, \ell)} := \Big(\frac{1}{t}(\fl{W}^{(\ell)})^\s{T}\fl{W}^{(\ell)}\Big)^{-1}\fl{w}^{\star(i)}(\fl{w}^{(i, \ell)})^{\s{T}}$. Then we can rewrite $\fl{V}'_1$ as, 
\begin{align}
    \s{vec}(\fl{V}'_1) &= \frac{1}{mt}\sum_{i \in [t]}\Big(\fl{J}^{(i, \ell)} \otimes (\fl{X}^{(i)})^{\s{T}}\fl{X}^{(i)}\Big)\s{vec}( \fl{U}^{\star}) \\
    \iff \fl{V}'_1 &= \frac{1}{mt}\sum_{i \in [t]}(\fl{X}^{(i)})^{\s{T}}\fl{X}^{(i)}\fl{U}^{\star}(\fl{J}^{(i, \ell)})^\s{T}\\
    \implies \E{\s{vec}(\fl{V}'_1)} &= \Big(\frac{1}{t}\sum_{i \in [t]}\fl{J}^{(i, \ell)} \otimes \fl{I}\Big)\s{vec}( \fl{U}^{\star}),\\
    \iff \E{\fl{V}'_1} &= \fl{U}^\star\Big(\frac{1}{t}\sum_{i \in [t]}\fl{J}^{(i, \ell)}\Big)^\s{T}\\
    \implies \|\E{\fl{V}'_1}\|_{2, \infty} &= \|\fl{U}^\star\Big(\frac{1}{t}\sum_{i \in [t]}\fl{J}^{(i, \ell)}\Big)^\s{T}\|_{2, \infty}\nonumber\\
    &= \max_{p \in [d]}\|\fl{U}^\star_p\|_2\|\frac{1}{t}\sum_{i \in [t]}\fl{J}^{(i, \ell)}\|_{2}\nonumber\\
    &\leq \|\fl{U}^\star\|_{2, \infty}\|\frac{1}{t}\sum_{i \in [t]}\fl{J}^{(i, \ell)}\|_{2}\label{eq:lemma-optimize_dp-lrs-2,infty-norm-EV1-infintynorm}
\end{align}
Now consider the $s$-th standard basis vector $\fl{e}_s \in \bR^{rd}$ s.t. $s$ falls under the $q^\star$-th fragment ($q \in [r]$), i.e. $\forall$ $\fl{a} \in \bR^{rd}$, if we write denote $\fl{a}^\s{T} = \begin{bmatrix}\fl{a}^\s{T}_1, \fl{a}^\s{T}_2, \dots, \fl{a}^\s{T}_r \end{bmatrix}$ where each $\fl{a}_i \in \bR^d$, then $\fl{e}_s^\s{T}\fl{a} = \begin{bmatrix}\fl{0}^\s{T}_1 \dots, \underbrace{\begin{bmatrix}0 \dots, 1_{s}, \dots 0 \end{bmatrix}}_{q^\star}, \dots \fl{0}^\s{T}_r \end{bmatrix}\begin{bmatrix}\fl{a}_1 \\ \dots \\ \fl{a}_{q^\star} \\ \dots \\ \fl{a}_r \end{bmatrix} = \begin{bmatrix}0 \dots, 1_{s}, \dots 0 \end{bmatrix} \fl{a}_{q^\star} = \s{vec}(\fl{a})_s$. Then following along similar lines of Lemma~\ref{lemma:a^TEb-concentration}, we have
\begin{align}
    \fl{e}_s^\s{T}\s{vec}(\fl{V}'_1) &= \fl{e}_s^\s{T}\frac{1}{mt}\sum_{i \in [t]}\Big(\fl{J}^{(i, \ell)} \otimes (\fl{X}^{(i)})^{\s{T}}\fl{X}^{(i)}\Big)\s{vec}( \fl{U}^{\star})\nonumber\\
    &= \frac{1}{mt}\sum_{i \in [t]}\sum_{j \in [m]}\sum_{p \in [r]}\sum_{q \in [r]} \fl{e}^\s{T}_p\Big(J^{(i, \ell)}_{p, q} \fl{x}^{(i)}_j(\fl{x}^{(i)}_j)^{\s{T}}\Big)\s{vec}(\fl{U}^\star)_q\nonumber\\
    &= \frac{1}{mt}\sum_{i \in [t]}\sum_{j \in [m]}(\fl{x}^{(i)}_j)^\s{T}\Big(\sum_{p \in [r]}\sum_{q \in [r]} J^{(i, \ell)}_{p, q} \s{vec}(\fl{U}^\star)_p\fl{e}_q^{\s{T}}\Big)\fl{x}^{(i)}_j\label{eq:lemma-optimize_dp-lrs-2,infty-norm-eTvecV1}
\end{align}
and
    \begin{align}
        \|\sum_{p \in [r]}\sum_{q \in [r]} J^{(i, \ell)}_{p, q} \s{vec}(\fl{U}^\star)_p\fl{e}_q^{\s{T}}\|_\s{F} &= \|\sum_{p \in [r]} J^{(i, \ell)}_{p, q^\star} \s{vec}(\fl{U}^\star)_p\fl{e}_{q^\star}\|_\s{F}\nonumber\\
        &= \|\sum_{p \in [r]} J^{(i, \ell)}_{p, q^\star} \s{vec}(\fl{U}^\star)_p\|_2\|\fl{e}_{q^\star}^{\s{T}}\|_2\nonumber\\
        &\leq \|\sum_{p \in [r]} |J^{(i, \ell)}_{p, q^\star}| \s{vec}(\fl{U}^\star)_p\|_2\\
        &\leq \sqrt{\Big(\sum_{p \in [r]}(J^{(i, \ell)}_{p, q^\star})^2\Big)\Big(\sum_{p \in [r]} \|\s{vec}(\fl{U}^\star)_p\|_2^2\Big)}\nonumber\\
        &= \|\fl{J}^{(i, \ell, q^\star)}\|_2\|\s{vec}(\fl{U}^\star)\|_2\nonumber\\
        &= \|\fl{J}^{(i, \ell, q^\star)}\|_2\|\fl{U}^\star\|_\s{F}.\label{eq:lemma-optimize_dp-lrs-2,infty-norm-vecV1-conc}
    \end{align}
Thus, using \eqref{eq:lemma-optimize_dp-lrs-2,infty-norm-eTvecV1} and \eqref{eq:lemma-optimize_dp-lrs-2,infty-norm-vecV1-conc} in Lemma~\ref{lemma:useful0} we have
\begin{align}
        \Big|\fl{e}^\s{T}_s\Big(\s{vec}(\fl{V}'_1) - \E{\s{vec}(\fl{V}'_1)}\Big)\Big| &\leq c\sqrt{\sum_{i \in [t]}\sum_{j \in [m]}\|\fl{J}^{(i, \ell, q^\star)}\|_2^2\|\fl{U}^\star\|_\s{F}^2\frac{\log (1/\delta_0)}{m^2t^2}}.\label{eq:lemma-optimize_dp-lrs-2,infty-norm-eT(vecV1-E)}
\end{align}
Now, note that 
\begin{align}
    \|\fl{V}'_1 - \E{\fl{V}'_1}\|_{2, \infty} &= \max_{p \in [d]} \|(\fl{V}'_1)_p - \E{\fl{V}'_1}_p\|_2\nonumber\\
    &= \max_{p \in [d]} \sqrt{\sum_{q \in [r]}|(V_1)_{p, q} - \E{V_1}_{p, q}|_2^2}\nonumber\\
    &= \max_{p \in [d]} \sqrt{\sum_{s = \{(q-1)d+p: q \in [r]\}}|\fl{e}_s^\s{T}(\s{vec}(\fl{V}'_1) - \fl{\E{\s{vec}(\fl{V}'_1}}|^2}\nonumber\\
    &\leq \max_{p \in [d]} \sqrt{\sum_{s = \{(q-1)d+p: q \in [r]\}}\Big(c\sqrt{\sum_{i \in [t]}\sum_{j \in [m]}\|\fl{J}^{(i, \ell, q^\star)}\|_2^2\|\fl{U}^\star\|_\s{F}^2\frac{\log (1/\delta_0)}{m^2t^2}}\Big)^2},\nonumber
\end{align}
where in the last step we use \eqref{eq:lemma-optimize_dp-lrs-2,infty-norm-eT(vecV1-E)}. Now note that as per the notation discussed above, $s = (q-1)d+p$ lies in the $q$-th (= $q^\star$) segment. Since $q \in [r]$, therefore summation over $s$ is equivalent to summation over $q^\star \in [r]$. Using this fact, the above becomes:
\begin{align}
    \|\fl{V}'_1 - \E{\fl{V}'_1}\|_{2, \infty} &\leq \max_{p \in [d]} \sqrt{\sum_{q^\star \in [r]}\Big(c^2\sum_{i \in [t]}\sum_{j \in [m]}\|\fl{J}^{(i, \ell, q^\star)}\|_2^2\|\fl{U}^\star\|_\s{F}^2\frac{\log (1/\delta_0)}{m^2t^2}\Big)}\nonumber\\
    &\leq \max_{p \in [d]} c\|\fl{U}^\star\|_\s{F}\sqrt{\sum_{q^\star \in [r]}\sum_{i \in [t]}\sum_{j \in [m]}\|\fl{J}^{(i, \ell, q^\star)}\|_2^2}\sqrt{\frac{\log (1/\delta_0)}{m^2t^2}}\nonumber\\
    &\leq c\|\fl{U}^\star\|_\s{F}\sqrt{\sum_{i \in [t]}\sum_{j \in [m]}\|\fl{J}^{(i, \ell)}\|_\s{F}^2}\sqrt{\frac{\log (1/\delta_0)}{m^2t^2}}.
\end{align}
Therefore, using \eqref{eq:lemma-optimize_dp-lrs-2,infty-norm-EV1-infintynorm} in the above, we have
\begin{align}
    \|\fl{V}'_1\|_{2, \infty} &\leq \|\E{\fl{V}'_1}\|_{2, \infty} + c\|\fl{U}^\star\|_\s{F}\sqrt{\sum_{i \in [t]}\sum_{j \in [m]}\|\fl{J}^{(i, \ell)}\|_\s{F}^2}\sqrt{\frac{\log (1/\delta_0)}{m^2t^2}}\nonumber\\
    &\leq \|\fl{U}^\star\|_{2, \infty}\|\frac{1}{t}\sum_{i \in [t]}\fl{J}^{(i, \ell)}\|_{2} + c\|\fl{U}^\star\|_\s{F}\sqrt{\sum_{i \in [t]}\sum_{j \in [m]}\|\fl{J}^{(i, \ell)}\|_\s{F}^2}\sqrt{\frac{\log (1/\delta_0)}{m^2t^2}}\nonumber\\
    &\leq \|\fl{U}^\star\|_{2, \infty}\|\frac{1}{t}\sum_{i \in [t]}\Big(\frac{1}{t}(\fl{W}^{(\ell)})^\s{T}\fl{W}^{(\ell)}\Big)^{-1}\fl{w}^{\star(i)}(\fl{w}^{(i, \ell)})^{\s{T}}\|_{2} \nonumber\\
    &\qquad + c\|\fl{U}^\star\|_\s{F}\|\Big(\frac{1}{t}(\fl{W}^{(\ell)})^\s{T}\fl{W}^{(\ell)}\Big)^{-1}\|\|\fl{w}^{\star(i)}(\fl{w}^{(i, \ell)})^{\s{T}}\|_\s{F}\sqrt{\frac{\log (1/\delta_0)}{mt}}.\label{eq:lemma-optimize_dp-lrs-2,infty-norm-V1-2,infinty-norm-value}
\end{align}

\textbf{Analysis for $\s{vec}(\fl{V}'_2)$:}\\
Let $\fl{L}^{(i, \ell)} := \Big(\frac{1}{t}(\fl{W}^{(\ell)})^\s{T}\fl{W}^{(\ell)}\Big)^{-1}$. Then,
\begin{align}
     \s{vec}(\fl{V}'_2) &= \frac{1}{mt}\sum_{i \in [t]}\Big(\fl{L}^{(i, \ell)} \otimes \fl{I}\Big)\s{vec}\Big( (\fl{X}^{(i)})^{\s{T}}\fl{X}^{(i)}(\fl{b}^{\star(i)} - \fl{b}^{(i, \ell)})(\fl{w}^{(i, \ell)})^{\s{T}}\Big) \\
     &= \frac{1}{mt}\sum_{i \in [t]}\Big(\fl{L}^{(i, \ell)} \otimes \fl{I}_{d\times d}\Big)\Big(\fl{I}_{r\times r} \otimes (\fl{X}^{(i)})^{\s{T}}\fl{X}^{(i)}\Big)\s{vec}\Big((\fl{b}^{\star(i)} - \fl{b}^{(i, \ell)})(\fl{w}^{(i, \ell)})^{\s{T}}\Big)\nonumber\\
    &= \frac{1}{mt}\sum_{i \in [t]}\Big(\fl{L}^{(i, \ell)} \otimes (\fl{X}^{(i)})^{\s{T}}\fl{X}^{(i)}\Big)\s{vec}\Big((\fl{b}^{\star(i)} - \fl{b}^{(i, \ell)})(\fl{w}^{(i, \ell)})^{\s{T}}\Big)\\
    \iff \fl{V}'_2 &= \frac{1}{mt}\sum_{i \in [t]}(\fl{X}^{(i)})^{\s{T}}\fl{X}^{(i)}(\fl{b}^{\star(i)} - \fl{b}^{(i, \ell)})(\fl{w}^{(i, \ell)})^{\s{T}}\Big(\fl{L}^{(i, \ell)}\Big)^\s{T},\\
    \implies \E{\s{vec}(\fl{V}'_2)} &= \frac{1}{t}\sum_{i \in [t]}\Big(\fl{L}^{(i, \ell)} \otimes \fl{I}\Big)\s{vec}\Big((\fl{b}^{\star(i)} - \fl{b}^{(i, \ell)})(\fl{w}^{(i, \ell)})^{\s{T}}\Big)\nonumber\\
     \iff \E{\fl{V}'_2} &= \frac{1}{t}\sum_{i \in [t]}\Big((\fl{b}^{\star(i)} - \fl{b}^{(i, \ell)})(\fl{w}^{(i, \ell)})^{\s{T}}\Big)\Big(\fl{L}^{(i, \ell)}\Big)^{\s{T}}\\
     \implies \|\E{\fl{V}'_2}\|_{2, \infty} &= \|\frac{1}{t}\sum_{i \in [t]}\Big((\fl{b}^{\star(i)} - \fl{b}^{(i, \ell)})(\fl{w}^{(i, \ell)})^{\s{T}}\Big)\Big(\fl{L}^{(i, \ell)}\Big)^{\s{T}}\|_{2, \infty}\nonumber\\
     &\leq \frac{1}{t}\sum_{i \in [t]} \max_{p \in [d]}\|\Big((\fl{b}^{\star(i)} - \fl{b}^{(i, \ell)})(\fl{w}^{(i, \ell)})^{\s{T}}\Big)_p\|_2\|\fl{L}^{(i, \ell)}\|_{2}\nonumber\\
     &\leq \frac{1}{t}\sum_{i \in [t]} \max_{p \in [d]}|\fl{b}^{\star(i)}_p - \fl{b}^{(i, \ell)}_p|\|\fl{w}^{(i, \ell)}\|_2\|\fl{L}^{(i, \ell)}\|_{2}\nonumber\\
     &\leq \zeta(\max_i \|\fl{b}^{\star(i)} - \fl{b}^{(i, \ell)}\|_\infty\|\fl{w}^{(i, \ell)}\|_2)\|\Big((\fl{W}^{(\ell)})^\s{T}\fl{W}^{(\ell)}\Big)^{-1}\|_{2}\label{eq:lemma-optimize_dp-lrs-2,infty-norm-EV2-infintynorm}
\end{align}
Further, we have 
\begin{align}
    \fl{e}_s^\s{T}\s{vec}(\fl{V}'_2) &= \fl{e}_s^\s{T}\frac{1}{mt}\sum_{i \in [t]}\Big(\fl{L}^{(i, \ell)} \otimes (\fl{X}^{(i)})^{\s{T}}\fl{X}^{(i)}\Big)\s{vec}\Big((\fl{b}^{\star(i)} - \fl{b}^{(i, \ell)})(\fl{w}^{(i, \ell)})^{\s{T}}\Big)\nonumber\\
    &= \frac{1}{mt}\sum_{i \in [t]}\sum_{j \in [m]}\sum_{p \in [r]}\sum_{q \in [r]} \fl{e}^\s{T}_p\Big(L^{(i, \ell)}_{p, q} \fl{x}^{(i)}_j(\fl{x}^{(i)}_j)^{\s{T}}\Big)\s{vec}\Big((\fl{b}^{\star(i)} - \fl{b}^{(i, \ell)})(\fl{w}^{(i, \ell)})^{\s{T}}\Big)_q\nonumber\\
    &= \frac{1}{mt}\sum_{i \in [t]}\sum_{j \in [m]}(\fl{x}^{(i)}_j)^\s{T}\Big(\sum_{p \in [r]}\sum_{q \in [r]} L^{(i, \ell)}_{p, q} \s{vec}\Big((\fl{b}^{\star(i)} - \fl{b}^{(i, \ell)})(\fl{w}^{(i, \ell)})^{\s{T}}\Big)_p\fl{e}_q^{\s{T}}\Big)\fl{x}^{(i)}_j\label{eq:lemma-optimize_dp-lrs-2,infty-norm-eTvecV2}
\end{align}
and $\|\sum_{p \in [r]}\sum_{q \in [r]} L^{(i, \ell)}_{p, q} \s{vec}\Big((\fl{b}^{\star(i)} - \fl{b}^{(i, \ell)})(\fl{w}^{(i, \ell)})^{\s{T}}\Big)_p\fl{e}_q^{\s{T}}\|_\s{F}$
    \begin{align}
        &= \|\sum_{p \in [r]} L^{(i, \ell)}_{p, q^\star} \s{vec}\Big((\fl{b}^{\star(i)} - \fl{b}^{(i, \ell)})(\fl{w}^{(i, \ell)})^{\s{T}}\Big)_p\fl{e}_{q^\star}\|_\s{F}\nonumber\\
        &= \|\sum_{p \in [r]} L^{(i, \ell)}_{p, q^\star} \s{vec}\Big((\fl{b}^{\star(i)} - \fl{b}^{(i, \ell)})(\fl{w}^{(i, \ell)})^{\s{T}}\Big)_p\|_2\|\fl{e}_{q^\star}^{\s{T}}\|_2\nonumber\\
        &\leq \|\sum_{p \in [r]} |L^{(i, \ell)}_{p, q^\star}| \s{vec}\Big((\fl{b}^{\star(i)} - \fl{b}^{(i, \ell)})(\fl{w}^{(i, \ell)})^{\s{T}}\Big)_p\|_2\nonumber\\
        &\leq \sqrt{\Big(\sum_{p \in [r]}(L^{(i, \ell)}_{p, q^\star})^2\Big)\Big(\sum_{p \in [r]} \|\s{vec}\Big((\fl{b}^{\star(i)} - \fl{b}^{(i, \ell)})(\fl{w}^{(i, \ell)})^{\s{T}}\Big)_p\|_2^2\Big)}\nonumber\\
        &= \|\fl{L}^{(i, \ell, q^\star)}\|_2\|\s{vec}\Big((\fl{b}^{\star(i)} - \fl{b}^{(i, \ell)})(\fl{w}^{(i, \ell)})^{\s{T}}\Big)\|_2\nonumber\\
        &= \|\fl{L}^{(i, \ell, q^\star)}\|_2\|(\fl{b}^{\star(i)} - \fl{b}^{(i, \ell)})(\fl{w}^{(i, \ell)})^{\s{T}}\|_\s{F}\label{eq:lemma-optimize_dp-lrs-2,infty-norm-vecV2-conc}
    \end{align}
Thus, using \eqref{eq:lemma-optimize_dp-lrs-2,infty-norm-eTvecV2} and \eqref{eq:lemma-optimize_dp-lrs-2,infty-norm-vecV2-conc} in Lemma~\ref{lemma:useful0} we have
\begin{align}
        \Big|\fl{e}^\s{T}_s\Big(\s{vec}(\fl{V}'_2) - \E{\s{vec}(\fl{V}'_2)}\Big)\Big| &\leq c\sqrt{\sum_{i \in [t]}\sum_{j \in [m]}\|\fl{L}^{(i, \ell, q^\star)}\|_2^2\|(\fl{b}^{\star(i)} - \fl{b}^{(i, \ell)})(\fl{w}^{(i, \ell)})^{\s{T}}\|_\s{F}^2\frac{\log (1/\delta_0)}{m^2t^2}}
\end{align}
Now, note that $\|\fl{V}'_2 - \E{\fl{V}'_2}\|_{2, \infty}$
\begin{align}
    &= \max_{p \in [d]} \|(\fl{V}'_1)_p - \E{\fl{V}'_1}_p\|_2\nonumber\\
    &= \max_{p \in [d]} \sqrt{\sum_{q \in [r]}|(V_1)_{p, q} - \E{V_1}_{p, q}|_2^2}\nonumber\\
    &= \max_{p \in [d]} \sqrt{\sum_{s = \{(q-1)d+p: q \in [r]\}}|\fl{e}_s^\s{T}(\s{vec}(\fl{V}'_1) - \fl{\E{\s{vec}(\fl{V}'_1}}|^2}\nonumber\\
    &\leq \max_{p \in [d]} \sqrt{\sum_{s = \{(q-1)d+p: q \in [r]\}}\Big(c\sqrt{\sum_{i \in [t]}\sum_{j \in [m]}\|\fl{L}^{(i, \ell, q^\star)}\|_2^2\|(\fl{b}^{\star(i)} - \fl{b}^{(i, \ell)})(\fl{w}^{(i, \ell)})^{\s{T}}\|_\s{F}^2\frac{\log (1/\delta_0)}{m^2t^2}}\Big)^2}\nonumber\\
\end{align}
Now note that as per the notation discussed above, $s = (q-1)d+p$ lies in the $q$-th (= $q^\star$) segment. Since $q \in [r]$, therefore summation over $s$ is equivalent to summation over $q^\star \in [r]$. Using this fact the above becomes, $\|\fl{V}'_2 - \E{\fl{V}'_2}\|_{2, \infty}$
\begin{align}
    &\leq \max_{p \in [d]} \sqrt{\sum_{q^\star \in [r]}\Big(c^2\sum_{i \in [t]}\sum_{j \in [m]}\|\fl{L}^{(i, \ell, q^\star)}\|_2^2\|(\fl{b}^{\star(i)} - \fl{b}^{(i, \ell)})(\fl{w}^{(i, \ell)})^{\s{T}}\|_\s{F}^2\frac{\log (1/\delta_0)}{m^2t^2}\Big)}\nonumber\\
    &\leq \max_{p \in [d]} c\sqrt{\sum_{q^\star \in [r]}\sum_{i \in [t]}\sum_{j \in [m]}\|\fl{L}^{(i, \ell, q^\star)}\|_2^2\|(\fl{b}^{\star(i)} - \fl{b}^{(i, \ell)})(\fl{w}^{(i, \ell)})^{\s{T}}\|_\s{F}^2}\sqrt{\frac{\log (1/\delta_0)}{m^2t^2}}\nonumber\\
    &\leq \max_{p \in [d]} c\sqrt{\sum_{i \in [t]}\sum_{j \in [m]}\|\fl{L}^{(i, \ell)}\|_\s{F}^2\|(\fl{b}^{\star(i)} - \fl{b}^{(i, \ell)})(\fl{w}^{(i, \ell)})^{\s{T}}\|_\s{F}^2}\sqrt{\frac{\log (1/\delta_0)}{m^2t^2}}\nonumber\\
    &= \max_{p \in [d]} c\sqrt{\sum_{i \in [t]}\sum_{j \in [m]}\|\Big(\frac{1}{t}(\fl{W}^{(\ell)})^\s{T}\fl{W}^{(\ell)}\Big)^{-1}\|_\s{F}^2\|(\fl{b}^{\star(i)} - \fl{b}^{(i, \ell)})(\fl{w}^{(i, \ell)})^{\s{T}}\|_\s{F}^2}\sqrt{\frac{\log (1/\delta_0)}{m^2t^2}}\nonumber\\
    &\leq c\|\Big(\frac{1}{t}(\fl{W}^{(\ell)})^\s{T}\fl{W}^{(\ell)}\Big)^{-1}\|_\s{F}\sqrt{\sum_{i \in [t]}\sum_{j \in [m]}\|\fl{b}^{\star(i)} - \fl{b}^{(i, \ell)}\|_2^2\|\fl{w}^{(i, \ell)}\|_\s{2}^2}\sqrt{\frac{\log (1/\delta_0)}{m^2t^2}}\nonumber\\
    &\leq c\|\Big(\frac{1}{t}(\fl{W}^{(\ell)})^\s{T}\fl{W}^{(\ell)}\Big)^{-1}\|_\s{F}\|\fl{b}^{\star(i)} - \fl{b}^{(i, \ell)}\|_2\|\fl{w}^{(i, \ell)}\|_\s{2}\sqrt{\frac{\log (1/\delta_0)}{mt}}.
\end{align}
Therefore, using \eqref{eq:lemma-optimize_dp-lrs-2,infty-norm-EV2-infintynorm} in the above, we have $\|\fl{V}'_2\|_{2, \infty}$
\begin{align}
    &\leq \|\E{\fl{V}'_2}\|_{2, \infty} + c\|\Big(\frac{1}{t}(\fl{W}^{(\ell)})^\s{T}\fl{W}^{(\ell)}\Big)^{-1}\|_\s{F}(\max_i \sqrt{\zeta k}\|\fl{b}^{\star(i)} - \fl{b}^{(i, \ell)}\|_\infty\|\fl{w}^{(i, \ell)}\|_2)\sqrt{\frac{\log (1/\delta_0)}{mt^2}}\nonumber\\
    &\leq \zeta(\max_i \|\fl{b}^{\star(i)} - \fl{b}^{(i, \ell)}\|_\infty\|\fl{w}^{(i, \ell)}\|_2)\|\Big((\fl{W}^{(\ell)})^\s{T}\fl{W}^{(\ell)}\Big)^{-1}\|_{2} \nonumber\\
    &\qquad c\|\Big(\frac{1}{t}(\fl{W}^{(\ell)})^\s{T}\fl{W}^{(\ell)}\Big)^{-1}\|_\s{F}\|\fl{b}^{\star(i)} - \fl{b}^{(i, \ell)}\|_2\|\fl{w}^{(i, \ell)}\|_\s{2}\sqrt{\frac{\log (1/\delta_0)}{mt}}.\label{eq:lemma-optimize_dp-lrs-2,infty-norm-V2-2,infinty-norm-value}
\end{align}
Analysis for $\fl{C}\s{vec}\Big(\frac{\fl{N}_2}{mt}\Big)$:

Note that:
\begin{align}
    \fl{C}\s{vec}\Big(\frac{\fl{N}_2}{mt}\Big) &= \Big(\frac{1}{t}(\fl{W}^{(\ell)})^\s{T}\fl{W}^{(\ell)} \otimes \fl{I}\Big)^{-1}\s{vec}\Big(\frac{\fl{N}_2}{mt}\Big) \nonumber\\
    &= \Big(\Big(\frac{1}{t}(\fl{W}^{(\ell)})^\s{T}\fl{W}^{(\ell)}\Big)^{-1} \otimes \fl{I}\Big)\s{vec}\Big(\frac{\fl{N}_2}{mt}\Big) := \s{vec}(\fl{V}'_3)\nonumber\\
    \implies \fl{V}'_3 &= \fl{I}\cdot \frac{\fl{N}_2}{mt}\cdot \Big(\frac{1}{t}(\fl{W}^{(\ell)})^\s{T}\fl{W}^{(\ell)}\Big)^\s{-T}\nonumber\\
    \implies \|\fl{V}'_3\|_{2, \infty} &= \|\frac{\fl{N}_2}{mt}\cdot \Big(\frac{1}{t}(\fl{W}^{(\ell)})^\s{T}\fl{W}^{(\ell)}\Big)^{-1}\|_{2, \infty}\nonumber\\
    &\leq \frac{1}{mt}\|\fl{N}_2\|_{2, \infty}\lr{\Big(\frac{1}{t}(\fl{W}^{(\ell)})^\s{T}\fl{W}^{(\ell)}\Big)^{-1}}_2\nonumber\\
    &\leq \frac{1}{mt}\|\fl{N}_2\|_{2, \infty}\frac{r}{\lambda_r\Big(\frac{r}{t}(\fl{W}^{(\ell)})^\s{T}\fl{W}^{(\ell)}\Big)}\nonumber\\
    &\leq \frac{2\sigma_2}{mt}\sqrt{\log(rd/\delta_0)}\cdot\frac{r\sqrt{r}}{\lambda_r\Big(\frac{r}{t}(\fl{W}^{(\ell)})^\s{T}\fl{W}^{(\ell)}\Big)}.\label{eq:lemma-optimize_dp-lrs-2,infty-norm-V3-2,infinty-norm-value}
\end{align}
Analysis for $\fl{C}\s{vec}(\fl{\Xi})$:

Note that:
\begin{align}
    \fl{C}\s{vec}(\fl{\Xi}) &= \Big(\frac{1}{t}(\fl{W}^{(\ell)})^\s{T}\fl{W}^{(\ell)} \otimes \fl{I}\Big)^{-1}\s{vec}(\fl{\Xi}) \nonumber\\
    &= \Big(\Big(\frac{1}{t}(\fl{W}^{(\ell)})^\s{T}\fl{W}^{(\ell)}\Big)^{-1} \otimes \fl{I}\Big)\s{vec}(\fl{\Xi}) := \s{vec}(\fl{V}'_\xi)\nonumber\\
    \implies \fl{V}'_\xi &= \fl{I}\cdot\fl{\Xi}\cdot \Big(\frac{1}{t}(\fl{W}^{(\ell)})^\s{T}\fl{W}^{(\ell)}\Big)^\s{-T}\nonumber\\
    \implies \|\fl{V}'_\xi\|_{2, \infty} &= \|\fl{\Xi}\cdot \Big(\frac{1}{t}(\fl{W}^{(\ell)})^\s{T}\fl{W}^{(\ell)}\Big)^{-1}\|_{2, \infty}\nonumber\\
    &\leq \|\fl{\Xi}\|_{2, \infty}\lr{\Big(\frac{1}{t}(\fl{W}^{(\ell)})^\s{T}\fl{W}^{(\ell)}\Big)^{-1}}_2\nonumber\\
    &\leq \frac{2\sigma\sqrt{\mu^{\star}\lambda^{\star}_r}\log(2rdmt/\delta_0)}{\sqrt{mt}}\frac{r}{\lambda_r\Big(\frac{r}{t}(\fl{W}^{(\ell)})^\s{T}\fl{W}^{(\ell)}\Big)}\label{eq:lemma-optimize_dp-lrs-2,infty-norm-Vxi-2,infinty-norm-value}
\end{align}

Combining $\fl{V}'_1$, $\fl{V}'_2$, $\fl{V}'_3$ and $\fl{V}'_\xi$ from \eqref{eq:lemma-optimize_dp-lrs-2,infty-norm-V1-2,infinty-norm-value}, \eqref{eq:lemma-optimize_dp-lrs-2,infty-norm-V2-2,infinty-norm-value}, \eqref{eq:lemma-optimize_dp-lrs-2,infty-norm-V3-2,infinty-norm-value} and \eqref{eq:lemma-optimize_dp-lrs-2,infty-norm-Vxi-2,infinty-norm-value} respectively in \eqref{eq:lemma-optimize_dp-lrs-2,infty-norm-CvecV-unplugged}, we have:
\begin{align}
    \fl{C}\s{vec}(\fl{V}') &= \s{vec}(\fl{V}'_1) + \s{vec}(\fl{V}'_2)\nonumber\\
    \iff \fl{C}\s{vec}\Big(\fl{V}' + \frac{\fl{N}_2}{mt} + \fl{\Xi}\Big) &= \s{vec}(\fl{V}'_1) + \s{vec}(\fl{V}'_2) + \s{vec}(\fl{V}'_3) + \s{vec}(\fl{V}'_\xi)\nonumber
\end{align}
\begin{align}
    &\implies \lr{\s{vec}^{-1}\Big(\fl{C}\s{vec}\Big(\fl{V}' + \frac{\fl{N}_2}{mt} + \fl{\Xi}\Big)\Big)}_{2, \infty} \nonumber\\
    &= \|\fl{V}'_1\|_{2, \infty} + \|\fl{V}'_2\|_{2, \infty} + \|\fl{V}'_3\|_{2, \infty} + \|\fl{V}'_\xi\|_{2, \infty}\nonumber
    \end{align}
    \begin{align}
    &\leq \|\fl{U}^\star\|_{2, \infty}\|\frac{1}{t}\sum_{i \in [t]}\Big(\frac{1}{t}(\fl{W}^{(\ell)})^\s{T}\fl{W}^{(\ell)}\Big)^{-1}\fl{w}^{\star(i)}(\fl{w}^{(i, \ell)})^{\s{T}}\|_{2} \nonumber\\
    &\qquad + c\|\fl{U}^\star\|_\s{F}\|\Big(\frac{1}{t}(\fl{W}^{(\ell)})^\s{T}\fl{W}^{(\ell)}\Big)^{-1}\|\|\fl{w}^{\star(i)}(\fl{w}^{(i, \ell)})^{\s{T}}\|_\s{F}\sqrt{\frac{\log (1/\delta_0)}{mt}}\nonumber\\
    & + \zeta(\max_i \|\fl{b}^{\star(i)} - \fl{b}^{(i, \ell)}\|_\infty\|\fl{w}^{(i, \ell)}\|_2)\|\Big((\fl{W}^{(\ell)})^\s{T}\fl{W}^{(\ell)}\Big)^{-1}\|_{2} \nonumber\\
    &\qquad + c\|\Big(\frac{1}{t}(\fl{W}^{(\ell)})^\s{T}\fl{W}^{(\ell)}\Big)^{-1}\|_\s{F}\|\fl{b}^{\star(i)} - \fl{b}^{(i, \ell)}\|_2\|\fl{w}^{(i, \ell)}\|_\s{2}\sqrt{\frac{\log (1/\delta_0)}{mt}}\nonumber\\
    & + \frac{2\sigma_2}{mt}\sqrt{\log(rd/\delta_0)}\cdot\frac{r\sqrt{r}}{\lambda_r\Big(\frac{r}{t}(\fl{W}^{(\ell)})^\s{T}\fl{W}^{(\ell)}\Big)} + \frac{2\sigma\sqrt{\mu^{\star}\lambda^{\star}_r}\log(2rdmt/\delta_0)}{\sqrt{mt}}\frac{r}{\lambda_r\Big(\frac{r}{t}(\fl{W}^{(\ell)})^\s{T}\fl{W}^{(\ell)}\Big)}\label{eq:lemma-optimize_dp-lrs-2,infty-norm-CvecV-value}
\end{align}
Therefore, using \eqref{eq:lemma-optimize_dp-lrs-2,infty-norm-CvecV-value}, \eqref{eq:lemma-optimize_dp-lrs-u-2norm-bound-type1-num-val} and \eqref{eq:lemma-optimize_dp-lrs-U-Fnorm-bound-N2-Fnorm}  in \eqref{eq:lemma-optimize_dp-lrs-2,infty-norm-AinvV-2,infty-unplugged-2}, we have $ \|\fl{U}^{(\ell)}\|_{2, \infty}$
\begin{align}
   &\leq \|\fl{U}^\star\|_{2, \infty}\|\frac{1}{t}\sum_{i \in [t]}\Big(\frac{1}{t}(\fl{W}^{(\ell)})^\s{T}\fl{W}^{(\ell)}\Big)^{-1}\fl{w}^{\star(i)}(\fl{w}^{(i, \ell)})^{\s{T}}\|_{2} \nonumber\\
    & + c\|\fl{U}^\star\|_\s{F}\|\Big(\frac{1}{t}(\fl{W}^{(\ell)})^\s{T}\fl{W}^{(\ell)}\Big)^{-1}\|\|\fl{w}^{\star(i)}(\fl{w}^{(i, \ell)})^{\s{T}}\|_\s{F}\sqrt{\frac{\log (1/\delta_0)}{mt}}\nonumber\\
    & + \zeta(\max_i \|\fl{b}^{\star(i)} - \fl{b}^{(i, \ell)}\|_\infty\|\fl{w}^{(i, \ell)}\|_2)\|\Big((\fl{W}^{(\ell)})^\s{T}\fl{W}^{(\ell)}\Big)^{-1}\|_{2} \nonumber\\
    & + c\|\Big(\frac{1}{t}(\fl{W}^{(\ell)})^\s{T}\fl{W}^{(\ell)}\Big)^{-1}\|_\s{F}\|\fl{b}^{\star(i)} - \fl{b}^{(i, \ell)}\|_2\|\fl{w}^{(i, \ell)}\|_\s{2}\sqrt{\frac{\log (1/\delta_0)}{mt}}\nonumber\\
    & + \frac{2\sigma_2}{mt}\sqrt{\log(rd/\delta_0)}\cdot\frac{r\sqrt{r}}{\lambda_r\Big(\frac{r}{t}(\fl{W}^{(\ell)})^\s{T}\fl{W}^{(\ell)}\Big)} + \frac{2\sigma\sqrt{\mu^{\star}\lambda^{\star}_r}\log(2rdmt/\delta_0)}{\sqrt{mt}}\frac{r}{\lambda_r\Big(\frac{r}{t}(\fl{W}^{(\ell)})^\s{T}\fl{W}^{(\ell)}\Big)} \nonumber
    \end{align}
    \begin{align}
    & + \frac{\sqrt{r}\Big(c\sqrt{\sum_{i \in [t]}\sum_{j \in [m]}\|\fl{w}^{(i, \ell)}\|_2^4}\sqrt{\frac{rd\log(rd/ \delta_0)}{m^2t^2}} + \frac{\sigma_1}{mt}\Big(2\sqrt{rd} + 2\sqrt{2rd\log(2rd/\delta_0)}\Big)\Big)}{\frac{1}{r}\lambda_r\Big(\frac{r}{t}(\fl{W}^{(\ell)})^\s{T}\fl{W}^{(\ell)} \Big)}\cdot\nonumber\\
    & \Big\{\frac{2}{t}\|\fl{U}^{\star}(\fl{W}^{\star})^\s{T}\fl{W}^{(\ell)}\|_\s{F} + \sqrt{\frac{4\zeta}{t}}(\max_i\|\fl{w}^{(i, \ell)}\|_2)\|\fl{b}^{\star(i)} - \fl{b}^{(i, \ell)}\|_2\nonumber\\ 
    & +4\Big(\|\fl{U}^{\star}\|\|\fl{w}^{\star(i)}\|_2\|\fl{w}^{(i, \ell)}\|_2 + \|\fl{b}^{\star(i)} - \fl{b}^{(i, \ell)}\|_2\|\fl{w}^{(i, \ell)}\|_2\Big)\sqrt{\frac{d\log (rd/\delta_0)}{mt}}\nonumber\\
    & + \frac{\sigma_2}{mt}6\sqrt{rd\log(rd)} + \frac{2\sigma\sqrt{d\mu^{\star}\lambda^{\star}_r}\log(2rdmt/\delta_0)}{\sqrt{mt}}\Big\}.\label{eq:lemma-optimize_dp-lrs-2,infty-val1}
\end{align}

\textbf{Calculation for $\|\fl{U}^{(\ell)} - \fl{U}^{\star}\fl{Q}^{(\ell-1)}\|_{2,\infty}$:}

The analysis will follow along similar lines as in the previous section except that we will now have:
    \begin{align}
        \s{vec}(\fl{U}^{(\ell)} - \fl{U}^{\star}\fl{Q}^{(\ell-1)}) &=  \Big(\underbrace{\fl{A'} + \frac{\fl{N}_1}{mt}}_{\fl{A}}\Big)^{-1}\Big(\s{vec}(\fl{V}) + \Big(\s{vec}\Big(\frac{\fl{N}_2}{mt}\Big) -  \frac{\fl{N}_1}{mt}\s{vec}(\fl{U}^\star\fl{Q}^{(\ell-1)})\Big)\Big) \label{eq:lemma-optimize_dp-lrs-U-2,infty-norm-bound-type2-exp1}
    \end{align}
    where
    \begin{align}
        \fl{A'}_{rd\times rd} &= \frac{1}{mt}\sum_{i \in [t]}\Big(\fl{w}^{(i, \ell)}(\fl{w}^{(i, \ell)})^{\s{T}} \otimes (\fl{X}^{(i)})^{\s{T}}\fl{X}^{(i)}\Big)\nonumber \\
        &= \frac{1}{mt}\sum_{i \in [t]}\Big(\fl{w}^{(i, \ell)}(\fl{w}^{(i, \ell)})^{\s{T}} \otimes \Big(\sum_{j=1}^{m}\fl{x}^{(i)}_j(\fl{x}^{(i)}_j)^{\s{T}}\Big)\Big),\nonumber\\
        \fl{V}_{d\times r} &= \frac{1}{mt}\sum_{i \in [t]} (\fl{X}^{(i)})^{\s{T}}\fl{X}^{(i)}\Big(\fl{U}^{\star}(\fl{w}^{\star(i)} - \fl{Q}^{(\ell-1)}\fl{w}^{(i, \ell)}) + (\fl{b}^{\star(i)} - \fl{b}^{(i, \ell)})\Big)(\fl{w}^{(i, \ell)})^{\s{T}}\nonumber\\
        &= \frac{1}{mt}\sum_{i \in [t]} (\fl{X}^{(i)})^{\s{T}}\fl{X}^{(i)}\Big(-\fl{U}^{\star}\fl{Q}^{(\ell-1)}\fl{h}^{(i, \ell)} + (\fl{b}^{\star(i)} - \fl{b}^{(i, \ell)})\Big)(\fl{w}^{(i, \ell)})^{\s{T}},\nonumber
    \end{align}
    i.e. we have the term $\fl{U}^{\star}\fl{w}^{\star(i)}$ replaced by $-\fl{U}^{\star}\fl{Q}^{(\ell-1)}\fl{h}^{(i, \ell)}$ where $\fl{h}^{(i, \ell)} = \fl{w}^{(i, \ell)} - (\fl{Q}^{(\ell-1)})^{-1}\fl{w}^{\star(i)} $.
    
    Therefore, following along similar lines as for the analysis of $\|\fl{U}^{(\ell)}\|_{2, \infty}$, we have $\|\fl{U}^{(\ell)} - \fl{U}^{\star}\fl{Q}^{(\ell-1)}\|_{2,\infty}$
    \begin{align}
    &\leq \lr{\s{vec}^{-1}\Big(\fl{C}\Big(\s{vec}(\fl{V}) + \Big(\s{vec}\Big(\frac{\fl{N}_2}{mt}\Big) -  \frac{\fl{N}_1}{mt}\s{vec}(\fl{U}^\star\fl{Q}^{(\ell-1)})\Big)\Big)}_{2, \infty} \nonumber\\
    &\qquad + \frac{\sqrt{r}\Big(c\sqrt{\sum_{i \in [t]}\sum_{j \in [m]}\|\fl{w}^{(i, \ell)}\|_2^4}\sqrt{\frac{rd\log(rd/ \delta_0)}{m^2t^2}} + \frac{\sigma_1}{mt}\Big(2\sqrt{rd} + 2\sqrt{2rd\log(2rd/\delta_0)}\Big)\Big)}{\frac{1}{r}\lambda_r\Big(\frac{r}{t}(\fl{W}^{(\ell)})^\s{T}\fl{W}^{(\ell)} \Big)}\cdot\nonumber\\
    &\qquad \lr{\Big(\s{vec}(\fl{V}) + \Big(\s{vec}\Big(\frac{\fl{N}_2}{mt}\Big) -  \frac{\fl{N}_1}{mt}\s{vec}(\fl{U}^\star\fl{Q}^{(\ell-1)})\Big)\Big)}_2\nonumber\\
    &\leq \lr{\s{vec}^{-1}\Big(\fl{C}\Big(\s{vec}(\fl{V}) + \Big(\s{vec}\Big(\frac{\fl{N}_2}{mt}\Big) -  \frac{\fl{N}_1}{mt}\s{vec}(\fl{U}^\star\fl{Q}^{(\ell-1)})\Big)\Big)}_{2, \infty} \nonumber\\
    &\qquad + \frac{\sqrt{r}\Big(c\sqrt{\sum_{i \in [t]}\sum_{j \in [m]}\|\fl{w}^{(i, \ell)}\|_2^4}\sqrt{\frac{rd\log(rd/ \delta_0)}{m^2t^2}} + \frac{\sigma_1}{mt}\Big(2\sqrt{rd} + 2\sqrt{2rd\log(2rd/\delta_0)}\Big)\Big)}{\frac{1}{r}\lambda_r\Big(\frac{r}{t}(\fl{W}^{(\ell)})^\s{T}\fl{W}^{(\ell)} \Big)}\cdot\nonumber\\
    &\qquad \Big(\|\fl{V}\|_\s{F} + \lr{\frac{\fl{N}_2}{mt}}_\s{F} + \lr{\frac{\fl{N}_1}{mt}}_2\|\fl{U}^{\star}\fl{Q}^{(\ell-1)}\|_\s{F}\Big).\label{eq:lemma-optimize_dp-lrs-2,infty-norm-AinvV,infty-unplugged-2-type2}
    \end{align}
    
    Now, note that:
\begin{align}
    \fl{C}\s{vec}(\fl{V}) &= \underbrace{\frac{1}{mt}\sum_{i \in [t]}\Big(\Big(\frac{1}{t}(\fl{W}^{(\ell)})^\s{T}\fl{W}^{(\ell)}\Big)^{-1}\fl{h}^{(i, \ell)}(\fl{w}^{(i, \ell)})^{\s{T}} \otimes (\fl{X}^{(i)})^{\s{T}}\fl{X}^{(i)}\Big)\s{vec}( \fl{U}^{\star}\fl{Q}^{(\ell-1)})}_{\s{vec}(\fl{V}_1)} \nonumber\\
    &\qquad + 
    \underbrace{\frac{1}{mt}\sum_{i \in [t]}\Big(\Big(\frac{1}{t}(\fl{W}^{(\ell)})^\s{T}\fl{W}^{(\ell)}\Big)^{-1} \otimes \fl{I}\Big)\s{vec}\Big( (\fl{X}^{(i)})^{\s{T}}\fl{X}^{(i)}(\fl{b}^{\star(i)} - \fl{b}^{(i, \ell)})(\fl{w}^{(i, \ell)})^{\s{T}}\Big)}_{\s{vec}(\fl{V}_2)}\label{eq:lemma-optimize_dp-lrs-2,infty-norm-CvecV-unplugged-type2}
\end{align}

Let $\fl{J}^{(i, \ell)} := \Big(\frac{1}{t}(\fl{W}^{(\ell)})^\s{T}\fl{W}^{(\ell)}\Big)^{-1}\fl{h}^{(i, \ell)}(\fl{w}^{(i, \ell)})^{\s{T}}$. Then \eqref{eq:lemma-optimize_dp-lrs-2,infty-norm-V1-2,infinty-norm-value} in this case becomes
\begin{align}
    \|\fl{V}_1\|_{2, \infty} &\leq \|\fl{U}^\star\|_{2, \infty}\|\frac{1}{t}\sum_{i \in [t]}\fl{J}^{(i, \ell)}\|_{2} + c\|\fl{U}^\star\|_\s{F}\sqrt{\sum_{i \in [t]}\sum_{j \in [m]}\|\fl{J}^{(i, \ell)}\|_\s{F}^2}\sqrt{\frac{\log (1/\delta_0)}{m^2t^2}}\nonumber\\
    &\leq \|\fl{U}^\star\|_{2, \infty}\|\frac{1}{t}\sum_{i \in [t]}\Big(\frac{1}{t}(\fl{W}^{(\ell)})^\s{T}\fl{W}^{(\ell)}\Big)^{-1}\fl{h}^{(i, \ell)}(\fl{w}^{(i, \ell)})^{\s{T}}\|_{2} \nonumber\\
    &\qquad + c\|\fl{U}^\star\|_\s{F}\|\Big(\frac{1}{t}(\fl{W}^{(\ell)})^\s{T}\fl{W}^{(\ell)}\Big)^{-1}\|\|\fl{h}^{(i, \ell)}(\fl{w}^{(i, \ell)})^{\s{T}}\|_\s{F}\sqrt{\frac{\log (1/\delta_0)}{mt}}\label{eq:lemma-optimize_dp-lrs-2,infty-norm-V1-2,infinty-norm-value-type2}
\end{align}
while \eqref{eq:lemma-optimize_dp-lrs-2,infty-norm-V2-2,infinty-norm-value}, \eqref{eq:lemma-optimize_dp-lrs-2,infty-norm-V3-2,infinty-norm-value} and \eqref{eq:lemma-optimize_dp-lrs-2,infty-norm-Vxi-2,infinty-norm-value} remain the same
\begin{align}
    \|\fl{V}_2\|_{2, \infty} &\leq \zeta(\max_i \|\fl{b}^{\star(i)} - \fl{b}^{(i, \ell)}\|_\infty\|\fl{w}^{(i, \ell)}\|_2)\|\Big((\fl{W}^{(\ell)})^\s{T}\fl{W}^{(\ell)}\Big)^{-1}\|_{2} \nonumber\\
    &\qquad c\|\Big(\frac{1}{t}(\fl{W}^{(\ell)})^\s{T}\fl{W}^{(\ell)}\Big)^{-1}\|_\s{F}\|\fl{b}^{\star(i)} - \fl{b}^{(i, \ell)}\|_2\|\fl{w}^{(i, \ell)}\|_\s{2}\sqrt{\frac{\log (1/\delta_0)}{mt}}.\label{eq:lemma-optimize_dp-lrs-2,infty-norm-V2-2,infinty-norm-value-type2}\\
    \|\fl{V}_3\|_{2, \infty} &\leq \frac{2\sigma_2}{mt}\sqrt{\log(rd/\delta_0)}\cdot\frac{r\sqrt{r}}{\lambda_r\Big(\frac{r}{t}(\fl{W}^{(\ell)})^\s{T}\fl{W}^{(\ell)}\Big)},\label{eq:lemma-optimize_dp-lrs-2,infty-norm-V3-2,infinty-norm-value-type2}\\
    \|\fl{V}_\xi\|_{2, \infty} &\leq \frac{2\sigma\sqrt{\mu^{\star}\lambda^{\star}_r}\log(2rdmt/\delta_0)}{\sqrt{mt}}\frac{r}{\lambda_r\Big(\frac{r}{t}(\fl{W}^{(\ell)})^\s{T}\fl{W}^{(\ell)}\Big)}\label{eq:lemma-optimize_dp-lrs-2,infty-norm-Vxi-2,infinty-norm-value-type2}
\end{align}
We also have the additional term $\fl{V}_4$ s.t.
\begin{align}
    \fl{C}\cdot\frac{\fl{N}_1}{mt}\s{vec}(\fl{U}^\star\fl{Q}^{(\ell-1)}) &= \Big(\frac{1}{t}(\fl{W}^{(\ell)})^\s{T}\fl{W}^{(\ell)} \otimes \fl{I}\Big)^{-1}\frac{\fl{N}_1}{mt}\s{vec}(\fl{U}^\star\fl{Q}^{(\ell-1)}) := \s{vec}(\fl{V}_4)\\
    \iff \fl{V}_4 &= \fl{I}\cdot \s{vec}^{-1}\Big(\frac{\fl{N}_1}{mt}\s{vec}(\fl{U}^\star\fl{Q}^{(\ell-1)})\Big) \cdot \Big(\frac{1}{t}(\fl{W}^{(\ell)})^\s{T}\fl{W}^{(\ell)}\Big)^\s{-T}\nonumber\\
    \implies \|\fl{V}_4\|_{2, \infty} &\leq \lr{\s{vec}^{-1}\Big(\frac{\fl{N}_1}{mt}\s{vec}(\fl{U}^\star\fl{Q}^{(\ell-1)})\Big) \cdot \Big(\frac{1}{t}(\fl{W}^{(\ell)})^\s{T}\fl{W}^{(\ell)}\Big)^\s{-T}}_{2, \infty}\nonumber\\
    &\leq \lr{\s{vec}^{-1}\Big(\frac{\fl{N}_1}{mt}\s{vec}(\fl{U}^\star\fl{Q}^{(\ell-1)})\Big)}_{2, \infty}\lr{\Big(\frac{1}{t}(\fl{W}^{(\ell)})^\s{T}\fl{W}^{(\ell)}\Big)^\s{-T}}_{2}\nonumber\\
    &\leq \lr{\s{vec}^{-1}\Big(\frac{\fl{N}_1}{mt}\s{vec}(\fl{U}^\star\fl{Q}^{(\ell-1)})\Big)}_{2, \infty}\frac{r}{\lambda_r\Big(\frac{r}{t}\fl{W}^{(\ell)})^\s{T}\fl{W}^{(\ell)}\Big)}.\label{eq:lemma-optimize_dp-lrs-2,infty-norm-V4-2,infinty-norm-temp1-type2}
\end{align}
Now, $\fl{N}_1 \sim \sigma_1\cM\cN(\vzero, \fl{I}_{rd\times rd}, \fl{I}_{rd\times rd}) \implies \frac{\fl{N}_1}{mt}\s{vec}(\fl{U}^\star\fl{Q}^{(\ell-1)}) = \frac{1}{mt} \begin{bmatrix} \fl{N}_{1,1}^\s{T}\s{vec}(\fl{U}^\star\fl{Q}^{(\ell-1)})\\ \vdots \\ \fl{N}_{1,j}^\s{T}\s{vec}(\fl{U}^\star\fl{Q}^{(\ell-1)})\\ \vdots \\\fl{N}_{1,rd}^\s{T}\s{vec}(\fl{U}^\star\fl{Q}^{(\ell-1)}) \end{bmatrix}$ where each $\fl{N}_{1,j} \sim \sigma_1\cN(\vzero, \fl{I}_{rd\times rd})$.
Therefore,
\begin{align}
    \lr{\s{vec}^{-1}\Big(\frac{\fl{N}_1}{mt}\s{vec}(\fl{U}^\star\fl{Q}^{(\ell-1)})\Big)}_{2, \infty}^2 = \frac{1}{m^2t^2}\max_{p \in [d]} \sum_{q \in [r]} \Big(\fl{N}_{1, p + q}^\s{T}\s{vec}(\fl{U}^\star\fl{Q}^{(\ell-1)})\Big)^2,\label{eq:lemma-optimize_dp-lrs-2,infty-norm-V4-2,infinty-norm-temp2-type2}
\end{align}
where each $(\fl{N}_{1, p + q})_{e,f}(\s{vec}(\fl{U}^\star\fl{Q}^{(\ell-1)}))_{e,f} \in \cN(0, \sigma_1^2)$  $\forall$ $e,f \in d \times r$. Therefore, using standard gaussian concentration results and taking union bound over $e, f$, we have w.p. atleast $1-\delta_0$,
\begin{align}
    \Big(\fl{N}_{1, p + q}^\s{T}\s{vec}(\fl{U}^\star\fl{Q}^{(\ell-1)})\Big)^2 &= \Big(\sum_e\sum_f (\fl{N}_{1, p + q})_{e,f}(\s{vec}(\fl{U}^\star\fl{Q}^{(\ell-1)}))_{e,f}\Big)^2\nonumber\\
    &\leq rd\sum_e\sum_f |(\fl{N}_{1, p + q})_{e,f}|^2|(\s{vec}(\fl{U}^\star\fl{Q}^{(\ell-1)}))_{e,f}|^2\nonumber\\
    &\leq rd\sum_e\sum_f \sigma_1^2\cdot2\log(2rd/\delta_0)|(\s{vec}(\fl{U}^\star\fl{Q}^{(\ell-1)}))_{e,f}|^2\nonumber\\
    &= rd\sigma_1^2\cdot2\log(2rd/\delta_0)\|\s{vec}(\fl{U}^\star\fl{Q}^{(\ell-1)})\|_2^2.\label{eq:lemma-optimize_dp-lrs-2,infty-norm-V4-2,infinty-norm-temp3-type2}
\end{align}
Using \eqref{eq:lemma-optimize_dp-lrs-2,infty-norm-V4-2,infinty-norm-temp3-type2} in \eqref{eq:lemma-optimize_dp-lrs-2,infty-norm-V4-2,infinty-norm-temp2-type2} and taking a Union Bound $\forall$ $p, q$, we have w.p. $\geq 1- \delta_0$
\begin{align}
    &\lr{\s{vec}^{-1}\Big(\frac{\fl{N}_1}{mt}\s{vec}(\fl{U}^\star\fl{Q}^{(\ell-1)})\Big)}_{2, \infty}^2 \leq \frac{1}{m^2t^2}\max_{p \in [d]} \sum_{q \in [r]} r^2d\sigma_1^2\cdot2\log(2r^2d^2/\delta_0)\|\s{vec}(\fl{U}^\star\fl{Q}^{(\ell-1)})\|_2^2\nonumber\\
    &\implies \lr{\s{vec}^{-1}\Big(\frac{\fl{N}_1}{mt}\s{vec}(\fl{U}^\star\fl{Q}^{(\ell-1)})\Big)}_{2, \infty} \leq \frac{1}{mt}r\sigma_1\sqrt{d\cdot2\log(2r^2d^2/\delta_0)}\|\fl{U}^\star\fl{Q}^{(\ell-1)}\|_\s{F}.\label{eq:lemma-optimize_dp-lrs-2,infty-norm-V4-2,infinty-norm-temp4-type2}
\end{align}
Using \eqref{eq:lemma-optimize_dp-lrs-2,infty-norm-V4-2,infinty-norm-temp4-type2} in \eqref{eq:lemma-optimize_dp-lrs-2,infty-norm-V4-2,infinty-norm-temp1-type2} we get
\begin{align}
    \|\fl{V}_4\|_{2, \infty} &\leq
    \frac{r^2\sigma_1\sqrt{rd\cdot2\log(2r^2d^2/\delta_0)}}{mt\lambda_r\Big(\frac{r}{t}\fl{W}^{(\ell)})^\s{T}\fl{W}^{(\ell)}\Big)}.\label{eq:lemma-optimize_dp-lrs-2,infty-norm-V4-2,infinty-norm-value-type2}
\end{align}

Combining $\fl{V}_1$, $\fl{V}_2$, $\fl{V}_3$, $\fl{V}_4$ and $\fl{V}_\xi$ from \eqref{eq:lemma-optimize_dp-lrs-2,infty-norm-V1-2,infinty-norm-value-type2}, \eqref{eq:lemma-optimize_dp-lrs-2,infty-norm-V2-2,infinty-norm-value-type2}, \eqref{eq:lemma-optimize_dp-lrs-2,infty-norm-V3-2,infinty-norm-value-type2}, \eqref{eq:lemma-optimize_dp-lrs-2,infty-norm-V4-2,infinty-norm-value-type2} and \eqref{eq:lemma-optimize_dp-lrs-2,infty-norm-Vxi-2,infinty-norm-value-type2} respectively in \eqref{eq:lemma-optimize_dp-lrs-2,infty-norm-CvecV-unplugged-type2}, we have $\lr{\s{vec}^{-1}\Big(\fl{C}\Big(\s{vec}(\fl{V}) + \s{vec}\Big(\frac{\fl{N}_2}{mt}\Big) -  \frac{\fl{N}_1}{mt}\s{vec}(\fl{U}^\star\fl{Q}^{(\ell-1)}) + \s{vec}(\fl{\Xi})\Big)}_{2, \infty}$
\begin{align}
    &= \|\fl{V}_1\|_{2, \infty} + \|\fl{V}_2\|_{2, \infty} + \|\fl{V}_3\|_{2, \infty} + \|\fl{V}_4\|_{2, \infty} + \|\fl{V}_\xi\|_{2, \infty}\nonumber\\
    &\leq \|\fl{U}^\star\|_{2, \infty}\|\frac{1}{t}\sum_{i \in [t]}\Big(\frac{1}{t}(\fl{W}^{(\ell)})^\s{T}\fl{W}^{(\ell)}\Big)^{-1}\fl{h}^{(i, \ell)}(\fl{w}^{(i, \ell)})^{\s{T}}\|_{2} \nonumber\\
    &\qquad + c\|\fl{U}^\star\|_\s{F}\|\Big(\frac{1}{t}(\fl{W}^{(\ell)})^\s{T}\fl{W}^{(\ell)}\Big)^{-1}\|\|\fl{h}^{(i, \ell)}(\fl{w}^{(i, \ell)})^{\s{T}}\|_\s{F}\sqrt{\frac{\log (1/\delta_0)}{mt}}\nonumber\\
    &\qquad + \zeta(\max_i \|\fl{b}^{\star(i)} - \fl{b}^{(i, \ell)}\|_\infty\|\fl{w}^{(i, \ell)}\|_2)\|\Big((\fl{W}^{(\ell)})^\s{T}\fl{W}^{(\ell)}\Big)^{-1}\|_{2} \nonumber\\
    &\qquad + c\|\Big(\frac{1}{t}(\fl{W}^{(\ell)})^\s{T}\fl{W}^{(\ell)}\Big)^{-1}\|_\s{F}\|\fl{b}^{\star(i)} - \fl{b}^{(i, \ell)}\|_2\|\fl{w}^{(i, \ell)}\|_\s{2}\sqrt{\frac{\log (1/\delta_0)}{mt}}\nonumber\\
    &\qquad + \frac{2\sigma_2}{mt}\sqrt{\log(rd/\delta_0)}\cdot\frac{r\sqrt{r}}{\lambda_r\Big(\frac{r}{t}(\fl{W}^{(\ell)})^\s{T}\fl{W}^{(\ell)}\Big)} + \frac{r^2\sigma_1\sqrt{rd\cdot2\log(2r^2d^2/\delta_0)}}{mt\lambda_r\Big(\frac{r}{t}\fl{W}^{(\ell)})^\s{T}\fl{W}^{(\ell)}\Big)}\nonumber\\
    &\qquad + \frac{2\sigma\sqrt{\mu^{\star}\lambda^{\star}_r}\log(2rdmt/\delta_0)}{\sqrt{mt}}\frac{r}{\lambda_r\Big(\frac{r}{t}(\fl{W}^{(\ell)})^\s{T}\fl{W}^{(\ell)}\Big)}.\label{eq:lemma-optimize_dp-lrs-2,infty-norm-CvecV-value-type2}
\end{align}
Therefore, using \eqref{eq:lemma-optimize_dp-lrs-2,infty-norm-CvecV-value-type2}, \eqref{eq:lemma-optimize_dp-lrs-u-2norm-bound-type1-num-val}, \eqref{eq:lemma-optimize_dp-lrs-U-Fnorm-bound-N2-Fnorm}  and \eqref{lemma-optimize_dp-lrs-U-Fnorm-bound-Xi-F-norm} in \eqref{eq:lemma-optimize_dp-lrs-2,infty-norm-AinvV,infty-unplugged-2-type2}, we have $\|\fl{U}^{(\ell)} - \fl{U}^{\star}\fl{Q}^{(\ell-1)}\|_{2,\infty}$
\begin{align}
    &\leq \|\fl{U}^\star\|_{2, \infty}\|\frac{1}{t}\sum_{i \in [t]}\Big(\frac{1}{t}(\fl{W}^{(\ell)})^\s{T}\fl{W}^{(\ell)}\Big)^{-1}\fl{h}^{(i, \ell)}(\fl{w}^{(i, \ell)})^{\s{T}}\|_{2} \nonumber\\
    &\qquad + c\|\fl{U}^\star\|_\s{F}\|\Big(\frac{1}{t}(\fl{W}^{(\ell)})^\s{T}\fl{W}^{(\ell)}\Big)^{-1}\|\|\fl{h}^{(i, \ell)}(\fl{w}^{(i, \ell)})^{\s{T}}\|_\s{F}\sqrt{\frac{\log (1/\delta_0)}{mt}}\nonumber\\
    &\qquad + \zeta(\max_i \|\fl{b}^{\star(i)} - \fl{b}^{(i, \ell)}\|_\infty\|\fl{w}^{(i, \ell)}\|_2)\|\Big((\fl{W}^{(\ell)})^\s{T}\fl{W}^{(\ell)}\Big)^{-1}\|_{2} \nonumber\\
    &\qquad + c\|\Big(\frac{1}{t}(\fl{W}^{(\ell)})^\s{T}\fl{W}^{(\ell)}\Big)^{-1}\|_\s{F}\|\fl{b}^{\star(i)} - \fl{b}^{(i, \ell)}\|_2\|\fl{w}^{(i, \ell)}\|_\s{2}\sqrt{\frac{\log (1/\delta_0)}{mt}}\nonumber\\
    &\qquad + \frac{2\sigma_2}{mt}\sqrt{\log(rd/\delta_0)}\cdot\frac{r\sqrt{r}}{\lambda_r\Big(\frac{r}{t}(\fl{W}^{(\ell)})^\s{T}\fl{W}^{(\ell)}\Big)} + \frac{r^2\sigma_1\sqrt{rd\cdot2\log(2r^2d^2/\delta_0)}}{mt\lambda_r\Big(\frac{r}{t}\fl{W}^{(\ell)})^\s{T}\fl{W}^{(\ell)}\Big)} \nonumber\\
    &\qquad + \frac{2\sigma\sqrt{\mu^{\star}\lambda^{\star}_r}\log(2rdmt/\delta_0)}{\sqrt{mt}}\frac{r}{\lambda_r\Big(\frac{r}{t}(\fl{W}^{(\ell)})^\s{T}\fl{W}^{(\ell)}\Big)}\nonumber\\
    &\qquad + \frac{\sqrt{r}\Big(c\sqrt{\sum_{i \in [t]}\sum_{j \in [m]}\|\fl{w}^{(i, \ell)}\|_2^4}\sqrt{\frac{rd\log(rd/ \delta_0)}{m^2t^2}} + \frac{\sigma_1}{mt}\Big(2\sqrt{rd} + 2\sqrt{2rd\log(2rd/\delta_0)}\Big)\Big)}{\frac{1}{r}\lambda_r\Big(\frac{r}{t}(\fl{W}^{(\ell)})^\s{T}\fl{W}^{(\ell)} \Big)}\cdot\nonumber\\
    &\qquad \Big\{\frac{2}{t}\|\fl{U}^{\star}\fl{Q}^{(\ell-1)}(\fl{H}^{(\ell)})^\s{T}\fl{W}^{(\ell)}\|_\s{F} + \sqrt{\frac{4\zeta}{t}}(\max_i\|\fl{w}^{(i, \ell)}\|_2)\|\fl{b}^{\star(i)} - \fl{b}^{(i, \ell)}\|_2\nonumber\\ 
    &\qquad +4\Big(\|\fl{U}^{\star}\fl{Q}^{(\ell-1)}\|\|\fl{h}^{(i, \ell)}\|_2\|\fl{w}^{(i, \ell)}\|_2 + \|\fl{b}^{\star(i)} - \fl{b}^{(i, \ell)}\|_2\|\fl{w}^{(i, \ell)}\|_2\Big)\sqrt{\frac{d\log (rd/\delta_0)}{mt}}\nonumber\\
    &\qquad + \frac{\sigma_2}{mt}6\sqrt{rd\log(rd)} + \frac{\sigma_1}{mt}\Big(2\sqrt{rd} + 4\sqrt{\log rd}\Big)\sqrt{r} + \frac{2\sigma\sqrt{d\mu^{\star}\lambda^{\star}_r}\log(2rdmt/\delta_0)}{\sqrt{mt}}\Big\}.\label{eq:lemma-optimize_dp-lrs-2,infty-val2}
\end{align}
\eqref{eq:lemma-optimize_dp-lrs-2,infty-val1} and \eqref{eq:lemma-optimize_dp-lrs-2,infty-val2} give us the required result.
\end{proof}

\begin{coro}\label{inductive-corollary:optimize_dp-lrs-U-infty-norm-bounds}
If $\s{B}_{\fl{U}^{(\ell-1)}} =   \cO\Big(\frac{1}{\sqrt{r\mu^\star}}\}\Big)$, $\sqrt{\frac{r\log(1/ \delta_0)}{m}} = \cO(1)$, $\sqrt{\nu^{(\ell-1)}} = \cO\Big(\frac{1}{\sqrt{r\mu^\star}}$, $\sqrt{\frac{r^2\log(r/\delta_0)}{m}} = \cO\Big(\frac{1}{\sqrt{\mu^\star}}\Big)$, $\epsilon < \sqrt{\mu^\star\lambda_r^\star}$, $\sqrt{\frac{r^2\zeta}{t}} = \min\{\cO\Big(\frac{1}{\sqrt{\mu^\star}}\Big), \cO\Big(\frac{1}{\mu^\star}\Big)\}$, $\Lambda' = \cO\Big(\sqrt{\frac{\lambda_r^\star}{\lambda_1^\star}}\Big)$, $\Lambda = \cO(\sqrt{\lambda_r^\star})$, $\sigma\sqrt{\frac{r^2\log^2 (r\delta^{-1})}{m}} = \cO(\sqrt{\lambda_r^\star})$, $\sqrt{\frac{r^3d\log(1/ \delta_0)}{mt}} = \min\{\cO\Big(\frac{1}{\mu^\star}\Big), \cO\Big(\frac{1}{\mu^\star\lambda_r^\star}\Big), \cO\Big(\frac{1}{\sqrt{r}(\mu^\star)^2\sqrt{\mu^\star}\lambda_r^\star\sqrt{k}}\Big)\}$,  $\frac{\sigma_1}{mt}\Big(2\sqrt{rd} + 4\sqrt{\log rd}\Big) = \min\{\cO(\lambda_r^\star), \cO\Big(\frac{1}{r}\Big), \cO\Big(\frac{1}{r\sqrt{r}\mu^\star\sqrt{\mu^\star}\sqrt{k}}\Big)\}$, $\sqrt{\nu^\star} = \cO\Big(\frac{1}{\sqrt{\mu^\star}}\frac{\lambda_r^\star}{\lambda_1^\star}\Big)$, $\Lambda = \cO(\sqrt{\lambda_r^{\star}})$, $mt=\Omega(\sigma_2 r(\mu^{\star})^{1/2}\sqrt{rd\log d}/\lambda_r^{\star})$ and $mt=\widetilde{\Omega}( \sigma^2dr^3\mu^{\star}(1+1/\lambda^{\star}_r))$, $t=\widetilde{\Omega}(\zeta(\mu^{\star}\lambda^{\star}_r)\max(1,\lambda^{\star}_r/r))$, $m=\widetilde{\Omega}(\sigma^2r^3/\lambda^{\star}_r)$. $\sqrt{\frac{k}{d}} = \cO(1)$ and Assumption \ref{assum:init} holds for iteration $\ell-1$, then, w.p. $1-\cO(\delta_0)$,


   \begin{align}
        \fl{U}^{(\ell)}_{2, \infty} &= \resizebox{.9\textwidth}{!}{$ \cO\Big(\frac{1}{\sqrt{k\mu^\star}}\Big) + \frac{1}{\sqrt{k}}\Big\{\cO\Big(\frac{\Lambda}{\sqrt{\mu^\star\lambda_r^\star}}\Big) + \cO\Big(\frac{\sigma_2r}{mt\lambda_r^\star}\sqrt{rd\log(rd)}\Big) +  \cO\Big(\sigma\sqrt{\frac{r^3d\log^2 (2rdmt/\delta_0)}{mt\lambda_r^\star}}\Big\}$}\nonumber
    \end{align}
    and $\|\fl{U}^{(\ell)} - \fl{U}^{\star}\fl{Q}^{(\ell-1)}\|_{2,\infty}$
    \begin{align}
         &= 
        \cO\Big(\s{B}_{\fl{U^{(\ell-1)}}}\sqrt{\frac{\lambda_r^\star}{\lambda_1^\star k}}\Big) + \frac{1}{\sqrt{k}}\Big\{\cO(\Lambda' + \cO\Big(\frac{\Lambda}{\sqrt{\mu^\star\lambda_r^\star}}\Big)+ \cO\Big(\frac{\sigma_2r}{mt\lambda_r^\star}\sqrt{rd\log(rd)}\Big)\nonumber\\
        &\qquad + \cO\Big(\frac{\sigma_1r^{3/2}}{mt\lambda_r^\star}\sqrt{rd\log rd}\Big) + \cO\Big(\sigma\sqrt{\frac{r^3d\mu^\star\log^2 (r\delta^{-1})}{mt\lambda_r^\star}}\Big)\Big\}.\nonumber
   \end{align}
\end{coro}
\begin{proof}
    From the Lemma statement we have,
    $ \lr{\fl{U}^{(\ell)}}_{2, \infty} $
    \begin{align}
       &\leq \|\fl{U}^\star\|_{2, \infty}\|\frac{1}{t}\sum_{i \in [t]}\Big(\frac{1}{t}(\fl{W}^{(\ell)})^\s{T}\fl{W}^{(\ell)}\Big)^{-1}\fl{w}^{\star(i)}(\fl{w}^{(i, \ell)})^{\s{T}}\|_{2} \nonumber\\
        & + c\|\fl{U}^\star\|_\s{F}\|\Big(\frac{1}{t}(\fl{W}^{(\ell)})^\s{T}\fl{W}^{(\ell)}\Big)^{-1}\|\|\fl{w}^{\star(i)}(\fl{w}^{(i, \ell)})^{\s{T}}\|_\s{F}\sqrt{\frac{\log (1/\delta_0)}{mt}}\nonumber\\
        & + \zeta(\max_i \|\fl{b}^{\star(i)} - \fl{b}^{(i, \ell)}\|_\infty\|\fl{w}^{(i, \ell)}\|_2)\|\Big((\fl{W}^{(\ell)})^\s{T}\fl{W}^{(\ell)}\Big)^{-1}\|_{2} \nonumber\\
        & + c\|\Big(\frac{1}{t}(\fl{W}^{(\ell)})^\s{T}\fl{W}^{(\ell)}\Big)^{-1}\|_\s{F}\|\fl{b}^{\star(i)} - \fl{b}^{(i, \ell)}\|_2\|\fl{w}^{(i, \ell)}\|_\s{2}\sqrt{\frac{\log (1/\delta_0)}{mt}}\nonumber\\
        & + \frac{2\sigma_2}{mt}\sqrt{\log(rd/\delta_0)}\cdot\frac{r\sqrt{r}}{\lambda_r\Big(\frac{r}{t}(\fl{W}^{(\ell)})^\s{T}\fl{W}^{(\ell)}\Big)} + \frac{2\sigma\sqrt{\mu^{\star}\lambda^{\star}_r}\log(2rdmt/\delta_0)}{\sqrt{mt}}\frac{r}{\lambda_r\Big(\frac{r}{t}(\fl{W}^{(\ell)})^\s{T}\fl{W}^{(\ell)}\Big)}\nonumber\\
        & + \frac{\sqrt{r}\Big(c\sqrt{\sum_{i \in [t]}\sum_{j \in [m]}\|\fl{w}^{(i, \ell)}\|_2^4}\sqrt{\frac{rd\log(rd/ \delta_0)}{m^2t^2}} + \frac{\sigma_1}{mt}\Big(2\sqrt{rd} + 2\sqrt{2rd\log(2rd/\delta_0)}\Big)\Big)}{\frac{1}{r}\lambda_r\Big(\frac{r}{t}(\fl{W}^{(\ell)})^\s{T}\fl{W}^{(\ell)} \Big)}\cdot\nonumber\\
        & \Big\{\frac{2}{t}\|\fl{U}^{\star}(\fl{W}^{\star})^\s{T}\fl{W}^{(\ell)}\|_\s{F} + \sqrt{\frac{4\zeta}{t}}(\max_i\|\fl{w}^{(i, \ell)}\|_2)\|\fl{b}^{\star(i)} - \fl{b}^{(i, \ell)}\|_2\nonumber\\ 
        & +4\Big(\|\fl{U}^{\star}\|\|\fl{w}^{\star(i)}\|_2\|\fl{w}^{(i, \ell)}\|_2 + \|\fl{b}^{\star(i)} - \fl{b}^{(i, \ell)}\|_2\|\fl{w}^{(i, \ell)}\|_2\Big)\sqrt{\frac{d\log (rd/\delta_0)}{mt}}\nonumber\\
        & + \frac{\sigma_2}{mt}6\sqrt{rd\log(rd)} + \frac{2\sigma\sqrt{d\mu^{\star}\lambda^{\star}_r}\log(2rdmt/\delta_0)}{\sqrt{mt}}\Big\}
    \end{align}
    As done in the analysis for Corollary~\ref{inductive-corollary:optimize_dp-lrs-U-F-norm-bounds}, using Assumption~\ref{assum:init} to plug in values for $\fl{b}^{(i, \ell)}$ and $\fl{Q}^{(\ell-1)}$, the fact that $\fl{U}^\star$ is orthonormal and norm and incoherence bounds for $\fl{H}^{(\ell)}$ and $\fl{W}^{(\ell)}$ from Corollaries \ref{inductive-corollary:optimize_dp-lrs-h,H-bounds} and \ref{inductive-inductive-corollary:optimize_dp-lrs-w-incoherence}, the above becomes $\fl{U}^{(\ell)}_{2, \infty}$
    \begin{align}
        &\leq \resizebox{.97\textwidth}{!}{$\sqrt{\frac{\nu^\star}{k}}\frac{\sqrt{\lambda_1\lambda_1^\star}}{\lambda_r} + c\frac{r\sqrt{r}}{\lambda_r}\sqrt{\mu\lambda_r}\sqrt{\mu^\star\lambda_r^\star}\sqrt{\frac{\log (1/\delta_0)}{mt}} + \zeta\Big(c'\sqrt{\mu^\star\lambda_r^\star}\s{B}_{\fl{U^{(\ell-1)}}}\sqrt{\frac{\lambda_r^\star}{\lambda_1^\star k}} + \frac{\Lambda}{\sqrt{k}}\Big)\frac{r\sqrt{\mu\lambda_r}}{t\lambda_r}$} \nonumber\\
        & + c\frac{r\sqrt{r}}{\lambda_r}\Big(c'\sqrt{\mu^\star\lambda_r^\star}\s{B}_{\fl{U^{(\ell-1)}}}\sqrt{\frac{\lambda_r^\star}{\lambda_1^\star}} + \Lambda\Big)\sqrt{\mu\lambda_r}\sqrt{\frac{\log (1/\delta_0)}{mt}}\nonumber\\
        & + \frac{2\sigma_2}{mt}\sqrt{\log(rd/\delta_0)}\cdot\frac{r\sqrt{r}}{\lambda_r} + \frac{2\sigma\sqrt{\mu^{\star}\lambda^{\star}_r}\log(2rdmt/\delta_0)}{\sqrt{mt}}\frac{r}{\lambda_r}\nonumber\\
        & + \sqrt{r}\Big(c\mu\sqrt{\frac{r^3d\log(rd/ \delta_0)}{mt}} + \frac{\sigma_1r}{mt\lambda_r}\Big(2\sqrt{rd} + 2\sqrt{2rd\log(2rd/\delta_0)}\Big)\Big)\cdot\nonumber\\
        & \Big\{2\sqrt{\mu^\star\lambda_r^\star}\sqrt{\mu\lambda_r} + \sqrt{\frac{4\zeta}{t}}\sqrt{\mu\lambda_r}\Big(c'\sqrt{\mu^\star\lambda_r^\star}\s{B}_{\fl{U^{(\ell-1)}}}\sqrt{\frac{\lambda_r^\star}{\lambda_1^\star}} + \Lambda\Big)\nonumber\\ 
        & +4\sqrt{\mu\lambda_r}\Big(\sqrt{\mu^\star\lambda_r^\star} + \Big(c'\sqrt{\mu^\star\lambda_r^\star}\s{B}_{\fl{U^{(\ell-1)}}}\sqrt{\frac{\lambda_r^\star}{\lambda_1^\star}} + \Lambda\Big)\Big)\sqrt{\frac{d\log (rd/\delta_0)}{mt}}\nonumber\\
        & + \frac{\sigma_2}{mt}6\sqrt{rd\log(rd)} + \frac{2\sigma\sqrt{d\mu^{\star}\lambda^{\star}_r}\log(2rdmt/\delta_0)}{\sqrt{mt}}\Big\}\\
        &= J_1 + J_2 \label{eq:inductive-corollary-optimize_dp-lrs-U-infty-norm-bounds-unplugged1}
    \end{align}
    where $J_1$ denotes the terms which arise from analysing the problem in the noiseless setting and $J_2$ denotes the contribution of noise terms ($\sigma_1, \sigma_2, \sigma, \Lambda, \Lambda'$). Now, $J_1$
    \begin{align}
    &= \sqrt{\frac{\nu^\star}{k}}\frac{\sqrt{\lambda_1\lambda_1^\star}}{\lambda_r} + c\sqrt{r}\frac{r}{\lambda_r}\sqrt{\mu\lambda_r}\sqrt{\mu^\star\lambda_r^\star}\sqrt{\frac{\log (1/\delta_0)}{mt}} + \zeta\cdot c'\sqrt{\mu^\star\lambda_r^\star}\s{B}_{\fl{U^{(\ell-1)}}}\sqrt{\frac{\lambda_r^\star}{\lambda_1^\star k}} \cdot \sqrt{\mu\lambda_r}\frac{r}{t\lambda_r} \nonumber\\
    &\qquad + c\frac{r\sqrt{r}}{\lambda_r}\cdot c'\sqrt{\mu^\star\lambda_r^\star}\s{B}_{\fl{U^{(\ell-1)}}}\sqrt{\frac{\lambda_r^\star}{\lambda_1^\star}} \cdot \sqrt{\mu\lambda_r}\sqrt{\frac{\log (1/\delta_0)}{mt}}\nonumber\\
    &\qquad + \sqrt{r}\Big(c\mu\sqrt{\frac{r^3d\log(rd/ \delta_0)}{mt}} + \frac{\sigma_1r}{mt\lambda_r}\Big(2\sqrt{rd} + 2\sqrt{2rd\log(2rd/\delta_0)}\Big)\Big)\cdot\nonumber\\
    &\qquad \Big\{2\sqrt{\mu^\star\lambda_r^\star}\sqrt{\mu\lambda_r} + \sqrt{\frac{4\zeta}{t}}\sqrt{\mu\lambda_r}\cdot c'\sqrt{\mu^\star\lambda_r^\star}\s{B}_{\fl{U^{(\ell-1)}}}\sqrt{\frac{\lambda_r^\star}{\lambda_1^\star}}\nonumber\\ 
    &\qquad +4\sqrt{\mu\lambda_r}\Big(\sqrt{\mu^\star\lambda_r^\star} + c'\sqrt{\mu^\star\lambda_r^\star}\s{B}_{\fl{U^{(\ell-1)}}}\sqrt{\frac{\lambda_r^\star}{\lambda_1^\star}}\Big)\sqrt{\frac{d\log (rd/\delta_0)}{mt}}\}.
    \end{align}
     Using $\sqrt{\nu^\star} = \cO\Big(\frac{\lambda_r^\star}{\lambda_1^\star\sqrt{\mu^\star}}\Big)$ and rearranging the terms in the above we get $J_1$
    \begin{align}
    &= \frac{1}{\sqrt{k\mu^\star}}\Big\{\cO\Big(\frac{\lambda_r^\star}{\lambda_r}\sqrt{\frac{\lambda_1^\star}{\lambda_1}}\Big) + \sqrt{\frac{k}{d}}\sqrt{\frac{\lambda_r^\star}{\lambda_r}}\sqrt{\frac{\mu}{\mu^\star}}\sqrt{\frac{r^3d\log (1/\delta_0)}{mt}} + \s{B}_{\fl{U^{(\ell-1)}}}\sqrt{\frac{\lambda_r^\star}{\lambda_1^\star }}\frac{r\zeta}{t} c'\sqrt{\mu^\star\mu}\sqrt{\frac{\lambda_r^\star}{\lambda_r}}\nonumber\\
    &\qquad + \s{B}_{\fl{U^{(\ell-1)}}}\sqrt{\frac{\lambda_r^\star}{\lambda_1^\star}}cc'\sqrt{\frac{\lambda_r^\star}{\lambda_r}}\sqrt{\mu^\star\mu}\sqrt{\mu^\star}\sqrt{\frac{r^3d\log (1/\delta_0)}{mt}}\sqrt{\frac{k}{d}}\nonumber\\
    &\qquad + \sqrt{r}\Big(c\mu\sqrt{\frac{r^3d\log(rd/ \delta_0)}{mt}} + \frac{\sigma_1r}{mt\lambda_r}\Big(2\sqrt{rd} + 2\sqrt{2rd\log(2rd/\delta_0)}\Big)\Big) \nonumber\\
    &\qquad \sqrt{\mu^\star}\sqrt{\mu^\star\lambda_r^\star}\sqrt{\mu\lambda_r}\Big\{2 + \sqrt{\frac{4\zeta}{t}}c'\s{B}_{\fl{U^{(\ell-1)}}}\sqrt{\frac{\lambda_r^\star}{\lambda_1^\star}} +4\Big(1 + c'\s{B}_{\fl{U^{(\ell-1)}}}\sqrt{\frac{\lambda_r^\star}{\lambda_1^\star}}\Big)\sqrt{\frac{d\log (rd/\delta_0)}{mt}}\Big\}\Big\}.\nonumber
    \end{align}
Using $\sqrt{\frac{r^3d\log(1/ \delta_0)}{mt}} = \cO\Big(\frac{1}{\mu^\star}\Big)$, $\sqrt{\frac{k}{d}} = \cO(1)$, $\s{B}_{\fl{U}^{(\ell-1)}} = \cO\Big(\frac{1}{\sqrt{r\mu^\star}}\Big)$, $\sqrt{\frac{r^2\log(r/\delta_0)}{m}} = \cO\Big(\frac{1}{\sqrt{\mu^\star}}\Big)$, $\sqrt{\frac{r^3d\log(1/ \delta_0)}{mt}} = \cO\Big(\frac{1}{\sqrt{r}(\mu^\star)^2\sqrt{\mu^\star}\lambda_r^\star\sqrt{k}}\Big)$, $\frac{\sigma_1}{mt}\Big(2\sqrt{rd} + 4\sqrt{\log rd}\Big) = \cO\Big(\frac{1}{r\sqrt{r}\mu^\star\sqrt{\mu^\star}\sqrt{k}}\Big)$,  $\sqrt{\frac{r^2\zeta}{t}} = \cO\Big(\frac{1}{\sqrt{\mu^\star}}\Big)$, eigenvalue ratios and incoherence bounds for $\fl{H}^{(\ell)}$ and $\fl{W}^{(\ell)}$ from Corollaries \ref{inductive-corollary:optimize_dp-lrs-h,H-bounds} and \ref{inductive-inductive-corollary:optimize_dp-lrs-w-incoherence} as well as $\s{B}_{\fl{U}^{(\ell-1)}} =  \cO\Big(\frac{1}{\sqrt{r\mu^\star}}\Big)$, $\lambda_r^\star \leq \lambda_1^\star$, $r\geq1$,$\mu^\star \geq 1$ for the bracketed terms, we get
    \begin{align}
    &= \frac{1}{\sqrt{k\mu^\star}}\Big\{\cO(1)\Big\}\nonumber\\
    &= \cO\Big(\frac{1}{\sqrt{k\mu^\star}}\Big).\label{eq:inductive-corollary-optimize_dp-lrs-U-infty-norm-bounds-J1-val}
\end{align}
Similarly, we have $J_2$ 
\begin{align}
    &\leq \zeta\frac{\Lambda}{\sqrt{k}}\sqrt{\mu\lambda_r}\frac{r}{t\lambda_r} + c\frac{r\sqrt{r}}{\lambda_r}\Lambda\sqrt{\mu\lambda_r}\sqrt{\frac{\log (1/\delta_0)}{mt}} + \frac{2\sigma_2}{mt}\sqrt{\log(rd/\delta_0)}\cdot\frac{r\sqrt{r}}{\lambda_r} \nonumber\\
    & + \resizebox{.97\textwidth}{!}{$\frac{2\sigma\sqrt{\mu^{\star}\lambda^{\star}_r }\log(2rdmt/\delta_0)r}{\sqrt{mt}\lambda_r} + \sqrt{r}\Big(c\mu\sqrt{\frac{r^3d\log(rd/ \delta_0)}{mt}} + \frac{\sigma_1r}{mt\lambda_r}\Big(2\sqrt{rd} + 2\sqrt{2rd\log(2rd/\delta_0)}\Big)\Big)\cdot$}\nonumber\\
    & \Big\{\sqrt{\frac{4\zeta}{t}}\sqrt{\mu\lambda_r}\Lambda +4\sqrt{\mu\lambda_r} \Lambda \sqrt{\frac{d\log (rd/\delta_0)}{mt}} + \frac{\sigma_2}{mt}6\sqrt{rd\log(rd)} + \frac{2\sigma\sqrt{d\mu^{\star}\lambda^{\star}_r}\log(2rdmt/\delta_0)}{\sqrt{mt}}\Big\}.
\end{align}
Using $\sqrt{\frac{r^3d\log(1/ \delta_0)}{mt}} = \cO\Big(\frac{1}{\sqrt{r}(\mu^\star)^2\sqrt{\mu^\star}\lambda_r^\star\sqrt{k}}\Big)$, $\frac{\sigma_1}{mt}\Big(2\sqrt{rd} + 4\sqrt{\log rd}\Big) = \cO\Big(\frac{1}{r\sqrt{r}\mu^\star\sqrt{\mu^\star}\sqrt{k}}\Big)$ and rearranging the terms in above gives
\begin{align}
    &= \resizebox{.97\textwidth}{!}{$\Lambda\Big\{\frac{r\zeta}{t}\frac{\sqrt{\mu\lambda_r}}{\lambda_r\sqrt{k}} + \frac{c\sqrt{\mu}}{\sqrt{\lambda_r}}\sqrt{\frac{r^3\log (1/\delta_0)}{mt}} + \sqrt{\mu\lambda_r}\cdot\cO\Big(\frac{1}{\sqrt{r}(\mu^\star)^2\sqrt{\mu^\star}\lambda_r^\star\sqrt{k}}\Big)\Big(\sqrt{\frac{4\zeta}{t}} +4\sqrt{\frac{d\log (rd/\delta_0)}{mt}} \Big)\Big\}$}\nonumber\\  
    &+ \sigma_2\Big\{\frac{2}{mt}\sqrt{\log(rd/\delta_0)}\cdot\frac{r\sqrt{r}}{\lambda_r} + \cO\Big(\frac{1}{\sqrt{r}(\mu^\star)^2\sqrt{\mu^\star}\lambda_r^\star\sqrt{k}}\Big)\cdot \frac{1}{mt}6\sqrt{rd\log(rd)}\Big\} \nonumber\\
    &+ \sigma\Big\{\frac{2\sqrt{\mu^{\star}\lambda^{\star}_r}\log(2rdmt/\delta_0)}{\sqrt{mt}}\frac{r}{\lambda_r} + \cO\Big(\frac{1}{\sqrt{r}(\mu^\star)^2\sqrt{\mu^\star}\lambda_r^\star\sqrt{k}}\Big)\cdot\frac{2\sqrt{d\mu^{\star}\lambda^{\star}_r}\log(2rdmt/\delta_0)}{\sqrt{mt}}\Big\}.\nonumber
\end{align}
Using  $\sqrt{\frac{r^2\zeta}{t}} = \cO\Big(\frac{1}{\sqrt{\mu^\star}}\Big)$, $\sqrt{\frac{k}{d}} = \cO(1)$, $\sqrt{\frac{r^3d\log(1/ \delta_0)}{mt}} = \cO\Big(\frac{1}{\mu^\star}\Big)$, $\Lambda = \cO(\sqrt{\lambda_r^{\star}})$, $mt=\Omega(\sigma_2 r(\mu^{\star})^{1/2}\sqrt{rd\log d}/\lambda_r^{\star})$ and $mt=\widetilde{\Omega}( \sigma^2dr^3\mu^{\star}/\lambda^{\star}_r)$, $t=\widetilde{\Omega}(\zeta(\mu^{\star}\lambda^{\star}_r)\max(1,\lambda^{\star}_r/r))$, eigenvalue ratios and incoherence bounds for $\fl{H}^{(\ell)}$ and $\fl{W}^{(\ell)}$ from Corollaries \ref{inductive-corollary:optimize_dp-lrs-h,H-bounds} and \ref{inductive-inductive-corollary:optimize_dp-lrs-w-incoherence} as well as $\lambda_r^\star \leq \lambda_1^\star$, $r\geq1$,$\mu^\star \geq 1$ for the bracketed terms, the above simplifies to   
\begin{align}
    J_2 &= \frac{1}{\sqrt{k}}\Big\{\cO\Big(\frac{\Lambda}{\sqrt{\mu^\star\lambda_r^\star}}\Big) + \cO\Big(\frac{\sigma_2r}{mt\lambda_r^\star}\sqrt{rd\log(rd)}\Big) +  \cO\Big(\sigma\sqrt{\frac{r^3d\log^2 (2rdmt/\delta_0)}{mt\lambda_r^\star}}\Big\}. \label{eq:inductive-corollary-optimize_dp-lrs-U-infty-norm-bounds-J2-val}
\end{align}
Using \eqref{eq:inductive-corollary-optimize_dp-lrs-U-infty-norm-bounds-J1-val} and \eqref{eq:inductive-corollary-optimize_dp-lrs-U-infty-norm-bounds-J2-val} in \eqref{eq:inductive-corollary-optimize_dp-lrs-U-infty-norm-bounds-unplugged1}
gives us the required bound for $\|\fl{U}^{(\ell)}\|_{2,\infty}$.

Similarly, using Assumption~\ref{assum:init} to plug in values for $\fl{b}^{(i, \ell)}$ and $\fl{Q}^{(\ell-1)}$, the fact that $\fl{U}^\star$ is orthonormal and norm and incoherence bounds for $\fl{H}^{(\ell)}$ and $\fl{W}^{(\ell)}$ from Corollaries \ref{inductive-corollary:optimize_dp-lrs-h,H-bounds} and \ref{inductive-inductive-corollary:optimize_dp-lrs-w-incoherence}, the above becomes $\fl{U}^{(\ell)}_{2, \infty}$, we can simplify and express $\|\fl{U}^{(\ell-1)} - \fl{U}^{\star}\fl{Q}^{(\ell-1)}\|_{2,\infty} $ as
\begin{align}
    &\leq \sqrt{\frac{\nu^\star}{k}}\frac{r}{t\lambda_r}\cdot \cO\Big(\sqrt{\frac{t}{r}}\sqrt{\lambda_1^\star}\Big)\cdot\nonumber\\
    &\qquad \Big(\cO\Big(\frac{\s{B}_{\fl{U^{(\ell-1)}}}\frac{\lambda_r^\star}{\lambda_1^\star}\sqrt{\frac{t}{r}\lambda^\star_1}}{\sqrt{r\mu^\star}}\Big) + \cO\Big(\sqrt{\frac{t}{r}\lambda^\star_r}\frac{\Lambda'}{\sqrt{r\mu^\star}}\Big) + \cO\Big(\frac{\Lambda}{\sqrt{r\mu^\star}}\sqrt{\frac{t}{r}}) + \cO\Big(\sigma\sqrt{\frac{r\log^2 (r\delta^{-1})}{m}}\Big)\Big) \nonumber\\
    &\qquad + \sqrt{r}\cO\Big(\frac{r}{\lambda_r}\sqrt{\mu\lambda_r}\Big)\sqrt{\frac{\log (1/\delta_0)}{mt}}\cdot\nonumber\\
    &\qquad \Big(\cO\Big(\frac{\sqrt{\mu^\star\lambda_r^\star}\s{B}_{\fl{U^{(\ell-1)}}}\sqrt{\frac{\lambda_r^\star}{\lambda_1^\star}}}{\sqrt{r\mu^\star}}\Big) + \cO\Big(\frac{\Lambda'\|\fl{w}^{\star(i)}\|_2}{\sqrt{r\mu^\star}}\Big) + \cO\Big(\frac{\Lambda}{\sqrt{r\mu^\star}}\Big) + \cO\Big(\sigma\sqrt{\frac{r\log^2 (r\delta^{-1})}{m}}\Big)\Big)\nonumber\\
    &\qquad + \zeta\Big(c'\sqrt{\mu^\star\lambda_r^\star}\s{B}_{\fl{U^{(\ell-1)}}}\sqrt{\frac{\lambda_r^\star}{\lambda_1^\star k}} + \frac{\Lambda}{\sqrt{k}}\Big)\sqrt{\mu\lambda_r}\frac{r}{t\lambda_r} \nonumber\\
    &\qquad + c\frac{r\sqrt{r}}{\lambda_r}\Big(c'\sqrt{\mu^\star\lambda_r^\star}\s{B}_{\fl{U^{(\ell-1)}}}\sqrt{\frac{\lambda_r^\star}{\lambda_1^\star}} + \Lambda\Big)\sqrt{\mu\lambda_r}\sqrt{\frac{\log (1/\delta_0)}{mt}}\nonumber\\
    &\qquad + \frac{2\sigma_2}{mt}\sqrt{\log(rd/\delta_0)}\cdot\frac{r\sqrt{r}}{\lambda_r} + \frac{r^2\sigma_1\sqrt{rd\cdot2\log(2r^2d^2/\delta_0)}}{mt\lambda_r} \nonumber\\
    &\qquad + \frac{2\sigma\sqrt{\mu^{\star}\lambda^{\star}_r}\log(2rdmt/\delta_0)}{\sqrt{mt}}\frac{r}{\lambda_r}\nonumber
    \end{align}
    \begin{align}
    &\qquad + r\sqrt{r}\Big(c\mu\sqrt{\frac{rd\log(rd/ \delta_0)}{mt}} + \frac{\sigma_1}{mt\lambda_r}\Big(2\sqrt{rd} + 2\sqrt{2rd\log(2rd/\delta_0)}\Big)\Big)\cdot\nonumber\\
    &\qquad \resizebox{.92\textwidth}{!}{$\Big\{\frac{2}{t}\sqrt{t\mu\lambda_r}\Big( \cO\Big(\frac{\s{B}_{\fl{U^{(\ell-1)}}}\frac{\lambda_r^\star}{\lambda_1^\star}\sqrt{\frac{t}{r}\lambda^\star_1}}{\sqrt{r\mu^\star}}\Big) + \cO\Big(\sqrt{\frac{t}{r}\lambda^\star_r}\frac{\Lambda'}{\sqrt{r\mu^\star}}\Big) + \cO\Big(\frac{\Lambda}{\sqrt{r\mu^\star}}\sqrt{\frac{t}{r}}\Big)+ \cO\Big(\sigma\sqrt{\frac{r\log^2 (r\delta^{-1})}{m}}\Big)\Big)$} \nonumber\\
    &\qquad + \sqrt{\frac{4\zeta}{t}}\sqrt{\mu\lambda_r}\Big(c'\sqrt{\mu^\star\lambda_r^\star}\s{B}_{\fl{U^{(\ell-1)}}}\sqrt{\frac{\lambda_r^\star}{\lambda_1^\star}} + \Lambda\Big)\nonumber\\ 
    &\qquad + \resizebox{.9\textwidth}{!}{$\sqrt{\mu\lambda_r}\Big(\cO\Big(\frac{\sqrt{\mu^\star\lambda_r^\star}\s{B}_{\fl{U^{(\ell-1)}}}\sqrt{\frac{\lambda_r^\star}{\lambda_1^\star}}}{\sqrt{r\mu^\star}}\Big) + \cO\Big(\frac{\Lambda'\|\fl{w}^{\star(i)}\|_2}{\sqrt{r\mu^\star}}\Big) + \cO\Big(\frac{\Lambda}{\sqrt{r\mu^\star}}\Big)+ \cO\Big(\sigma\sqrt{\frac{r\log^2 (r\delta^{-1})}{m}}\Big)$}\nonumber\\
    &\qquad + \Big(c'\sqrt{\mu^\star\lambda_r^\star}\s{B}_{\fl{U^{(\ell-1)}}}\sqrt{\frac{\lambda_r^\star}{\lambda_1^\star}} + \Lambda\Big)\Big)\sqrt{\frac{d\log (rd/\delta_0)}{mt}}\nonumber\\
    &\qquad + \frac{\sigma_2}{mt}6\sqrt{rd\log(rd)} + \frac{\sigma_1}{mt}\Big(2\sqrt{rd} + 4\sqrt{\log rd}\Big)\sqrt{r} + \frac{2\sigma\sqrt{d\mu^{\star}\lambda^{\star}_r}\log(2rdmt/\delta_0)}{\sqrt{mt}}\Big\}\\
    &= J_1' + J_2'\label{eq:inductive-corollary-optimize_dp-lrs-U-infty-norm-bounds-unplugged2}
\end{align}
where as before, $J_1'$ denotes the terms which arise from analysing the problem in the noiseless setting and $J_2'$ denotes the contribution of noise terms ($\sigma_1, \sigma_2, \sigma, \Lambda, \Lambda'$). Now, $J'_1$
\begin{align}
    &= \sqrt{\frac{\nu^\star}{k}}\frac{r}{t\lambda_r}\cdot\cO\Big(\sqrt{\frac{t}{r}}\sqrt{\lambda_1^\star}\Big)\cdot\cO\Big(\frac{\s{B}_{\fl{U^{(\ell-1)}}}\frac{\lambda_r^\star}{\lambda_1^\star}\sqrt{\frac{t}{r}\lambda^\star_1}}{\sqrt{r\mu^\star}}\Big) \nonumber\\
    & + \sqrt{r}\cO\Big(\frac{r}{\lambda_r}\sqrt{\mu\lambda_r}\Big)\sqrt{\frac{\log (1/\delta_0)}{mt}}\cdot\cO\Big(\frac{\sqrt{\mu^\star\lambda_r^\star}\s{B}_{\fl{U^{(\ell-1)}}}\sqrt{\frac{\lambda_r^\star}{\lambda_1^\star}}}{\sqrt{r\mu^\star}}\Big) \nonumber\\
    & + \zeta c'\sqrt{\mu^\star\lambda_r^\star}\s{B}_{\fl{U^{(\ell-1)}}}\sqrt{\frac{\lambda_r^\star}{\lambda_1^\star k}}\sqrt{\mu\lambda_r}\frac{r}{t\lambda_r} + c\frac{r\sqrt{r}}{\lambda_r}c'\sqrt{\mu^\star\lambda_r^\star}\s{B}_{\fl{U^{(\ell-1)}}}\sqrt{\frac{\lambda_r^\star}{\lambda_1^\star}} \sqrt{\mu\lambda_r}\sqrt{\frac{\log (1/\delta_0)}{mt}}\nonumber\\
    & + r\sqrt{r}\Big(c\mu\sqrt{\frac{rd\log(rd/ \delta_0)}{mt}} + \frac{\sigma_1}{mt\lambda_r}\Big(2\sqrt{rd} + 2\sqrt{2rd\log(2rd/\delta_0)}\Big)\Big)\cdot \nonumber\\
    & \Big\{\frac{2}{t}\sqrt{t\mu\lambda_r}\cdot \cO\Big(\frac{\s{B}_{\fl{U^{(\ell-1)}}}\frac{\lambda_r^\star}{\lambda_1^\star}\sqrt{\frac{t}{r}\lambda^\star_1}}{\sqrt{r\mu^\star}}\Big) + \sqrt{\frac{4\zeta}{t}}\sqrt{\mu\lambda_r}\cdot c'\sqrt{\mu^\star\lambda_r^\star}\s{B}_{\fl{U^{(\ell-1)}}}\sqrt{\frac{\lambda_r^\star}{\lambda_1^\star}} \nonumber\\
    & +4\sqrt{\mu\lambda_r}\cdot \Big(\cO\Big(\frac{\sqrt{\mu^\star\lambda_r^\star}\s{B}_{\fl{U^{(\ell-1)}}}\sqrt{\frac{\lambda_r^\star}{\lambda_1^\star}}}{\sqrt{r\mu^\star}}\Big) + c'\sqrt{\mu^\star\lambda_r^\star}\s{B}_{\fl{U^{(\ell-1)}}}\sqrt{\frac{\lambda_r^\star}{\lambda_1^\star}} \Big)\sqrt{\frac{d\log (rd/\delta_0)}{mt}}\Big\}.\nonumber
    \end{align}
    Taking $\s{B}_{\fl{U^{(\ell-1)}}}\sqrt{\frac{\lambda_r^\star}{\lambda_1^\star k}}$ common across all terms, using eigenvalue ratio and incoherence bounds, $\s{B}_{\fl{U}^{(\ell-1)}} =   \cO\Big(\frac{1}{\sqrt{r\mu^\star}}\Big)$,  and rearranging the terms in the above we get $J_1'$
    \begin{align}
    &= \s{B}_{\fl{U^{(\ell-1)}}}\sqrt{\frac{\lambda_r^\star}{\lambda_1^\star k}}\Big\{\sqrt{\nu^\star}\frac{1}{\lambda_r}\cdot\cO(\sqrt{\lambda_1^\star})\cdot\cO\Big(\frac{\sqrt{\lambda^\star_r}}{\sqrt{r\mu^\star}}\Big) \nonumber\\
    &+ c\sqrt{r}\cO\Big(\frac{r}{\lambda_r}\sqrt{\mu\lambda_r}\Big)\cdot\cO\Big(\frac{\sqrt{\mu^\star\lambda_r^\star}}{\sqrt{r\mu^\star}}\Big)\cdot \sqrt{\frac{k\log (1/\delta_0)}{mt}} \nonumber\\
    & + \frac{\zeta}{t} c'\sqrt{\mu^\star\lambda_r^\star}\sqrt{\mu\lambda_r}\frac{r}{\lambda_r}  + cc'\frac{r\sqrt{r}}{\lambda_r}\sqrt{\mu^\star\lambda_r^\star}\sqrt{\frac{k\log (1/\delta_0)}{mt}}\nonumber\\
    &+ r\sqrt{r}\sqrt{k}\Big(c\mu\sqrt{\frac{rd\log(rd/ \delta_0)}{mt}} + \frac{\sigma_1}{mt\lambda_r}\Big(2\sqrt{rd} + 2\sqrt{2rd\log(2rd/\delta_0)}\Big)\Big)\cdot \mu^\star\lambda_r^\star\cdot\nonumber\\
    & \Big\{\frac{2}{r}\cdot \cO\Big(\frac{1}{\mu^\star}\Big) + \sqrt{\frac{4\zeta}{t}}*\cdot \cO(1) + 4 \cdot \Big(\cO\Big(\frac{1}{\sqrt{r\mu^\star}}\Big) + \cO(1)\Big)\sqrt{\frac{d\log (rd/\delta_0)}{mt}}\Big\}.\nonumber
    \end{align}
    Using $\sqrt{\nu^\star} = \cO\Big(\frac{1}{\sqrt{\mu^\star}}\frac{\lambda_r^\star}{\lambda_1^\star}\Big)$, $\sqrt{\frac{k}{d}} = \cO(1)$, $\sqrt{\frac{r^3d\log(1/ \delta_0)}{mt}} = \cO\Big(\frac{1}{\mu^\star}\Big)$, $\sqrt{\frac{r^2\zeta}{t}} = \cO\Big(\frac{1}{\mu^\star}\Big)$, $\sqrt{\frac{r^3d\log(1/ \delta_0)}{mt}} = \cO\Big(\frac{1}{\sqrt{r}(\mu^\star)^2\sqrt{\mu^\star}\lambda_r^\star\sqrt{k}}\Big)$, $\frac{\sigma_1}{mt}\Big(2\sqrt{rd} + 4\sqrt{\log rd}\Big) = \cO\Big(\frac{1}{r\sqrt{r}\mu^\star\sqrt{\mu^\star}\sqrt{k}}\Big)$ for the bracketed terms the above becomes
    \begin{align}
    &= \s{B}_{\fl{U^{(\ell-1)}}}\sqrt{\frac{\lambda_r^\star}{\lambda_1^\star k}}\cdot \cO(1)\nonumber\\
    &= \cO\Big(\s{B}_{\fl{U^{(\ell-1)}}}\sqrt{\frac{\lambda_r^\star}{\lambda_1^\star k}}\Big).\label{eq:inductive-corollary-optimize_dp-lrs-U-infty-norm-bounds-J1'-val} 
\end{align}
Similarly, we have $J_2'$ 
\begin{align}
    &= \sqrt{\frac{\nu^\star}{k}}\frac{r}{t\lambda_r}\cdot\cO\Big(\sqrt{\frac{t}{r}}\sqrt{\lambda_1^\star}\Big)\Big(\cO\Big(\sqrt{\frac{t}{r}\lambda^\star_r}\frac{\Lambda'}{\sqrt{r\mu^\star}}\Big) + \cO\Big(\frac{\Lambda}{\sqrt{r\mu^\star}}\sqrt{\frac{t}{r}}\Big) + \cO\Big(\sigma\sqrt{\frac{r\log^2 (r\delta^{-1})}{m}}\Big)\Big) \nonumber\\
    & + \sqrt{r}\cO\Big(\frac{r}{\lambda_r}\sqrt{\mu\lambda_r}\Big)\cdot\Big(\cO\Big(\frac{\Lambda'\|\fl{w}^{\star(i)}\|_2}{\sqrt{r\mu^\star}}\Big) + \cO\Big(\frac{\Lambda}{\sqrt{r\mu^\star}}\Big) + \cO\Big(\sigma\sqrt{\frac{r\log^2 (r\delta^{-1})}{m}}\Big)\Big)\sqrt{\frac{\log (1/\delta_0)}{mt}}\nonumber\\
    & + \zeta\frac{\Lambda}{\sqrt{k}}\sqrt{\mu\lambda_r}\frac{r}{t\lambda_r} + c\frac{r\sqrt{r}}{\lambda_r}\Lambda\sqrt{\mu\lambda_r}\sqrt{\frac{\log (1/\delta_0)}{mt}} + \frac{2\sigma_2}{mt}\sqrt{\log(rd/\delta_0)}\cdot\frac{r\sqrt{r}}{\lambda_r} \nonumber\\
    & + \frac{r^2\sigma_1\sqrt{rd\cdot2\log(2r^2d^2/\delta_0)}}{mt\lambda_r} + \frac{2\sigma\sqrt{\mu^{\star}\lambda^{\star}_r}\log(2rdmt/\delta_0)}{\sqrt{mt}}\frac{r}{\lambda_r}\nonumber\\
    & + r\sqrt{r}\Big(c\mu\sqrt{\frac{rd\log(rd/ \delta_0)}{mt}} + \frac{\sigma_1}{mt\lambda_r}\Big(2\sqrt{rd} + 2\sqrt{2rd\log(2rd/\delta_0)}\Big)\Big)\cdot\nonumber\\
    & \Big\{\frac{2}{t}\sqrt{t\mu\lambda_r}\Big(\cO\Big(\sqrt{\frac{t}{r}\lambda^\star_r}\frac{\Lambda'}{\sqrt{r\mu^\star}}\Big) + \cO\Big(\frac{\Lambda}{\sqrt{r\mu^\star}}\sqrt{\frac{t}{r}}\Big)+ \cO\Big(\sigma\sqrt{\frac{r\log^2 (r\delta^{-1})}{m}}\Big)\Big)\nonumber \\
    & + \resizebox{.97\textwidth}{!}{$\sqrt{\frac{4\zeta}{t}}\sqrt{\mu\lambda_r}\Lambda+4\sqrt{\mu\lambda_r}\Big(\cO\Big(\frac{\Lambda'\|\fl{w}^{\star(i)}\|_2}{\sqrt{r\mu^\star}}\Big) + \cO\Big(\frac{\Lambda}{\sqrt{r\mu^\star}}\Big) + \cO\Big(\sigma\sqrt{\frac{r\log^2 (r\delta^{-1})}{m}}\Big) + \Lambda\Big)\sqrt{\frac{d\log (rd/\delta_0)}{mt}}$}\nonumber\\
    & + \frac{\sigma_2}{mt}6\sqrt{rd\log(rd)} + \frac{\sigma_1}{mt}\Big(2\sqrt{rd} + 4\sqrt{\log rd}\Big)\sqrt{r} + \frac{2\sigma\sqrt{d\mu^{\star}\lambda^{\star}_r}\log(2rdmt/\delta_0)}{\sqrt{mt}}\Big\}.\nonumber
\end{align}
Using $\sqrt{\frac{r^3d\log(1/ \delta_0)}{mt}} = \cO\Big(\frac{1}{\sqrt{r}(\mu^\star)^2\sqrt{\mu^\star}\lambda_r^\star\sqrt{k}}\Big)$, $\frac{\sigma_1}{mt}\Big(2\sqrt{rd} + 4\sqrt{\log rd}\Big) = \cO\Big(\frac{1}{r\sqrt{r}\mu^\star\sqrt{\mu^\star}\sqrt{k}}\Big)$, eigenvalue ratio bounds and rearranging the terms in above gives
\begin{align}
    &= \Lambda\Big\{\sqrt{\frac{\nu^\star}{k}}\cO\Big(\frac{\sqrt{\lambda_1^\star}}{\lambda_r^\star\sqrt{r\mu^\star}}\Big) + \cO\Big(\frac{1}{\sqrt{r\lambda_r^\star}}\sqrt{\frac{r^3\log (1/\delta_0)}{mt}} + \frac{r\zeta}{t}\cO\Big(\frac{1\sqrt{\mu^\star}}{\sqrt{k}\sqrt{\lambda_r^\star}}\Big)\nonumber\\
    & + \resizebox{.97\textwidth}{!}{$\cO\Big(\frac{1}{\mu^\star\sqrt{\mu^\star}\lambda_r^\star\sqrt{k}}\Big)\Big(\frac{2}{r}\sqrt{\mu\lambda_r}\cO\Big(\frac{1}{\sqrt{r\mu^\star}}\Big) + \sqrt{\frac{4\zeta}{t}}\sqrt{\mu\lambda_r} + 4\sqrt{\mu\lambda_r}\cO\Big(\frac{1}{\sqrt{r\mu^\star}}+1\Big)\sqrt{\frac{d\log (rd/\delta_0)}{mt}}\Big)\Big\}$}\nonumber\\
    & +\Lambda'\Big\{\sqrt{\frac{\nu^\star}{k}}\cdot\cO\Big(\sqrt{\frac{\lambda_1^\star}{\lambda_r^\star}}\frac{1}{\sqrt{r\mu^\star}}\Big) + c\sqrt{r}\frac{r}{\lambda_r}\sqrt{\mu\lambda_r}\cdot\cO\Big(\frac{\|\fl{w}^{\star(i)}\|_2}{\sqrt{r\mu^\star}}\Big) \nonumber\\
    & + \cO\Big(\frac{1}{\mu^\star\sqrt{\mu^\star}\lambda_r^\star\sqrt{k}}\Big)\Big(\frac{2}{t}\sqrt{t\mu\lambda_r}\sqrt{\frac{t}{r}\lambda^\star_r}\cO\Big(\frac{1}{\sqrt{r\mu^\star}}\Big) + 4\sqrt{\mu\lambda_r}\cO\Big(\frac{\|\fl{w}^{\star(i)}\|_2}{\sqrt{r\mu^\star}}\Big)\Big)\Big\}\nonumber\\
    & + \sigma_1\Big\{\frac{r^2\sqrt{rd\cdot2\log(2r^2d^2/\delta_0)}}{mt\lambda_r} + \cO\Big(\frac{1}{\mu^\star\sqrt{\mu^\star}\lambda_r^\star\sqrt{k}}\Big)\cdot \frac{\sqrt{r}}{mt}\Big(2\sqrt{rd} + 4\sqrt{\log rd}\Big)\Big\}\nonumber\\
    & + \sigma_2\Big\{\frac{2}{mt}\sqrt{\log(rd/\delta_0)}\cdot\frac{r\sqrt{r}}{\lambda_r} + \cO\Big(\frac{1}{\mu^\star\sqrt{\mu^\star}\lambda_r^\star\sqrt{k}}\Big)\cdot \frac{1}{mt}6\sqrt{rd\log(rd)}\Big\}
    \end{align}
    \begin{align}
    & + \sigma\Big\{\sqrt{\frac{\nu^\star}{k}}\frac{r}{t\lambda_r}\cdot\cO\Big(\sqrt{\frac{t}{r}}\sqrt{\lambda_1^\star}\Big)\cdot \Big(\sqrt{\frac{r\log^2 (r\delta^{-1})}{m}}\Big) \nonumber\\
    & + \resizebox{.97\textwidth}{!}{$c\sqrt{r}\frac{r}{\lambda_r}\cdot \cO\Big(\sqrt{\frac{r\log^2 (r\delta^{-1})}{m}}\sqrt{\frac{\log (1/\delta_0)}{mt}}\Big)
    + \frac{2\sqrt{\mu^{\star}\lambda^{\star}_r}\log(2rdmt/\delta_0)}{\sqrt{mt}}\frac{r}{\lambda_r} + \cO\Big(\frac{1}{\mu^\star\sqrt{\mu^\star}\lambda_r^\star\sqrt{k}}\Big)\cdot\sqrt{\mu^{\star}\lambda^{\star}_r}$}\nonumber\\
    &\qquad \resizebox{.95\textwidth}{!}{$\Big(\frac{2}{\sqrt{t}}\cdot\cO\Big(\sqrt{\frac{r\log^2 (r\delta^{-1})}{m}}\Big) + 4\cO\Big(\sqrt{\frac{r\log^2 (r\delta^{-1})}{m}}\sqrt{\frac{d\log (rd/\delta_0)}{mt}}\Big) + \frac{2\sigma\sqrt{d}\log(2rdmt/\delta_0)}{\sqrt{mt}}\Big)\Big\}$}.\nonumber
\end{align}
Taking $\sqrt{\frac{1}{k}}$ common and using  $\sqrt{\nu^\star} = \cO\Big( \frac{1}{\sqrt{\mu^\star}}\frac{\lambda_r^\star}{\lambda_1^\star}\Big)$, $\sqrt{\frac{r^3d\log(1/ \delta_0)}{mt}} = \cO\Big(\frac{1}{\mu^\star}\Big)$, $\sqrt{\frac{k}{d}} = \cO(1)$, $\sqrt{\frac{r^2\zeta}{t}} = \cO\Big(\frac{1}{\sqrt{\mu^\star}}\Big)$, $\sqrt{\frac{k}{d}} = \cO(1)$, $\sqrt{\frac{r^3d\log(1/ \delta_0)}{mt}} = \cO\Big(\frac{1}{\mu^\star}\Big)$, $\Lambda = \cO(\sqrt{\lambda_r^{\star}})$, $mt=\Omega(\sigma_2 r(\mu^{\star})^{1/2}\sqrt{rd\log d}/\lambda_r^{\star})$ and $mt=\widetilde{\Omega}( \sigma^2dr^3\mu^{\star}/\lambda^{\star}_r)$, $t=\widetilde{\Omega}(\zeta(\mu^{\star}\lambda^{\star}_r)\max(1,\lambda^{\star}_r/r))$, eigenvalue ratios and incoherence bounds for $\fl{H}^{(\ell)}$ and $\fl{W}^{(\ell)}$ from Corollaries \ref{inductive-corollary:optimize_dp-lrs-h,H-bounds} and \ref{inductive-inductive-corollary:optimize_dp-lrs-w-incoherence} as well as $\lambda_r^\star \leq \lambda_1^\star$, $r\geq1$,$\mu^\star \geq 1$ for the bracketed terms, the above simplifies to   
\begin{align}
    J_2' &= \frac{1}{\sqrt{k}}\Big\{\cO(\Lambda' + \cO\Big(\frac{\Lambda}{\sqrt{\mu^\star\lambda_r^\star}}\Big) + \cO\Big(\frac{\sigma_2r}{mt\lambda_r^\star}\sqrt{rd\log(rd)}\Big) + \cO\Big(\frac{\sigma_1r^{3/2}}{mt\lambda_r^\star}\sqrt{rd\log rd}\Big) \nonumber\\
    &\qquad + \cO\Big(\sigma\sqrt{\frac{r^3d\mu^\star\log^2 (r\delta^{-1})}{mt\lambda_r^\star}}\Big)\Big\}\label{eq:inductive-corollary-optimize_dp-lrs-U-infty-norm-bounds-J2'-val}
\end{align}
Using \eqref{eq:inductive-corollary-optimize_dp-lrs-U-infty-norm-bounds-J1'-val} and \eqref{eq:inductive-corollary-optimize_dp-lrs-U-infty-norm-bounds-J2'-val} in \eqref{eq:inductive-corollary-optimize_dp-lrs-U-infty-norm-bounds-unplugged2}
gives us the required bound for $\|\fl{U}^{(\ell-1)} - \fl{U}^{\star}\fl{Q}^{(\ell-1)}\|_{2,\infty} $.
\end{proof}

\begin{lemma}\label{lemma:alg2-innerloop-iteration-bounds}
For some constant $c>0$ and for any iteration indexed by $\ell>0, j>0$, we  set
\begin{align}
    \Delta^{(\ell-1, j)} :=  \alpha^{(\ell-1)}+c\sqrt{\frac{\log (d/\delta)}{m}}\Big(\gamma^{(\ell-1, j-1)}+\beta^{(\ell-1)}\Big) + \sigma\sqrt{\frac{\log (d\delta^{-1})}{m}}.  \nonumber 
\end{align},
where $\alpha^{(\ell-1)}$, $\gamma^{(\ell-1, j-1)}$ and $\beta^{(\ell-1)}$ are known upper-bounds on $\|\fl{U}^{+(\ell-1)}\fl{w}^{(i, \ell-1)} - \fl{U}^{\star}\fl{w}^{\star(i)}\|_{\infty}$, $\|\fl{b}^{(i,\ell-1,j-1)}-\fl{b}^{\star (i)}\|_2$ and $\|\fl{U}^{+(\ell-1)}\fl{w}^{(i, \ell-1)} - \fl{U}^{\star}\fl{w}^{\star(i)}\|_2$ respectively. Subsequently, we have 
\begin{align}
    \lr{\fl{b}^{(i,\ell-1,j)}-\fl{b}^{\star (i)} }_{\infty} &\le 2\Delta^{(\ell-1, j)},\nonumber\\
    \lr{\fl{b}^{(i,\ell-1, j)}-\fl{b}^{\star (i)} }_{2} &\le 2\sqrt{k}\Delta^{(\ell-1, j)} := \gamma^{(\ell-1, j)},\nonumber\\
    \text{and } \s{support}(\fl{b}^{(i,\ell-1, j}) &\subseteq \s{support}(\fl{b}^{\star(i)})\nonumber
\end{align} 
with probability at least $1-\delta$.
\end{lemma}

\begin{proof}
It is easy to see that update step of the algorithm gives us $\fl{c}^{(i,\ell-1, j)}-\fl{b}^{\star (i)}$
\begin{align}
    &= \Big(\fl{I}-\frac{1}{m} (\fl{X}^{(i)})^{\s{T}}\fl{X}^{(i)} \Big)\Big(\fl{b}^{(i,\ell-1, j-1)}-\fl{b}^{\star (i)}\Big)\nonumber\\
    &\qquad +\frac{1}{m}(\fl{X}^{(i)})^{\s{T}}\fl{X}^{(i)}(\fl{U}^{\star}\fl{w}^{\star(i)}-\fl{U}^{+(\ell-1)}\fl{w}^{(i, \ell-1)})+\frac{1}{m}(\fl{X}^{(i)})^{\s{T}}\fl{\xi}^{(i)}\nonumber\\
    \implies &\fl{c}^{(i,\ell-1,j)}-\fl{b}^{\star (i)} -\fl{U}^{\star}\fl{w}^{\star(i)}+\fl{U}^{+(\ell-1)}\fl{w}^{(i, \ell-1)} \nonumber\\
    &\qquad = \Big(\fl{I}-\frac{1}{m} (\fl{X}^{(i)})^{\s{T}}\fl{X}^{(i)} \Big)\Big(\fl{b}^{(i,\ell-1,j-1)}-\fl{b}^{\star (i)}\Big)\nonumber\\
    &\qquad +\Big(\fl{I}-\frac{1}{m}(\fl{X}^{(i)})^{\s{T}}\fl{X}^{(i)}\Big)(\fl{U}^{+(\ell-1)}\fl{w}^{(i, \ell-1)} - \fl{U}^{\star}\fl{w}^{\star(i)})+\frac{1}{m}(\fl{X}^{(i)})^{\s{T}}\fl{\xi}^{(i)}. \label{eq:alg2-lemma-unknown-support-c-unplugged}
\end{align}

Note that $\lr{\frac{1}{m}(\fl{X}^{(i)})^{\s{T}}\fl{\xi}^{(i)}}_{\infty}\le \sqrt{\frac{\log (d\delta_0^{-1})}{m}}$ with probability at least $1-\delta_0$. Let $\fl{e}_s\in \bb{R}^d$ denote the $s^{\s{th}}$ basis vector for which the $s^{\s{th}}$ coordinate entry is $1$ and all other coordinate entries are $0$. Then, note that $\left|\Big(\fl{c}^{(i,\ell-1,j)}-\fl{b}^{\star (i)} -\fl{U}^{\star}\fl{w}^{\star(i)}+\fl{U}^{+(\ell-1)}\fl{w}^{(i, \ell-1)}\Big)_s\right|$
\begin{align}
    &= \resizebox{0.97\hsize}{!}{$\left|\fl{e}_s^{\s{\s{T}}}\Big(\fl{I}-\frac{1}{m} (\fl{X}^{(i)})^{\s{T}}\fl{X}^{(i)} \Big)\Big(\fl{b}^{(i,\ell-1,j-1)}-\fl{b}^{\star (i)}\Big)+\fl{e}_s^{\s{T}}\Big(\fl{I}-\frac{1}{m}(\fl{X}^{(i)})^{\s{T}}\fl{X}^{(i)}\Big)(\fl{U}^{+(\ell-1)}\fl{w}^{(i, \ell-1)} - \fl{U}^{\star}\fl{w}^{\star(i)})\right|$}\nonumber\\
    &\qquad +\sigma\sqrt{\frac{\log (d\delta^{-1})}{m}} \nonumber\\
    &\leq \left|\fl{e}_s^{\s{T}}\Big(\fl{I}-\frac{1}{m} (\fl{X}^{(i)})^{\s{T}}\fl{X}^{(i)} \Big)\Big(\fl{b}^{(i,\ell-1, j-1)}-\fl{b}^{\star (i)}\Big)\right|\nonumber\\
    &\qquad +\left|\fl{e}_s^{\s{T}}\Big(\fl{I}-\frac{1}{m}(\fl{X}^{(i)})^{\s{T}}\fl{X}^{(i)}\Big)(\fl{U}^{+(\ell-1)}\fl{w}^{(i, \ell-1)} - \fl{U}^{\star}\fl{w}^{\star(i)})\right|+\sigma\sqrt{\frac{\log (d\delta^{-1})}{m}} \nonumber\\
    &\leq \left|\frac{1}{m} \fl{e}_s^{\s{T}}(\fl{X}^{(i)})^{\s{T}}\fl{X}^{(i)} (\fl{b}^{(i,\ell-1, j-1)}-\fl{b}^{\star (i)}) - \fl{e}_s^{\s{T}}(\fl{b}^{(i,\ell-1, j-1)}-\fl{b}^{\star (i)})\right| +\sigma\sqrt{\frac{\log (d\delta^{-1})}{m}}\nonumber\\
    &\qquad +\left|\frac{1}{m}\fl{e}_s^{\s{T}}(\fl{X}^{(i)})^{\s{T}}\fl{X}^{(i)}(\fl{U}^{+(\ell-1)}\fl{w}^{(i, \ell-1)} - \fl{U}^{\star}\fl{w}^{\star(i)}) - \fl{e}_s^{\s{T}}(\fl{U}^{+(\ell-1)}\fl{w}^{(i, \ell-1)} - \fl{U}^{\star}\fl{w}^{\star(i)})\right| \nonumber\\
    &\leq c\sqrt{\frac{\log (1/\delta)}{m}}\Big(\|\fl{b}^{(i,\ell-1,j-1)}-\fl{b}^{\star (i)}\|_2+\|\fl{U}^{+(\ell-1)}\fl{w}^{(i, \ell-1)} - \fl{U}^{\star}\fl{w}^{\star(i)}\|_2\Big)+\sigma\sqrt{\frac{\log (d\delta^{-1})}{m}},
\end{align}
w.p. $\geq 1-\delta$,
where we invoke Lemma \ref{lemma:useful1} in the last step and plugging $\fl{a} = \fl{e}_s$ and $\fl{b} = \fl{b}^{(i,\ell-1,j-1)}-\fl{b}^{\star (i)}$ and $w^{\star}_i\fl{u}^{\star}-w_i^{(\ell-1)}\widehat{\fl{u}}^{(\ell-1)}$ for the two terms respectively. Therefore, by taking a union bound over all entries $s\in [d]$, we can conclude that
\begin{align}
    &\lr{\fl{c}^{(i,\ell-1,j)}-\fl{b}^{\star (i)} -\fl{U}^{\star}\fl{w}^{\star(i)}+\fl{U}^{+(\ell-1)}\fl{w}^{(i, \ell-1)}}_{\infty} \nonumber\\
    &\leq c\sqrt{\frac{\log (d/\delta)}{m}}\Big(\|\fl{b}^{(i,\ell-1,j-1)}-\fl{b}^{\star (i)}\|_2+\|\fl{U}^{+(\ell-1)}\fl{w}^{(i, \ell-1)} - \fl{U}^{\star}\fl{w}^{\star(i)}\|_2\Big)+\sigma\sqrt{\frac{\log (d\delta^{-1})}{m}}\nonumber
\end{align}
\begin{align}
    &\lr{\fl{c}^{(i,\ell-1,j)}-\fl{b}^{\star (i)} }_{\infty}\nonumber \\
    &\leq \|\fl{U}^{+(\ell-1)}\fl{w}^{(i, \ell-1)} - \fl{U}^{\star}\fl{w}^{\star(i)}\|_{\infty} \nonumber\\
    &\quad +c\sqrt{\frac{\log (d/\delta)}{m}}\Big(\|\fl{b}^{(i,\ell-1,j-1)}-\fl{b}^{\star (i)}\|_2+\|\fl{U}^{+(\ell-1)}\fl{w}^{(i, \ell-1)} - \fl{U}^{\star}\fl{w}^{\star(i)}\|_2+\sigma\sqrt{\frac{\log (d\delta^{-1})}{m}}\Big)\nonumber\\
    & \triangleq \widehat{\Delta}^{(\ell-1,j)}\nonumber\\
    &\leq \underbrace{\alpha^{(\ell-1)} +c\sqrt{\frac{\log (d/\delta)}{m}} \Big(\gamma^{(\ell-1, j-1)}+\beta^{(\ell-1)}\Big)+\sigma\sqrt{\frac{\log (d\delta^{-1})}{m}}}_{\triangleq \Delta^{(\ell-1,j)}}
    \label{eq:alg2-lemma-unknown-support-term1}
\end{align}
w.p. $\geq 1-\delta$. Now, we have 
\begin{align}
    \fl{b}^{(i,\ell-1,j)} &= \s{HT}(\fl{c}^{(i,\ell-1,j)},\Delta^{(\ell-1, j)})\nonumber \\
    \implies \fl{b}^{(i,\ell-1,j)}_s &= \begin{cases}
    \fl{c}^{(i,\ell-1,j)}_s & \text{if } |\fl{c}^{(i,\ell-1,j)}_s| > \Delta^{(\ell-1, j)},\\
    0 & \text{ otherwise},
  \end{cases}\label{eq:alg2-lemma-unknown-support-term-temp}\\
  \implies |\fl{b}^{(i,\ell-1,j)}_s - \fl{b}^{\star(i)}_s | &= \begin{cases}
    |\fl{c}^{(i,\ell-1,j)}_s - \fl{b}^{\star(i)}_s| & \text{if } |\fl{c}^{(i,\ell-1,j)}_s| > \Delta^{(\ell-1, j)},\\
    |\fl{b}^{\star(i)}_s| & \text{ otherwise}.
  \end{cases}\label{eq:alg2-lemma-unknown-support-term2}
\end{align}
Therefore, by using \eqref{eq:alg2-lemma-unknown-support-term1} and \eqref{eq:alg2-lemma-unknown-support-term2}, we have
$\lr{\fl{b}^{(i,\ell-1,j)}-\fl{b}^{\star (i)} }_{\infty}\le \Delta^{(\ell-1, j)}$ and $\lr{\fl{b}^{(i,\ell-1,j)}-\fl{b}^{\star (i)} }_{2}\le 2\sqrt{k}\Delta^{(\ell-1, j)}$.
Further, from equation \eqref{eq:alg2-lemma-unknown-support-term1} we have for any coordinate $s$
\begin{align}
    \Big|\Big(\fl{c}^{(i,\ell-1,j)}-\fl{b}^{\star (i)}\Big)_s\Big| &\leq \Delta^{(\ell-1, j)}.\nonumber
\end{align}
Thus, if $s \notin	 \s{support}(\fl{b}^{\star(i)})$, then the above gives $|\fl{c}^{(i,\ell-1,j)}| \leq \Delta^{(\ell-1, j)}$. Using this in \eqref{eq:alg2-lemma-unknown-support-term-temp} then gives $\fl{b}^{(i,\ell-1,j)}_s = 0$, i.e., $\forall \, s \notin	 \s{support}(\fl{b}^{\star(i)}) \implies s \notin	 \s{support}(\fl{b}^{\star(i, \ell-1, j)})$. Hence, $\s{support}(\fl{b}^{(i,\ell-1,j)}) \subseteq \s{support}(\fl{b}^{\star(i)})$.
\end{proof}

\begin{lemma}\label{inductive-corollary:optimize_dp-lrs-innerloop-output-bounds}
Suppose $c>0$ and $c_1=c\sqrt{\frac{k\log (d/\delta)}{m}} \leq \frac{1}{2}$ be positive constants.
For any iteration indexed by $\ell>0$, after 
\begin{align}
     T^{(\ell)}=\Omega\Big(\ell\max_i \log\Big(\frac{\gamma^{(\ell-1)}}{\epsilon}\Big)\Big)\nonumber
\end{align}
iterations of the inner loop at Step 3 in the $\ell^{\s{th}}$ iteration of the outer loop, we have
\begin{align}
    \lr{\fl{b}^{(i,\ell)}-\fl{b}^{\star (i)} }_{2} \leq 2\varphi^{(i)} + \epsilon
    \text{ and } \lr{\fl{b}^{(i,\ell)}-\fl{b}^{\star (i)} }_{\infty} \leq \frac{1}{\sqrt{k}}\Big(2\varphi^{(i)} + \epsilon\Big)\nonumber
\end{align}
with probability at least $1-T^{(\ell)}\delta$, where $\varphi^{(i)}$ is an upperbound on $\widehat{\varphi}^{(i)}$ s.t. 
\begin{align}
    \widehat{\varphi}^{(i)} &=     \resizebox{.9\textwidth}{!}{$2\Big(\sqrt{k}\|\fl{U}^{+(\ell-1)}\fl{w}^{(i, \ell-1)} - \fl{U}^{\star}\fl{w}^{\star(i)}\|_{\infty}+c_1\|\fl{U}^{+(\ell-1)}\fl{w}^{(i, \ell-1)} - \fl{U}^{\star}\fl{w}^{\star(i)}\|_2+\sigma\sqrt{\frac{k\log (d\delta^{-1})}{m}}\Big)$} \nonumber\\
    &\leq \varphi^{(i)} = 2\Big(\sqrt{k}\alpha^{(\ell-1)}+c_1\beta^{(\ell-1)}+\sigma\sqrt{\frac{k\log (d\delta^{-1})}{m}}\Big).\nonumber
\end{align}
and $\alpha^{(\ell-1)}$, $\gamma^{(\ell-1, j-1)}$ and $\beta^{(\ell-1)}$ denote upperbounds on $\|\fl{U}^{+(\ell-1)}\fl{w}^{(i, \ell-1)} - \fl{U}^{\star}\fl{w}^{\star(i)}\|_{\infty}$, $\|\fl{b}^{(i,\ell-1,j-1)}-\fl{b}^{\star (i)}\|_2$ and $\|\fl{U}^{+(\ell-1)}\fl{w}^{(i, \ell-1)} - \fl{U}^{\star}\fl{w}^{\star(i)}\|_2$ respectively. Furthermore, we will also have that
$\s{support}(\fl{b}^{(i,\ell}) \subseteq \s{support}(\fl{b}^{\star(i)})$.
\end{lemma}

\begin{proof}
Let $\varphi^{(i)}$ be an upperbound on $\widehat{\varphi}^{(i)}$ where, 
\begin{align}
    \widehat{\varphi}^{(i)} &= 2\Big(\sqrt{k}\|\fl{v}\|_{\infty}+c_1\|\fl{v}\|_2+\sigma\sqrt{\frac{k\log (d\delta^{-1})}{m}}\Big) \leq \varphi^{(i)}.\nonumber
\end{align}
where $\fl{v} := \fl{U}^{+(\ell-1)}\fl{w}^{(i, \ell-1)} - \fl{U}^{\star}\fl{w}^{\star(i)}$. Then $\varphi^{(i)} := 2\Big(\sqrt{k}\alpha^{(\ell-1)}+c_1\beta^{(\ell-1)}\Big)$ From Lemma~\ref{lemma:alg2-innerloop-iteration-bounds}, we have for each iteration $j$,
\begin{align}
    \gamma^{(\ell-1, j)} := \lr{\fl{b}^{(i,\ell-1,j)}-\fl{b}^{\star (i)} }_{2} &\leq \varphi^{(i)} + 2c_1\gamma^{(\ell-1, j-1)}\label{eq:lemma-alg2-innerloop-output-bounds-2norm-rec}\\
    \text{and } \lr{\fl{b}^{(i,\ell-1,j)}-\fl{b}^{\star (i)} }_{\infty} &\leq  \frac{\varphi^{(i)}}{\sqrt{k}} + \frac{2c_1}{\sqrt{k}}\gamma^{(\ell-1, j-1)}\label{eq:lemma-alg2-innerloop-output-bounds-inftynorm-rec}
\end{align}
with probability at least $1-\delta_0$, where $c_1 = c\sqrt{\frac{k\log(d/\delta_0)}{m}}$.
Therefore after $T^{(\ell)}$ iterations of Step 3 inner loop at the $\ell^{\s{th}}$ iteration of the outer loop, we have using \eqref{eq:lemma-alg2-innerloop-output-bounds-2norm-rec}:
\begin{align}
    \lr{\fl{b}^{(i,\ell)}-\fl{b}^{\star (i)} }_{2} &=     \lr{\fl{b}^{(i,\ell-1,T^{(\ell)})}-\fl{b}^{\star (i)} }_{2} \nonumber\\
    &\leq \varphi^{(i)} + 2c_1\gamma^{(\ell-1, T^{(\ell)}-1)}\nonumber\\
    &\leq \varphi^{(i)} + 2c_1\varphi^{(i)} + (2c_1)^2\gamma^{(\ell-1, T^{(\ell)}-2)}\nonumber\\
    & \dots\nonumber\\
    &\leq \varphi^{(i)}(1 + (2c_1)\varphi^{(i)} + (2c_1)^2 + \dots + (2c_1)^{T^{(\ell)}-1})+ (2c_1)^{T^{(\ell)}}\gamma^{(\ell-1, 0)}\nonumber\\
    &= \varphi^{(i)}\frac{1 - (2c_1)^{T^{(\ell)}}}{1-2c_1} + (2c_1)^{T^{(\ell)}}\gamma^{(\ell-1)}\nonumber\\
    &\leq \varphi^{(i)}\frac{1}{1-2c_1} + (2c_1)^{T^{(\ell)}}\gamma^{(\ell-1)}\label{eq:lemma-alg2-innerloop-output-bounds-2norm-closed},
\end{align}
w.p. $\geq 1-T^{(\ell)}\delta_0$ where $\gamma^{(\ell-1)} = \gamma^{(\ell-1, 0)}$ is the upper bound on $\|\fl{b}^{(i,\ell-1)}-\fl{b}^{\star (i)}\|_2 = \|\fl{b}^{(i,\ell-1,0)}-\fl{b}^{\star (i)}\|_2$. Similarly, unfolding \eqref{eq:lemma-alg2-innerloop-output-bounds-inftynorm-rec} gives
\begin{align}
    \lr{\fl{b}^{(i,\ell)}-\fl{b}^{\star (i)} }_{\infty} &= \lr{\fl{b}^{(i,\ell-1,T^{(\ell)})}-\fl{b}^{\star (i)} }_{\infty}\nonumber \\
    &\leq \frac{\varphi^{(i)}}{\sqrt{k}} + \frac{2c_1}{\sqrt{k}}\gamma^{(\ell-1, T^{(\ell)}-1)}\nonumber\\
    &\leq \frac{\varphi^{(i)}}{\sqrt{k}} + \frac{2c_1\varphi^{(i)}}{\sqrt{k}} + \frac{(2c_1)^2\gamma^{(\ell-1, T^{(\ell)}-2)}}{\sqrt{k}}\nonumber\\
    &\dots\nonumber\\
    &\leq \frac{\varphi^{(i)}}{\sqrt{k}}(1 + (2c_1)\varphi^{(i)} + (2c_1)^2 + \dots (2c_1)^{T^{(\ell)}-1})+ \frac{(2c_1)^{T^{(\ell)}}\gamma^{(\ell-1, 0)}}{\sqrt{k}}\nonumber\\
    &\leq \frac{\varphi^{(i)}}{\sqrt{k}}\frac{1}{1-2c_1} + (2c_1)^{T^{(\ell)}}\frac{\gamma^{(\ell-1)}}{\sqrt{k}}\label{eq:lemma-alg2-innerloop-output-bounds-inftynorm-closed}
\end{align}
w.p. $\geq 1-T^{(\ell)}\delta_0$. Therefore, if we set $T^{(\ell)} \geq \max_i \frac{1}{\log (1/2c_1)}\Big(\frac{\gamma^{(\ell-1)}}{\epsilon}\Big)$ and $c_1 < \frac{1}{2}$ is sufficiently small then  \eqref{eq:lemma-alg2-innerloop-output-bounds-2norm-closed} gives us
\begin{align}
    \lr{\fl{b}^{(i,\ell)}-\fl{b}^{\star (i)} }_{2} &\leq 2\varphi^{(i)} + \epsilon\label{eq:lemma-alg2-innerloop-output-bounds-2norm-value}
\end{align}
and \eqref{eq:lemma-alg2-innerloop-output-bounds-inftynorm-closed} gives us
\begin{align}
    \lr{\fl{b}^{(i,\ell)}-\fl{b}^{\star (i)} }_{\infty} &\leq \frac{1}{\sqrt{k}}\Big(2\varphi^{(i)} + \epsilon\Big)\label{eq:lemma-alg2-innerloop-output-bounds-inftynorm-value}
\end{align}
w.p. $\geq 1-T^{(\ell)}\delta_0$. Equations \eqref{eq:lemma-alg2-innerloop-output-bounds-2norm-value} \eqref{eq:lemma-alg2-innerloop-output-bounds-inftynorm-value} give us the required result. Also, note that we set 
\begin{align}
    T^{(\ell)}=\Omega\Big(\ell\max_i \log\Big(\frac{\gamma^{(\ell-1)}}{\epsilon}\Big)\Big)\label{eq:lemma-alg2-innerloop-output-bounds-T(ell)-unplugged}
\end{align}
\end{proof}

\begin{coro}\label{inductive-corollary:optimize_dp-lrs-b-bounds}
    Using Corollaries~\ref{inductive-corollary:optimize_dp-lrs-h,H-bounds},
    \ref{inductive-inductive-corollary:optimize_dp-lrs-w-incoherence}, \ref{inductive-corollary:optimize_dp-lrs-U-F-norm-bounds}, \ref{inductive-corollary:optimize_dp-lrs-R-bounds} and \ref{inductive-corollary:optimize_dp-lrs-U-infty-norm-bounds}, we have
    \begin{align}
        \lr{\fl{b}^{(i,\ell+1)}-\fl{b}^{\star (i)} }_{2} &\leq c'\max\{\|\fl{w}^{\star(i)}\|_2, \epsilon\}\s{B}_{\fl{U^{(\ell)}}}\sqrt{\frac{\lambda_r^\star}{\lambda_1^\star}}\nonumber\\
        \text{ and } \lr{\fl{b}^{(i,\ell+1)}-\fl{b}^{\star (i)} }_{\infty} &\leq c'\max\{\|\fl{w}^{\star(i)}\|_2, \epsilon\}\s{B}_{\fl{U^{(\ell)}}}\sqrt{\frac{\lambda_r^\star}{\lambda_1^\star k}}\nonumber
    \end{align}
    with $c' = \max\Big(\cO(1), O\Big(\frac{1} {\s{B}_{\fl{U}^{(0)}}}\frac{\lambda^{\star}_1}{\lambda^{\star}_r}\Big)\Big)$, and for sufficiently large constants $\tilde{c},\hat{c}>0$
\begin{align}
    &\Lambda = \tilde{c}\Big( \sqrt{\lambda^{\star}_r\mu^{\star}}\Big(\frac{\sigma_2 r}{mt\lambda^{\star}_r}+\frac{\sigma_1r^{3/2}}{mt\lambda_r^\star}\sqrt{rd\log rd}+\sigma\sqrt{\frac{r^3d\mu^\star\log^2 (r\delta^{-1})}{mt\lambda_r^\star}}\Big)\nonumber\\
    &\qquad + \sigma\Big(\sqrt{\frac{r^3\log^2(r\delta^{-1})}{m\lambda^{\star}_r}}\Big) + \sqrt{\frac{k\log (d\delta^{-1})}{m}}\Big)\Big) \nonumber\\
    &\Lambda' = \hat{c}\Big(\frac{\Lambda}{\sqrt{\mu^{\star}\lambda^{\star}_r}}\Big).\nonumber
\end{align}
\end{coro}

\begin{proof}
    Using Corollary~\ref{inductive-corollary:optimize_dp-lrs-U-infty-norm-bounds} we have $\|\fl{U}^{(\ell)}\fl{w}^{(i, \ell)} - \fl{U}^{\star}\fl{w}^{\star(i)}\|_{\infty}$
        \begin{align}
            &= \|\fl{U}^{(\ell)}\fl{w}^{(i, \ell)} - \fl{U}^{(\ell)}(\fl{Q}^{(\ell-1)})^{-1}\fl{w}^{\star(i)} + \fl{U}^{(\ell)}(\fl{Q}^{(\ell-1)})^{-1}\fl{w}^{\star(i)} - \fl{U}^{\star}\fl{w}^{\star(i)}\|_\infty\nonumber\\
            &\leq \|\fl{U}^{(\ell)}(\fl{w}^{(i, \ell)} - (\fl{Q}^{(\ell-1)})^{-1}\fl{w}^{\star(i)})\|_\infty + \|(\fl{U}^{(\ell)} - \fl{U}^{\star}\fl{Q}^{(\ell-1)})(\fl{Q}^{(\ell-1)})^{-1}\fl{w}^{\star(i)}\|_\infty\nonumber\\
            &\leq \|\fl{U}^{(\ell)}\|_{2, \infty}\|\fl{w}^{(i, \ell-1)} - (\fl{Q}^{(\ell-1)})^{-1}\fl{w}^{\star(i)})\|_2 + \|\fl{U}^{(\ell)} - \fl{U}^{\star}\fl{Q}^{(\ell-1)}\|_{2,\infty}\|(\fl{Q}^{(\ell-1)})^{-1}\fl{w}^{\star(i)}\|_2\nonumber\\
            &= \cO\Big(\frac{1}{\sqrt{k\mu^\star}}\Big)\cO\Big(\frac{\max\{\epsilon, \|\fl{w}^{\star(i)}\|_2\}\s{B}_{\fl{U^{(\ell-1)}}}\sqrt{\frac{\lambda_r^\star}{\lambda_1^\star}}}{\sqrt{r\mu^\star}} + \frac{\Lambda'\|\fl{w}^{\star(i)}\|_2}{\sqrt{r\mu^\star}} + \frac{\Lambda}{\sqrt{r\mu^\star}} + \sigma\sqrt{\frac{r\log^2 (r\delta^{-1})}{m}}\Big) \nonumber\\
            &\qquad + \cO\Big( \s{B}_{\fl{U^{(\ell-1)}}}\sqrt{\frac{\lambda_r^\star}{\lambda_1^\star k}} + \frac{1}{\sqrt{k}}\Big\{\Lambda' + \frac{\Lambda}{\sqrt{\mu^\star\lambda_r^\star}} \nonumber\\
            &\qquad + \frac{\sigma_2r}{mt\lambda_r^\star}\sqrt{rd\log(rd)} + \frac{\sigma_1r^{3/2}}{mt\lambda_r^\star}\sqrt{rd\log rd} + \sigma\sqrt{\frac{r^3d\mu^\star\log^2 (r\delta^{-1})}{mt\lambda_r^\star}}\Big\}\Big)2\|\fl{w}^{\star(i)}\|_2\nonumber\\
            &=\frac{1}{\sqrt{k}}\Big\{\cO\Big(\max\{\epsilon, \|\fl{w}^{\star(i)}\|_2\}\s{B}_{\fl{U^{(\ell-1)}}}\sqrt{\frac{\lambda_r^\star}{\lambda_1^\star}}\Big) + \cO(\Lambda'\|\fl{w}^{\star(i)}\|_2) + \cO(\Lambda)  \nonumber\\
            &\qquad + \cO\Big(\frac{\|\fl{w}^{\star(i)}\|_2\sigma_2r}{mt\lambda_r^\star}\sqrt{rd\log(rd)}\Big) + \cO\Big(\frac{\|\fl{w}^{\star(i)}\|_2\sigma_1r^{3/2}}{mt\lambda_r^\star}\sqrt{rd\log rd}\Big) \nonumber\\
            &\qquad + \cO\Big(\sigma\Big(\frac{1}{\sqrt{\mu^\star}}\sqrt{\frac{r\log^2 (r\delta^{-1})}{m}} + \|\fl{w}^{\star(i)}\|_2\sqrt{\frac{r^3d\mu^\star\log^2 (r\delta^{-1})}{mt\lambda_r^\star}}\Big)\Big)\Big\}\nonumber\\ 
            &:= \alpha^{(\ell-1)}.\label{eq:inductive-corollary-b-bound-1}
        \end{align}
    Similarly, Using Corollaries~\ref{inductive-corollary:optimize_dp-lrs-U-F-norm-bounds} and \ref{inductive-corollary:optimize_dp-lrs-R-bounds} we have $\|\fl{U}^{(\ell)}\fl{w}^{(i, \ell)} - \fl{U}^{\star}\fl{w}^{\star(i)}\|_{2}$ 
    \begin{align}
        &= \|\fl{U}^{(\ell)}\fl{w}^{(i, \ell)} - \fl{U}^{(\ell)}(\fl{Q}^{(\ell-1)})^{-1}\fl{w}^{\star(i)} + \fl{U}^{(\ell)}(\fl{Q}^{(\ell-1)})^{-1}\fl{w}^{\star(i)} - \fl{U}^{\star}\fl{w}^{\star(i)}\|_2\nonumber\\
        &\leq \|\fl{U}^{(\ell)}(\fl{w}^{(i, \ell)} - (\fl{Q}^{(\ell-1)})^{-1}\fl{w}^{\star(i)})\|_2 + \|(\fl{U}^{(\ell)} - \fl{U}^{\star}\fl{Q}^{(\ell-1)})(\fl{Q}^{(\ell-1)})^{-1}\fl{w}^{\star(i)}\|_2\nonumber\\
        &\leq \|\fl{U}^{(\ell)}\|_{2}\|\fl{w}^{(i, \ell)} - (\fl{Q}^{(\ell-1)})^{-1}\fl{w}^{\star(i)})\|_2 + \|\fl{U}^{(\ell)} - \fl{U}^{\star}\fl{Q}^{(\ell-1)}\|_{\s{F}}\|(\fl{Q}^{(\ell-1)})^{-1}\fl{w}^{\star(i)}\|_2\nonumber\\
        &\leq (1+c'')\cdot\cO\Big(\frac{\max\{\epsilon, \|\fl{w}^{\star(i)}\|_2\}\s{B}_{\fl{U^{(\ell-1)}}}\sqrt{\frac{\lambda_r^\star}{\lambda_1^\star}}}{\sqrt{r\mu^\star}} + \frac{\Lambda'\|\fl{w}^{\star(i)}\|_2}{\sqrt{r\mu^\star}} + \frac{\Lambda}{\sqrt{r\mu^\star}} + \sigma\sqrt{\frac{r\log^2 (r\delta^{-1})}{m}}\Big) \nonumber\\
        &\qquad + \cO\Big(\s{B}_{\fl{U^{(\ell-1)}}}\sqrt{\frac{\lambda_r^\star}{\lambda_1^\star}} + \Lambda' + \frac{\Lambda}{\sqrt{\mu^\star\lambda_r^\star}} + \frac{\sigma_2r}{mt\lambda_r^\star}\sqrt{rd\log(rd)} + \frac{\sigma_1r\sqrt{r}}{mt\lambda_r^\star}\sqrt{rd\log rd} \nonumber\\ 
        &\qquad + \sigma\Big(\sqrt{\frac{r^3d\mu^\star\log^2 (r\delta^{-1})}{mt\lambda_r^\star}}+\sqrt{\frac{r^3\log^2(r\delta^{-1})}{m\lambda^{\star}_r}}\Big)\Big)2\|\fl{w}^{\star(i)}\|_2\Big)\nonumber\\
        &= \cO\Big(\max\{\epsilon, \|\fl{w}^{\star(i)}\|_2\}\s{B}_{\fl{U^{(\ell-1)}}}\sqrt{\frac{\lambda_r^\star}{\lambda_1^\star}}\Big) + \cO(\Lambda'\|\fl{w}^{\star(i)}\|_2) + \cO(\Lambda)  \nonumber\\
        &\qquad + \cO\Big(\frac{\|\fl{w}^{\star(i)}\|_2\sigma_2r}{mt\lambda_r^\star}\sqrt{rd\log(rd)}\Big)+ \cO\Big(\frac{\|\fl{w}^{\star(i)}\|_2\sigma_1r^{3/2}}{mt\lambda_r^\star}\sqrt{rd\log rd}\Big) \nonumber\\
        &\qquad  + \cO\Big(\sigma\Big(\frac{1}{\sqrt{\mu^\star}}\sqrt{\frac{r\log^2 (r\delta^{-1})}{m}} + \|\fl{w}^{\star(i)}\|_2\Big(\sqrt{\frac{r^3d\mu^\star\log^2 (r\delta^{-1})}{mt\lambda_r^\star}}+\sqrt{\frac{r^3\log^2(r\delta^{-1})}{m\lambda^{\star}_r}}\Big)\Big)\Big)\nonumber\\ 
        &:= \beta^{(\ell-1)}.\label{eq:inductive-corollary-b-bound-2}
    \end{align}
    
    Using \eqref{eq:inductive-corollary-b-bound-1} and \eqref{eq:inductive-corollary-b-bound-2}, we have:
    \begin{align}
       \varphi^{(i)} &= 2\Big(\sqrt{k}\alpha^{(\ell-1)} + c_1\beta^{(\ell-1)} + \sigma\sqrt{\frac{k\log (d\delta^{-1})}{m}}\Big)\\
       &\leq \cO\Big(\max\{\epsilon, \|\fl{w}^{\star(i)}\|_2\}\s{B}_{\fl{U^{(\ell-1)}}}\sqrt{\frac{\lambda_r^\star}{\lambda_1^\star}}\Big) + \cO(\Lambda'\|\fl{w}^{\star(i)}\|_2) + \cO(\Lambda) \nonumber\\
       &\qquad +  \cO\Big(\frac{\|\fl{w}^{\star(i)}\|_2\sigma_2r}{mt\lambda_r^\star}\sqrt{rd\log(rd)}\Big) + \cO\Big(\frac{\|\fl{w}^{\star(i)}\|_2\sigma_1r^{3/2}}{mt\lambda_r^\star}\sqrt{rd\log rd}\Big)\nonumber\\
       &\qquad + \cO\Big(\sigma\Big(\frac{1}{\sqrt{\mu^\star}}\sqrt{\frac{r\log^2 (r\delta^{-1})}{m}} + \sqrt{\frac{k\log (d\delta^{-1})}{m}}\nonumber\\
       &\qquad\qquad + \|\fl{w}^{\star(i)}\|_2\Big(\sqrt{\frac{r^3d\mu^\star\log^2 (r\delta^{-1})}{mt\lambda_r^\star}}+\sqrt{\frac{r^3\log^2(r\delta^{-1})}{m\lambda^{\star}_r}}\Big)\Big)\Big).\label{eq:inductive-corollary-b-bound-temp99}
    \end{align}
    Using \eqref{eq:inductive-corollary-b-bound-temp99} in Lemma~\ref{inductive-corollary:optimize_dp-lrs-innerloop-output-bounds} and setting $\epsilon' \gets \cO\Big(\s{B}_{\fl{U^{(\ell-1)}}}\sqrt{\frac{\lambda_r^\star}{\lambda_1^\star}}\cdot \epsilon$\Big), we have:
    \begin{align}
        &\lr{\fl{b}^{(i,\ell+1)}-\fl{b}^{\star (i)} }_{2} \leq 2\varphi^{(i)} + \epsilon'\nonumber\\
        &= \cO\Big(\max\{\epsilon, \|\fl{w}^{\star(i)}\|_2\}\s{B}_{\fl{U^{(\ell-1)}}}\sqrt{\frac{\lambda_r^\star}{\lambda_1^\star}}\Big) + \cO(\Lambda'\|\fl{w}^{\star(i)}\|_2) + \cO(\Lambda) \nonumber\\
       &\qquad +  \cO\Big(\frac{\|\fl{w}^{\star(i)}\|_2\sigma_2r}{mt\lambda_r^\star}\sqrt{rd\log(rd)}\Big) + \cO\Big(\frac{\|\fl{w}^{\star(i)}\|_2\sigma_1r^{3/2}}{mt\lambda_r^\star}\sqrt{rd\log rd}\Big)\nonumber\\
       &\qquad + \cO\Big(\sigma\Big(\frac{1}{\sqrt{\mu^\star}}\sqrt{\frac{r\log^2 (r\delta^{-1})}{m}} +  \sqrt{\frac{k\log (d\delta^{-1})}{m}}\nonumber\\
       &\qquad\qquad +\|\fl{w}^{\star(i)}\|_2\Big(\sqrt{\frac{r^3d\mu^\star\log^2 (r\delta^{-1})}{mt\lambda_r^\star}}+\sqrt{\frac{r^3\log^2(r\delta^{-1})}{m\lambda^{\star}_r}}\Big)\Big)\Big).
    \end{align}
    Recall that from Corollaries~\ref{inductive-corollary:optimize_dp-lrs-U-F-norm-bounds} and \ref{inductive-corollary:optimize_dp-lrs-R-bounds}, we have
    \begin{align}
        \lr{\Delta(\fl{U}^{+(\ell)}, \fl{U}^{\star})}_\s{F}  &\leq  (1+c'')\cO\Big(\Big\{\s{B}_{\fl{U^{(\ell-1)}}}\sqrt{\frac{\lambda_r^\star}{\lambda_1^\star}}+\Lambda'\sqrt{\lambda^{\star}_r} + \frac{\Lambda}{r} + \frac{\sigma_2r}{mt\lambda_r^\star}\sqrt{rd\log(rd)} \nonumber\\
        &+
        \frac{\sigma_1r\sqrt{r}}{mt\lambda_r^\star}\sqrt{rd\log rd}+ \sigma\Big(\sqrt{\frac{r^3d\mu^\star\log^2 (rdmt/\delta_0)}{mt\lambda_r^\star}}\Big)\Big\}\Big)
    \end{align}

Therefore, it is sufficient to have for sufficiently large constants $\tilde{c},\hat{c}>0$
\begin{align}
    \Lambda &= \tilde{c}\Big( \sqrt{\lambda^{\star}_r\mu^{\star}}\Big(\frac{\sigma_2 r}{mt\lambda^{\star}_r}+\frac{\sigma_1r^{3/2}}{mt\lambda_r^\star}\sqrt{rd\log rd} +\sigma\sqrt{\frac{r^3d\mu^\star\log^2 (r\delta^{-1})}{mt\lambda_r^\star}}\Big)\nonumber\\
    &\qquad + \sigma\Big(\sqrt{\frac{r^3\log^2(r\delta^{-1})}{m\lambda^{\star}_r}}\Big) + \sqrt{\frac{k\log (d\delta^{-1})}{m}}\Big)\Big) \\
    \Lambda' &= \hat{c}\Big(\frac{\Lambda}{\sqrt{\mu^{\star}\lambda^{\star}_r}}\Big).
\end{align}
such that
$\lr{\fl{b}^{(i,\ell+1)}-\fl{b}^{\star(i)}}_2 \le \frac{1}{10}\max\{\epsilon, \|\fl{w}^{\star(i)}\|_2\}\s{B}_{\fl{U^{(\ell-1)}}}\sqrt{\frac{\lambda_r^\star}{\lambda_1^\star}} +\Lambda$ and $ \lr{\Delta(\fl{U}^{+(\ell)}, \fl{U}^{\star})}_\s{F} \le  \frac{\s{B}_{\fl{U^{(\ell-1)}}}}{100}\sqrt{\frac{\lambda_r^\star}{\lambda_1^\star}}+\Lambda'$ which satisfies the induction assumption and therefore completes the proof.

Comparing the contribution of noise-deficit terms on both sides for the next iteration, we also get the value of c' as 
\begin{align}
    c'''\max\{\|\fl{w}^{\star(i)}\|_2, \epsilon\}\s{B}_{\fl{U^{(\ell-1)}}}\sqrt{\frac{\lambda_r^\star}{\lambda_1^\star}} &:= c'\max\{\|\fl{w}^{\star(i)}\|_2\nonumber, \epsilon\}\s{B}_{\fl{U^{(\ell)}}}\sqrt{\frac{\lambda_r^\star}{\lambda_1^\star}}\nonumber\nonumber\\ &\leq \frac{51}{50*200}c'\max\{\|\fl{w}^{\star(i)}\|_2, \epsilon\}\s{B}_{\fl{U^{(\ell-1)}}}\sqrt{\frac{\lambda_r^\star}{\lambda_1^\star}}\nonumber\\
    \implies c' &:= \frac{50*200*c'''}{51} < 5.\nonumber
\end{align}
using sufficiently large $m$ and $t$ to pull down the value of $c'''$. Combining with the Base Case we have $c' = \max\Big(\cO(1), O\Big(\frac{1} {\s{B}_{\fl{U}^{(0)}}}\frac{\lambda^{\star}_1}{\lambda^{\star}_r}\Big)\Big)$.
\end{proof}

\begin{thm*}[Restatement of Theorem \ref{thm:private2} (Parameter Estimation)]
Consider the LRS problem \eqref{prob:general} with all parameters $m,t,\zeta$ obeying the bounds stated in Theorem \ref{thm:main} with $\zeta = O\Big(t(r^2(\mu^{\star})^2)^{-1}(\frac{\lambda_r^{\star}}{\lambda_1^{\star}})^2\Big)$, $k=O\Big(d\cdot  (\frac{\lambda_r^{\star}}{\lambda_1^{\star}})^2\Big)$, $m=\widetilde{\Omega}\Big(k + r^2\mu^\star\Big(\frac{\lambda_1^\star}{\lambda_r^\star}\Big)^2 + \frac{\sigma^2r^3}{\lambda_r^\star}\Big)$, $mt=\widetilde{\Omega}\Big(r^3d\mu^{\star}\Big(r(\mu^\star)^4(\lambda_r^\star)^2k + \mu^\star\Big(\frac{\lambda_1^\star}{\lambda_r^\star}\Big)^2 + \mu^\star(\lambda_r^\star)^2 + \sigma^2\Big(1+\frac{1}{\lambda_r^\star}\Big)\Big)\Big)$ and furthermore, $t=\widetilde{\Omega}\Big((rd)^{3/2}\mu^\star\Big(1+\lambda_r^\star + \sqrt{rk}(\mu^\star)^{3/2}\lambda_r^\star + \sqrt{\mu^\star}\Big(1+\frac{(\max_i \|\fl{b}^\star{(i)}\|_2}{\sqrt{\mu^\star\lambda_r^\star}} + \sqrt{\frac{\lambda_r^\star}{\lambda_1^\star}}\Big)\Big)\frac{\sqrt{\log(1/\delta)+\epsilon}}{\epsilon}\Big)$.
Suppose we run Algorithm \ref{algo:optimize_dp-lrs}  for $\s{L}=\log\left(\frac{\lambda^{\star}_r}{\sigma\sqrt{\lambda^\star_1}}\cdot \sqrt{\frac{mt}{\mu^{\star} r d}}\right)$ iterations with parameters: 
\begin{align}
\s{A}_1 = \widetilde{O}(\sqrt{d}), \s{A}_2 = \widetilde{O}(\sqrt{\mu^\star\lambda_r^\star} + (\max_i \|\fl{b}^{\star(i)}\|_2)), 
\s{A}_3= \widetilde{O}\Big(\lambda_r^\star\sqrt{\frac{\mu^\star}{\lambda_1^\star}}\Big), \s{A}_w = \widetilde{O}(\sqrt{\mu^\star\lambda_r^\star}). \nonumber
\end{align}
Then, w.p. $\geq 1-O(\delta_0)$, the outputs $\fl{U}^{+(\s{L})},\{\fl{b}^{(i,\s{L})}\}_{i\in [t]}$ satisfies: 
\begin{align}
     &\lr{(\fl{I}-\fl{U}^{\star}(\fl{U}^{\star})^{\s{T}})\fl{U}^{+(\s{L})}}_{\s{F}} = \widetilde{O}\Big(\frac{\sigma}{\sqrt{\mu^\star\lambda_r^\star}}\Big(\mu^\star\sqrt{\frac{r^3d}{mt}}+ \sqrt{\frac{r^3}{m\lambda^{\star}_r}} + \sqrt{\frac{k}{m}}\Big)\Big)+\frac{\sqrt{k}\eta}{\sqrt{\mu^\star\lambda_r^\star}}\nonumber \\  &\left|\left|\fl{b}^{(i, \s{L})}-\fl{b}^{\star (i)}\right|\right|_{\infty} \le \widetilde{O}\Big(\frac{\sigma}{\sqrt{k}}\Big(\mu^\star\sqrt{\frac{r^3d}{mt}}+ \sqrt{\frac{r^3}{m\lambda^{\star}_r}} + \sqrt{\frac{k}{m}}\Big)\Big)+\eta ,\text{ for all } i\in [t],
     \nonumber
\end{align}
where $\sqrt{k}\eta=\widetilde{O}\Big(t^{-1}(\mu^\star)^{3/2}\sqrt{\lambda_r^\star} r\sqrt{d}\Big(1+ \max_{i\in[t]}\frac{\|\fl{b}^{\star(i)}\|_2}{\sqrt{\mu^\star\lambda_r^\star}}+\sqrt{\frac{\lambda_r^\star}{\lambda_1^\star}} + rd\Big)\frac{\sqrt{\log(1/\delta)+\epsilon}}{\epsilon}\Big)$. 
\end{thm*}

\begin{proof}
We will denote the DP noise by $\sigma_{\s{DP}}>0$. 
Using standard gaussian concentration inequalities, we set $\s{A}_1, \s{A}_2, \s{A}_3$ and $\s{A}_w$ as written in the theorem statement which ensures that for all $i,j,\ell$ in $\fl{U}$ update of Algorithm, let $\|\fl{x}^{(i)}_j\|_2\leq A_1$, $\|\fl{w}^{(i, \ell)}\|_2 \leq A_w$, $|y^{(i)}_j|\leq A_2$, and $\|(\fl{x}^{(i)}_j)^\top\fl{b}^{(i,\ell)}\|_2\leq A_3$ with probability $1-\cO(\frac{1}{\s{Poly}(mt\s{L})})$.
Setting each entry of $\fl{N}_1$ independently according to $\mathcal{N}\left(0,m^2\cdot A_1^4\cdot A_w^4\cdot {\sf L}\cdot\sigma_{\s{DP}}^2\right)$ ($\sigma_1^2=m^2\cdot A_1^4\cdot A_w^4\cdot {\sf L}\cdot\sigma_{\s{DP}}^2$) and each entry of $\fl{N}_2$ is independently set  $\mathcal{N}\left(0,m^2\cdot A_1^2(A_2+A_3)^2A_w^2\cdot {\sf L}\cdot\sigma_\s{DP}^2\right)$ ($\sigma_2^2=m^2\cdot A_1^2(A_2+A_3)^2A_w^2\cdot {\sf L}\cdot\sigma_\s{DP}^2$) ensures that the algorithm
satisfies $\frac{1}{\sigma_\s{DP}}^2$-zCDP and equivalently $(\epsilon, \delta)$ Approximate Differential Privacy if $\sigma_\s{DP} \geq \frac{\sqrt{\log(1/\delta)+\epsilon}}{\epsilon}$ (Theorem~\ref{thm:private1}).

Using the bounds on $m,t,mt,\zeta, k$ in terms of the ground truth model parameters $\mu^{\star},\lambda^{\star}_1,\lambda^{\star}_r$ expressed in the theorem statement, we invoke Corollaries~\ref{inductive-corollary:optimize_dp-lrs-h,H-bounds}, \ref{inductive-inductive-corollary:optimize_dp-lrs-w-incoherence}, \ref{inductive-corollary:optimize_dp-lrs-U-F-norm-bounds}, \ref{inductive-corollary:optimize_dp-lrs-R-bounds}, ~\ref{inductive-corollary:optimize_dp-lrs-U-infty-norm-bounds} and \ref{inductive-corollary:optimize_dp-lrs-b-bounds} as well as the Base Case  \ref{inductive-corollary:optimize_dp-lrs-base-case} ($\ell = 1$) to show that our Inductive Assumption ~\ref{assum:init} holds for each iteration of $\ell$ and complete out proof using the Principle of Induction.

Now, note that the error bound guarantees in ~\ref{assum:init} have two terms in the upper bounds: the first one (a multiple of $\s{B}_{\fl{U^{(\ell-1)}}}$, which stems from analysing the problem in the noiseless setting) decreases exponentially with the number of iterations the second unchanging one ($\Lambda$ and $\Lambda'$ depends on the inherent noise $\sigma$ and DP noise $\sigma_{\s{DP}}$). Plugging $\s{L}=\log\left(\frac{\lambda^{\star}_r}{\sigma\sqrt{\lambda^\star_1}}\cdot \sqrt{\frac{mt}{\mu^{\star} r d}}\right)$ in the geometric series expression, we obtain the guarantees as stated in the main theorem.
\end{proof}

\begin{coro}[Restatement of Theorem \ref{thm:main} (Parameter Estimation)]
Consider the \lrs problem \eqref{prob:general} with $t$ linear regression tasks and samples obtained by \eqref{eq:samples}. Let model parameters satisfy assumptions A1, A2. Also, let the row sparsity of $\fl{B}^{\star}$ satisfy $\zeta = O\Big(t(r^2(\mu^{\star})^2)^{-1}(\frac{\lambda_r^{\star}}{\lambda_1^{\star}})^2\Big)$, $k=O\Big(d\cdot  (\frac{\lambda_r^{\star}}{\lambda_1^{\star}})^2\Big)$, $m=\widetilde{\Omega}\Big(k + r^2\mu^\star\Big(\frac{\lambda_1^\star}{\lambda_r^\star}\Big)^2 + \frac{\sigma^2r^3}{\lambda_r^\star}\Big)$. 
Suppose Algorithm \ref{algo:optimize_lrs1} is initialized with $\fl{U}^{+(0)}$ such that $\lr{(\fl{I}-\fl{U}^{\star}(\fl{U}^{\star})^{\s{T}})\fl{U}^{+(0)}}_{\s{F}} = O\Big(\sqrt{\frac{\lambda_r^{\star}}{\lambda_1^{\star}}}\Big)$ and  $\lr{\fl{U}^{+(0)}}_{2,\infty}=O(\sqrt{\mu^{\star}r/d})$, and is run for $\s{L}=\log\left(\frac{\lambda^{\star}_r}{\sigma\sqrt{\lambda^\star_1}}\cdot \sqrt{\frac{mt}{\mu^{\star} r d}}\right)$ iterations. Then, w.p. $\geq 1-O(\delta_0)$, the outputs $\fl{U}^{+(\s{L})},\{\fl{b}^{(i,\s{L})}\}_{i\in [t]}$ satisfies:  
\begin{align}
     &\lr{(\fl{I}-\fl{U}^{\star}(\fl{U}^{\star})^{\s{T}})\fl{U}^{+(\s{L})}}_{\s{F}} = \widetilde{O}\Big(\frac{\sigma}{\sqrt{\mu^\star\lambda_r^\star}}\Big(\mu^\star\sqrt{\frac{r^3d}{mt}}+ \sqrt{\frac{r^3}{m\lambda^{\star}_r}} + \sqrt{\frac{k}{m}}\Big)\Big), \nonumber\\  &\left|\left|\fl{b}^{(i, \s{L})}-\fl{b}^{\star (i)}\right|\right|_{\infty} \le \widetilde{O}\Big(\frac{\sigma}{\sqrt{k}}\Big(\mu^\star\sqrt{\frac{r^3d}{mt}}+ \sqrt{\frac{r^3}{m\lambda^{\star}_r}} + \sqrt{\frac{k}{m}}\Big)\Big) ,\text{ for all } i\in [t],\nonumber
\end{align}
where, the total number of samples satisfies: 
\begin{align}
&m=\widetilde{\Omega}\Big(k + r^2\mu^\star\Big(\frac{\lambda_1^\star}{\lambda_r^\star}\Big)^2 + \frac{\sigma^2r^3}{\lambda_r^\star}\Big), \nonumber\\
&mt=\widetilde{\Omega}\Big(r^3d\mu^{\star}\Big(r(\mu^\star)^4(\lambda_r^\star)^2k + \mu^\star\Big(\frac{\lambda_1^\star}{\lambda_r^\star}\Big)^2 + \mu^\star(\lambda_r^\star)^2 + \sigma^2\Big(1+\frac{1}{\lambda_r^\star}\Big)\Big)\Big)\nonumber.
\end{align}
\end{coro}

\begin{proof}
The proof follows by substituting $\sigma_{\s{DP}}=0$ (hence $\sigma_1,\sigma_2=0$) in the proof of Theorem \ref{thm:private2}.
\end{proof}

\subsection{Proof of Theorem~\ref{thm:private1}}
Following along similar lines of proof techniques used for privacy guarantees used in \cite{pmlr-v178-varshney22a}, our proof will broadly involve computing the Zero Mean Concentrated Differential Privacy (zCDP) parameters and then using them to prove the Approximate Differential Privacy. The Update Step for $\fl{U}^{(\ell)}$ without the additive DP Noise is:
\begin{align}
    & \widehat{\fl{x}^{(i)}_j}\leftarrow {\sf clip}\left(\fl{x}^{(i)}_j,\s{A}_1\right), \widehat{y^{(i)}_j} \leftarrow {\sf clip}\left(y^{(i)}_j,\s{A}_2\right),\widehat{(\fl{x}^{(i)}_j)^\s{T}\fl{b}^{(i, \ell)}}\leftarrow {\sf clip}\left((\fl{x}^{(i)}_j)^\s{T}\fl{b}^{(i, \ell)},\s{A}_3\right)\nonumber\\
    &\text{ and } \widehat{\fl{w}^{(i, \ell)}}\leftarrow {\sf clip}\left(\fl{w}^{(i, \ell)},\s{A}_w\right)\label{line:dp-proof-udpate1}\\
    &\fl{A} := \frac{1}{mt}\sum_{i \in [t]}\Big(\widehat{\fl{w}^{(i, \ell)}}(\widehat{\fl{w}^{(i, \ell)}})^{\s{T}} \otimes \Big(\sum_{j=1}^{m}\widehat{\fl{x}^{(i)}_j}(\widehat{\fl{x}^{(i)}_j})^{\s{T}}\Big)\Big)\label{line:dp-proof-udpate2}\\
    &\fl{V} := \frac{1}{mt}\sum_{i \in [t]}\sum_{j \in [m]} \widehat{\fl{x}^{(i)}_j}\Big(\widehat{y^{(i)}_j} - (\widehat{\fl{x}^{(i)}_j)^{\s{T}}\fl{b}^{(i, \ell)}} \Big)(\widehat{\fl{w}^{(i, \ell)}})^{\s{T}}\label{line:dp-proof-udpate3}\\
    &\fl{U}^{(\ell)} \gets  \s{vec}^{-1}_{d\times r}(\fl{A}^{-1}\s{vec}(\fl{V}))
    \label{line:dp-proof-udpate4}.
\end{align}
where $\s{clip}(\dot,\dot)$ denotes the clipping function. Therefore, the sensitivity of $\fl{A}$ and $\fl{V}$ due to samples from $i^{\s{th}}$-task (w.r.t. the Frobenius norm) is 
$\Gamma_1 = m\s{A}_1^2\s{A}_w^2$, and $\Gamma_2 = m\s{A}_1(\s{A}_2 + \s{A}_3)\s{A_w}$ respectively.
Now, since each entry of $\fl{N}^{(1)}$ is independently generated from $\mathcal{N}\left(0,m^2\cdot \s{A}_1^4\cdot \s{A}_w^4\cdot {\sf L}\cdot\sigma_{\s{DP}}^2\right)$ and each entry of $\fl{N}^{(2)}$ is independently generated from  $\mathcal{N}\left(0,m^2\cdot \s{A}_1^2(\s{A}_2+\s{A}_3)^2\s{A}_w^2\cdot {\sf L}\cdot\sigma_{\s{DP}}^2\right)$, the update steps \eqref{line:dp-proof-udpate2} and \eqref{line:dp-proof-udpate3} are $\Big(\rho_{\ell, 1} = \frac{\Gamma_1^2}{2\cdot m^2\cdot \s{A}_1^4\cdot \s{A}_w^4\cdot {\sf L}\cdot\sigma_{\s{DP}}^2} = \frac{1}{2{\sf L}\cdot\sigma_{\s{DP}}^2}\Big)$-zCDP and $\Big(\rho_{\ell, 2} = \frac{\Gamma_2^2}{2\cdot m^2\cdot \s{A}_1^4\cdot \s{A}_w^4\cdot {\sf L}\cdot\sigma_{\s{DP}}^2} = \frac{1}{2{\sf L}\cdot\sigma_{\s{DP}}^2}\Big)$-zCDP respectively by virtue of the DP noise standard deviations~\cite{bun2016concentrated}. Therefore by composition and robustness to post-processing, each iteration step is $\Big(\rho_\ell = \rho_{\ell,1} + \rho_{\ell,2} = \frac{1}{{\sf L}\cdot\sigma_{\s{DP}}^2}\Big)$-zCDP.  By composition of zCDPs, the overall $\rho$ for the algorithm  is given by $\rho = \sum_{\ell = 1}^{\s{L}} \rho_{\ell} = \frac{1}{\sigma_\s{DP}^2}$.

Recall $\rho$-zCDP  for an algorithm is equivalent to obtaining a $(\mu,\mu\rho)$-Renyi differential privacy (RDP)~\cite{mironov2017renyi} guarantee. Now, we will optimize for $\mu\in[1,\infty)$ and demonstrate that for the choice of the noise multiplier $\sigma_\s{DP}$ mentioned in the theorem statement satisfies $(\epsilon,\delta)$-DP. Our analysis is similar to that of Theorem 1 of \cite{chien2021private}.

Note that $(\mu,\mu\rho)$-(RDP) $\implies$ $(\epsilon, \delta)$ Approximate Privacy where $\epsilon = \mu\rho + \frac{\log(1/\delta)}{\mu - 1}$ $\forall \mu > 1$. Also note that $\epsilon_{\min} = \rho + 2\sqrt{\rho\log(1/\delta)}$ is attained at $\frac{\text{d}\epsilon}{\text{d}\mu} = 0 \implies \mu = 1 + \sqrt{\frac{\log(1/\delta)}{\rho}}$.

Consider a fixed $\epsilon$. Since we want to minimize $\sigma_\s{DP}$ (which scales as $1/\sqrt{\rho}$), we need to compute the maximum permissable $\rho$ s.t. $\epsilon_{\min}(\rho) \leq \epsilon$. Since $\epsilon_{\min}(\rho)$ is an increasing function of $\rho$ (thus an increasing function of $\sigma_\s{DP}$) and a second order polynomial in $\sqrt{\rho}$ with root at $\sqrt{\rho} = \sqrt{\log (1/\delta) + \epsilon_{\min}} - \sqrt{\log(1/\delta)}$, the maximum is achieved at $\epsilon_{\min}(\rho) = \epsilon$. Therefore,
\begin{align}
    \frac{1}{\sigma_\s{DP}^2} &= (\sqrt{\log (1/\delta) + \epsilon} - \sqrt{\log(1/\delta)})^2= \frac{\epsilon^2}{(\sqrt{\log (1/\delta) + \epsilon} + \sqrt{\log(1/\delta)})^2}.
\end{align}
Since the above value of $\sigma_\s{DP}$ satisfies $(\epsilon, \delta)$-DP and 
\begin{align}
    \frac{\epsilon^2}{(\sqrt{\log (1/\delta) + \epsilon} + \sqrt{\log(1/\delta)})^2} \geq \frac{\epsilon^2}{4(\log(1/\delta) + \epsilon)},
\end{align}
choosing $\sigma_\s{DP}\geq \frac{2\sqrt{(\log(1/\delta)+\epsilon)}}{\epsilon}$ ensures $(\epsilon, \delta)$-DP.

\section{Algorithm and Proof of Theorem \ref{thm:main} (Generalization Guarantees) }\label{app:generalization}

\begin{algorithm}[!b]
\caption{\textsc{AM-New Task}\label{algo:generalize}}
\begin{algorithmic}[1]
\REQUIRE Data $\{(\fl{X}\in \bb{R}^{m'\times d},\fl{y}\in \bb{R}^{m'})\}$, known bounds $\lr{\fl{b}^{\star}}_{\infty}\le \s{C}$. Set parameter $\epsilon>0$ appropriately. Estimate $\fl{U}^{+}$ of $\fl{U}^{\star}$ satisfying $\lr{(\fl{I}-\fl{U}^{\star}(\fl{U}^{\star})^{\s{T}})\fl{U}^{+}}_{\s{F}} \le \rho$. Parameter $\s{A}$.

\FOR{$\ell = 1, 2, \dots$}
    \STATE Initialize $\fl{w}^{(0)},\fl{b}^{(0)}=\f{0}$. Set $\phi^{(0)}=2$ since $\lr{\fl{w}^{(0)}-(\fl{U}^{\star})^{\s{T}}\fl{U}^{+})^{-1}\fl{w}^{\star}}_2 \le \phi^{(0)}\lr{\fl{w}^{\star}}_2\le 2\lr{\fl{w}^{\star}}_2$. Set $\gamma^{(0)} \ge \lr{\fl{b}^{\star}}_{\infty}$.
    \FOR{$i = 1,2,3, \dots, t$}
        \STATE Set $T^{(\ell)} = \Omega\Big(\ell\log\Big(\frac{\gamma^{(\ell-1)}}{\epsilon}\Big)\Big)$.
        \STATE  $\fl{w}^{(\ell)} = \Big((\fl{X}^{(i)}\fl{U}^{+(\ell-1)})^{\s{T}}(\fl{X}^{(i)}\fl{U}^{+(\ell-1)})\Big)^{-1}\Big((\fl{X}^{(i)}\fl{U}^{+(\ell-1)})^{\s{T}}(\fl{y}^{(i)} -  \fl{X}^{(i)}\fl{b}^{(i, \ell)})\Big)$ \COMMENT{Use a fresh batch of data samples}
        \STATE $\fl{b}^{(\ell)} \gets \s{Optimize Sparse Vector}(\fl{X}, \fl{y}, \alpha = \s{A}+ c_1\phi^{(\ell-1)}\lr{\fl{w}^{\star}}_2+ \frac{2\rho\lr{\fl{w}^{\star}}_2}{\sqrt{k}}, \beta = \s{A}+\phi^{(\ell-1)}\lr{\fl{w}^{\star}}_2+2\rho\lr{\fl{w}^{\star}}_2, \gamma = \s{A}+\frac{\lr{\fl{w}^{\star}}_2}{\sqrt{k}}\Big(\phi^{(\ell-1)}c'+\lr{\fl{w}^{\star}}_2\rho(1+c'')\Big), \s{T} = T^{(\ell)})$  \COMMENT{Use a fresh batch of data samples, constants $c_1,c',c''$ set appropriately.}
        \STATE Set $\Phi^{(\ell)} \leftarrow \lr{\fl{w}^{\star}}_2\Phi^{(\ell-1)}c_3+2\rho\lr{\fl{w}^{\star}}_2\Big(1+c_4\Big)+\s{A}$. \COMMENT{$c_3,c_4$ can be made arbitrarily small by increasing number of samples $m'$.}
    \ENDFOR
\ENDFOR
\STATE Return  $\fl{w^{(\ell)}}$ and $\fl{b}^{( \ell)}$.
\end{algorithmic}
\end{algorithm}

Consider a new task for which we get the samples $\{(\fl{x}_i,y_i)\}_{i=1}^{m'}$ i.e. $y_i= \langle\fl{x}_i, \fl{U}^{\star}\fl{w}^{\star}+\fl{b}^{\star}\rangle$ for all $i\in [m']$. 
Suppose we have an estimate $\fl{U}^{+}$ of $\fl{U}^{\star}$ such that $(\fl{U}^{+})^{\s{T}}\fl{U}^{+}=\fl{I}$ and 
\begin{align}
    &\lr{(\fl{I}-\fl{U}^{\star}(\fl{U}^{\star})^{\s{T}})\fl{U}^{+}}_{\s{F}} \le \rho, \;
     \lr{(\fl{I}-\fl{U}^{\star}(\fl{U}^{\star})^{\s{T}})\fl{U}^{+}}_{2,\infty} \le \frac{\rho}{\sqrt{k}} \text{ and } \lr{\fl{U}^{+}}_{2,\infty} \le \sqrt{\frac{\nu}{k}} 
\end{align}
for some known parameters $\nu,\rho$.
Our goal is to recover the vectors $\fl{w}^{\star}\in \bb{R}^r$ and $\fl{b}^{\star}\in \bb{R}^d$ satisfying $\lr{\fl{b}}_0 \le k$. 
We will again use an Alternating Minimization algorithm for recovery of $\fl{w}^{\star},\fl{b}^{\star}$. In the $\ell^{\s{th}}$ iteration, with probability at least $1-O(\delta/\s{L})$ for $m=\Omega(k\log (d\s{L}\delta^{-1}))$ we have the following updates for some constant $c>0$, (note that the $\ell^{\s{th}}$ iterates of $\fl{w}^{\star},\fl{b}^{\star}$ are given by $\fl{w}^{(\ell)},\fl{b}^{(\ell)}$). 

At the $\ell^{\s{th}}$ iteration, we will denote a known upper bound 
\begin{align}
    \lr{\fl{w}^{(\ell-1)}-\fl{Q}^{-1}\fl{w}^{\star}}_2 \le \phi^{(\ell-1)}\lr{\fl{w}^{\star}}_2+2\sigma\frac{\sqrt{k\log (d\delta^{-1})}}{\sqrt{m}}+2\sigma\frac{\sqrt{r\log^2 (r\delta^{-1})}}{\sqrt{m}}
\end{align}
where $\phi^{(\ell)}$ is known. We can use Lemma \ref{inductive-corollary:optimize_dp-lrs-innerloop-output-bounds} to have 
\begin{align}
    \lr{\fl{b}^{(\ell)}-\fl{b}^{\star}}_{2} \leq 2\varphi^{(i)} + \epsilon
    \text{ and } \lr{\fl{b}^{(\ell)}-\fl{b} }_{\infty} \leq \frac{1}{\sqrt{k}}\Big(2\varphi^{(i)} + \epsilon\Big)
\end{align}
with probability at least $1-T^{(\ell)}\delta$, where $\varphi$ is an upperbound on $\widehat{\varphi}$ s.t. 
\begin{align}
    \widehat{\varphi} &= 2\Big(\sqrt{k}\|\fl{U}^{+}\fl{w}^{(\ell-1)} - \fl{U}^{\star}\fl{w}^{\star}\|_{\infty}+c_1\|\fl{U}^{+}\fl{w}^{(\ell-1)} - \fl{U}^{\star}\fl{w}^{\star}\|_2+\sigma\frac{\sqrt{k\log (d\delta^{-1})}}{\sqrt{m}}\Big) \\
    &\leq \varphi = 2\Big(\sqrt{k}\alpha^{(\ell-1)}+c_1\beta^{(\ell-1)}+\sigma\frac{\sqrt{k\log (d\delta^{-1})}}{\sqrt{m}}+2\sigma\frac{\sqrt{r\log^2 (r\delta^{-1})}}{\sqrt{m}}\Big).
\end{align}
and $\alpha^{(\ell-1)}$,  $\beta^{(\ell-1)}$ denote upper bounds on $\|\fl{U}^{+}\fl{w}^{(\ell-1)} - \fl{U}^{\star}\fl{w}^{\star}\|_{\infty}$, and $\|\fl{U}^{+}\fl{w}^{(\ell-1)} - \fl{U}^{\star}\fl{w}^{\star}\|_2$ respectively. Furthermore, we will also have that
$\s{support}(\fl{b}^{(\ell}) \subseteq \s{support}(\fl{b}^{\star})$. We denote $\fl{Q}=(\fl{U}^{\star})^{\s{T}}\fl{U}^{+}$. Using a similar analysis as in Corollary \ref{inductive-corollary:optimize_dp-lrs-b-bounds}, we have:
        \begin{align}
            &\|\fl{U}^{+}\fl{w}^{(\ell-1)} - \fl{U}^{\star}\fl{w}^{\star}\|_{\infty} = \|\fl{U}^{+}\fl{w}^{(\ell-1)} - \fl{U}^{(\ell-1)}\fl{Q}^{-1}\fl{w}^{\star} + \fl{U}^{(\ell-1)}\fl{Q}^{-1}\fl{w}^{\star} - \fl{U}^{\star}\fl{w}^{\star}\|_\infty\\
            &\leq \|\fl{U}^{+}\fl{w}^{(\ell-1)} - \fl{Q}^{-1}\fl{w}^{\star})\|_\infty + \|(\fl{U}^{(\ell-1)} - \fl{U}^{\star}\fl{Q})\fl{Q}^{-1}\fl{w}^{\star}\|_\infty\\
            &\leq \|\fl{U}^{+}\|_{2, \infty}\|\fl{w}^{(\ell-1)} - \fl{Q}^{-1}\fl{w}^{\star}\|_2 + \|\fl{U}^{+} - \fl{U}^{\star}\fl{Q}\|_{2,\infty}\|\fl{Q}^{-1}\fl{w}^{\star}\|_2\\
            &\leq \sqrt{\frac{\nu}{k}}\|\fl{w}^{(\ell-1)} - \fl{Q}^{-1}\fl{w}^{\star}\|_2+ \frac{2\rho\lr{\fl{w}^{\star}}_2}{\sqrt{k}} \\ 
            &\le \sqrt{\frac{\nu}{k}}\Big(\phi^{(\ell-1)}\lr{\fl{w}^{\star}}_2+2\sigma\frac{\sqrt{k\log (d\delta^{-1})}}{\sqrt{m}}+2\sigma\frac{\sqrt{r\log^2 (r\delta^{-1})}}{\sqrt{m}}\Big)+ \frac{2\rho\lr{\fl{w}^{\star}}_2}{\sqrt{k}}\\
            &:= \alpha^{(\ell-1)}.\label{eq:inductive-corollary-b-bound-1-gen}
        \end{align}

    Similarly, we have: 
    \begin{align}
        &\|\fl{U}^{(\ell-1)}\fl{w}^{( \ell-1)} - \fl{U}^{\star}\fl{w}^{\star}\|_{2} \\
        &= \|\fl{U}^{+}\fl{w}^{(\ell-1)} - \fl{U}^{(\ell-1)}\fl{Q}^{-1}\fl{w}^{\star} + \fl{U}^{(\ell-1)}\fl{Q}^{-1}\fl{w}^{\star} - \fl{U}^{\star}\fl{w}^{\star}\|_2  \\
            &\leq \|\fl{U}^{+}\fl{w}^{(\ell-1)} - \fl{Q}^{-1}\fl{w}^{\star})\|_2 + \|(\fl{U}^{(\ell-1)} - \fl{U}^{\star}\fl{Q})\fl{Q}^{-1}\fl{w}^{\star}\|_2\\
            &\leq \|\fl{U}^{+}\|_{2}\|\fl{w}^{(\ell-1)} - \fl{Q}^{-1}\fl{w}^{\star}\|_2 + \|\fl{U}^{+} - \fl{U}^{\star}\fl{Q}\|_{2}\|\fl{Q}^{-1}\fl{w}^{\star}\|_2\\
            &\leq \|\fl{w}^{(\ell-1)} - \fl{Q}^{-1}\fl{w}^{\star}\|_2+ 2\rho\lr{\fl{w}^{\star}}_2 \\
            &\le \phi^{(\ell-1)}\lr{\fl{w}^{\star}}_2+2\rho\lr{\fl{w}^{\star}}_2+2\sigma\frac{\sqrt{k\log (d\delta^{-1})}}{\sqrt{m}}+2\sigma\frac{\sqrt{r\log^2 (r\delta^{-1})}}{\sqrt{m}}\\
            &:= \beta^{(\ell-1)}.\label{eq:inductive-corollary-b-bound-2-gen}
    \end{align}
    
    Using \eqref{eq:inductive-corollary-b-bound-1-gen} and \eqref{eq:inductive-corollary-b-bound-2-gen}, we have:
    \begin{align}
       \varphi =  \phi^{(\ell-1)}\lr{\fl{w}^{\star}}_2(2\sqrt{\nu}+c_1)+\lr{\fl{w}^{\star}}_2(4\rho+4c_1\rho)+\sigma\frac{\sqrt{k\log (d\delta^{-1})}}{\sqrt{m}}\Big)
    \end{align}
    since $c_1 < \frac{1}{40}$, $\sqrt{\nu}\le \frac{1}{40}$ and $\rho \le \frac{1}{80}$.

    Using above in Lemma~\ref{inductive-corollary:optimize_dp-lrs-b-bounds} and setting $\epsilon \gets \varphi^{(i)}$, we have:
    \begin{align}
        \lr{\fl{b}^{(\ell)}-\fl{b}^{\star } }_{2} &\le 3\lr{\fl{w}^{\star}}_2\Big(\phi^{(\ell-1)}(2\sqrt{\nu}+c_1)+4\rho(1+c_1)\Big)\\
        &\qquad +3\sigma\frac{\sqrt{k\log (d\delta^{-1})}}{\sqrt{m}}+3\sigma\frac{\sqrt{r\log^2 (r\delta^{-1})}}{\sqrt{m}}\Big)
    \end{align}

Similarly, we will also have from our updates (with $\fl{S}=\frac{1}{m}\sum_{i=1}^{m'}\fl{x}_i(\fl{x}_i)^{\s{T}}$).
\begin{align}
    \fl{w}^{(\ell)}-\fl{Q}^{-1}\fl{w}^{\star} &= \Big(\fl{U}^{+\s{T}}\fl{S}\fl{U}^{+}\Big)^{-1}\Big(\fl{U}^{+\s{T}}\fl{S}(\fl{b}^{\star}-\fl{b}^{(\ell)})+\fl{U}^{+\s{T}}\fl{S}(\fl{U}^{\star}\fl{Q}-\fl{U}^{+})\fl{Q}^{-1}\fl{w}^{\star}\Big)\\
    &+\Big(\fl{U}^{+\s{T}}\fl{S}\fl{U}^{+}\Big)^{-1}\Big(\fl{U}^{\s{T}}(\fl{X}^{(i)})^{\s{T}}\fl{z}^{(i)})
\end{align}

We already know by using an $\epsilon$-net argument that $\lr{\Big(\fl{U}^{+\s{T}}\fl{S}\fl{U}^{+}\Big)^{-1}}_2\le 2$. We also know that 
\begin{align}
    &\lr{\bb{E}\fl{U}^{+\s{T}}\fl{S}(\fl{b}^{\star}-\widehat{\fl{b}}^{(\ell)})}_2 \le \sqrt{\nu}\lr{\fl{b}^{\star}-\fl{b}^{(\ell)}}_2 \\
    &\lr{\bb{E}\fl{U}^{+\s{T}}\fl{S}(\fl{U}^{\star}\fl{Q}-\fl{U}^{+})\fl{Q}^{-1}\fl{w}^{\star}}_2 =\lr{\fl{U}^{+\s{T}}(\fl{U}^{\star}\fl{Q}-\fl{U}^{+})\fl{Q}^{-1}\fl{w}^{\star}}_2 \le 2\rho\lr{\fl{w}^{\star}}_2 
\end{align}
and moreover, 
\begin{align}
    \lr{\fl{U}^{\star\s{T}}(\fl{S}-\fl{I})(\fl{b}^{\star}-\fl{b}^{(\ell)})}_2 &\le \lr{\fl{b}^{\star}-\fl{b}^{(\ell)}}_2 \sqrt{\frac{r\log \delta^{-1}}{m}} \\
    \lr{\fl{U}^{\star\s{T}}(\fl{S}-\fl{I})(\fl{U}^{\star}\fl{Q}-\fl{U}^{+})\fl{Q}^{-1}\fl{w}^{\star}}_2 &\le \lr{\fl{U}^{+\s{T}}(\fl{U}^{\star}\fl{Q}-\fl{U}^{+})\fl{Q}^{-1}\fl{w}^{\star}}_2 \sqrt{\frac{r\log \delta^{-1}}{m}} \\
    &\le 2\rho\sqrt{\frac{r\log \delta^{-1}}{m}}\lr{\fl{w}^{\star}}_2 \\
    \lr{\Big(\fl{U}^{+\s{T}}\fl{S}\fl{U}^{+}\Big)^{-1}\Big(\fl{U}^{\s{T}}(\fl{X}^{(i)})^{\s{T}}\fl{z}^{(i)})}_2 &\le \frac{\sigma\sqrt{r\log^2 (r\delta^{-1})}}{\sqrt{m}}
\end{align}
with probability at least $1-\delta/\s{L}$. 
Hence, we get that 

$\lr{\fl{w}^{(\ell)}-\fl{Q}^{-1}\fl{w}^{\star}}_2 \le  3\lr{\fl{w}^{\star}}_2\Big(\phi^{(\ell-1)}(2\sqrt{\nu}+c_1)+4\rho(1+c_1)\Big)\Big(\sqrt{\nu}+\sqrt{\frac{r\log \delta^{-1}}{m}}\Big)+2\rho\lr{\fl{w}^{\star}}_2\Big(1+\sqrt{\frac{r\log \delta^{-1}}{m}}\Big)+2\frac{\sigma\sqrt{r\log^2 (r\delta^{-1})}}{\sqrt{m}}+2\frac{\sigma\sqrt{k\log (d\delta^{-1})}}{\sqrt{m}}$. 
Therefore, for $m'=\Omega\Big(\max(k\log (d\s{L}\delta^{-1}), r\log\delta^{-1})\Big)$, we get a decrease along with a bias term. We can have $\phi^{(0)}=2\lr{\fl{w}^{\star}}_2$ by using $\fl{w}^{(0})=0$. After $\s{L}$ iterations,  we will get 
$\lr{\fl{w}^{(\s{L})}-\fl{Q}^{-1}\fl{w}^{\star}}_2 = O\Big(\rho\lr{\fl{w}^{\star}}_2+c'^{\s{L}-1}\lr{\fl{w}^{\star}}_2+\frac{\sigma\sqrt{k\log (d\delta^{-1})}}{m'}\Big)$; hence, we will have with $\s{L}=O\Big(\log \rho^{-1})$ that $\lr{\fl{w}^{(\s{L})}-\fl{Q}^{-1}\fl{w}^{\star}}_2 = O\Big(\rho\lr{\fl{w}^{\star}}_2+\frac{\sigma\sqrt{k\log (d\delta^{-1})}}{\sqrt{m}}+\frac{\sigma\sqrt{r\log^2 (r\delta^{-1})}}{\sqrt{m}}\Big)$.
The generalization error is given by 
\begin{align}
   \ca{L}(\fl{U}^{+},\fl{w}^{(\s{L})},\fl{b}^{(\s{L})})-\ca{L}(\fl{U}^{\star},\fl{w}^{\star},\fl{b}^{\star})
\end{align}
where $\ca{L}(\fl{U},\fl{w},\fl{b}) \triangleq \bb{E}_{(\fl{x},y)}(y-\langle \fl{x},\fl{U}\fl{w}+\fl{b}\rangle)^2$. Hence, we have that 
\begin{align}\label{eq:generalize}
    \ca{L}(\fl{U}^{+},\fl{w}^{(\s{L})},\fl{b}^{(\s{L})})-\ca{L}(\fl{U}^{\star},\fl{w}^{\star},\fl{b}^{\star}) \le \widetilde{O}\Big(\rho^2\lr{\fl{w}^{\star}}^2_2+\frac{\sigma^2(r+k)}{m}\Big). 
\end{align}

\begin{thm}[Restatement of Theorem \ref{thm:private2} (Generalization properties in private setting)]
Generalization error for a new task scales as: 
\begin{align}
&\ca{L}(\fl{U},\fl{w},\fl{b})-\ca{L}(\fl{U}^{\star},\fl{w}^{\star},\fl{b}^{\star})\\ &=\widetilde{O}\Big(\sigma^2\Big(\frac{r^3d(\mu^\star)^2}{mt}+ \frac{r^3}{m\lambda^{\star}_r} + \frac{k+r}{m}\Big)+\frac{dr^2(\log(1/\delta)+\epsilon)(\lambda_r^{\star}\mu^{\star})^2}{\epsilon^2 t^2}\cdot (\kappa^2+r^2d^2) \Big)    
\end{align}
where $\kappa=1+\sqrt{\frac{\lambda_r^\star}{\lambda_1^\star}} + \max_{i\in[t]}\frac{\|\fl{b}^{\star(i)}\|_2}{\sqrt{\mu^\star\lambda_r^\star}}$.
\end{thm}

\begin{proof}
We assume that $\lr{\fl{w}^{\star}}_2\le \sqrt{\mu^{\star}\lambda^{\star}_r}$ due to the incoherence (see Assumption A2).
We substitute $\rho$ to be the guarantee that we had obtained in Theorem \ref{thm:private2}; hence we immediately obtain our desired guarantees by using equation \ref{eq:generalize}.
\end{proof}

\begin{coro}[Restatement of Theorem \ref{thm:main} (Generalization Properties in non-private setting)]
Furthermore, for a new task, Algorithm~\ref{algo:generalize} ensures the following generalization error bound: 
$$\ca{L}(\fl{U},\fl{w},\fl{b})-\ca{L}(\fl{U}^{\star},\fl{w}^{\star},\fl{b}^{\star}) =\widetilde{O}\Big(\sigma^2\Big(\frac{r^3d(\mu^\star)^2}{mt}+ \frac{r^3}{m\lambda^{\star}_r} + \frac{k+r}{m}\Big)\Big).$$
\end{coro}

\begin{proof}
The proof follows again by substituting $\sigma_{\s{DP}}=0$ (hence $\sigma_1,\sigma_2=0$) which removes the last term in the generalization properties in Theorem \ref{thm:private2}.
\end{proof}

\section{Discussion on obtaining initial estimates using Method of Moments}\label{app:mom}

\paragraph{Overview:}
Note that Algorithm \ref{algo:optimize_lrs1} has local convergence properties as described in Theorem \ref{thm:main}. In practice, typically we use random initialization for $\fl{U}^{+(0)}$. However, similar to the representation learning framework \cite{tripuraneni2021provable}, we can use the Method of Moments to obtain a good initialization. i.e. when the representation matrix $\fl{U}^{\star}$ is of rank $r$, we can compute the Singular Value Decomposition (SVD) of the matrix $(mt)^{-1}\sum_{i\in [t]}(y^{(i)}_j)^2\fl{x}^{(i)}_j(\fl{x}^{(i)}_j)^{\s{T}}$. This is similar to the Method of Moments technique used in \cite{tripuraneni2021provable} and has been used as an initialization technique in the AM framework of \cite{thekumparampil2021statistically} as well. Even in the presence of additional sparse vectors, the SVD decomposition is robust. Such a phenomenon has been also been characterized theoretically in the robust PCA setting \cite{netrapalli2014non}.  

\paragraph{Details for Rank-$1$:}
 Assume $\|\fl{u}^{\star}\|_2=\|\fl{w}^{\star}\|_2=1$ and $\|\fl{b}^{\star(i)}\|_0\le k$ for all $i\in [t]$. Moreover, for some constant $\mu > 0$, we will have $\|\fl{u}^{\star}\|_{\infty} \le \sqrt{\mu/d},\|\fl{w}^{\star}\|_{\infty} \le \sqrt{\mu/t}, \max_{i\in [t]}\|\fl{B}\|_{\infty} \le \mu/\sqrt{dt}$  where $\fl{B}$ is the matrix whose columns correspond to $\fl{b}^{\star (i)}$'s. Suppose, we obtain samples $(\fl{x},y)\in \bb{R}^{d}\times \bb{R}$ where each sample is randomly generated from the $t$ data generating models corresponding to each task. In order to generate the $i^{\s{th}}$  sample  we first draw a latent variable $j \sim_U [t]$ and subsequently generate the tuple according to the following process:
\begin{align}\label{eq:mixtures}
    \fl{x}^{(i)} \mid j  \sim \ca{N}(0,\fl{I}_d) \text{ and } y^{(i)} \mid \fl{x}^{(i)}, j \sim \ca{N}(\langle \fl{x}^{(i)}, w^{\star}_j \fl{u}^{\star}+\fl{b}^{\star (j)}  \rangle,\sigma^2)
\end{align}

We look at the quantity $y^2\fl{x}\fl{x}^{\s{T}}$. Our first result is the following lemma:
\begin{lemma}
Suppose we obtain samples $\{(\fl{x}^{(i)},y^{(i)})\}$ generated according to the model described in equation \ref{eq:mixtures}. In that case we have
\begin{align}\label{lem:mom_expectation}
    \E{y^2\fl{x}\fl{x}^{\s{T}}} = \fl{I}+\frac{2}{t}\sum_j \Big(w^{\star}_j \fl{u}^{\star}+\fl{b}^{\star (j)}\Big)\Big(w^{\star}_j \fl{u}^{\star}+\fl{b}^{\star (j)}\Big)^{\s{T}}
\end{align}
where $\fl{I}$ denotes the $d$-dimensional identity matrix.
\end{lemma}

The proof follows from simple calculations. From the data $\{(\fl{x}^{(i)}_j,y_j^{(i)})\}_{j=1}^{m}$ for the $i^{\s{th}}$ task, we can compute an unbiased estimate $\fl{A}\triangleq \frac{1}{mt}\sum_{i=1}^{t}\sum_{j=1}^{m} (y_j^{(i)})^2 \fl{x}^{(i)}_j(\fl{x}^{(i)}_j)^\s{T}$ of the matrix $\E{y^2\fl{x}\fl{x}^\s{T}}$. Let us write $\fl{A}=\E{\fl{A}}+2\fl{F}$ where $2\fl{F}$ is the error in estimating $\E{\fl{A}}$. Also, let us denote $0.5t(\E{y^2\fl{x}\fl{x}^\s{T}}-\fl{I}) \triangleq \fl{L}$. In that case, we will have $0.5t(\fl{A}-\fl{I}) = 0.5t(\fl{A}-\E{  \fl{A}}+\E{\fl{A}}-\fl{I})=\fl{L}+\fl{F}$. We will also denote $$\fl{E}\triangleq \underbrace{\sum_{j=1}^{t} \Big(w^{\star}_j \fl{b}^{\star (j)}(\fl{u}^{\star})^{\s{T}}+w^{\star}_i\fl{u}^{\star}(\fl{b}^{\star (j)})^{\s{T}}+\fl{b}^{\star(j)}(\fl{b}^{\star (j)})^{\s{T}}\Big)}_{\fl{G}}+\fl{F}.$$ Our goal is to show that any eigenvector of $\fl{L}+\fl{F}$ is close to $\fl{u}^{\star}$ in infinity norm. Note that $(\fl{L}+\fl{F})\fl{z}=(\fl{u}(\fl{u}^{\star })^{\s{T}}+\fl{E})\fl{z}=\lambda \fl{z}$. Hence, we have 
\begin{align}\label{eq:spectral}
    \fl{z}=\Big(\fl{I}-\frac{\fl{E}}{\lambda}\Big)^{-1}\frac{\fl{u}^{\star}(\fl{u}^{\star})^{\s{T}}\fl{z}}{\lambda}. 
\end{align}
First, note that 
\begin{align}
     &\lambda \fl{z}\fl{z}^\s{T} - \fl{u}^{\star}(\fl{u}^{\star})^{\s{T}} = \frac{\fl{u}^{\star}(\fl{u}^{\star})^{\s{T}}\fl{z}\fl{z}^\s{T}\fl{u}^{\star}(\fl{u}^{\star})^{\s{T}}}{\lambda}+\sum_{p,q:p+q \ge 1} \frac{\fl{E}^p\fl{u}^{\star}(\fl{u}^{\star})^{\s{T}}\fl{z}\fl{z}^\s{T}\fl{u}^{\star}(\fl{u}^{\star})^{\s{T}}\fl{E}^q}{\lambda^{p+q+1}} \\
     & \resizebox{0.98\textwidth}{!}{$\|\lambda \fl{z}\fl{z}^{\s{T}} - \fl{u}^{\star}(\fl{u}^{\star})^{\s{T}}\|_{\infty} = \|\frac{\fl{u}^{\star}(\fl{u}^{\star})^{\s{T}}\fl{z}\fl{z}^\s{T}\fl{u}^{\star}(\fl{u}^{\star})^{\s{T}}}{\lambda}-\fl{u}^{\star}(\fl{u}^{\star})^{\s{T}}\|_{\infty}+\|\sum_{p,q:p+q \ge 1} \frac{\fl{E}^p\fl{u}^{\star}(\fl{u}^{\star})^{\s{T}}\fl{z}\fl{z}^{\s{T}}\fl{u}^{\star}(\fl{u}^{\star})^{\s{T}}\fl{E}^q}{\lambda^{p+q+1}}\|_{\infty}$}.
\end{align}

We have that 
\begin{align}
&\lr{\frac{\fl{u}^{\star}(\fl{u}^{\star})^{\s{T}}\fl{z}\fl{z}^\s{T}\fl{u}^{\star}(\fl{u}^{\star})^{\s{T}}}{\lambda}-\fl{u}^{\star}(\fl{u}^{\star})^{\s{T}}}_{\infty} \nonumber\\
= &\max_{ij} \fl{e}_i^{\s{T}}\Big(\frac{\fl{u}^{\star}(\fl{u}^{\star})^{\s{T}}\fl{z}\fl{z}^{\s{T}}\fl{u}^{\star}(\fl{u}^{\star})^{\s{T}}}{\lambda}-\fl{u}^{\star}(\fl{u}^{\star})^{\s{T}}\Big)\fl{e}_j \\
= &\fl{e}_i^{\s{T}}\Big(\fl{u}^{\star}(\fl{u}^{\star})^{\s{T}}+\fl{U}^{\star}_{\perp}(\fl{U}^{\star}_{\perp})^{\s{T}}\Big)\Big(\frac{\fl{u}^{\star}(\fl{u}^{\star})^{\s{T}}\fl{z}\fl{z}^{\s{T}}\fl{u}^{\star}(\fl{u}^{\star})^{\s{T}}}{\lambda}-\fl{u}^{\star}(\fl{u}^{\star})^{\s{T}}\Big)\Big(\fl{u}^{\star}(\fl{u}^{\star})^{\s{T}}+\fl{U}^{\star}_{\perp}(\fl{U}^{\star}_{\perp})^{\s{T}}\Big)\fl{e}_j \\
\le &\lr{\fl{u}^{\star}}_{\infty}^2\lr{\frac{\fl{u}^{\star}(\fl{u}^{\star})^{\s{T}}\fl{z}\fl{z}^\s{T}\fl{u}^{\star}(\fl{u}^{\star})^{\s{T}}}{\lambda}-\fl{u}^{\star}(\fl{u}^{\star})^{\s{T}}}_{2}  
\le \frac{\mu^2}{d}\Big(\frac{((\fl{u}^{\star})^{\s{T}}\fl{z})^2}{\lambda}-1\Big) 
\end{align}
where $\fl{U}^{\star}_{\perp}$ is the subspace orthogonal to the vector $\fl{u}^{\star}$. 

First, we will show an upper bound on $\lr{\fl{F}}_{\infty}$. 
Recall that according to the data generating mechanism, each co-variate $\fl{x}$ is generated according to $\ca{N}(\vzero,\fl{I}_d)$ and given the co-variate, the response $y\mid \fl{x} \sim \ca{N}(\langle \fl{x},w^\star \fl{u}^{\star}+\fl{b}^{\star} \rangle,\sigma^2)$ where $w^{\star},\fl{b}^{\star}$ is uniformly chosen at random from the set $\{w^{\star}_j,\fl{b}^{\star(j)}\}_{j}$. Hence, we can bound the magnitude of $y$ as follows:  $E{y^2|\fl{x}} = t^{-1}\sum_{j=1}^t\sigma^2+\langle \fl{x},w^{\star}_j \fl{u}^{\star}+\fl{b}^{\star(j)} \rangle^2$ and therefore $\E{y^2} = t^{-1}\sum_{j=1}^t \sigma^2+\lr{w^{\star}_j \fl{u}^{\star}+\fl{b}^{\star(j)}}_2^2$. Hence, $y\sim t^{-1}\sum_{j=1}^t \ca{N}(0,\sigma^2+\lr{w^{\star}_j \fl{u}^{\star}+\fl{b}^{\star(j)}}_2^2)$ and therefore, by using standard Gaussian concentration, we will have $|y|\le \sqrt{\sigma^2+\max_{j \in [t]}\lr{w^{\star}_j \fl{u}^{\star}+\fl{b}^{\star(j)}}_2^2}\log (mt) \le \sqrt{\sigma^2+4\mu t^{-1}}\log (mt)$ for all $mt$ samples w.p. at least $1-\s{poly}((mt)^{-1})$. Moreover, $|\fl{x}^{(i)}_{p,j}| \le \log (dmt)$ for all $i\in [t], p \in [m], j \in [d]$. Hence, with probability at least $1-\s{poly}((dmt)^{-1})$, by using standard concentration inequalities, we have $\lr{\fl{F}}_{\infty} \le \sqrt{\frac{\sigma^2+4\mu t^{-1}}{mt}}\log^3 (dmt)$. We will now bound $\lr{\fl{E}}_2 \le \lr{\fl{G}}_2+\lr{\fl{F}}_2$. 
In order to do so, fix unit vectors $\fl{x},\fl{y}$ such that $\lr{\fl{E}}_2 = \fl{x}^{\s{T}} \fl{E} \fl{y}=\sum_{is} x_i y_s \fl{E}_{is} =\frac12\sum_{is}(x_i^2+y_s^2)\fl{E}_{is} $. We have the following:
\begin{align}
  &\sum_{i} x_i^2 \sum_{j=1}^{t}\sum_{s=1}^{d} \fl{b}_i^{\star (j)} \fl{b}_s^{\star (j)} \le \zeta k \lr{\fl{B}}_{\infty}^2 \quad \text{ and } \quad  \sum_{s} y_s^2 \sum_{j=1}^{t} \fl{b}_s^{\star (j)} \sum_{i=1}^{d} \fl{b}_i^{\star (j)} \le \zeta k \lr{\fl{B}}_{\infty}^2 \le \frac{\mu^2 \zeta k}{dt} \\
  &\implies \sum_{i} x_i^2 \sum_{j=1}^{t} w^{\star}_j \fl{b}_i^{\star (j)} \sum_{s=1}^{d} \fl{u}_s^{\star} \le \zeta d \lr{\fl{B}}_{\infty}\lr{\fl{u}}_{\infty}\lr{\fl{w}}_{\infty}\nonumber\\
  &\quad \text{ and } \quad  \sum_{s} y_s^2 \sum_{j=1}^{t} w^{\star}_j \fl{u}_s^{\star} \sum_{i=1}^{d} \fl{b}_i^{\star (j)} \le  kt \lr{\fl{B}}_{\infty}\lr{\fl{u}}_{\infty}\lr{\fl{w}}_{\infty} \le \frac{\mu^2 k}{d} \\
&\implies \sum_{i} x_i^2 \sum_{j=1}^{t} w^{\star}_j \fl{u}_i^{\star (j)} \sum_{s=1}^{d} \fl{b}_s^{\star (j)} \le k t \lr{\fl{B}}_{\infty}\lr{\fl{u}}_{\infty}\lr{\fl{w}}_{\infty} \nonumber\\
&\quad \text{ and } \quad  \sum_{s} y_s^2 \sum_{j=1}^{t} w^{\star}_j \fl{b}_s^{\star (j)} \sum_{i=1}^{d} \fl{u}_i^{\star} \le  \zeta d \lr{\fl{B}}_{\infty}\lr{\fl{u}}_{\infty}\lr{\fl{w}}_{\infty}\le \frac{\mu^2 \zeta}{t}.
\end{align}
and similarly $\lr{\fl{F}}_2 \le \sqrt{d}\lr{\fl{F}}_{\infty}$. Hence $\lr{\fl{F}}_2 \le 1/800$ provided $mt=\Omega(d\sigma^2)$. 
 In that case, we have $\lr{\fl{E}}_2 \le 1/400$ provided $\zeta\le c_1 t$ and $k \le c_2 d$ for appropriate constants $0 \le c_1,c_2 \le 1$. 
   Therefore, $\lambda$ must be at least $399/400$ (Weyl's inequality) and $(\langle \fl{u}^{\star},\fl{z} \rangle^2-1) \le 4\lr{\fl{E}}_2$ (Davis Kahan). Hence, we have the following inequality: $\Big(\frac{((\fl{u}^{\star})^\s{T}\fl{z})^2}{\lambda}-1\Big)\le 1/100$. Again, we have 
\begin{align}
&\lr{\frac{\fl{E}^p\fl{u}^{\star}(\fl{u}^{\star})^{\s{T}}\fl{z}\fl{z}^\s{T}\fl{u}^{\star}(\fl{u}^{\star})^{\s{T}}\fl{E}^q}{\lambda^{p+q+1}}}_{\infty} \\
&= \max_{ij} \fl{e}_i^{\s{T}}\Big(\frac{\fl{E}^p\fl{u}^{\star}(\fl{u}^{\star})^{\s{T}}\fl{z}\fl{z}^\s{T}\fl{u}^{\star}(\fl{u}^{\star})^{\s{T}}\fl{E}^q}{\lambda^{p+q+1}}\Big)\fl{e}_j \\
&= \max_{ij} \fl{e}_i^{\s{T}}\fl{E}^p\Big(\fl{u}^{\star}(\fl{u}^{\star})^{\s{T}}+\fl{U}^{\star}_{\perp}(\fl{U}^{\star}_{\perp})^{\s{T}}\Big)\Big(\frac{\fl{u}^{\star}(\fl{u}^{\star})^{\s{T}}\fl{z}\fl{z}^\s{T}\fl{u}^{\star}(\fl{u}^{\star})^{\s{T}}}{\lambda^{p+q+1}}\Big)\Big(\fl{u}^{\star}(\fl{u}^{\star})^{\s{T}}+\fl{U}^{\star}_{\perp}(\fl{U}^{\star}_{\perp})^{\s{T}}\Big)\fl{E}^q\fl{e}_j \\
&\le \lr{\fl{E}^p\fl{u}^{\star}}_{\infty}\lr{\fl{E}^q\fl{u}^{\star}}_{\infty}\lr{\frac{\fl{u}^{\star}(\fl{u}^{\star})^{\s{T}}\fl{z}\fl{z}^\s{T}\fl{u}^{\star}(\fl{u}^{\star})^{\s{T}}}{\lambda^{p+q+1}}}_{2} \le \lr{\fl{E}^p\fl{u}^{\star}}_{\infty}\lr{\fl{E}^q\fl{u}^{\star}}_{\infty}\Big(\frac{((\fl{u}^{\star})^\s{T}\fl{z})^2}{\lambda^{p+q+1}}\Big)
\end{align}
where $\fl{U}^{\star}_{\perp}$ is the subspace orthogonal to the vector $\fl{u}^{\star}$.

\begin{lemma}
Let $\fl{e}_i\in \bb{R}^d$ denote the $i^{\s{th}}$ standard basis vector. In that case, we will have $$\max_i \lr{\fl{e}_i^{\s{T}}\fl{E}^p\fl{u}^{\star}}\le \frac{\mu}{\sqrt{d}}\Big(\frac{\mu^2\zeta k}{dt}+\frac{\mu^2 k}{d}+\frac{\mu^2\zeta}{t}+\sqrt{\frac{d}{mt}}\lr{\fl{F}}_{\infty}\Big)^{p}.$$
\end{lemma}   

\begin{proof}
We can prove this statement via induction. For $p=1$, the statement follows from the incoherence of $\fl{u}^{\star}$. Suppose the statement holds for $p=k$ for some $k > 1$. Under this induction hypothesis, we are going to show that the statement holds for $p=k+1$. We will have 
\begin{align}
&\lr{\fl{e}_i^{\s{T}}\fl{E}^{k+1}\fl{u}^{\star}}_2^2 = \sum_{\ell} (\fl{e}_i^{\s{T}}\fl{E}\fl{E}^{k}\fl{u}^{\star}\fl{e}_{\ell})^2 = \sum_{\ell} (\sum_j\fl{E}_{ij}\fl{E}^{k}\fl{u}^{\star}\fl{e}_{\ell})^2 \\
&= \sum_{j_1j_2}\fl{E}_{ij_1}\fl{E}_{ij_2}\lr{\fl{e}_{j_1}^{\s{T}}\fl{E}^{k}\fl{u}^{\star}}_2\lr{\fl{e}_{j_2}^{\s{T}}\fl{E}^{k}\fl{u}^{\star}}_2 \le \frac{\mu^2}{d}\Big(\frac{\mu^2\zeta k}{dt}+\frac{\mu^2 k}{d}+\frac{\mu^2\zeta}{t}+\sqrt{\frac{d}{mt}}\lr{\fl{F}}_{\infty}\Big)^{2k+2}
\end{align}
\end{proof}
Hence, we must have $\|\sum_{p,q:p+q \ge 1} \frac{\fl{E}^p\fl{u}^{\star}(\fl{u}^{\star})^{\s{T}}\fl{z}\fl{z}^{\s{T}}\fl{u}^{\star}(\fl{u}^{\star})^{\s{T}}\fl{E}^q}{\lambda^{p+q+1}}\|_{\infty} $
\begin{align}
    &\le \sum_{p,q:p+q \ge 1} \frac{\mu^2}{d}\Big(\frac{\mu^2\zeta k}{dt}+\frac{\mu^2 k}{d}+\frac{\mu^2\zeta}{t}\Big)^{p+q} \Big(\frac{1}{\lambda}\Big)^{p+q}\Big(\frac{((\fl{u}^{\star})^\s{T}\fl{z})^2}{\lambda}\Big) \\
    &\le \frac{\mu^2}{d}\Big(\frac{((\fl{u}^{\star})^\s{T}\fl{z})^2}{\lambda}\Big)\sum_{p,q:p+q \ge 1} \alpha^{p+q} = \frac{\mu^2}{d}\Big(\frac{((\fl{u}^{\star})^\s{T}\fl{z})^2}{\lambda}\Big)\Big(\Big(\frac{1}{1-\alpha}\Big)^2-1\Big) 
\end{align}
where $\alpha=\frac{\mu^2\zeta k}{dt\lambda}+\frac{\mu^2 k}{d\lambda}+\frac{\mu^2\zeta}{t\lambda}+\sqrt{\frac{d}{mt}}\lr{\fl{F}}_{\infty}$. Again, if $\zeta \le c_1 t$ and $k\le c_2 d$ for appropriate constants $0 \le c_1,c_2 \le 1$, we will have $\lr{\lambda\fl{z}\fl{z}^{\s{T}}-\fl{u}^{\star}(\fl{u}^{\star})^{\s{T}}}_{\infty} = O\Big(\frac{\mu^2}{d}\Big)$ and similarly, from our previous calculations on the operator norms, we will have $\lr{\lambda\fl{z}\fl{z}^{\s{T}}-\fl{u}^{\star}(\fl{u}^{\star})^{\s{T}}}_{2} = O(1)\sigma \sqrt{\frac{d}{mt}}$. Hence, provided $mt=\Omega(d\sigma^2)$, by using Davis Kahan inequality, we obtain the initialization guarantees that we need for the rank-$1$ setting (see Theorem \ref{thm:main}).


\section{Useful Lemmas}

\begin{lemma}[Hanson-Wright lemma]\label{lemma:useful0}
Let $\fl{x}^{(1)},\fl{x}^{(2)},\dots,\fl{x}^{(m)}\sim \ca{N}(\vzero,\fl{I}_{d\times d})$ be $m$ i.i.d. standard isotropic Gaussian random vectors of dimension $d$. Then, for some universal constant $c\ge0$, the following holds true with a probability of at least $1-\delta_0$
\begin{align}
    \left|\frac{1}{m}\sum_{j=1}^{m}\fl{x}_j^{\s{T}}\fl{A}_j\fl{x}_j-\frac{1}{m}\sum_{j=1}^{m}\s{Tr}(\fl{A}_j)\right| \leq
    c\max\Big(\sqrt{\sum_{j=1}^{m}\|\fl{A}_j\|_{\s{F}}^{2}\frac{\log \delta_0^{-1}}{m^2}},\max_{j=1,\dots,m}\|\fl{A}_j\|_2\frac{\log \delta_0^{-1}}{m}\Big). \nonumber
\end{align}
\end{lemma}

\begin{lemma}\label{lemma:useful1}
Let $\fl{x}^{(1)},\fl{x}^{(2)},\dots,\fl{x}^{(m)}\sim \ca{N}(\vzero,\fl{I}_{d\times d})$ be $m$ i.i.d. standard isotropic Gaussian random vectors of dimension $d$. Then, for some universal constant $c\ge0$, the following holds true with a probability of at least $1-\delta_0$. 
\begin{align}
    \left|\frac{1}{m}\sum_{j=1}^{m}\fl{a}^{\s{T}}(\fl{x}^{(j)}(\fl{x}^{(j)})^{\s{T}})\fl{b}-\fl{a}^{\s{T}}\fl{b}\right| \leq c\lr{\fl{a}}_2\lr{\fl{b}}_2\max\Big(\sqrt{\frac{\log \delta_0^{-1}}{m}},\frac{\log \delta_0^{-1}}{m}\Big). \nonumber
\end{align}
\end{lemma}

\begin{lemma}\label{lemma:weyl}
For three real r-rank matrices, satisfying $\fl{A} - \fl{B} = \fl{C}$, Weyl’s inequality tells that
\begin{align}
    \sigma_k(\fl{A}) - \sigma_k(\fl{B}) \leq \|\fl{C}\|  \nonumber  
\end{align}
$\forall$ $k \in [r]$ where $\sigma_k(\cdot)$ is the k-th largest singular value operator.
\end{lemma}

\begin{lemma}\label{lemma:useful2}
Let $\fl{x}^{(1)},\fl{x}^{(2)},\dots,\fl{x}^{(m)}\sim \ca{N}(\f{0},\fl{I}_{d\times d})$ be $m$ i.i.d. standard isotropic Gaussian random vectors of dimension $d$. Then, for some universal constant $c\ge0$, the following holds true with a probability of at least $1-\delta_0$, $\left|\left|\frac{1}{m}\sum_{j=1}^{m}\fl{a}_j(\fl{x}^{(j)}(\fl{x}^{(j)})^{\s{T}})-\frac{1}{m}\sum_{j=1}^{m}\fl{a}_j\fl{I}\right|\right|_2 $
\begin{align}
    \leq c\max\Big(\frac{\lr{\fl{a}}_2}{\sqrt{m}}\sqrt{\frac{d\log 9+\log \delta_0^{-1}}{m}},\lr{\fl{a}}_{\infty}\frac{d\log 9+\log \delta_0^{-1}}{m}\Big). \nonumber
\end{align}
\end{lemma}

\begin{lemma}\label{lemma:cauchy-matrix}
Let $\fl{a}_i, \fl{b}_i \in \bR^d$ $\forall$ $i \in [t]$. Then, 
\begin{align}
    \|\sum_i \fl{a}_i\fl{b}_i^\s{T}\|_p^2 \leq \|\sum_i \fl{a}_i\fl{a}_i^\s{T}\|_p\|\sum_i \fl{b}_i\fl{b}_i^\s{T}\|_p.  \nonumber
\end{align}
\end{lemma}

\begin{lemma}\label{lemma:psd-sigma-lambda}
    For a real matrix $\fl{A} \in \bR^{m\times n}$ and a real symmetric positive semi-definite (PSD) matrix $\fl{B} \in \bR^{n\times n}$, the following holds true: $\sigma^2_{\min}(\fl{A})\lambda_{\min}(\fl{B}) \leq \lambda_{\min}(\fl{A}\fl{B}\fl{A}^\s{T})$, where $\sigma_{\min}(\cdot)$ and $\lambda_{\min}(\cdot)$ represents the minimum singular value and minimum eigenvalue operators respectively.
\end{lemma}

\begin{lemma}\label{tensor-product:ABC}
For any three matrices $\fl{A}, \fl{B}$, and $\fl{C}$ for which the matrix product $\fl{ABC}$ is defined,
\begin{align}
    \s{vec}(\fl{ABC}) = (\fl{C}^{\s{T}} \otimes \fl{A})\s{vec}(\fl{B}).\nonumber
\end{align}
\end{lemma}

\begin{lemma}\label{tail:sub_exp}
For a $(\nu^2,\alpha)$ sub-exponential random variable, we have the following tail bound
\begin{align}
    \bb{P}(|X - \E{X}| \geq t) \leq e^{-\frac{1}{2}\min\{t^2/\nu^2, t/\alpha\}}.\nonumber
\end{align}
\end{lemma}

\end{document}